\documentclass[final,12pt,authoryear]{elsarticle}

\usepackage{amssymb}
\journal{arXiv.}

\usepackage{caption, amsmath, amssymb, mathtools, extpfeil, float, yfonts, csquotes, lipsum, fancyhdr, algorithm, algorithmic, amsthm, xcolor, hyperref, centernot, tikz}

\usepackage{ragged2e}
\usetikzlibrary{automata, positioning}

\usepackage{thmtools}
\usepackage{thm-restate}

\newif\ifshowproofs
\showproofstrue

\begin{document}

% Components of Markov Decision Process (apart from reward)
\newcommand{\Environment}{M}
\newcommand{\States}{\mathcal{S}}
\newcommand{\Actions}{\mathcal{A}}
\newcommand{\reward}{R}
\newcommand{\tfunc}{\tau}
\newcommand{\init}{\mu_0} %{\iota}
\newcommand{\discount}{\gamma}
\newcommand{\TransitionSet}{\mathcal{T}}
\newcommand{\InitStateSet}{\mathcal{I}}

\newcommand{\norm}[1]{\lVert#1\rVert}

\newcommand{\MDP}{\langle \States, \Actions, \tfunc, \init, \reward, \discount \rangle}
\newcommand{\MDPd}{\langle \States, \Actions, \tfunc, \init, \reward, d \rangle}

\newcommand{\yrcite}[1]{\citeyearpar{#1}}

\theoremstyle{plain}
\newtheorem{theorem}{Theorem}%[section]
\newtheorem{proposition}[theorem]{Proposition}
\newtheorem{lemma}[theorem]{Lemma}
\newtheorem{corollary}[theorem]{Corollary}

\theoremstyle{definition}
\newtheorem{definition}[theorem]{Definition}
\newtheorem{assumption}[theorem]{Assumption}

\theoremstyle{remark}
\newtheorem{remark}[theorem]{Remark}

\newcommand{\red}[1]{\textcolor{red}{#1}}
\newcommand{\green}[1]{\textcolor{teal}{#1}}
\newcommand{\orange}[1]{\textcolor{orange}{#1}}
\newcommand{\blue}[1]{{#1}}

% Reward function domains
\newcommand{\SxS}{{\States{\times}\States}}
\newcommand{\SxA}{{\States{\times}\Actions}}
\newcommand{\SxAxS}{{\States{\times}\Actions{\times}\States}}

% value functions etc.
\newcommand{\M}{\mathcal{M}}
\newcommand{\Q}{Q}
\newcommand{\V}{V}
\newcommand{\A}{A}
\newcommand{\Return}{G}
\newcommand{\Evaluation}{J}
\newcommand{\EvaluationMCE}{\Evaluation^{\mathrm{H}}_\beta}
\newcommand{\Qfor}[1]{\Q^{#1}}
\newcommand{\Vfor}[1]{\V^{#1}}
\newcommand{\Afor}[1]{\A^{#1}}
\newcommand{\QStar}{\Q^\star}
\newcommand{\VStar}{\V^\star}
\newcommand{\AStar}{\A^\star}
\newcommand{\QSoft}{\Q^{\mathrm{S}}_\alpha}
\newcommand{\QSoftN}[1]{\Q^{\mathrm{S}}_{\alpha,#1}}

\newcommand{\J}{\Evaluation}

\newcommand{\softmax}{\mathrm{softmax}}

% policies
\newcommand{\policy}{\pi}
\newcommand{\OptimalPolicy}{\policy_\star}
\newcommand{\basePolicy}{\policy_0}
\newcommand{\greedyPolicyWrt}[1]{\policy^{#1}_{\star}}
\newcommand{\epsGreedyPolicyWrt}[1]{\policy^{#1}_{\epsilon}}
\newcommand{\BoltzmannPolicyWrt}[1]{\policy^{#1}_{\beta}}
\newcommand{\BoltzmannRationalPolicy}{\policy^\star_\beta}
\newcommand{\MaximumEntropyPolicy}{\policy_\beta}
\newcommand{\MCEPolicy}{\policy^{\mathrm{H}}_\beta}

\newcommand{\Expect}[2]{\mathbb{E}_{#1}\left[{#2}\right]}

\newcommand{\PS}{\mathrm{PS}_\gamma}
\newcommand{\SR}{S'\mathrm{R}_\tfunc}
\newcommand{\LS}{\mathrm{LS}}
\newcommand{\OP}{\mathrm{OP}_{\tfunc,\gamma}}
\newcommand{\Mx}{\mathrm{M}_\mathcal{X}}
\newcommand{\CS}{\mathrm{CS}}

\newcommand{\kPS}[1]{\mathrm{PS}_{\discount,\init}^{#1}}

\newcommand{\R}{\mathcal{R}}
\newcommand{\Rspace}{{\hat{\R}}}

\newcommand{\ORD}{\mathrm{ORD}_{\tfunc, \gamma}}
\newcommand{\OPT}{\mathrm{OPT}_{\tfunc, \gamma}}
\newcommand{\ORDstar}{\mathrm{ORD}_{\star}}

\newcommand{\starc}{d^{\mathrm{STARC}}_{\tfunc, \gamma}}

\newcommand{\eq}[1]{\equiv_{#1}}

\newcommand{\refines}{\preceq}
\newcommand{\refinesStrict}{\prec}

\begin{frontmatter}



\title{Partial Identifiability and Misspecification\\ in Inverse Reinforcement Learning}


\author{Joar Skalse, Alessandro Abate}

\affiliation{organization={Department of Computer Science, Oxford University},%Department and Organization
            %addressline={}, 
            %city={},
            %postcode={}, 
            %state={},
            country={The United Kingdom}}

\begin{abstract}
The aim of Inverse Reinforcement Learning (IRL) is to infer a reward function $R$ from a policy $\pi$. This problem is difficult, for several reasons. First of all, there are typically multiple reward functions which are compatible with a given policy; this means that the reward function is only \emph{partially identifiable}, and that IRL contains a certain fundamental degree of ambiguity. Secondly, in order to infer $R$ from $\pi$, an IRL algorithm must have a \emph{behavioural model} of how $\pi$ relates to $R$. However, the true relationship between human preferences and human behaviour is very complex, and practically impossible to fully capture with a simple model. This means that the behavioural model in practice will be \emph{misspecified}, which raises the worry that it might lead to unsound inferences if applied to real-world data. 
In this paper, we provide a comprehensive mathematical analysis of partial identifiability and misspecification in IRL. Specifically, we fully characterise and quantify the ambiguity of the reward function for all of the behavioural models that are most common in the current IRL literature. We also provide necessary and sufficient conditions that describe precisely how the observed demonstrator policy may differ from each of the standard behavioural models before that model leads to faulty inferences about the reward function $R$. In addition to this, we introduce a cohesive framework for reasoning about partial identifiability and misspecification in IRL, together with several formal tools that can be used to easily derive the partial identifiability and misspecification robustness of new IRL models, or analyse other kinds of reward learning algorithms.
\end{abstract}

\end{frontmatter}



\newpage
\tableofcontents

\newpage
\section{Introduction}

%\textcolor{red}{overall comment: the terms `misspecification' and `partial identifiability' are clearly defined, whereas `ambiguity' is used both vis-a-vis rewards functions and behavioural models in IRL. maybe we could be crisper with the latter term? }\blue{I have added the following paragraph to Section 3.1 -- what do you think of this solution?}\textcolor{red}{excellent, thanks.} 
%\blue{[[At this point, we should briefly remark on our use of the term \enquote{ambiguity} in this article. We will mostly use this as a semi-technical term, corresponding to a number of related technical notions in different contexts. In particular, \enquote{the ambiguity of $f$} and \enquote{the ambiguity of the reward function under $f$} should both normally be taken to refer to $\mathrm{Am}(f)$. Similarly, \enquote{$f$ is less ambiguous than $g$} and \enquote{$g$ tolerates the ambiguity of $f$} should be understood as \enquote{$\mathrm{Am}(f) \refines \mathrm{Am}(g)$}, as per Definition~\ref{def:refinement}. If we say that $f$ is \enquote{too ambiguous} (for some purpose), then this should normally be taken to mean that $\mathrm{Am}(f) \not\refines P$ (for some partition $P$ of $\R$). The intended meaning should generally be clear in each given context. However, to avoid any risk of confusion, all theorems and proofs are stated in terms of unambiguous technical terms.]]}

%\red{[keep next paragraph? I'd cut as self explanatory.]} 
In this section we provide the background and context for our work, an overview of major related work from the existing literature, and an overview of our contributions and the structure of this article. 
%In this section, we first provide a general background of the applications of and challenges related to inverse reinforcement learning; this will be helpful for understanding the motivations underpinning our work, and for contextualising our results. We also provide an overview of major related work from literature, as well as a more detailed statement of the contributions and results in this work. 

\subsection{Background and Context}

Inverse Reinforcement Learning (IRL) is an area of machine learning that is concerned with the problem of inferring what \emph{objective} an agent is pursuing based on the \emph{actions} which that agent takes within some environment \citep{ng2000}. 
IRL can be related to the notion of \emph{revealed preferences} in psychology and economics, since it aims to infer \emph{preferences} from \emph{behaviour} \citep{rothkopf2011}.  
%There are many studies that make IRL a relevant problem that is worth investigating.
There are many possible applications of IRL.
For example, it has been used in natural science contexts, as a tool for understanding animal behaviour \citep{celegansIRL}.
It can also be used in various engineering contexts; many important tasks can be represented as sequential decision-making problems, where the goal is to maximise a given \emph{reward function} over several steps \citep{sutton2018}. However, for many complex tasks it can be very challenging to manually specify a reward function that robustly incentivises the intended behaviour \citep[see e.g.\ ][]{clark2016faulty, paulus2017deep, ibarz2018reward, manheim_categorizing_2019, spec_gaming_DM, knox2022reward, rewardhacking, pang2023reward, karwowski2023goodharts}. 
In those contexts, IRL can be employed to automatically \emph{learn} a good reward function, based on demonstrations of correct behaviour \cite[e.g. ][]{abbeel2010,singh2019}. 
IRL can also be used as a tool for \emph{imitation learning}, where the goal is to use machine learning to clone the behaviour of an agent. In these cases, IRL can improve metrics such as out-of-distribution robustness \cite[e.g.][]{imitation2017}. 
%More generally, the problem of inferring preferences from behaviour is also relevant to cognitive psychology and the philosophy of mind
Overall, IRL relates to many fundamental questions about goal-directed behaviour and agent-based modelling.

It is important to note that the properties which we desire an IRL method to have will depend on the context in which that IRL method will be applied. For example, when IRL is used as a tool for imitation learning, it is not fundamentally important that the inferred preferences actually correspond to the true intentions of the demonstrator, as long as they help the imitation learning process. 
However, when IRL is used to understand the preferences and motivations of an agent \cite[as in e.g.][etc]{CIRL}, then it is crucial that the inferred preferences actually capture the true intentions of the observed agent as faithfully as possible. We should note that this paper is written with mainly this latter motivation in mind.

IRL faces several fundamental challenges. First of all, the IRL problem is typically formalised as the problem of inferring a reward function $R$ from a policy $\pi$.\footnote{Throughout this section, we will sometimes refer to technical terms that we expect to be familiar to most readers. For a rigorous definition of these terms, see Section~\ref{sec:preliminaries}.} To do this, an IRL algorithm needs a model of how $\pi$ relates to $R$, which is referred to as a \emph{behavioural model}. However, under most behavioural models, there are typically multiple reward functions that are consistent with each given policy. For example, two different reward functions may result in exactly the same optimal policy. In that case, we cannot distinguish between those reward functions by observing their optimal policy. This means that the reward function is ambiguous, or \emph{partially identifiable}, based on this data source. 
This ought to be intuitive: informally, there can be multiple different reasons for doing something, which means that observed behaviour sometimes can be explained in multiple different ways.
%\textcolor{red}{[can you also relate to the other way around, namely having multiple (optimal) policies corresponding to the same reward]}\blue{[yes, but that doesn't imply that the IRL problem is partially identifiable]}
For this reason, the IRL problem is fundamentally ambiguous, and we should \emph{prima facie} expect this ambiguity to be irreducible. As such, it is important to fully characterise and quantify this ambiguity in order to clearly understand its impact. 

Another core challenge for IRL is that it must assume a specific relationship between the observed policy and the underlying reward function, i.e., it requires a specific behavioural model. In reality, the relationship between human preferences and human behaviour is incredibly complex: indeed, a complete account of this relationship would amount to a solution to many of the main questions in fields such as cognitive science, behavioural psychology, decision science, and artificial intelligence, etc.
%\textcolor{red}{[also on downstream studies, such as those in game theory (e.g. issue of rationality) and decision science?]}. 
By contrast, most IRL algorithms are based on rather simple behavioural models, that typically correspond to some form of noisy optimality (c.f.\ Section~\ref{sec:preliminaries}).
In fact, there are observable differences between human data and data synthesised using these standard assumptions \citep[see, e.g.,][]{orsini2021}. 
This means that these behavioural models are \emph{misspecified}, which raises the concern that they might systematically lead to flawed inferences if applied to real-world data.

Resolving the issue of misspecification in IRL is fundamentally difficult. %, \textcolor{red}{from both a technical standpoint, a scientific standpoint, and even a philosophical standpoint (since it is ambiguous what \enquote{preferences} even mean in certain contexts). [I'd keep this more factual: perhaps difficult technically and scientifically? ]}
Of course, we can incorporate findings from behavioural psychology to create behavioural models that are more and more accurate (and hence subject to less and less misspecification). Similarly, we can use machine learning to \emph{learn} behavioural models from data \citep[an approach pioneered by][]{shah2019feasibility}, which may also yield more accurate models.
However, it will never be realistically possible to create a behavioural model that is completely free from all forms of misspecification. For this reason, it is important to understand how sensitive the IRL problem is to misspecification of the underlying behavioural model:  is a mostly accurate behavioural model sufficient to ensure that the inferred reward function likewise is mostly accurate, or can a slight error in the behavioural model lead to a large error in the inferred reward? In the former case misspecification may be a manageable issue, whereas in the latter case it may be practically insurmountable. 

%\medskip 

In this paper, we provide a comprehensive theoretical study of \emph{partial identifiability} and of  \emph{misspecification} in inverse reinforcement learning. 
To do this, we first introduce a cohesive theoretical framework for analysing partial identifiability and misspecification robustness in IRL, and derive a number of core results and formal tools within this framework.
We then apply these tools to exactly characterise the \emph{ambiguity} of the reward function given several popular behavioural models, and derive necessary and sufficient conditions which exactly describe what forms of misspecification these behavioural models will tolerate.
%\red{We also derive several results that can apply to much larger classes of behavioural models. [can you elaborate this? I've anticipated the next sentence:]} For instance, t
The tools we introduce can also be used to easily derive the partial identifiability and misspecification robustness of new behavioural models, beyond those we consider explicitly. 
Our analysis is general, as it is carried out in terms of \emph{behavioural models}, rather than \emph{algorithms}. This means that our results will apply to any IRL algorithm based on these behavioural models. 

The motivation behind our work is to provide a theoretically principled understanding of whether and when IRL methods are (or are not) applicable to the problem of inferring a person's (true) preferences and intentions. 
It will never be realistically possible to fully eliminate ambiguity and misspecification from IRL, except possibly in very narrow domains. 
Therefore, if we wish to use IRL as a tool for preference elicitation, then it is crucial to have a good understanding of how IRL is affected by partial identifiability and misspecified behavioural models.
In this paper, we aim to contribute towards building this formal understanding. 

\subsection{Related Work}\label{sec:related_work}

This paper is based on a number of earlier conference papers, namely \citet{skalse2022, misspecification1, skalse2023starc, skalse2024quantifyingsensitivityinversereinforcement}. Specifically, the results in Section~\ref{sec:comparing_reward_functions} are based on the work in \citet{misspecification1} and \citet{skalse2023starc}; the results in Section~\ref{sec:partial_identifiability} are grounded in work from \citet{skalse2022}; the results in Section~\ref{sec:misspecification_1} are in large part based on those in \citet{misspecification1}; and the results in Section~\ref{sec:misspecification_2} depend on those in \citet{skalse2024quantifyingsensitivityinversereinforcement}. However, this paper also contains a large number of novel results that cannot be found in any earlier work, especially in Section~\ref{sec:misspecification_1_wider_classes}, \ref{appendix:generalising_analysis}, and part of Section~\ref{sec:partial_identifiability}. 
%, but also strewn throughout other sections. 
In addition to these new results, this paper also contributes by presenting all results with a mature and cohesive narrative and with unified terminology, which will help with making the results more accessible. %and easier. 

\smallskip 

The issue of partial identifiability in IRL is well-known, and has been studied in a number of previous works. Indeed, the first paper to formally introduce the IRL problem \citep{ng2000} acknowledges the issue of partial identifiability, and characterises the ambiguity of the reward function under the assumption that the observed policy is optimal and the assumption that the reward of a transition $\langle s,a,s' \rangle$ only depends on the state $s$. This work is extended by \citet{dvijotham2010}, who study partial identifiability in IRL for a particular type of environment called linearly-solvable Markov decision processes (LMDPs). Partial identifiability in IRL is also studied by \citet{cao2021}. In this paper, it is assumed that the observed policy maximises causal entropy (c.f.\ Section~\ref{sec:preliminaries}), and that the reward of a transition $\langle s,a,s' \rangle$ only depends on the state $s$ and action $a$ (but not the subsequent state $s'$). \citet{cao2021} also show that the ambiguity of the reward function in this setting can be reduced by combining information from multiple environments.
Also relevant is \citet{towardstheoreticalunderstandingofIRL}, who generalise the results of \citet{cao2021} by also considering environments with constraints, as well as other types of regularisation. They also provide an analysis of the sample complexity of the IRL problem in this setting.

We extend this previous work on partial identifiability in IRL by providing a more complete analysis, and by integrating our analysis into the study of misspecification robustness. In particular, our analysis explicitly considers three types of policies --- optimal policies, maximal causal entropy policies, and Boltzmann-rational policies. Of these, only the first two have been considered by previous works. Moreover, unlike \citet{ng2000} and \citet{cao2021}, we allow the reward of a transition $\langle s,a,s' \rangle$ to depend on each of $s$, $a$, and $s'$, and show that this reveals important additional structure that is not captured by the analysis of \citet{ng2000} or \citet{cao2021}. In addition to this, we provide a general, unified framework for reasoning about both partial identifiability and misspecification robustness, and integrate our results into this framework. However, unlike \citet{dvijotham2010}, we will not consider LMDPs. Moreover, unlike \citet{towardstheoreticalunderstandingofIRL}, we will not consider environments with constraints, other types of regularisation, or finite-sample bounds. Extending our analysis to cover these cases will be a direction for future work.

It is well-known that the standard behavioural models of IRL are misspecified in most applications.
However, there has nonetheless so far not been much research on how sensitive IRL is to misspecification, and what forms of misspecification it can tolerate. 
There are previous papers which aim to \emph{reduce} misspecification in IRL, by creating more realistic behavioural models.
For example, most work in IRL assumes that the observed agent discounts exponentially. However, there is an extensive body of work in the behavioural sciences which suggests that humans are better modelled as discounting \emph{hyperbolically} \cite[see e.g.\ ][]{inconsistencyempirical,mazur1987adjusting, green1996exponential, againstnormativediscounting, discounting_review}. For this reason, \citet{IgnorantAndInconsistent} analyse IRL for agents that use an approximation of hyperbolic discounting.  
Similarly, most work in IRL assumes that the observed agent is \emph{risk-neutral}, whereas humans often are \emph{risk-sensitive} \citep[see e.g.\ ][]{allais, ellsberg, prospecttheory}. For this reason, \citet{singh2018risksensitive} analyse IRL for agents with different forms of risk-sensitivity. Also relevant is \citet{chan2021human}, who provide an empirical study of IRL which incorporates many different models from the behavioural psychology literature. \citet{chan2021human} also empirically confirm that misspecified behavioural models can lead to large errors in the inferred reward, but that this error can be reduced when the misspecification is reduced.

These approaches to reducing misspecification rely on creating more accurate behavioural models by manually incorporating more information about human behaviour. Another approach to reducing misspecification is to try to \emph{learn} a behavioural model from data. \citet{shah2019feasibility} carry out an empirical analysis of IRL where the behavioural model and the underlying reward function are learnt in two different steps, but conclude that this approach comes with significant practical challenges. By contrast, \citet{armstrong2017} carry out a theoretical analysis of the setting where the reward function and the behavioural model are learnt at the same time, from a single stream of data. Notably, \citet{armstrong2017} derive several impossibility theorems for this setting. In particular, they show that this problem setting always will admit several degenerate solutions that fail to solve the problem in a satisfactory way, given that the learning algorithm has an inductive bias towards joint simplicity.

These earlier works all aim to \emph{reduce} misspecification in IRL, by creating more accurate behavioural models. By contrast, we are not focusing on the problem of \emph{reducing} misspecification. Rather, our work aims to understand how \emph{sensitive} IRL is to misspecification of the behavioural model.  As such, our analysis of misspecification is distinct from this earlier work, although it is very relevant to it. Our work aims to answer whether or not IRL will yield accurate inferences given that the behavioural model is misspecified, which in turn would tell us how much misspecification has to be removed (be that manually or by a learning algorithm) before we can make accurate inferences about the reward function through IRL. 

There are some previous papers that (like this paper) study the question of how robust IRL is to misspecification of the behavioural model. In particular, \citet{choicesetmisspecification} study the effects of \emph{choice set misspecification} in IRL (and reward inference more broadly), following the formalism of \citet{RRIC}. They also show that choice set misspecification in some cases can be catastrophic. Also relevant is \cite{viano2021robust}, who study the effects of misspecified \emph{environment dynamics}. They also propose a bespoke IRL algorithm that is meant to be more robust to such misspecification. 
By contrast, we present a broader analysis that covers \emph{all} forms of misspecification, within a single framework. Our work is therefore much wider in scope, and aims to provide necessary and sufficient conditions which fully describe all kinds of misspecification to which each behavioural model is robust.

Another relevant paper is \citet{hong2022sensitivity}, who also study how sensitive IRL is to misspecification of the behavioural model. Our work is more complete than this earlier work in several important respects. 
To start with, our problem setup is both more realistic, and more general. 
In particular, in order to quantify how robust IRL is to misspecification, we first need a way to formalise what it means for two reward functions to be \enquote{close}. \citet{hong2022sensitivity} formalise this in terms of the $L_2$-distance between the reward functions. However, this choice is problematic, because two reward functions can be very dissimilar even though they have a small $L_2$-distance, and vice versa (cf.\ Section~\ref{sec:STARC}). By contrast, our analysis is carried out in terms of specially selected \emph{metrics} on the space of all reward functions, which are backed by strong theoretical guarantees (cf.\ Section~\ref{sec:STARC}). 
Moreover, \citet{hong2022sensitivity} assume that there is a \emph{unique} reward function that maximises fit to the training data, but this is violated in most real-world cases \citep{ng2000, dvijotham2010, cao2021, kim2021, schlaginhaufen2023identifiability}. 
In addition to this, many of their results also assume \enquote{strong log-concavity}, which is a rather opaque condition that is left mostly unexamined. Indeed, \citet{hong2022sensitivity} explicitly do not answer if strong log-concavity should be expected to hold under typical circumstances. 
Our work is not subject to any of these limitations. Moreover, unlike \citet{hong2022sensitivity}, we also integrate our analysis of misspecification with the study of partial identifiability, which is crucial for gaining a complete understanding of the problem. In addition to this, we also present a large number of novel results that are not analogous to any results derived by \citet{hong2022sensitivity}. This includes --- among other things --- necessary and sufficient conditions that fully describe what kinds of misspecification many behavioural models will (or will not) tolerate. 

In our analysis, we will provide a method for quantifying the difference between reward functions. Previous works have also considered this problem. In particular, \citet{epic} provide a pseudometric on the space of all reward functions, which they call EPIC (Equivalent Policy Invariant Comparison). They also show that EPIC induces a regret bound. Similarly, \citet{dard} also provide a pseudometric for reward functions, which they call DARD (Dynamics-Aware Reward Distance). While EPIC is invariant to the transition dynamics of the environment, DARD incorporates some information about the transition function, which can lead to tighter correlation to worst-case regret. However, unlike \citet{epic}, \citet{dard} do not show that DARD induces a bound on worst-case regret. We will also introduce a family of pseudometrics on the space of all reward functions. However, unlike EPIC and DARD, our pseudometrics induce much stronger theoretical guarantees. %\red{[More?]}

There is also other work which studies the question of what happens if a reward function is changed or misspecified.
For example, \citet{rewardhacking} show that if two reward functions $R_1, R_2$ are \emph{unhackable}, in the sense that there are no policies $\pi_1, \pi_2$ such that $J_1(\pi_1) > J_1(\pi_2)$ but $J_2(\pi_1) < J_2(\pi_2)$, then either $R_1$ and $R_2$ induce the same ordering of policies, or at least one of them assigns the same value to all policies. Similarly, \citet{consequences_of_misaligned_AI} consider the case when a reward function $R_2$ depends on a strict subset of the features which are relevant to another reward function $R_1$, and show that optimising $R_2$ in this case may lead to a policy that is arbitrarily bad according to $R_1$, given certain assumptions. Related to this work is also e.g.\ \citet{reward_misspecification} and \citet{pang2023reward}, who carry out an empirical investigation of the consequences of misspecified reward functions in certain environments. Another relevant paper is \citet{karwowski2023goodharts}, who study the effects of reward misspecification through the lens of \emph{Goodhart's Law}, and \citet{limitations_of_Markov_rewards}, who provide examples of natural preference structures which cannot be expressed by reward functions at all.

\subsection{Contributions and Structure of This Article %\red{[keep this, or cut it?]}
}

This paper makes several core contributions. First of all, in Section~\ref{sec:frameworks}, we introduce a framework for reasoning about partial identifiability and misspecification in IRL. This includes a number of formal definitions that describe what it means for an application to tolerate the ambiguity of a reward learning method, and what it means for a behavioural model to be robust to a given form of misspecification, as well as methods for quantifying partial identifiability and misspecification robustness. We also derive a number of lemmas and general results within this framework, that make it easy to reason about partial identifiability and misspecification. 

In Section~\ref{sec:comparing_reward_functions}, we provide several results related to the issue of comparing reward functions. Specifically, we 
%introduce a number of important transformations that can be applied to reward functions, and characterise their properties. We also 
provide necessary and sufficient conditions that describe when two reward functions have the same optimal policies, or the same ordering of policies. We additionally introduce a family of pseudometrics for continuously quantifying the difference between reward functions. We show that these pseudometrics induce both an upper and a lower bound on worst-case regret, and that any pseudometric with this property must be bilipschitz equivalent to ours. Our later analysis builds on these results.

In Section~\ref{sec:partial_identifiability}, we fully characterise the ambiguity of the reward function given several different behavioural models, and we describe the practical consequences of this ambiguity. 
Notably, we show that the ambiguity of the reward is unproblematic for each of the standard behavioural models as long as the learnt reward is used in the same environment it was learnt in, but that we cannot guarantee robust transfer to new environments. 
%Notably, we show that this ambiguity usually is unproblematic, \textcolor{red}{but that it is too great [? pls clarify/rephrase]} to guarantee robust transfer to new environments.
In Sections~\ref{sec:misspecification_1} and \ref{sec:misspecification_2}, we analyse the question of misspecification, and derive necessary and sufficient conditions that fully describe what forms of misspecification each of the standard behavioural models will tolerate. We also study a few specific types of misspecification in greater depth, such as misspecification of the parameters of the behavioural model or perturbations of the observed policy. We find that the standard behavioural models do tolerate some forms of misspecification, but that they are highly sensitive to other forms of misspecification: notably, we find that even mild misspecification of the discount factor $\gamma$ or transition function $\tfunc$ can lead to very large errors in the inferred reward function. 

The proofs of all theorems stated in the main text are provided in \ref{appendix:proofs}. In \ref{appendix:generalising_analysis}, we discuss how to extend our analysis further, by generalising the definitions we introduce in Section~\ref{sec:frameworks}. We show that most of our analysis from Sections~\ref{sec:comparing_reward_functions}-\ref{sec:misspecification_2} carries over if our definitions are generalised.

               % DONE     
\section{Technical Background}\label{sec:preliminaries}

In this section, we introduce the technical prerequisites that are needed to understand the rest of our paper, together with our choice of notation. We also introduce all the assumptions we will make about the environment. For a more in-depth overview of reinforcement learning, see e.g.\ \citet{sutton2018}, and for a more in-depth overview of inverse reinforcement learning, see e.g.\ \citet{arora2020survey} or \citet{Adams2022-ji}.

\subsection{Reinforcement Learning}\label{sec:RL_preliminaries}

A \emph{Markov Decision Processes} (MDP) is a tuple
$(\States, \Actions, \tfunc, \init, 
\reward, \discount)$ where 
$\States$ is a set of \emph{states},
$\Actions$ is a set of \emph{actions},
$\tfunc : \SxA \rightsquigarrow \States$ is a \emph{transition function},
$\init \in \Delta(\States)$ is an \emph{initial state distribution}, 
$\reward : \SxAxS \to \mathbb{R}$ is a \emph{reward function}, and
$\discount \in (0, 1)$ is a \emph{discount rate}.
Here $f : X \rightsquigarrow Y$ denotes a probabilistic mapping from $X$ to $Y$.
A (stationary) \textit{policy} is a function $\policy : \States \rightsquigarrow \Actions$, which encodes the behaviour of an agent in each state of an MDP.
We use $\Pi$ to denote the set of all stationary policies.
A triple $\langle s,a,s' \rangle \in \SxAxS$ is a \emph{transition}, and a \emph{trajectory} $\xi = \langle s_0, a_0, s_1, a_1 \dots \rangle$ is an infinite (potentially repeating) path through an MDP, i.e.\ an element of $(\SxA)^\omega$. If $s_0 \in \mathrm{supp}(\init)$ and $s_{t+1} \in \mathrm{supp}(\tfunc(s_t, a_t))$ for each $t \in \mathbb{N}$, then we say that $\xi$ is a \emph{possible} trajectory, and otherwise it is \emph{impossible}.

In this paper, we assume that $\States$ and $\Actions$ are finite. Moreover, we also assume that all states in $S$ are reachable under $\tfunc$ and $\init$ (i.e., for all states $s$, there exists a possible trajectory which includes $s$). This is primarily a theoretical convenience. Also note that if an MDP has unreachable states, then we may simply remove these states from $\States$. %\red{[we might want to comment more on this assumption, as tau depends on actions/policy, and this raises strict requirements as we need to quantify universally over policies.]}

The \emph{return function} $\Return : (\SxA)^\omega \to \mathbb{R}$ gives the
cumulative discounted reward of each trajectory, i.e.\
$\Return(\xi) = \sum_{t=0}^{\infty} \discount^t \reward(s_t, a_t, s_{t+1})$. Similarly, the \emph{evaluation function} $\Evaluation : \Pi \to \mathbb{R}$ gives the expected trajectory return of each policy, $\Evaluation(\pi) = \Expect{\xi \sim \pi}{G(\xi)}$.
The \emph{value function} $V^\pi : \States \to \mathbb{R}$ of a policy $\pi$ encodes the expected future cumulative discounted reward from each state when following that policy $\pi$. 
The $Q$-function $Q^\pi : \SxA \to \mathbb{R}$ of a policy $\pi$ is given by $Q^\pi(s,a) = \Expect{S' \sim \tfunc(s,a)}{R(s,a,S') + \discount \Vfor\policy(S')}$, i.e.\ the expected future cumulative discounted reward conditional on taking action $a$ in state $s$, and then following the policy $\pi$. Similarly, the \emph{advantage function} of $\pi$ is given by $A^\pi(s, a) = Q^\pi(s, a) - V^\pi(s)$. We say that the \emph{ordering of policies} in an MDP is the ordering on $\Pi$ that is induced by $\J$.

Both $V^\pi$ and $Q^\pi$ can also be defined in terms of fixed points, because they   
%. Specifically, both $V^\pi$ and $Q^\pi$ 
satisfy the following \emph{Bellman equations}:  
\begin{equation}\label{equation:V_pi_recursion}
V^\pi(s) = \Expect{A \sim \pi(s), S' \sim \tfunc(s,A)}{R(s,A,S') + \gamma V^\pi(S')},
\end{equation}
\begin{equation}\label{equation:Q_pi_recursion}
Q^\pi(s,a) = \Expect{S' \sim \tfunc(s,a), A' \sim \pi(S')}{R(s,a,S') + \gamma Q^\pi(S',A')}.
\end{equation}
Both of these equations specify a recursion, and these recursions can be shown to be contraction maps.  Thus, $V^\pi$ and $Q^\pi$ are the only functions which satisfy Equations~\ref{equation:V_pi_recursion} and \ref{equation:Q_pi_recursion} respectively. Similarly, we can specify a unique function $V^\star : \States \to \mathbb{R}$ and a unique function $Q^\star : \SxA \to \mathbb{R}$ via the following two Bellman recursions:
\begin{equation}\label{equation:optimal_V_recursion}
V^\star(s) = \max_{a \in \Actions}\Expect{S' \sim \tfunc(s,a)}{R(s,a,S') + \gamma V^\star(S')},
\end{equation}
\begin{equation}\label{equation:optimal_Q_recursion}
Q^\pi(s,a) = \Expect{S' \sim \tfunc(s,a)}{R(s,a,S') + \gamma \max_{a \in \Actions} Q^\pi(S',a)}. 
\end{equation}
We refer to $V^\star$ as the \emph{optimal value function}, and to $Q^\star$ as the \emph{optimal $Q$-function}. We can also define an \emph{optimal advantage function} $A^\star$ as $A^\star(s,a) = Q^\star(s,a) - V^\star(s)$. Note that $A^\star$ always is non-positive.

If an action $a$ maximises $Q^\star(s,a)$ (or, equivalently, $A^\star(s,a)$) in some state $s$, then we say that $a$ is an \emph{optimal action} in $s$. If a policy $\pi$ only takes optimal actions with positive probability, then we say that $\pi$ is an \emph{optimal policy}. We will sometimes denote an optimal policy as $\pi^\star$. If $\pi$ is optimal, then $\pi$ maximises the evaluation function $\J$. %\red{[next needed?]}\blue{[It is relevant to e.g.\ thm 57.]} 
However, the converse does not hold. To see this, note that $\pi$ may maximise $\J$, even if $\pi$ takes sub-optimal actions in states that $\pi$ visits with probability $0$. Also note that if $\pi$ is optimal, then $Q^\pi = Q^\star$ and $V^\pi = V^\star$. 
Since Equations~\ref{equation:optimal_V_recursion} and \ref{equation:optimal_Q_recursion} always have a solution, there is always at least one optimal policy. Moreover, the set of all optimal policies form a convex set, given by all distributions over the optimal actions in each state. %\red{[is this obvious?]}\blue{[it follows from the fact that a policy is optimal iff it only gives support to optimal actions]}. 

In this paper, we will often talk about pairs or sets of reward functions. In these cases, we will give each reward function a subscript $R_i$, and use $\Evaluation_i$, $\VStar_i$, and $\Vfor{\pi}_i$, and so on, to denote $R_i$'s evaluation function, optimal value function, and $\pi$ value function, and so on. We reserve $R_0$ for the reward function that is zero everywhere, i.e.\ $R_0(s,a,s') = 0$ for all $s,a,s'$. Moreover, if a reward function $R$ satisfies that $J(\pi) = J(\pi')$ for all policies $\pi, \pi'$, then we say that $R$ is \emph{trivial}. $R_0$ is trivial, but there are other trivial reward functions as well (c.f.\ e.g.\ Proposition~\ref{prop:change_from_potentials} or Theorem~\ref{thm:policy_ordering}). We will also use $\R$ to denote the set of all possible reward functions.

Note that we have have defined reward functions $R$ as having the type signature $\SxAxS \to \mathbb{R}$. In practice, it is common to instead consider reward functions with the type signature $\SxA \to \mathbb{R}$. The reason for this is that, for any reward function $R_1 : \SxAxS \to \mathbb{R}$, we can define a second reward function $R_2 : \SxA \to \mathbb{R}$ as
$$
R_2(s,a) = \Expect{S' \sim \tfunc(s,a)}{R_1(s,a,S')}.
$$
It is now easy to see that $J_2 = J_1$, $Q^\star_2 = Q^\star_1$, and so on.  Thus, we arguably do not gain any expressive power from allowing the reward of a transition $\langle s,a,s' \rangle$ to depend on $s'$. However, it is important to note that $R_1$ and $R_2$ only are equivalent relative to one particular transition function $\tfunc$. Moreover, in some of our results, we will quantify over multiple transition functions. We must therefore allow the reward to depend on $s'$, to ensure that our results are fully general.

The \emph{occupancy measure} $\eta^\pi$ of a policy $\pi$ is the $(|\States||\Actions||\States|)$-dimensional vector in which the value of the $(s,a,s')$'th dimension is given by
$$
\sum_{t=0}^\infty \gamma^t \mathbb{P}_{\xi \sim \pi}\left(S_t, A_t, S_{t+1} = s,a,s' \right), 
$$
where the probability is over a trajectory $\xi$ sampled from $\pi$ (assuming the first state is sampled from $\init$ and transitions are sampled from $\tfunc$).
%\textcolor{red}{$\mathbb{P}_{\xi \sim \pi}$ - has this notation been formally introduce elsewhere - is its meaning clear?} 
In other words, the occupancy measure $\eta^\pi$ of $\pi$ measures the cumulative discounted probability with which $\pi$ visits each transition. To prove our results, it will sometimes be useful to map policies to their occupancy measures. One reason for this is that, if we represent the reward function $R$ as an $(|\States||\Actions||\States|)$-dimensional vector, then $\J(\pi) = \eta^\pi \cdot R$. 
%\red{[clarify claim, re:transient]}\blue{[what do you mean by re:transient?]}. 
In other words, occupancy measures allow us to decompose $\J$ into two separate steps, the first of which is independent of the reward function, and the second of which is linear. We will sometimes use $\Omega$ to denote the set of all occupancy measures, i.e.\ $\Omega = \{\eta^\pi : \pi \in \Pi\}$, where $\Pi$ is the set of all (stationary) policies.

\subsection{Inverse Reinforcement Learning}\label{sec:IRL_preliminaries}

The aim of an IRL algorithm is to infer a representation of an agent's \emph{preferences} based on their \emph{behaviour}. It is typically assumed that these preferences can be represented as a reward function, and that the observed behaviour has the form of a policy. It is also typically assumed that the environment of the agent can be modelled as an MDP. The IRL problem can thus loosely be stated as follows. There is an unknown reward function $R$. You get to observe a policy $\pi$, which has been computed from $R$ relative to some transition function $\tfunc$, initial state distribution $\init$, and discount factor $\gamma$. We may or may not assume that $\tfunc$, $\init$, and $\gamma$ are known.\footnote{Note that we generally have to assume that the set of states $\States$ and the set of actions $\Actions$ are known, since these are the domain and codomain of $\pi$.} The goal is then to infer a reward function $R_H$, that is as similar as possible (in some relevant sense) to the true reward function $R$. 

An IRL algorithm must make assumptions about how the observed policy $\pi$ relates to the underlying reward function, $R$. These assumptions are referred to as the \emph{behavioural model}. In some cases, the behavioural model simply assumes that $\pi$ is optimal under $R$ \citep[e.g.\ ][]{ng2000}. However, this assumption is often unrealistic; people sometimes make mistakes, and are subject to limited information and limited cognitive resources. As such, many IRL algorithms make use of other behavioural models. One common model is \emph{Boltzmann rationality} \citep[e.g.\ ][]{ramachandran2007}, which says that
$$
\mathbb{P}(\pi(s) = a) = \left( \frac{\exp \beta Q^\star(s,a)}{\sum_{a' \in \Actions} \exp \beta Q^\star(s,a')} \right).
$$
Here $\beta \in \mathbb{R}^+$ is known as a \emph{temperature parameter}.
When $\pi$ satisfies this relationship, we refer to it as a \emph{Boltzmann-rational policy}.
The function $f : \mathbb{R}^n \to \mathbb{R}^n$ given by $f(v)_i = \exp \beta v_i / \sum_{j=1}^n \exp \beta v_j$ is known as the \emph{softmax function} for temperature $\beta$. Thus, a Boltzmann-rational policy is given by applying a softmax function to the optimal $Q$-function. Intuitively speaking, such a policy takes every action with positive probability, but is more likely to take actions with high value than actions with low value. Boltzmann-rationality can therefore be seen as a form of noisy optimality.

Another common behavioural model is \emph{causal entropy maximisation} \citep[e.g.\ ][]{ziebart2010thesis}. This behavioural model specifies an alternative optimisation criterion, known as the \emph{maximal causal entropy} (MCE) objective:
$$
\J^\mathrm{MCE}(\pi) = \Expect{\xi \sim \pi}{\sum_{t=0}^{\infty} \discount^t (\reward(s_t, a_t, s_{t+1}) + \alpha H(\pi(s_t)))}.
$$
Here $\alpha \in \mathbb{R}^+$ is a \emph{weight}, and $H$ is the Shannon entropy function. A policy $\pi$ which maximises the MCE objective is referred to as an MCE policy. Moreover, let the \emph{soft $Q$-function} $\QSoft : \SxA \to \mathbb{R}$ be the function that is defined by the following Bellman recursion:
\begin{equation}\label{equation:soft_Q_recursion}
\QSoft(s,a) = \Expect{S' \sim \tfunc(s,a)}{R(s,a,S') + \gamma \alpha \log
        \sum_{a' \in \Actions} \exp\left(\left(\frac{1}{\alpha}\right)\QSoft(S', a')\right)}.     
\end{equation}
This recursion is a contraction map, and thus has a unique solution. Moreover, it can be shown that the MCE policy is given by  
$$
\mathbb{P}(\pi(s) = a) = \left( \frac{\exp (1/\alpha) \QSoft(s,a)}{\sum_{a' \in \Actions} \exp (1/\alpha) \QSoft(s,a')} \right),
$$
i.e., by applying the softmax function with temperature $1/\alpha$ to $\QSoft$ \citep[see][their Theorem 1 and 2]{haarnoja2017}. Note that this implies that the MCE policy always is unique. %\blue{[this is immediate from the fact that the soft Q-function is unique, and that the MCE policy is given by softmaxing the soft Q-function, which is in Haarnoja et al]}. 
Intuitively speaking, the MCE policy maximises expected cumulative discounted reward, subject to a regularisation term that encourages the policy to be as stochastic as possible. One way to justify the MCE objective as a model of human behaviour is to note that a boundedly rational agent presumably is less likely to solve a given problem using a strategy that is highly sensitive to mistakes.\footnote{As an intuitive example, suppose you are choosing between two different train routes between a point A and some destination B, where the first route is a direct connection, and the second route involves several stops. Suppose also that the second route is slightly faster if you do not miss any connecting trains, but longer if you do miss one or more of the connections. A boundedly rational agent may then be more likely to pick the first route, even if an optimal agent would pick the second route. The entropy regularisation in the MCE objective roughly captures this kind of reasoning, which may make it a plausible model of bounded rationality.}%\blue{[This explanation can be cut. I included it since it seems like many people find the MCE model unintuitive as a model of bounded rationality, so I thought this might be helpful.]}

In the current literature, most IRL algorithms assume that the observed policy is either optimal, Boltzmann-rational, or MCE-optimal. Therefore, we will refer to these behavioural models as the \emph{standard} behavioural models, and focus on them in our analysis. We will however additionally present many results that hold for wider classes of behavioural models. 

There are many ways to design an IRL algorithm around a given behavioural model \citep[see e.g.\ ][etc]{ng2000, ramachandran2007, ziebart2010thesis, haarnoja2017}. However, the details of these algorithms will not be important for understanding our paper, because our analysis will be carried out primarily in terms of behavioural models, rather than specific algorithms. In this way, we can derive results that will apply to \emph{any} IRL algorithm that is based on a given behavioural model.

\subsection{Metrics, Pseudometrics, and Norms}\label{sec:prelims_metrics_and_norms}

In our analysis, we will often quantify the difference between different kinds of objects (especially reward functions). To do this, we will make use of \emph{metrics}, \emph{pseudometrics}, and \emph{norms}. Given a set $X$, a function $m : X \times X \to \mathbb{R}$ is a \emph{pseudometric} on $X$ if it satisfies the following axioms:
\begin{enumerate}
    \item Indiscernibility of identicals: $m(x,x) = 0$ for all $x \in X$.
    \item Positivity: $m(x,y) \geq 0$ for all $x, y \in X$.
    \item Symmetry: $m(x,y) = m(y,x)$ for all $x, y \in X$.
    \item Triangle inequality: $m(x,z) \leq m(x,y) + m(y,z)$ for all $x, y, z \in X$.
\end{enumerate}
If $m$ additionally satisfies the identity of indiscernibles, which says that $m(x,y) \neq 0$ for all $x, y \in X$ such that $x \neq y$, then $m$ is a \emph{metric}. Every metric is a pseudometric, but not vice versa.

Given a vector space $V$, a function $n : V \to \mathbb{R}$ is a \emph{norm} on $V$ if $n$ satisfies the following axioms:
\begin{enumerate}
    \item Non-negativity: $n(v) \geq 0$ for all $v \in V$.
    \item Positive definiteness: $n(v) = 0$ if and only if $v$ is the zero vector.
    \item Absolute homogeneity: $n(c \cdot v) = c \cdot m(v)$ for all $v \in V$ and $c \in \mathbb{R}$.
    \item Triangle inequality: $n(v+w) \leq n(v) + n(w)$ for all $v,w \in V$.
\end{enumerate}
For any real number $p \geq 1$, the function $L_p : \mathbb{R}^d \to \mathbb{R}$ given by
$$
L_p(v) = \left(\sum_{i=1}^{d} |v_i|^p\right)^{1/p}
$$
is a norm on $\mathbb{R}^d$. The function $L_\infty : \mathbb{R}^d \to \mathbb{R}$, given by $L_\infty(v) = \max_i |v_i|$, is also a norm. Moreover, if $n : \mathbb{R}^d \to \mathbb{R}$ is a norm , and $M : \mathbb{R}^d \to \mathbb{R}^d$ is an invertible matrix, then $n \circ M$ is also a norm.

If $n : \mathbb{R}^n \to \mathbb{R}$ is a norm, then the function $m : \mathbb{R}^n \times  \mathbb{R}^n \to \mathbb{R}$ given by $m(v,w) = n(v-w)$ is a metric on $\mathbb{R}$. For convenience, we will (in a mild overload of notation) also denote this metric using $n$, so that e.g.\ 
$$
L_2(v,w) = \left(\sum_{i=1}^{|v|} |v_i - w_i|^2\right)^{1/2},
$$
and so on. Every norm corresponds to a metric in this way, but not every metric corresponds to a norm.

Given a set $X$, and two metrics $m_1, m_2$ on $X$, if there exists positive constants $\ell, u \in \mathbb{R}^+$ such that
$$
\ell \cdot m_1(x,y) \leq m_2(x,y) \leq u \cdot m_1(x,y)
$$
for all $x,y \in X$, then $m_1$ and $m_2$ are said to be \emph{bilipschitz equivalent}. All norms (but not all metrics) are bilipschitz equivalent on any finite-dimensional vector space.

%\subsection{Occupancy Measures}\label{sec:occupancy_measure_prelims}

       % DONE     
\section{New Definitions and Formalisms}\label{sec:frameworks}

In this section, we introduce the theoretical frameworks that underpin our further analysis. First, we will introduce a number of definitions that formalise the notion of \emph{partial identifiability}. After this, we will introduce two related but distinct ways of formalising \emph{misspecification robustness}, and discuss the benefits of each approach. In addition, we will present several relevant 
%lemmas and 
intermediate results about our framework. These lemmas will be used to prove our later results, but are also insightful in their own right. 
%We will end this section with an in-depth discussion of some of the choices that we have made, 
%in creating these definitions, 
%and how they may be modified and broadened.  

Some of the definitions we provide in this section will be given relative to an equivalence relation $\equiv$ or a pseudometric $d^\R$ on $\R$, the set of all possible rewards. The purpose of these is to quantify differences between reward functions (in particular, the difference between the learnt reward function and the true reward function). 
%For example, we may wish to learn a reward function $R_H$ that is \emph{close} to the true reward function $R^\star$, in the sense that $d^\R(R_H, R^\star) \leq \epsilon$ for some $\epsilon$; alternatively, we may wish that $R_H$ is \emph{equivalent} to $R^\star$ in a sense that can be captured by some equivalence relation $\equiv$. 
In this section, we will not discuss the issue of which equivalence relation $\equiv$ or pseudometric $d^\R$ to use --- this question will instead be addressed in Section~\ref{sec:comparing_reward_functions}.

\subsection{Partial Identifiability}\label{sec:frameworks_ambiguity}

In this section, we describe the framework that we will use to analyse \emph{partial identifiability}. 
Before going into the specifics, let us recall the details of the problem. In IRL, there are typically multiple reward functions that are consistent with a given data source, even in the limit of infinite data.
%For example, two reward functions may induce exactly the same MCE policy; in that case, no amount of data from their MCE policy can distinguish them. 
This means that the reward function is ambiguous, or \emph{partially identifiable}, based on such data sources. We wish to characterise this ambiguity. %\red{[AA note to self: clarify usage of term `ambiguity' above.]}\blue{[I have added a paragraph about this a bit further down]}
%, and understand its consequences. 

At the same time, it is important to note that it often is unnecessary to identify a reward function uniquely, because all plausible reward functions might lead to the same outcome in a given application. For example, if we want to learn a reward function in order to compute an optimal policy, then it is enough to learn a reward function that has the same optimal policies as the true reward function. %In general, ambiguity is not problematic if all admissible reward functions lead to identical downstream outcomes. 
It is therefore important to also consider the \emph{ambiguity tolerance} of %specific models. 
various applications.%, as this allows us to determine whether or not the ambiguity of a given data source is problematic for a given application.

%Ambiguity and ambiguity tolerance are formally related. Both concern \emph{invariances} of objects that can be computed from reward functions to \emph{transformations} of those reward functions. Thus, our main contribution is to catalogue the invariances of various mathematical objects derived from the reward function. We explore a \emph{partial order} on these invariances and its implications for selecting and evaluating data sources.

%We wish to characterise the ambiguity of the reward function for several common behavioural models. Our results dscribe the infinite-data bounds for the information that can be recovered from these types of data.
%This can be used to evaluate algorithms relative to their limits, and data sources relative to each other.

%With this context in mind, we can now introduce our 
Our framework for characterising partial identifiability in IRL is based on the following three definitions:

%\red{[Maybe: prime with background about ambiguity tolerance.]}

\begin{definition}\label{def:reward_object_behavioural_model}
    We say that a \emph{reward object} is a function $f : \R \to X$, where $\R$ is the set of all reward functions, and $X$ is any set. If $X$ is the set of all policies $\Pi$, then we refer to $f$ as a \emph{behavioural model}. 
\end{definition}

\begin{definition}\label{def:invariance_partition}
Given a reward object $f : \mathcal{R} \to X$, the \emph{invariance partition} $\mathrm{Am}(f)$ of $f$ is the partition of $\mathcal{R}$ according to the equivalence relation $\eq{f}$ where $R_1 \eq{f} R_2$ if and only if $f(R_1) = f(R_2)$.
\end{definition} 

\begin{definition}\label{def:refinement}
Given two partitions $P$, $Q$ of $\R$, if $R_1 \eq{P} R_2 \implies R_1 \eq{Q} R_2$ then we write $P \refines Q$. Given two reward objects $f : \mathcal{R} \to X$, $g : \mathcal{R} \to Y$, if $\mathrm{Am}(f) \refines \mathrm{Am}(g)$ then we say that $f$ is \emph{no more ambiguous} than $g$.
If $\mathrm{Am}(f) \refines \mathrm{Am}(g)$ but not $\mathrm{Am}(g) \refines \mathrm{Am}(f)$, then we write $\mathrm{Am}(f) \refinesStrict \mathrm{Am}(g)$ and say that $f$ is \emph{strictly less ambiguous} than $g$.
\end{definition}

\begin{figure}[H]
    \centering
    \includegraphics[width=3\textwidth/4]{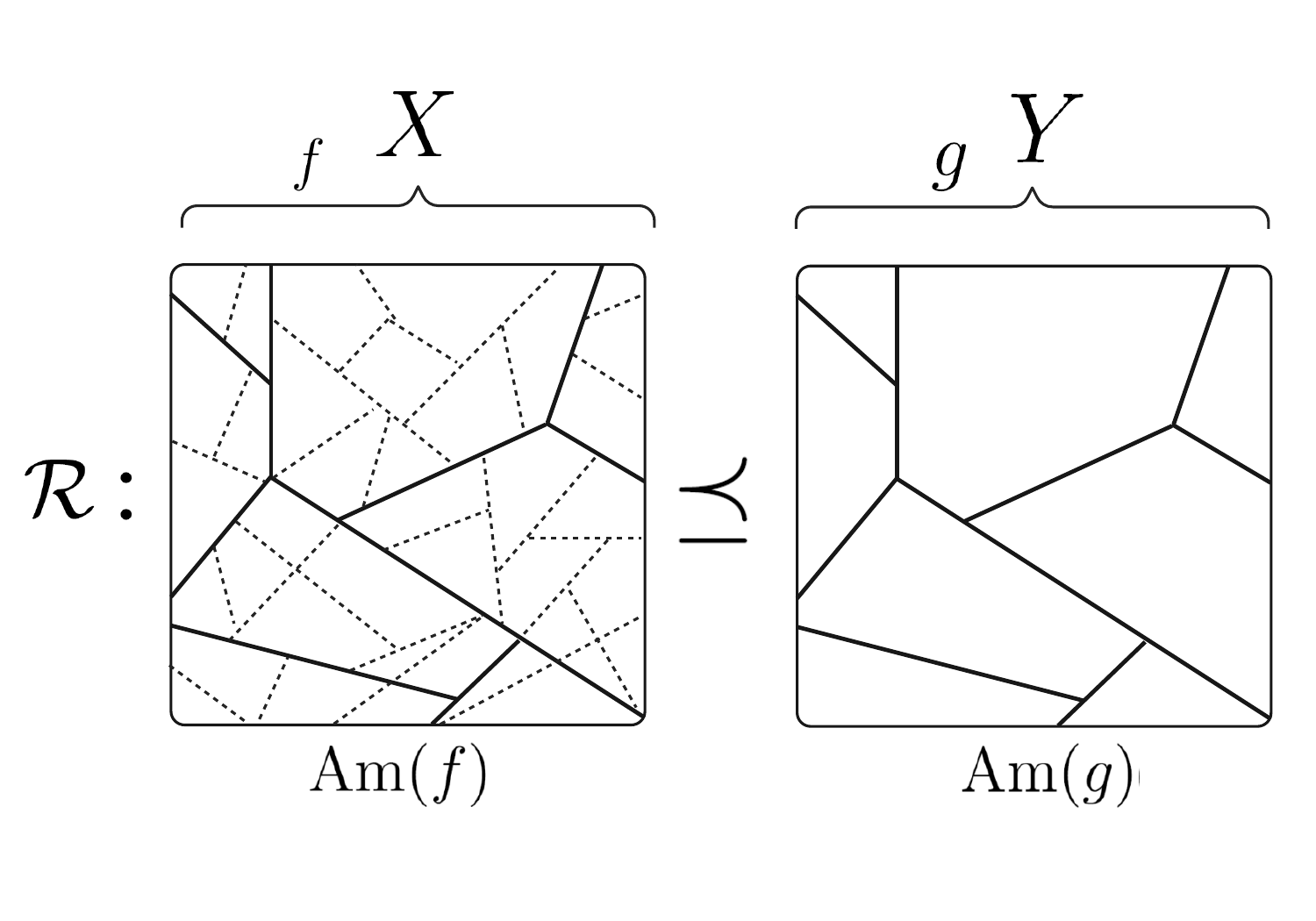}
    \caption{This figure illustrates Definition~\ref{def:reward_object_behavioural_model}-\ref{def:refinement} visually. Specifically, suppose $f : \R \to X$ and $g : \R \to Y$ are functions (or \enquote{reward objects} in our terminology). Now $f$ induces a partitioning $\mathrm{Am}(f)$ of $\R$ according to which $R_1$ and $R_2$ belong to the same partition if (and only if) $f(R_1) = f(R_2)$, and likewise for $g$ and $\mathrm{Am}(g)$. If $g(R_1) = g(R_2)$ whenever $f(R_1) = f(R_2)$, then $\mathrm{Am}(f)$ is a partition refinement of $\mathrm{Am}(g)$, which can be visualised as in the figure above. This corresponds to the case when $\mathrm{Am}(f) \refines \mathrm{Am}(g)$, where $f$ is \emph{less ambiguous} than $g$.}
    \label{fig:frameworks_refinement}
\end{figure}

Before moving on, let us provide an intuitive explanation of these definitions. First of all, anything that can be computed from a reward function can be seen as a \emph{reward object}. For example, we could consider the function $f_{\tfunc,\gamma}$ that, given a reward function $R$, returns the optimal $Q$-function given transition function $\tfunc$ and discount factor $\gamma$. In this case, $X$ would be the set of functions $\SxA \to \mathbb{R}$. 
Similarly, we could consider the function $b_{\tfunc,\gamma,\beta}$ that returns the Boltzmann-rational policy for temperature $\beta$,  given transition function $\tfunc$ and discount factor $\gamma$. In this case $X$ would be the set $\Pi$ of all policies, which means that $b_{\tfunc,\gamma,\beta}$ is a behavioural model. 
%As we can see, reward objects 
%(as defined in Definition~\ref{def:reward_object_behavioural_model}) 
%are a versatile abstract building block that can be used for complex constructions. 
We will mainly, but not exclusively, consider reward objects with the type $\mathcal{R} \rightarrow \Pi$, i.e.\ behavioural models.

We can use reward objects to create an abstract model of a reward learning algorithm $\mathcal{L}$ as follows; first, we assume that there is a true underlying reward function $R^\star$. We model the data source as a function $f : \mathcal{R} \to X$, for some data space $X$, so that the learning algorithm observes $f(R^\star)$.
Note that $f(R^\star)$ could be a distribution, which models the case where the data comprises a set of samples from some source, but it could also be a single finite object.
Next, we assume that $\mathcal{L}$ learns (or converges to) a reward function $R_H$ that is compatible with the observed data, which means that $f(R_H) = f(R^\star)$. Note that this primarily is a model of the \emph{asymptotic} behaviour of learning algorithms, in the limit of \emph{infinite data}.

%\red{[INSERT IMAGE]}

Hence, the defined \emph{invariance partition} $\mathrm{Am}(f)$ of $f$ groups together all reward functions that a learning algorithm $\mathcal{L}$ that is based on $f$ could converge to. For example, let $b_{\beta,\tfunc,\gamma} : \R \to \Pi$ be the function that returns the Boltzmann-rational policy for temperature $\beta$ given transition function $\tfunc$ and discount factor $\gamma$. If two reward functions $R_1$, $R_2$ have the same Boltzmann-rational policy --- i.e., if $b_{\beta,\tfunc,\gamma}(R_1) = b_{\beta,\tfunc,\gamma}(R_2)$ --- then $R_1$ and $R_2$ cannot be distinguished by $b_{\beta,\tfunc,\gamma}$. Thus, $\mathrm{Am}(b_{\beta,\tfunc,\gamma})$ partitions $\R$ according to which reward functions can and cannot be separated by a learning algorithm based on Boltzmann-rational policies. This means that $\mathrm{Am}(f)$ describes the \emph{ambiguity} of the reward $R$ given the data $f(R)$.

Next, note that we can also interpret the invariance partition of $f$ as a characterisation of the information that we need to have about $R$ to construct $f(R)$. Specifically, let $g : \mathcal{R} \to Y$ be a function whose output we wish to compute. If $R^\star$ is the true reward function, then it is acceptable to instead learn a reward function $R_H$ as long as $g(R_H) = g(R^\star)$.
This means that the invariance partition of $g$ also groups together all reward functions that it would be acceptable to learn, for the purpose of computing the output of $g$. Stated differently, $\mathrm{Am}(g)$ describes the \emph{ambiguity tolerance} of $R$ when computing the value of $g(R)$.

We can now see that $\refines$ formalises two important relationships between reward objects. First of all, if $f$ and $g$ correspond to two different reward learning data sources, and $\mathrm{Am}(f) \refinesStrict \mathrm{Am}(g)$, then we get strictly more information about the underlying reward function by observing data from $f$ than we get by observing data from $g$. 
%\textcolor{red}{[next sentence is less clear than earlier:]}
Moreover, if $f$ is a reward learning data source and $g$ is a downstream application, then $\mathrm{Am}(f) \refines \mathrm{Am}(g)$ is precisely the condition of $g$ tolerating the ambiguity of the data source $f$ (i.e., any two reward functions that cannot be distinguished by data from $f$ lead to identical outputs when computing the value of $g$).

To make this more intuitive, let us discuss an example. 
%\textcolor{red}{[employ `example' environment?]} 
Consider first a reward learning data source, such as trajectory comparisons.
In this case, we can let $X$ be the set of all (strict, partial) orderings of the set of all trajectories, and $f$ be the function that returns the ordering of the trajectories that is induced by the trajectory return function, $G$. Let $R^\star$ be the true reward function. In the limit of infinite data, the reward learning algorithm will learn a reward function $R_H$ that induces the same trajectory ordering as $R^\star$, which means that $f(R_H) = f(R^\star)$. Furthermore, if we want to use the learnt reward function to compute a policy, then we may consider a function $g : \mathcal{R} \to \Pi$ that takes a reward function $R$, and returns a policy $\pi^\star$ that is optimal under $R$ (given some $\tfunc$ and $\gamma$). Then if $f(R_H) = f(R^\star) \implies g(R') = g(R^\star)$, we will compute a policy that is optimal under the true reward $R^\star$. This corresponds to the condition that $\mathrm{Am}(f) \refines \mathrm{Am}(g)$.

Before moving on, we should also briefly add a remark on our use of the term \enquote{ambiguity} in this article. We will mostly use this as a general  term, corresponding to a number of related technical notions in different contexts. In particular, \enquote{the ambiguity of $f$} and \enquote{the ambiguity of the reward function under $f$} should both normally be taken to refer to $\mathrm{Am}(f)$. Similarly, \enquote{$f$ is less ambiguous than $g$} and \enquote{$g$ tolerates the ambiguity of $f$} should be understood as \enquote{$\mathrm{Am}(f) \refines \mathrm{Am}(g)$}, as per Definition~\ref{def:refinement}. If we say that $f$ is \enquote{too ambiguous} (for some purpose), then this should normally be taken to mean that $\mathrm{Am}(f) \not\refines P$ (for some partition $P$ of $\R$). The intended meaning should generally be clear in each given context. However, to avoid any risk of confusion, all theorems and proofs are stated in terms of unambiguous technical terms.

Next, it is useful to note that the ambiguities of an object are inherited by all objects that can be computed from it. More formally: 

\begin{restatable}[]{lemma}{ambiguityinherited}
\label{lemma:ambiguity_inherited}
Consider two reward objects $f : \mathcal{R} \to X$, $g : \mathcal{R} \to Y$. If there exists a function $h : X \to Y$ such that $h \circ f = g$, then $\mathrm{Am}(f) \refines \mathrm{Am}(g)$.
\end{restatable}

%\begin{proof}
%If $f(R_1) = f(R_2)$, then $h \circ f(R_1) = h \circ f(R_2)$, so $g(R_1) = g(R_2)$. Thus $f(R_1) = f(R_2) \implies g(R_1) = g(R_2)$, so $\mathrm{Am}(f) \refines \mathrm{Am}(g)$.
%\end{proof}

This simple observation has the important consequence that if there is an intermediate object (e.g., a $Q$-function) that is too ambiguous for a given application, then this ambiguity will also hold for any object that can be computed from this intermediate object. 

Note that $\refines$ is transitive; if $P \refines Q$ and $Q \refines R$ then $P \refines R$. It is also antisymmetric; if $P \refines Q$ and $Q \refines P$ then $P = Q$. This means that our framework endows all reward learning data sources and applications with a lattice structure, where $f \rightarrow g$ if $\mathrm{Am}(f) \refines \mathrm{Am}(g)$. This lattice structure enables reading out several important relationships graphically: %\textcolor{red}{[possibly cluster 1-to-3?]}\blue{[I think it might be useful to highlight this explicitly]}
\begin{enumerate}
    \item If $f \rightarrow g$ then a data source based on $f$ is at least as informative as a data source based on $g$.
    \item If $f \rightarrow g$ then a data source based on $f$ contains enough information to compute the output of $g$.
    \item If $f \rightarrow g$ then it is in principle possible to compute $g(R)$ from $f(R)$.
    \item If $f \rightarrow g$ and $f \not\rightarrow h$ then $g \not\rightarrow h$. In other words, if $f$ is a data source that does not contain enough information to compute $h$, and $g$ is a data source that can be derived from $f$, then $g$ does not contain enough information to compute $h$.
\end{enumerate}
As such, our definitions make it easy to reason about partial identifiability, ambiguity, and ambiguity tolerance, within a single unified framework.

\begin{figure}[H]
\centering
\begin{tikzpicture}[shorten >=1pt,node distance=2.6cm,on grid,auto]
   \node[] (R)   {$R$}; 
   \node[] (Q) [below=of R] {$Q^\star$};
   \node[] (pi) [below left=of Q] {$\pi^\star$};
   \node[] (V) [below right=of Q] {$V^\star$};
    \path[->] 
    (R) edge [swap] node {} (Q)
    (Q) edge [swap] node {} (pi)
    (Q) edge [swap] node {} (V)
    ;
\end{tikzpicture}
\caption{This figure gives a simple illustration of how Definition~\ref{def:reward_object_behavioural_model}-\ref{def:refinement} induces a partial order over objects that can be computed from reward functions. For example, let $q$ be the function that, given a reward function $R$, returns the optimal $Q$-function $Q^\star$, and let $v$ be the function that, given a reward function $R$, returns the optimal value-function $V^\star$. Since $V^\star$ can be computed from $Q^\star$, we have that $\mathrm{Am}(q) \refines \mathrm{Am}(v)$, which can be represented as $q \to v$ (or $Q^\star \to V^\star$) in a figure. Important relationships between data sources can then be read out graphically --- for example, if $Q^\star$ is too ambiguous for a given application, then $V^\star$ must be too ambiguous as well.}
%\caption{This figure summarises our results from Section~\ref{sec:partial_identifiability}. On the left-hand side, we list several reward objects and equivalence relations on $\R$. We write $f \to g$ if $\mathrm{Am}(f) \refines \mathrm{Am}(g)$. Since ambiguity refinement is transitive and antisymmetric, this lets us place all reward objects in a lattice structure. Using this structure, we can read out several important relationships graphically: if $f \to g$, then a data source that is based on $g$ is at least as ambiguous as a data source based on $f$, the information contained in a data source based on $f$ is sufficient to derive the value of $g$ as an application, and it is in principle possible to compute $g$ based on $f$. Note that the lattice structure in this case forms a linear order --- this is a special property of the reward objects and equivalence relations we have studied, and does not hold in general. On the right-hand side of the figure we list the reward transformations that characterise the ambiguity of the reward objects to the left.}
\label{fig:ambiguity_lattice_example}
\end{figure}
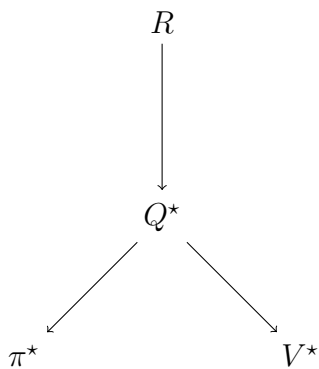

Next, it will often be useful to express $\mathrm{Am}(f)$ in terms of the set of all transformations of $R$ that preserve $f(R)$.\footnote{Note that reward transformations are used for many purposes in reinforcement learning, such as \emph{reward shaping} \citep[e.g.,][]{ng1999}. In this paper, we will (primarily) use them to characterise and describe the structure of partitions of $\R$.} Formally:

\begin{definition}
\label{def:transformations_and_invariances}
A \emph{reward transformation} is a map $t : \mathcal{R} \to \mathcal{R}$.
We say that the \emph{invariances} of $f$ is a set of reward transformations $T$ if for all $R_1, R_2 \in \mathcal{R}$, we have that $f(R_1) = f(R_2)$ if and only if there is a $t \in T$ such that $t(R_1) = R_2$.
We then say that $f$ \emph{determines $R$ up to} $T$.
\end{definition}

Moreover, when talking about a particular kind of object, we will for the sake of brevity sometimes leave the function $f$ implicit, and instead just mention the relevant object. 
For example, we might say that \enquote{the Boltzmann-rational policy determines $R$ up to $T$}. This should be understood as saying that \enquote{$f$ determines $R$ up to $T$, where $f$ is the function that takes a reward and returns the corresponding Boltzmann-rational policy}.
It is also worth noting that $f$ and $T$ often will be parameterised by $\tfunc$, $\init$, or $\gamma$; this dependence will usually be spelt out, but it may in some cases be omitted when it is unambiguous from the context. Moreover, we will sometimes express the invariances of a function $f$ in terms of several sets of reward transformations -- for example, we might say that \enquote{$f$ determines $R$ up to $T_1$ and $T_2$}. This should be understood as saying that $f$ determines $R$ up to $T$, where $T$ is the set of all transformations that can be formed by composing transformations in $T_1$ and $T_2$ (in any order). For more details, see Section~\ref{sec:reward_transformations}.

It is worth noting that if $f$ determines $R$ up to $T_1$, and $g$ determines $R$ up to $T_2$, where $T_1 \subseteq T_2$, then $\mathrm{Am}(f) \refines \mathrm{Am}(g)$. Similarly, if $f$ determines $R$ up to $T_1$, and $g$ determines $R$ up to $T_1$ \emph{and} $T_2$, then we also have that $\mathrm{Am}(f) \refines \mathrm{Am}(g)$. This ought to be quite intuitive, but noting this will make it easier to compare our theorem statements in Section~\ref{sec:partial_identifiability}.

\medskip

Note that the notions introduced in Definition~\ref{def:reward_object_behavioural_model}-\ref{def:refinement} only let us compare the ambiguity of different data sources $f$ and $g$ by examining whether or not the ambiguity of $f$ is strictly greater than the ambiguity of $g$, or vice versa. This means that Definition~\ref{def:reward_object_behavioural_model}-\ref{def:refinement} do not let us \emph{quantify} the absolute ambiguity of a data source. To address this, we introduce the following definition:
%This is what the notions given in Definition~\ref{def:ambiguity_diameter} are meant to address. 

%\textcolor{red}{[next, a bit of a change in pace and context from above]}
%In addition to the definitions we have introduced above, we will also make use of the following definitions:

\begin{definition}\label{def:ambiguity_diameter}
    Given a set of reward functions $S \subseteq \R$, and a pseudometric $d^\R$ on $\R$, we say that the \emph{diameter} $\mathrm{diam}(S)$ of $S$ is the supremum of the distance between pairs of reward functions in $S$ under $d^\R$, i.e.\
    $$
    \mathrm{diam}(S) = \mathrm{sup}\{d^\R(R_1, R_2) : R_1, R_2 \in S\}.
    $$
    Moreover, given a reward object $f : \R \to X$, we say that the \emph{upper diameter} of $\mathrm{Am}(f)$ is the \emph{greatest} diameter of any set in $\mathrm{Am}(f)$, i.e.\
    $$
    \mathrm{sup}\{\mathrm{diam}(S) : S \in \mathrm{Am}(f)\}.
    $$
    Similarly, we say that the \emph{lower diameter} of $\mathrm{Am}(f)$ is the \emph{smallest} diameter of any set in $\mathrm{Am}(f)$, i.e.\
    $$
    \mathrm{inf}\{\mathrm{diam}(S) : S \in \mathrm{Am}(f)\}.
    $$
\end{definition}

To understand these definitions, 
note that if we have a pseudometric $d^\R$ on $\R$ that provides a quantification of how different any two reward functions are, then we can use this pseudometric to measure the \enquote{size} of $\mathrm{Am}(f)$, where a larger size corresponds to greater ambiguity. Since not every set in $\mathrm{Am}(f)$ may have the same size, we further distinguish between the upper and the lower diameter of $\mathrm{Am}(f)$. Intuitively, the upper diameter measures the worst-case ambiguity of the reward function under $f$, whereas the lower diameter measures the best-case ambiguity. For example, if the lower diameter of $\mathrm{Am}(f)$ is $\epsilon$, then that means that there for \emph{any} $x \in \mathrm{Im}(f)$ exists two reward functions $R_1, R_2$ such that $f(R_1) = f(R_2) = x$, but such that the distance between $R_1$ and $R_2$ is $\epsilon$ (or arbitrarily close to $\epsilon$). By contrast, if the upper diameter of $\mathrm{Am}(f)$ is $\epsilon$, then that means that there is \emph{some} $x \in \mathrm{Im}(f)$ for which there exists two reward functions $R_1, R_2$ such that $f(R_1) = f(R_2) = x$, but such that the distance between $R_1$ and $R_2$ is $\epsilon$ (or arbitrarily close to $\epsilon$). Also note that the upper diameter of $\mathrm{Am}(f)$ always is at least as great as the lower diameter of $\mathrm{Am}(f)$.

It is also important to note that, while the (upper and lower) diameter of $\mathrm{Am}(f)$ provides a way of quantifying the size of $\mathrm{Am}(f)$, this does not capture all of the structure of $\mathrm{Am}(f)$. For example, even if the lower diameter of $\mathrm{Am}(f)$ is greater than the upper diameter of $\mathrm{Am}(g)$, this does \emph{not} guarantee that $\mathrm{Am}(g) \refines \mathrm{Am}(f)$. For this reason, it may in some cases be more informative to characterise $\mathrm{Am}(f)$ in terms of reward transformations (rather than in terms of its upper and lower diameter), since this provides a complete description of the structure of $\mathrm{Am}(f)$.

\subsection{Misspecification Robustness}\label{sec:frameworks_misspecification}

%\textcolor{red}{[ambiguity tolerance vs robustness to misspecification]}\blue{[What do you mean by this?]}

In this section, we introduce the frameworks that we will use for analysing robustness to misspecification. Do do this, we will fist give an abstract model of a reward learning algorithm $\mathcal{L}$ that is slightly more general than that provided in Section~\ref{sec:frameworks_ambiguity}. As before, we assume that there is a true underlying reward function $R^\star$, and that the training data is generated by a function $g : \R \to X$, so that the learning algorithm observes $g(R^\star)$. Moreover, we assume that the learning algorithm $\mathcal{L}$ has a model $f : \R \to X$ of how the observed data relates to $R^\star$, such that $\mathcal{L}$ converges to a reward function $R_H$ that satisfies $f(R_H) = g(R^\star)$. However, unlike in Section~\ref{sec:frameworks_ambiguity}, we will not assume that $f = g$; this allows us to reason about the impact of misspecification. If $f \neq g$, then $f$ is \emph{misspecified}, otherwise $f$ is correctly specified.

Intuitively, we want to say that $f$ is robust to misspecification with $g$ if a learning algorithm $\mathcal{L}$ that is based on $f$ is guaranteed to learn a reward function that is \enquote{close} to the true reward function if it is trained on data generated from $g$. To make this statement formal, we need a definition of what it means for two reward functions to be \enquote{close}. Our first formalisation defines this in terms of \emph{equivalence classes}. Specifically, we assume that we have a partition $P$ of $\mathcal{R}$ (which, of course, corresponds to an equivalence relation), and that the learnt reward function $R_H$ is \enquote{close enough} to the true reward $R^\star$ if they are in the same equivalence class, $R_H \equiv_P R^\star$. 
We will for the time being leave out the question of how to pick the partition $P$, and later revisit this question in Section~\ref{sec:comparing_reward_functions}.
Given this, we can now provide our first definition of robustness to misspecification. 

\begin{definition}\label{def:misspecification_eq}
    Given a partition $P$ of $\mathcal{R}$, and two reward objects $f, g : \R \to X$, we say that $f$ is \emph{$P$-robust to misspecification} with $g$ if each of the following conditions are satisfied:
    \begin{enumerate}
        \item If $f(R_1) = g(R_2)$ then $R_1 \eq{P} R_2$.
        \item $\mathrm{Im}(g) \subseteq \mathrm{Im}(f)$.
        \item $\mathrm{Am}(f) \refines P$.
        \item $f \neq g$.
    \end{enumerate}
\end{definition}

Let us explain each of these conditions. %\blue{in Definition~\ref{def:misspecification_eq}}.
The first condition says that if $f$ is $P$-robust to misspecification with $g$, then any learning algorithm $\mathcal{L}$ based on $f$ is guaranteed to learn a reward function that is $P$-equivalent to the true reward function when trained on data generated from $g$. This is the core property of misspecification robustness, which ensures that the mismatch between $f$ and $g$ is unproblematic. 

The second condition ensures that $\mathcal{L}$ can never observe data that is impossible according to its assumed model. For example, suppose $f$ maps each reward function to a deterministic policy; in that case, the learning algorithm $\mathcal{L}$ will assume that the observed policy must be deterministic. What happens if such an algorithm is given data from a nondeterministic policy? This is undefined, absent further details about $\mathcal{L}$. Since we do not want to make any strong assumptions about $\mathcal{L}$, it is reasonable to require that any data that could be produced by $g$, can be explained under $f$. 

The third condition says that any learning algorithm $\mathcal{L}$ based on $f$ is guaranteed to learn a reward function that is $P$-equivalent to the true reward function when trained on data generated by $f$, i.e.\ when there is no misspecification. 
In other words, $f$ is \emph{no more ambiguous} than $P$, in the sense of Definition~\ref{def:refinement}.
This condition is included to rule out certain uninteresting edge cases (c.f.\ \ref{appendix:additional_comments_on_definitions}). The final condition simply says that $f$ and $g$ are distinct --- if they are not, then $f$ is not misspecified!\footnote{Note that if we were to drop condition 4, and set $f = g$, then Definition~\ref{def:misspecification_eq} would be equivalent to Definition~\ref{def:refinement}, which is the definition we use to formalise ambiguity tolerance. Definition~\ref{def:misspecification_eq} is thus an extension of Definition~\ref{def:refinement}, designed to cover the case where the true data generating process (i.e.\ $g$) is different from the model assumed by $\mathcal{L}$ (i.e., $f$).}

\begin{figure}[H]
    \centering
    \includegraphics[width=3\textwidth/4]{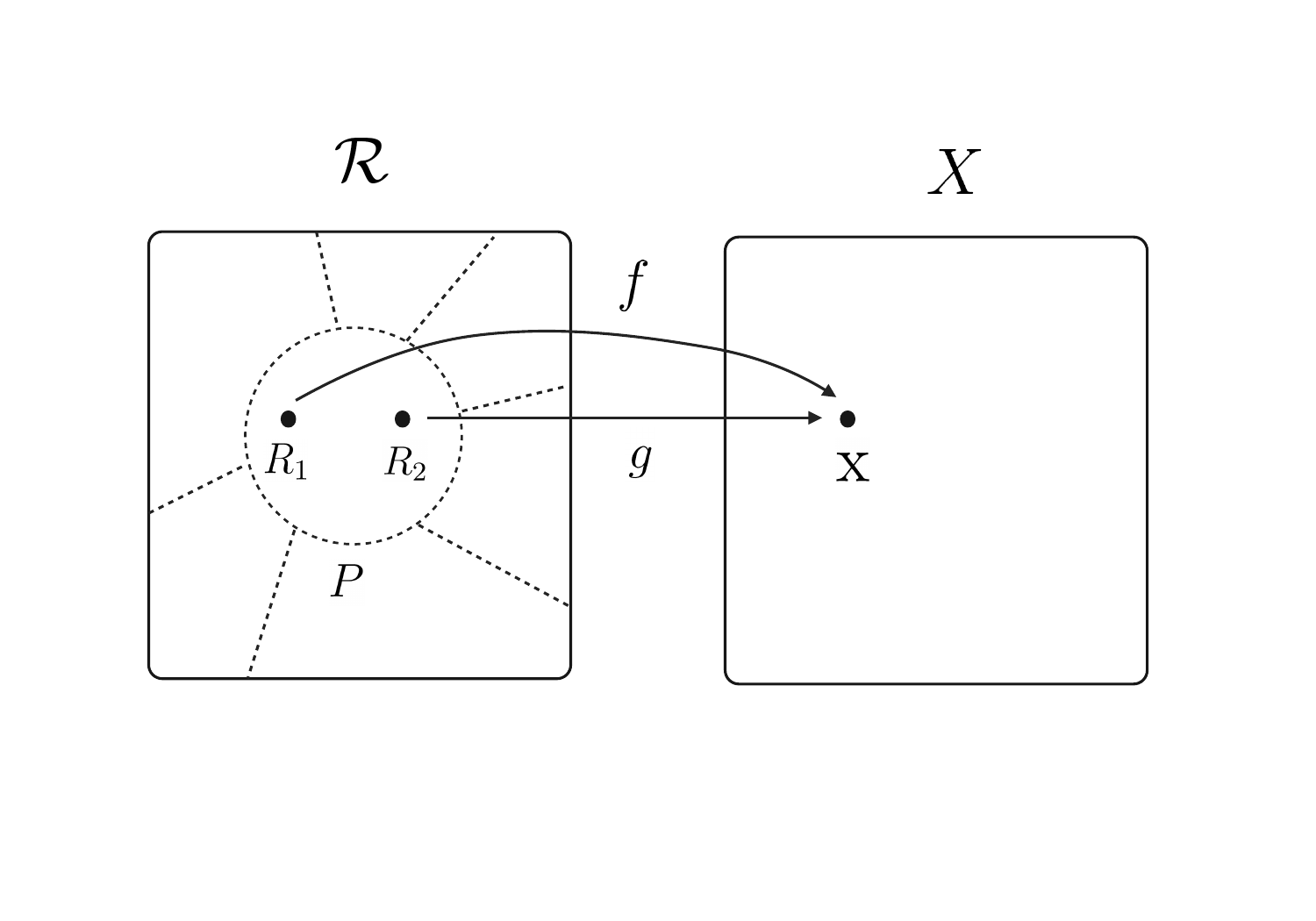}
    \caption{This figure illustrates the conditions in Definition~\ref{def:misspecification_eq}. Both $f$ and $g$ are functions from the space of all rewards $\R$ to some set $X$, and $P$ is a partitioning of $\R$. The learning algorithm $\mathcal{L}$ observes $x = g(R^\star)$ for some unknown reward function $R^\star$, and will find a reward function $R_H$ such that $f(R_H) = x$. We wish to ensure that $R_H \eq{P} R^\star$. If this holds for all $R_H$ and $R^\star$ such that $f(R_H) = g(R^\star)$, together with the other conditions in Definition~\ref{def:misspecification_eq}, when we say that $f$ is $P$-robust to misspecification with $g$.}
    \label{fig:frameworks_mis_eq}
\end{figure}

%\blue{If $f$ is $P$-robust to misspecification with $g$, then we will sometimes say that $f$ \emph{tolerates} misspecification with $g$. Similarly, if $f$ is \emph{not} $P$-robust to misspecification with $g$, then we will sometimes say that $f$ is \emph{sensitive} misspecification with $g$. More informally, if $g$ corresponds to a particular \enquote{type} of misspecification (e.g., misspecification of the }

Our next definition formalises misspecification robustness in terms of pseudometrics on $\R$ (rather than equivalence relations). While Definition~\ref{def:misspecification_eq} captures many important properties of misspecification, it is also limited by the fact that it quantifies the differences between reward functions in terms of equivalence relations. With this definition, two reward functions are either equivalent or not, which means that Definition~\ref{def:misspecification_eq} cannot distinguish between small and large errors in the learned reward function. %This might be an important limitation. 
To alleviate this limitation, we introduce a second definition of misspecification robustness that is based on \emph{pseudometrics} on $\R$; this will let us quantify the error in the learnt reward function in a fine-grained and continuous manner. 

\begin{definition}\label{def:misspecification_metric}
Given a pseudometric $d^\mathcal{R}$ on $\R$, and two reward objects $f, g : \R \to X$, we say that $f$ is $\epsilon$-robust to misspecification with $g$ as measured by $d^\mathcal{R}$ if each of the following conditions are satisfied:
\begin{enumerate}
    \item If $f(R_1) = g(R_2)$ then $d^\mathcal{R}(R_1, R_2) \leq \epsilon$.
    \item $\mathrm{Im}(g) \subseteq \mathrm{Im}(f)$.
    \item If $f(R_1) = f(R_2)$ then $d^\mathcal{R}(R_1, R_2) \leq \epsilon$.
    \item $f \neq g$.
\end{enumerate}
\end{definition}

The conditions in Definition~\ref{def:misspecification_metric} mirror the conditions in Definition~\ref{def:misspecification_eq}; the second and fourth conditions are identical in both definitions, and the first and third conditions are restated in terms of a pseudometric $d^\mathcal{R}$. We will for the time being leave out the question of how to pick a pseudometric $d^\mathcal{R}$, and later revisit this question in Section~\ref{sec:comparing_reward_functions}.

\begin{figure}[H]
    \centering
    \includegraphics[width=3\textwidth/4]{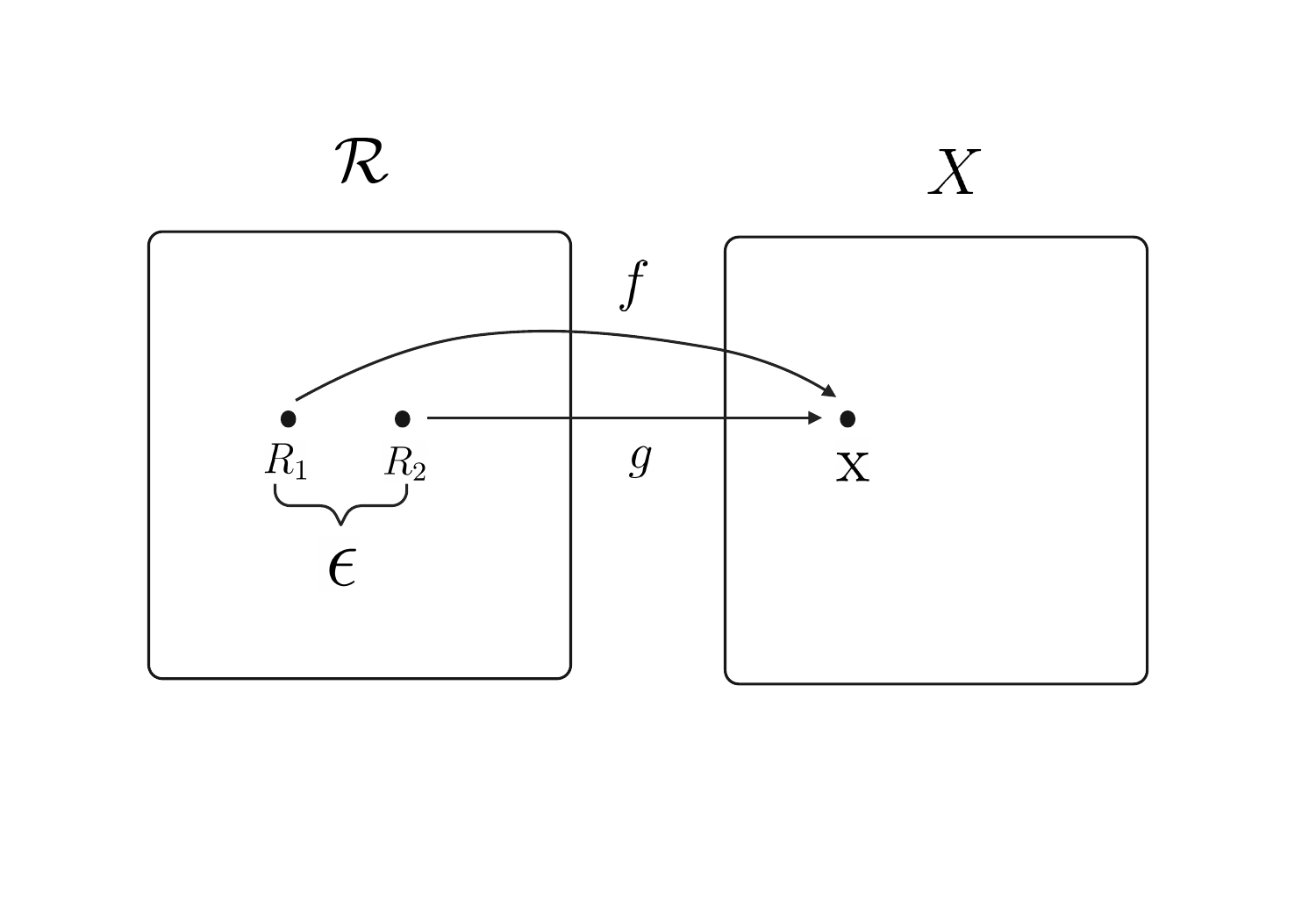}
    \caption{This figure illustrates the conditions in Definition~\ref{def:misspecification_metric}. Both $f$ and $g$ are functions from the space of all rewards $\R$ to some set $X$, and we have some pseudometric $d^\R$ on $\R$. The learning algorithm $\mathcal{L}$ observes $x = g(R^\star)$ for some unknown reward function $R^\star$, and will find a reward function $R_H$ such that $f(R_H) = x$. We wish to ensure that $d^\R(R_H, R^\star) \leq \epsilon$. If this holds for all $R_H$ and $R^\star$ such that $f(R_H) = g(R^\star)$, together with the other conditions in Definition~\ref{def:misspecification_metric}, when we say that $f$ is $\epsilon$-robust to misspecification with $g$ (as measured by the pseudometric $d^\R$).}
    \label{fig:frameworks_mis_metric}
\end{figure}

%Note that Definition~\ref{def:misspecification_robustness} is existential, in the sense that a single counter-example in principle is enough to prevent $f$ from being $\epsilon$-robust to misspecification with $g$, even if $f(R_1) = g(R_2)$ implies $d^\mathcal{R}(R_1, R_2) \leq \epsilon$ for \enquote{most} $R_1$ and $R_2$, etc. As such, even if $f$ is not $\epsilon$-robust to misspecification with $g$, it could in theory still be the case that a learning algorithm $\mathcal{L}$ based on $f$ is guaranteed to learn a reward function $R_h$ that is close to the true reward function $R^\star$ for most choices of $R^\star$. We can rule out this possibility by restricting $\hat{\mathcal{R}}$ in a way that excludes gerrymandered counter-examples.

%Finally, 

Next, note that any result expressed in terms of Definition~\ref{def:misspecification_eq} can be translated into a corresponding result expressed in terms of Definition~\ref{def:misspecification_metric}: %, depending on the introduced pseudometric:
%, which is expressed in terms of Definition~\ref{def:misspecification_metric}:

\begin{restatable}[]{proposition}{metricequivalencerelationconnection}
\label{prop:metric_equivalence_relation_connection}
    Consider a pseudometric $d^\mathcal{R}$ and an equivalence relation $\eq{P}$ on $\R$ such that $R_1 \eq{P} R_2$ if and only if $d^\mathcal{R}(R_1, R_2) = 0$. Then $f$ is $P$-robust to misspecification with $g$ if and only if $f$ is $0$-robust to misspecification with $g$ as measured by $d^\mathcal{R}$.
\end{restatable}
%\begin{proof}
%    Immediate from Definition~\ref{def:misspecification_eq} and \ref{def:misspecification_metric}.
%\end{proof}

Also note that if $f$ is $0$-robust to misspecification with $g$, then it of course follows that $f$ is $\epsilon$-robust to misspecification with $g$ for all $\epsilon > 0$:

\begin{restatable}[]{proposition}{metricequivalencerelationconnectiontwo}
\label{prop:metric_equivalence_relation_connection_two}
    If $f$ is $\epsilon$-robust to misspecification with $g$ measured by $d^\mathcal{R}$, and $\delta > \epsilon$, then $f$ is $\delta$-robust to misspecification with $g$ measured by $d^\mathcal{R}$.
\end{restatable}
%\begin{proof}
%    Immediate from Definition~\ref{def:misspecification_metric}.
%\end{proof}

In light of this, one might ask why we should use Definition~\ref{def:misspecification_eq} if Definition~\ref{def:misspecification_metric} is strictly more expressive.
There are several reasons for this.
First of all, Definition~\ref{def:misspecification_eq} still captures most of what we care about in practice, while also being notably easier to work with.
Moreover, while any pseudometric can be straightforwardly translated into an equivalence relation, it is not always straightforward to translate an equivalence relation into a metric, other than by letting this metric be equal to $0$ for equivalent reward functions and $1$ for non-equivalent reward functions. 
Additionally, Definition~\ref{def:misspecification_eq} will let us derive results that are both stronger and easier to interpret qualitatively, than what is possible using Definition~\ref{def:misspecification_metric}. For this reason, we will make use of both Definition~\ref{def:misspecification_eq} and Definition~\ref{def:misspecification_metric} throughout this paper.

In \ref{appendix:generalising_analysis}, we provide a more extensive discussion of Definition~\ref{def:misspecification_eq} and \ref{def:misspecification_metric}, including ways in which these definitions may be modified or generalised, and whether such modifications would have a meaningful impact on any of our results. We show that many natural generalisations would lead to results that are either identical or closely analogous to the results that we will provide for Definition~\ref{def:misspecification_eq} and \ref{def:misspecification_metric}.

\subsection{Intermediate Results About Our Definitions}\label{sec:frameworks_properties_of_definitions}

%Let us briefly comment on the requirement that $\mathrm{Im}(g) \subseteq \mathrm{Im}(f)$. If a learning algorithm $\mathcal{L}$ is based on a model $f$, then it may or may not be $P$-robust to misspecification with $g$ even if $\mathrm{Im}(g) \not\subseteq \mathrm{Im}(f)$, depending on the exact properties of $\mathcal{L}$. For example, the Boltzmann-rational model requires that every action is taken with positive probability in every state. What happens if a learning algorithm based on this model is shown trajectories from a policy that takes some actions with probability $0$? For most sensible $\mathcal{L}$, the result should be that $\mathcal{L}$ assumes that those actions are taken with a positive but very low probability. This means that the requirement that $\mathrm{Im}(g) \subseteq \mathrm{Im}(f)$ is a somewhat \enquote{soft} requirement in practice.

In this section, we provide several lemmas and intermediate results about Definition~\ref{def:misspecification_eq} and \ref{def:misspecification_metric}. These results give insight into the properties of our problem setting, and will also be used to prove our object-level results. We begin by listing a number of interesting properties of Definition~\ref{def:misspecification_eq}:

\begin{restatable}[]{lemma}{Probustnessinheritance}
\label{lemma:P_robustness_inheritance}
If $f$ is not $P$-robust to misspecification with $g$, and $\mathrm{Im}(g) \subseteq \mathrm{Im}(f)$, then for any $h$, $h \circ f$ is not $P$-robust to misspecification with $h \circ g$.
\end{restatable}

%\begin{proof}
%If $f$ is not $P$-robust to misspecification with $g$, and $\mathrm{Im}(g) \subseteq \mathrm{Im}(f)$, then either $f \not \refines P$, or $f = g$, or $f(R_1) = g(R_2)$ but $R_1 \not\equiv_P R_2$ for some $R_1, R_2$.
%
%In the first case, if $f \not \refines P$ then $h \circ f \not \refines P$, as per Lemma~\ref{lemma:ambiguity_inherited}. This implies that $h \circ f$ is not $P$-robust to misspecification with any reward object (including with $h \circ g$).
%
%In the second case, if $f = g$ then $h \circ f = h \circ g$. This implies that $h \circ f$ is not $P$-robust to misspecification with $h \circ g$.
%
%In the last case, suppose $f(R_1) = g(R_2)$ but $R_1 \not\equiv_P R_2$ for some $R_1, R_2$. If $f(R_1) = g(R_2)$ then $h \circ f(R_1) = h \circ g(R_2)$, so there are $R_1, R_2$ such that $h \circ f(R_1) = h \circ g(R_2)$, but $R_1 \not\equiv_P R_2$. This implies that $h \circ f$ is not $P$-robust to misspecification with $h \circ g$.
%\end{proof}

Lemma~\ref{lemma:P_robustness_inheritance} says that if $f$ lacks robustness to a given form of misspecification, then any object that can be computed from $f$ inherits a lack of robustness to its corresponding misspecification. This lemma can be seen as analogous to Lemma~\ref{lemma:ambiguity_inherited}, and will later be used to show that broad classes of data models lack robustness to some forms of misspecification.

\begin{restatable}[]{lemma}{Probustnessimpliesrefinement}
\label{lemma:P_robustness_implies_refinement}
If $f$ is $P$-robust to misspecification with $g$ then $\mathrm{Am}(g) \refines P$.
\end{restatable}

%\begin{proof}
%Suppose $f$ is $P$-robust to misspecification with $g$, and let $R_1,R_2$ be any two reward functions such that $g(R_1) = g(R_2)$. Since $\mathrm{Im}(g) \subseteq \mathrm{Im}(f)$ there is an $R_3$ such that $f(R_3) = g(R_1) = g(R_2)$. Since $f$ is $P$-robust to misspecification with $g$, it must be the case that $R_3 \equiv_P R_1$ and $R_3 \equiv_P R_2$. By transitivity, we thus have that $R_1 \equiv_P R_2$. Since $R_1$ and $R_2$ were chosen arbitrarily, it must be that $R_1 \equiv_P R_2$ whenever $g(R_1) = g(R_2)$.
%\end{proof}

It may be easier to understand Lemma~\ref{lemma:P_robustness_implies_refinement} by considering the contrapositive statement; if $\mathrm{Am}(g) \not\refines P$ then no $f$ is $P$-robust to misspecification with $g$. In other words, if data from $g$ is insufficient for identifying the $P$-class of the true reward function when there is no misspecification, then we cannot identify the correct $P$-class by using a misspecified data model. This means that we can never gain anything from misspecification.

\begin{restatable}[]{proposition}{Probustnesssymmetry}
\label{prop:P_robustness_symmetry}
If $f$ is $P$-robust to misspecification with $g$ and $\mathrm{Im}(f) = \mathrm{Im}(g)$ then $g$ is $P$-robust to misspecification with $f$.
\end{restatable}

%\begin{proof}
%If $f$ is $P$-robust to misspecification with $g$ then this immediately implies that $f \neq g$, and that if $f(R_1) = g(R_2)$ for some $R_1,R_2$ then $R_1 \equiv_P R_2$. Lemma~\ref{lemma:P_robustness_implies_refinement} implies that $\mathrm{Am}(g) \refines P$, and if $\mathrm{Im}(f) = \mathrm{Im}(g)$ then $\mathrm{Im}(f) \subseteq \mathrm{Im}(g)$. This means that $g$ is $P$-robust to misspecification with $f$.
%\end{proof}

Proposition~\ref{prop:P_robustness_symmetry} says that misspecification robustness is symmetric under many typical circumstances. For example, if $f$ and $g$ are both surjective, then $\mathrm{Im}(f) = \mathrm{Im}(g)$. This means that there are equivalence classes of behavioural models that are all robust to misspecification with each other.

\begin{restatable}[]{lemma}{lessambiguitylessrobustness}
\label{lemma:less_ambiguity_less_robustness}
Let $\mathrm{Am}(f) \refines P$. Then there is no $g$ such that $f$ is $P$-robust to misspecification with $g$ if and only if $\mathrm{Am}(f) = P$.
\end{restatable}

%\begin{proof}
%First consider the case when $\mathrm{Am}(f) = P$, and assume for contradiction that $f$ is $P$-robust to misspecification with $g$. Let $R_1$ be any reward function, and consider $g(R_1)$. Since $\mathrm{Im}(g) \subseteq \mathrm{Im}(f)$, there is an $R_2$ such that $f(R_2) = g(R_1)$. Since $f$ is $P$-robust to misspecification with $g$, this implies that $R_2 \equiv_P R_1$. Moreover, if $\mathrm{Am}(f) = P$ then $R_2 \equiv_P R_1$ if and only if $f(R_2) = f(R_1)$, so it must be the case that $f(R_2) = f(R_1)$. Now, since $f(R_2) = f(R_1)$ and $f(R_2) = g(R_1)$, we have that $g(R_1) = f(R_1)$. Since $R_1$ was chosen arbitrarily, this implies that $f = g$, which is a contradiction. Hence, if $\mathrm{Am}(f) = P$ then there is no $g$ such that $f$ is $P$-robust to misspecification with $g$.
%
%Next, consider the case when $\mathrm{Am}(f) \refines P$ and $\mathrm{Am}(f) \neq P$. This implies that there are $R_1, R_2$ such that $R_1 \equiv_P R_2$ but $f(R_1) \neq f(R_2)$. We can then construct a $g$ as follows; let $g(R_1) = f(R_2)$, $g(R_2) = f(R_1)$, and $g(R) = f(R)$ for all $R \neq R_1,R_2$. Now $f$ is $P$-robust to misspecification with $g$. Hence, if $\mathrm{Am}(f) \refines P$ and $\mathrm{Am}(f) \neq P$ then there is a $g$ such that $f$ is $P$-robust to misspecification with $g$, which in turn implies that if $\mathrm{Am}(f) \refines P$ and there is no $g$ such that $f$ is $P$-robust to misspecification with $g$ then $\mathrm{Am}(f) = P$.
%\end{proof}

Lemma~\ref{lemma:less_ambiguity_less_robustness} has a few interesting implications. First of all, note that it means that we should expect most well-behaved data models to be robust to some forms of misspecification, assuming that $\mathrm{Am}(f) \neq P$. Moreover, it also suggests that data models that are less ambiguous also are less robust to misspecification, and vice versa. One way to interpret this is to note that if $\mathrm{Am}(f) \refinesStrict P$, then $f$ is sensitive to properties of $R$ that are irrelevant from the point of view of $P$. Specifically, it means that there are reward functions $R_1, R_2$ such that $f(R_1) \neq f(R_2)$ but $R_1 \eq{P} R_2$. Informally, we may then expect $f$ to be robust to misspecification with $g$ if $f$ and $g$ only differ in terms of such \enquote{irrelevant details} (c.f.\ Section~\ref{sec:misspecification_1_wider_classes}).

\begin{restatable}[]{lemma}{Probustnessfunctioncomposition}
\label{lemma:P_robustness_function_composition}
Let $\mathrm{Am}(f) \refines P$. 
Then $f$ is $P$-robust to misspecification with $g$ if and only if $f \neq g$ and $g = f \circ t$ for some $t : \R \to \R$ such that $R \eq{P} t(R)$ for all $R$.
\end{restatable}

%\begin{proof}
%First suppose that $f$ is $P$-robust to misspecification with $g$ --- we will construct a $t$ that fits our description. Since $\mathrm{Im}(g) \subseteq \mathrm{Im}(f)$, we have that there for each $R$ exists an $R'$ such that $g(R) = f(R')$. Let $t : \R \to \R$ be a function that maps each $R$ to one such $R'$. Since by construction $g(R) = f(t(R))$ for each $R$, we have that $g = f \circ t$. Moreover, since $f$ is $P$-robust to misspecification with $g$, we have that $R \eq{P} t(R)$. This completes the first direction.
%
%For the other direction, suppose $f \neq g$ and $g = f \circ t$ for some $t : \R \to \R$ such that $R \eq{P} t(R)$ for all $R$.  By assumption we have that $\mathrm{Am}(f) \refines P$. Moreover, we clearly have that $\mathrm{Im}(g) \subseteq \mathrm{Im}(f)$. Finally, if $g(R_1) = f(R_2)$ then $f(t(R_1)) = f(R_2)$. Since $\mathrm{Am}(f) \refines P$, this means that $t(R_1) \equiv_P R_2$. Moreover, since $R \eq{P} t(R)$ for all $R$, we have that $R_1 \eq{P} t(R_1)$. By transitivity, this means that $R_1 \eq{P} R_2$. Thus $f$ is $P$-robust to misspecification with $g$, and we are done.
%\end{proof}

Lemma~\ref{lemma:P_robustness_function_composition} is very important, because it provides us with an easy method for deriving necessary and sufficient conditions that completely describe what forms of misspecification any given data model $f$ is robust to. In particular, given an equivalence relation $P$, if we can find the set $T$ of all functions $t : \R \to \R$ such that $R \eq{P} t(R)$ for all $R$, then we can completely characterise the misspecification robustness of any data model $f$ by simply composing $f$ with each element of $T$. We will later use this method to characterise the misspecification robustness of several important data models.

%\textcolor{red}{[should any of the lemmas above be actual theorems/propositions?]}\blue{[They could be! What is your opinion?]}

Let us next consider Definition~\ref{def:misspecification_metric}. We will show that Definition~\ref{def:misspecification_metric} mostly fails to induce results analogous to those given in Lemma~\ref{lemma:P_robustness_inheritance}-\ref{lemma:P_robustness_function_composition}.

\begin{restatable}[]{proposition}{noepsilonrobustnessimpliesrefinement}
\label{prop:no_epsilon_robustness_implies_refinement}
There exists a pseudometric $d^\R$ on $\R$ such that for each $\epsilon > 0$ there are reward objects $f, g : \R \to X$ where $f$ is $\epsilon$-robust to misspecification with $g$ as measured by $d^\R$, but there are reward functions $R_1, R_2$ such that $g(R_1) = g(R_2)$ but $d^\R(R_1, R_2) > \epsilon$.
\end{restatable}

%\begin{proof}
%For example, let $d^\R$ be the metric induced by the $L_2$-norm, let $X$ be any set such that $|X| \geq |\R|$, and let $f : \R \to X$ be any injective function. Pick two reward functions $R_1, R_2$ such that $d^\R(R_1,R_2) = 2\epsilon$, let $g(R_1) = g(R_2) = f((R_1 + R_2)/2)$, and let $g(R) = f(R)$ for $R \neq R_1,R_2$.
%
%\begin{figure}[H]
%    \centering
%    \includegraphics[width=3\textwidth/4]{figures/frameworks/prop 17.pdf}
    %\caption{\red{TODO}}
%    \label{fig:frameworks_prop17}
%\end{figure}
%
%\end{proof}

As such, Definition~\ref{def:misspecification_metric} does not induce a result analogous to Lemma~\ref{lemma:P_robustness_implies_refinement}; there are $f$ and $g$ such that a learning algorithm that is based on $f$ is guaranteed to learn a reward function that is close to the true reward function if trained on data generated from $g$, but where this is \emph{not} true if we instead use a learning algorithm that is based on $g$, even though the former algorithm is misspecified and the latter is not. This is somewhat pathological. However, we can prove a similar but weaker result:

\begin{restatable}[]{lemma}{epsilonrobustnessimpliesweakrefinement}
\label{lemma:epsilon_robustness_implies_weak_refinement}
Let $f, g : \R \to $ be two reward objects, and let $d^\R$ be a pseudometric on $\R$. Suppose $f$ is $\epsilon$-robust to misspecification with $g$ (as measured by $d^\R$). Then if $g(R_1) = g(R_2)$, we have that $d^\R(R_1, R_2) \leq 2\epsilon$.
\end{restatable}

%\begin{proof}
%Let $R_1, R_2$ be any two reward functions such that $g(R_1) = g(R_2)$. From condition 2 in Definition~\ref{def:misspecification_metric}, we have that there is a reward $R_3$ such that $f(R_3) = g(R_1) = g(R_2)$. From condition 1 in Definition~\ref{def:misspecification_metric}, we have that $d^\mathcal{R}(R_3, R_1) \leq \epsilon$ and that $d^\mathcal{R}(R_3, R_2) \leq \epsilon$. The triangle inequality then implies that $d^\mathcal{R}(R_1, R_2) \leq 2\epsilon$. 
%\end{proof}

%Similarly,
Definition~\ref{def:misspecification_metric} does also not induce a result analogous to Proposition~\ref{prop:P_robustness_symmetry}:

\begin{restatable}[]{proposition}{noepsilonrobustnesssymmetry}
\label{prop:no_epsilon_robustness_symmetry}
There exists a pseudometric $d^\R$ on $\R$ such that for each $\epsilon > 0$ there are reward objects $f, g : \R \to X$ where $f$ is $\epsilon$-robust to misspecification with $g$ as measured by $d^\R$, and $\mathrm{Im}(f) = \mathrm{Im}(g)$, but where $g$ is not $\epsilon$-robust to misspecification with $f$ as measured by $d^\R$. 
\end{restatable}

%\begin{proof}
%For example, let $d^\R$ be the metric induced by the $L_2$-norm, let $X$ be any set such that $|X| \geq |\R|$, and let $h : \R \to X$ be any injective function. Pick four reward functions $R_1,R_2,R_a,R_b$ such that $d^\R(R_1,R_2) = 2\epsilon$, $d^\R(R_1,R_a) < \epsilon$, and $d^\R(R_2,R_b) < \epsilon$. Let $g(R_1) = g(R_2) = h((R_1 + R_2)/2)$, and let $g(R) = h(R)$ for $R \neq R_1,R_2$. Let $f(R_1) = h(R_a)$, $f(R_2) = h(R_b)$, and $f(R) = h(R)$ for $R \neq R_1,R_2$.

%\begin{figure}[H]
%    \centering
%    \includegraphics[width=3\textwidth/4]{figures/frameworks/prop 19.pdf}
    %\caption{\red{TODO}}
%    \label{fig:frameworks_prop19}
%\end{figure}

%Now $g$ is not $\epsilon$-robust to misspecification with $f$, since $g(R_1) = g(R_2)$ even though $d^\R(R_1, R_2) = 2\epsilon$.  However, $f$ is $\epsilon$-robust to misspecification with $g$.  First, if $f(R) = g(R')$, then either $R = R'$, or $R' = (R_1 + R_2)/2$ and $R$ is either $R_1$ or $R_2$. In the former case $d^\R(R, R') = 0$, and in the latter $d^\R(R_1, R_2) = \epsilon$. Moreover, if $f(R) = f(R')$, then either $R = R'$, or $R = R_1$ and $R' = R_a$ (or vice versa), or $R = R_2$ and $R' = R_b$ (or vice versa). In the first case $d^\R(R, R') = 0$, and in the latter two cases $d^\R(R, R') < \epsilon$. Next, $f \neq g$, since $f(R_1) \neq g(R_1)$ and $f(R_2) \neq g(R_2)$. Finally, $\mathrm{Im}(f) = \mathrm{Im}(g)$, since both $\mathrm{Im}(f)$ and $\mathrm{Im}(g)$ are equal to $\mathrm{Im}(h) \setminus \{h(R_1), h(R_2)\}$.
%\end{proof}

In other words, Definition~\ref{def:misspecification_metric} fails to be symmetric even if $\mathrm{Im}(f) = \mathrm{Im}(g)$. This will make it more difficult to establish equivalence classes of behavioural models that are internally robust to misspecification. Lemma~\ref{lemma:less_ambiguity_less_robustness} cannot even be straightforwardly translated into Definition~\ref{def:misspecification_metric} for $\epsilon > 0$. Definition~\ref{def:misspecification_metric} also fails to induce a result analogous to Lemma~\ref{lemma:P_robustness_function_composition}:

\begin{restatable}[]{proposition}{noepsilonrobustnessfunctioncomposition}
\label{prop:no_epsilon_robustness_function_composition}
There exists a pseudometric $d^\R$ on $\R$ such that for each $\epsilon > 0$ there is an $f: \R \to X$ such that if $f(R_1) = f(R_2)$ then $d^\R(R_1,R_2) \leq \epsilon$, and a $t : \R \to \R$ such that $d^\R(R,t(R)) \leq \epsilon$, where $f \neq f \circ t$, but where $f$ is not $\epsilon$-robust to misspecification with $f \circ t$ as measured by $d^\R$.
\end{restatable}

%\begin{proof}
%For example, let $d^\R$ be the metric induced by the $L_2$-norm, and let $X$ be any set such that $|X| \geq |\R|$. Pick three reward functions $R_1, R_2, R_3$ such that $d^\R(R_1, R_2) = \epsilon$, $d^\R(R_2, R_3) = \epsilon$, and $d^\R(R_1, R_3) = 2\epsilon$. Let $f$ be injective, except that $f(R_1) = f(R_2)$, and let $t(R) = R$ for all $R$, except that $t(R_3) = R_2$. Now $f(R_1) = f \circ t(R_3)$, even though $d^\R(R_1,R_2) = 2\epsilon > \epsilon$.

%\begin{figure}[H]
%    \centering
%    \includegraphics[width=3\textwidth/4]{figures/frameworks/prop 20.pdf}
    %\caption{\red{TODO}}
%    \label{fig:frameworks_prop20}
%\end{figure}

%\end{proof}

This is particularly unfortunate, since Lemma~\ref{lemma:P_robustness_function_composition} is very useful for easily deriving necessary and sufficient conditions for $P$-robustness.
However, we can derive a \emph{sufficient} condition for $\epsilon$-robustness that mirrors Lemma~\ref{lemma:P_robustness_function_composition}; if $f$ satisfies that $f(R_1) = f(R_2)$ implies $d^\R(R_1,R_2) \leq \epsilon$, and $t$ satisfies that $d^\R(R,t(R)) \leq \epsilon$, then $f$ is guaranteed to be $2\epsilon$-robust to misspecification with $f \circ t$. However, there can be $g$ such that $f$ is $\epsilon$-robust to misspecification with $g$, but where $g$ cannot be expressed in this way.
We can also derive a necessary and sufficient condition by imposing stronger requirements on $f$:

\begin{restatable}[]{lemma}{weakepsilonrobustnessfunctioncomposition}
\label{lemma:weak_epsilon_robustness_function_composition}
Let $f : \R \to X$ be a reward object, and let $d^\mathcal{R}$ be a pseudometric on $\R$.
Assume that $f(R_1) = f(R_2) \implies d^\mathcal{R}(R_1, R_2) = 0$.
Then $f$ is $\epsilon$-robust to misspecification with $g$ as measured by $d^\mathcal{R}$ if and only if $g = f \circ t$ for some $t : \R \to \R$ such that $d^\mathcal{R}(R, t(R)) \leq \epsilon$ for all $R$, and such that $f \neq g$.
\end{restatable}

Moreover, Definition~\ref{def:misspecification_metric} \emph{does} induce a result analogous to Lemma~\ref{lemma:P_robustness_inheritance}:

\begin{restatable}[]{proposition}{epsilonrobustnessinheritance}
\label{prop:epsilon_robustness_inheritance}
For any pseudometric $d^\R$ on $\R$ and any $\epsilon \geq 0$, if $f$ is not $\epsilon$-robust to misspecification with $g$ as measured by $d^\R$, and $\mathrm{Im}(g) \subseteq \mathrm{Im}(f)$, then for any $h$, $h \circ f$ is not $\epsilon$-robust to misspecification with $h \circ g$ as measured by $d^\R$.
\end{restatable}

%\begin{proof}
%If $f$ is not $\epsilon$-robust to misspecification with $g$ as measured by $d^\R$, and $\mathrm{Im}(g) \subseteq \mathrm{Im}(f)$, then either there are $R_1, R_2$ such that $f(R_1) = g(R_2)$ but $d^\R(R_1, R_2) > \epsilon$, or there are $R_1, R_2$ such that $f(R_1) = f(R_2)$ but $d^\R(R_1, R_2) > \epsilon$, or $f = g$. 
%
%In the first case, if $f(R_1) = g(R_2)$ but $d^\R(R_1, R_2) > \epsilon$ then $h \circ f(R_1) = h \circ g(R_2)$ but $d^\R(R_1, R_2) > \epsilon$. In the second case, if $f(R_1) = f(R_2)$ but $d^\R(R_1, R_2) > \epsilon$ then $h \circ f(R_1) = h \circ f(R_2)$ but $d^\R(R_1, R_2) > \epsilon$. In the third case, if $f = g$ then $h \circ f = h \circ g$. In each case, we thus have that $h \circ f$ is not $\epsilon$-robust to misspecification with $h \circ g$ as measured by $d^\R$.
%\end{proof}

The comparison between Lemma~\ref{lemma:P_robustness_inheritance}-\ref{lemma:P_robustness_function_composition} and 
the results provided above
%Lemma~\ref{lemma:no_epsilon_robustness_implies_refinement}-\ref{lemma:epsilon_robustness_inheritance}
exemplify the fact that Definition~\ref{def:misspecification_eq} sometimes lets us derive results that are stronger and more informative than what is possible using Definition~\ref{def:misspecification_metric}, unless we assume that $\epsilon = 0$ (in which case Definitions~\ref{def:misspecification_eq} and \ref{def:misspecification_metric} are equivalent). For this reason, we will make use of both Definition~\ref{def:misspecification_eq} and Definition~\ref{def:misspecification_metric}.

\subsection{Reward Transformations}\label{sec:reward_transformations}

Recall that a \emph{reward transformation} is a map $t : \R \to \R$. In this section, we introduce several important classes of reward transformations, that we will later use to express our results. First recall \emph{potential shaping}, which was first introduced by \cite{ng1999}:

\begin{definition}\label{def:potential_shaping}%[Potential Shaping]
A \emph{potential function} is a function $\Phi : \States \to \mathbb{R}$.
Given a discount $\discount$, we say that $R_1$ and $R_2$ differ by \emph{potential shaping} with $\gamma$ if for some potential $\Phi$,
$$
R_2(s,a,s') = R_1(s,a,s') + \discount\cdot\Phi(s') - \Phi(s).
$$
for all $s,s' \in \States$ and all $a \in \Actions$.
\end{definition}

\cite{ng1999} proved that if $R_1$ and $R_2$ differ by potential shaping, then they have the same optimal policies for any choice of $\tfunc$ and $\init$. Potential shaping also has many other interesting properties, which we discuss in more detail in \ref{appendix:reward_transformation_properties}. We next define two new classes of transformations, starting with \emph{$S'$-redistribution}.
\begin{definition}
\label{def:sprime-redistribution} %[$S'$-Redistribution]
Given a transition function $\tfunc$,
we say that $R_1$ and $R_2$ differ by \emph{$S'$-redistribution} with $\tfunc$
if
    $$
    \Expect{S' \sim \tfunc(s,a)}{\reward_1(s,a,S')}
    = \Expect{S' \sim \tfunc(s,a)}{\reward_2(s,a,S')}
    $$
for all $s,s' \in \States$ and all $a \in \Actions$.
\end{definition}

% If $\tfunc$ is deterministic then $S'$-redistribution does not allow any change of $\reward$ for any possible transition.

$S'$-redistribution accounts for any difference between $R_1$ and $R_2$ that does not affect the expected reward. If $s_1$, $s_2 \in \mathrm{Supp}(\tfunc(s,a))$ then $S'$-redistribution can increase $\reward(s,a,s_1)$ if it decreases $\reward(s,a,s_2)$ proportionally. $S'$-redistribution can also change $\reward$ arbitrarily for transitions that occur with probability 0. We next consider \emph{optimality-preserving transformations}:
% 
%Note that $S'$-redistribution depends crucially on the reward function's dependence on the successor state. 

\begin{definition}\label{def:optimality_preserving_transformation}
Given a transition function $\tfunc$ and a discount $\discount$, we say that $R_1$ and $R_2$ differ by an \emph{optimality-preserving transformation} with $\tfunc$ and $\gamma$ if there exists a function $\psi : \States \rightarrow \mathbb{R}$ such that
$$
\mathbb{E}_{S' \sim \tfunc(s,a)}[R_2(s,a,S') + \discount\cdot\psi(S')] \leq \psi(s), 
$$
with equality if and only if $a \in \mathrm{argmax}_{a \in \Actions} A^\star_1(s,a)$. 
%\textcolor{blue}{[aha: thanks for fixing this, I was confused earlier...]}
%\forall s,a: 
\end{definition}

%This definition is somewhat dense, so let us explain it in more detail. The function $\psi$ specifies the value function of the new reward $R_2$, and the requirement that $\mathbb{E}_{S' \sim \tfunc(s,a)}[R_2(s,a,S') + \discount\cdot\psi(S')] \leq \psi(s)$ mirrors the Bellman recursion of $\VStar$.
%In other words, a
As the name suggests, an optimality-preserving transformation preserves optimal policies (c.f.\ Theorem~\ref{thm:OPT_ambiguity}). 
Intuitively speaking, an optimality-preserving transformation lets us pick an arbitrary new value function $\psi$, and then adjust $R_2$ in any way that respects the new value function and the argmax of $A^\star_1$ --- the latter condition ensures that the same actions (and hence the same policies) stay optimal. %We will also use transformations that allow the reward to vary freely for a given set of transitions.

%\begin{definition}[Masking]
%Given a set of transitions $\mathcal{X} \subseteq \SxAxS$, we say that
%    $\reward_1 \in \mathcal{R}$ is produced by a \emph{mask of $\mathcal{X}$} from $\reward_2 \in \mathcal{R}$ if $\reward_1(x) = \reward_2(x)$ for all $x \notin \mathcal{X}$.
%\end{definition}

In addition to these transformations, we also say that $R_1$ and $R_2$ differ by \emph{positive linear scaling} if $R_2 = c \cdot R_1$ for some positive constant $c$, and that they differ by \emph{constant shift} if $R_2 = R_1 + c$ for some constant $c$.
Based on these definitions, we can now specify several \emph{sets} of reward transformations:
\begin{enumerate}
    \item Let $\PS$ be the set of all  reward transformations $t$ such that $t(R)$ is given by potential shaping of $R$ relative to the discount $\discount$.
    \item Let $\SR$ be the set of all reward transformations $t$ such that $t(R)$ is given by $S'$-redistribution of $R$ relative to the transition function $\tfunc$.
    \item Let $\LS$ be the set of all reward transformations $t$ that scale each reward function by some positive constant, i.e.\ for each $R$ there is a $c \in \mathbb{R}^+$ such that $t(R)(s,a,s') = c \cdot R(s,a,s')$.
    \item Let $\CS$ be the set of all reward transformations $t$ that shift each reward function by some constant, i.e.\ for each $R$ there is a $c \in \mathbb{R}$ such that $t(R)(s,a,s') = R(s,a,s') + c$.
    \item Let $\OP$ be the set of all reward transformations $t$ such that $t(R)$ is given by an optimality-preserving transformation of $R$ with $\tfunc$ and $\discount$.
    %\item Let $\Mx$ be the set of all reward transformations $t$ such that $t(R)$ is given by a mask of $\mathcal{X}$.
\end{enumerate}
Note that these sets are defined in a way that allows their transformations to be \enquote{sensitive} to the reward function it takes as input. For example, a transformation $t \in \mathrm{PS}_\gamma$ might apply one potential function $\Phi_1$ to $R_1$, and a different potential function $\Phi_2$ to $R_2$, and a transformation $t \in \mathrm{LS}$ might scale $R_1$ by a positive constant $c_1$, and $R_2$ by a different constant $c_2$, etc. Many of our results will be expressed in terms of these sets.

We will also combine sets of reward transformations to form bigger sets. Specifically, if $T_1$ and $T_2$ are sets of reward transformations, then we use $T_1 \bigodot T_2$ to denote the set of all transformations that can be obtained by composing transformations in $T_1$ and $T_2$ arbitrarily, in any order. Formally, we define this operator in the following way:

\begin{definition}
Let $T_1$ and $T_2$ be two (non-empty) sets of reward transformations. Define $S_0$ as $T_1 \cup T_2$, and
$$
S_{i+1} = \{ t_1 \circ t_2 : t_1 \in T_1 \cup T_2, t_2 \in S_i\}.
$$
Then $T_1 \bigodot T_2 = \bigcup_{i = 0}^\infty\{S_i : i \in \mathbb{N}\}$.
\end{definition}

For example, this means that $\PS \bigodot \SR$ is the set of all reward transformations that can be created by composing potential shaping and $S'$-redistribution, in any order. In the text, we will sometimes refer to this set as \enquote{potential shaping and $S'$-redistribution}. This means that the statement \enquote{$R_1$ and $R_2$ differ by potential shaping and $S'$-redistribution} should be understood as saying that there is a $t \in \PS \bigodot \SR$ such that $R_2 = t(R_1)$, and so on. Also note that $\bigodot$ is both commutative and associative.

In \ref{appendix:reward_transformation_properties}, we list and prove several key properties of the reward transformations we have introduced above. These properties are primarily used in our proofs, but may also be helpful for gaining an intuitive understanding  of how these reward transformations work.

\subsection{Behavioural Models}\label{sec:additional_definitions}

In this section, we introduce some special notation for the three behavioural models that are most common in the current IRL literature, i.e.\ optimal policies, Boltzmann-rational policies, and MCE policies.
Given a transition function $\tfunc$ and a discount parameter $\gamma$, let $b_{\tfunc, \gamma, \beta} : \R \to \Pi$ be the function that returns the Boltzmann-rational policy of $R$ with temperature $\beta$, and let $c_{\tfunc, \gamma, \alpha} : \R \to \Pi$ be the function that returns the MCE policy of $R$ with weight $\alpha$. These policies exist and are unique for each $\tfunc$, $\gamma$, $\beta$, and $\alpha$, and so $b_{\tfunc, \gamma, \beta}$ and $c_{\tfunc, \gamma, \alpha}$ are well-defined.

%$o_{\tau,\gamma} : \mathcal{R} \to \Pi$

The optimality model requires a bit more care, because there may in general be more than one policy that is optimal under a given reward function.
To resolve this, recall that a policy is optimal if and only if it only gives support to optimal actions, where the \enquote{optimal actions} are the actions that maximise $Q^\star$.
A state may have multiple optimal actions, so we can get multiple optimal policies by breaking ties in different ways.
However, if an optimal policy gives support to multiple actions in some state, then we would normally not expect the exact probability it assigns to each action to convey any information about the reward function. We will therefore only look at the actions that the optimal policy takes, and ignore the relative probability it assigns to those actions. 
Formally, we will treat optimal policies as functions $\OptimalPolicy : \States \rightarrow \mathcal{P}(\mathrm{argmax}_{a \in \Actions} \AStar)-\{\varnothing\}$; i.e.\ as functions that for each state return a non-empty subset of the set of all actions that are optimal in that state.
Let $\mathcal{O}_{\tfunc,\gamma}$ be the set of all functions that return such policies (relative to transition function $\tfunc$ and discount factor $\gamma$).
Moreover, let $o_{\tfunc,\gamma}^\star \in \mathcal{O}_{\tfunc,\gamma}$ be the function that, given $R$, returns the function that maps each state to the set of \emph{all} actions which are optimal in that state. Intuitively, $o_{\tfunc,\gamma}^\star$ corresponds to optimal policies that take all optimal actions with positive probability. Alternatively, we can also think of $o_{\tfunc,\gamma}^\star$ as corresponding to the set of all optimal policies (noting that this set determines the set of optimal actions for each state, and vice versa).
                 % DONE
\section{Comparing Reward Functions}\label{sec:comparing_reward_functions}

When analysing a reward learning algorithm, we wish to derive claims that compare the learnt reward function to the underlying true reward function, given different setups and conditions. To do this, we must first have principled methods for comparing reward functions. In this section, we discuss different methods for doing this. 
First, we will 
%introduce a number of \emph{transformations} that can be applied to reward functions, and discuss the effects of these transformations. Then, we will 
introduce two natural \emph{equivalence relations} on the space of all reward functions, and characterise the corresponding equivalence classes. We will also introduce a family of \emph{pseudometrics} on the space of all reward functions, and show that these pseudometrics satisfy several desirable properties. In later sections, we will use these reward transformations, equivalence classes, and metrics, to express and prove our results about IRL algorithns.

\subsection{Equivalent Reward Functions}\label{sec:equivalent_rewards}

In this section, we will introduce and study two important equivalence relations on $\R$. The first equivalence relation considers two reward functions to be equivalent if they have the same \emph{ordering of policies}, and the second equivalence relation considers two reward functions to be equivalent if they have the same \emph{optimal policies}. We will also characterise these equivalence relations in terms of reward transformations.

Given a discount $\gamma$ and transition function $\tfunc$, we say that $\ORD$ is the equivalence relation under which $R_1 \eq{\ORD} R_2$ if and only if $R_1$ and $R_2$ have the same \emph{policy ordering} under $\tfunc$ and $\init$.\footnote{Note that while the policy ordering of $R$ may depend on the initial state distribution $\init$, we have that $R_1$ and $R_2$ have the same policy order for one $\init$ if and only if they have the same policy order for all $\init$, c.f.\ Theorem~\ref{thm:policy_ordering}.} Moreover, we say that $\OPT$ is the equivalence relation under which $R_1 \eq{\OPT} R_2$ if and only if $R_1$ and $R_2$ have the same \emph{optimal policies} under $\tfunc$ and $\init$. Note that if $R_1 \eq{\ORD} R_2$ then $R_1 \eq{\OPT} R_2$, but the converse does not hold --- if two reward functions have the same policy ordering, then they have the same optimal policies, but they may have the same optimal policies, without having the same policy ordering. This means that $\ORD$ is stronger than $\OPT$.

We will characterise these equivalence relations in terms of necessary and sufficient conditions on $R_1$ and $R_2$ (relative to a particular choice of transition function $\tfunc$ and discount factor $\gamma$): 
%\red{[BREAK]}
%Using these results, we can now derive necessary and sufficient conditions for two reward functions $R_1$, $R_2$ to have the same ordering of policies (relative to a particular choice of $\tfunc$ and $\init$):

\begin{restatable}[]{theorem}{policyordering}
\label{thm:policy_ordering}
%$(\States,\Actions,\tfunc,\init,R_1,\gamma)$ and $(\States,\Actions,\tfunc,\init,R_2,\gamma)$ have the same ordering of policies
$R_1 \eq{\ORD} R_2$ if and only if $R_2 = t(R_1)$ for some $t \in \SR \bigodot \PS \bigodot \LS$.
\end{restatable}

Stated differently, Theorem~\ref{thm:policy_ordering} says that the MDPs $(\States,\Actions,\tfunc,\init,R_1,\gamma)$ and $(\States,\Actions,\tfunc,\init,R_2,\gamma)$ have the same ordering of policies if and only if $R_1$ and $R_2$ differ by potential shaping, positive linear scaling, and $S'$-redistribution. 
This result will be very important for our analysis.
We next show that $\OPT$ corresponds to optimality-preserving transformations:

\begin{restatable}[]{theorem}{OPTambiguity}
\label{thm:OPT_ambiguity}
$R_1 \eq{\OPT} R_2$ if and only if $R_2 = t(R_1)$ for some $t \in \OP$.
\end{restatable}

Stated differently, Theorem~\ref{thm:OPT_ambiguity} says that the MDPs $(\States,\Actions,\tfunc,\init,R_1,\gamma)$ and $(\States,\Actions,\tfunc,\init,R_2,\gamma)$ have the same optimal policies if and only if $R_1$ and $R_2$ differ by an optimality-preserving transformation. 

In \ref{appendix:stronger_equivalence_conditions_for_rewards}, we discuss the question of how to define and characterise even stronger equivalence relations on $\R$.

%From this, we can also derive the following:

%\begin{proposition}\label{prop:ord_refines_opt}
%    $\ORD \refines \OPT$
%\end{proposition}
%\begin{proof}
%    As per Theorem~\ref{thm:policy_ordering}, if $R_1 \eq{\ORD} R_2$ then $R_1$ and $R_2$ differ by potential shaping, $S'$-redistribution, and positive linear scaling. By the linearity of expectation, $S'$-redistribution does not affect the optimal $Q$-function. Moreover, potential shaping leads to constant shift of the optimal $Q$-function (Proposition~\ref{prop:change_from_potentials}), and positive linear scaling leads to positive linear scaling of the optimal $Q$-function. All of these changes preserve the optimal actions, and hence the optimal policies.
%\end{proof}

%Note that Proposition~\ref{prop:ord_refines_opt} is not entirely trivial, because of the way optimal policies are defined.

\subsection{STARC Metrics}\label{sec:STARC}

In this section, we introduce a family of \emph{pseudometrics} on $\R$, which we can use to get a fine-grained quantification of the difference between any two reward functions.
First, we note that it is not straightforward to quantify the difference between reward functions in an informative way. A simple method might be to measure their $L_2$-distance. However, this is unsatisfactory, because two reward functions can have a large $L_2$-distance, even if they induce the \emph{same} ordering of policies, or a small $L_2$-distance, even if they induce the \emph{opposite} ordering of policies. For example, given an arbitrary reward function $R$ and an arbitrary positive constant $c$, we have that $R$ and $c \cdot R$ have the same ordering of policies, even though their $L_2$-distance may be arbitrarily large. Similarly, for any $\epsilon$, we have that $\epsilon \cdot R$ and $-\epsilon \cdot R$ have the opposite ordering of policies, unless $R$ is trivial, even though their $L_2$-distance may be arbitrarily small. Constructing a new  pseudometric on $\R$ which provides an informative quantification of a useful difference between two reward functions will therefore require some care.

Before proceeding, we should consider what properties a function $d : \R \times \R \to \mathbb{R}$ needs to have, in order to provide a useful way of quantifying the differences between reward functions. First of all, it would certainly be desirable for $d$ to be a \emph{pseudometric}, since pseudometrics provide a well-defined notion of \enquote{distance} that can be used in mathematical analysis. Moreover, it seems reasonable to permit $d$ to be a pseudometric, rather than to require $d$ to be a (proper) metric, because we may want to consider distinct reward functions to be equivalent. For example, if $R_1$ and $R_2$ are distinct but have the same ordering of policies, then it would be natural to consider their distance to be $0$.

Additionally, it would be highly desirable for $d$ to induce an upper bound on worst-case regret. Specifically, we want it to be the case that if $d(R_1, R_2)$ is small, then the impact of optimising $R_2$ instead of $R_1$ should also be small.
When a pseudometric has this property, we say that it is \emph{sound}:

\begin{definition}\label{def:soundness}
A pseudometric $d$ on $\R$ is \emph{sound} if there exists a positive constant $U$, such that for any reward functions $R_1$ and $R_2$, if two policies $\pi_1$ and $\pi_2$ satisfy that $J_2(\pi_2) \geq J_2(\pi_1)$, then
$$
J_1(\pi_1) - J_1(\pi_2) \leq U \cdot (\max_\pi J_1(\pi) - \min_\pi J_1(\pi)) \cdot d(R_1, R_2).
$$
%and moreover, if $d(R_1, R_2) = 0$ then $R_1$ and $R_2$ have the same policy order.
\end{definition}

Let us unpack this definition. 
$J_1(\pi_1) - J_1(\pi_2)$ is the regret, as measured by $R_1$, of using policy $\pi_2$ instead of $\pi_1$.
A division by the quantity $(\max_\pi J_1(\pi) - \min_\pi J_1(\pi))$ normalises the regret based on the total range of $R_1$ %to lie between $0$ and $1$
(though the term is put on the right-hand side of the inequality, instead of being used as a denominator, in order to avoid division by zero when $R_1$ is trivial).
The condition that $J_2(\pi_2) \geq J_2(\pi_1)$ says that $R_2$ prefers $\pi_2$ over $\pi_1$.
Taken together, this means that a pseudometric $d$ on $\R$ is sound if $d(R_1, R_2)$ gives an upper bound on the maximal regret that could be incurred under $R_1$ if an arbitrary policy $\pi_1$ is optimised to another policy $\pi_2$ according to $R_2$. 
It is worth noting that this includes the special case when $\pi_1$ is optimal under $R_1$ and $\pi_2$ is optimal under $R_2$. It is also worth noting that Definition~\ref{def:soundness} implicitly is given relative to a particular choice of $\tfunc$ and $\gamma$ (via $\J_1$ and $\J_2$).

Moreover, it would also be desirable for $d$ to induce a \emph{lower} bound on worst-case regret.
It may not be immediately obvious why this property is preferable. 
To see why this is the case, note that if a pseudometric $d$ on $\mathcal{R}$ does not induce a lower bound on worst-case regret, then there are reward functions that have a \emph{low} worst-case regret, but a \emph{large} distance under $d$. This would in turn mean that $d$ is not tight, and that it should be possible to improve upon it. In other words, if we want a small distance under $d$ to be both sufficient \emph{and necessary} for low worst-case regret, then $d$ must induce both an upper \emph{and a lower} bound on worst-case regret.
As such, we also introduce the following definition:

\begin{definition}\label{def:STARC_complete}
A pseudometric $d$ on $\R$ is \emph{complete} if there exists a positive constant $L$, such that for any reward functions $R_1$ and $R_2$, there exists two policies $\pi_1$ and $\pi_2$ such that $J_2(\pi_2) \geq J_2(\pi_1)$ and
$$
J_1(\pi_1) - J_1(\pi_2) \geq L \cdot (\max_\pi J_1(\pi) - \min_\pi J_1(\pi)) \cdot d(R_1, R_2),
$$
%and moreover, if $R_1$ and $R_2$ have the same policy order then $d(R_1, R_2) = 0$.
and moreover, if both $R_1$ and $R_2$ are trivial, then $d(R_1, R_2) = 0$.
\end{definition}

The last condition is included to rule out certain pathological edge-cases.
Intuitively, if $d$ is sound, then a small $d$ is \emph{sufficient} for low regret, and if $d$ is complete, then a small $d$ is \emph{necessary} for low regret. 
Soundness implies the absence of false positives, and completeness the absence of false negatives. Soundness and completeness also implies the following property:

\begin{restatable}[]{proposition}{soundandcompletemeansdistancezeroiffsameorder}
\label{prop:sound_and_complete_means_distance_0_iff_same_order}
    If a pseudometric $d$ on $\R$ is both sound and complete, then $d(R_1, R_2) = 0$ if and only if $R_1 \eq{\ORD} R_2$.
\end{restatable}

In other words, a pseudometric that is sound and complete must consider two reward functions to be equivalent exactly when they induce the same ordering of policies.
Next, it is worth noting that if two pseudometrics $d_1, d_2$ on $\R$ are both sound and complete, then $d_1$ and $d_2$ are bilipschitz equivalent. This means that if there is a pseudometric on $\R$ that is both sound and complete, then this pseudometric is unique up to bilipschitz equivalence:

\begin{restatable}[]{proposition}{bilipschitz}
\label{prop:bilipschitz}
Any pseudometrics on $\R$ that are both sound and complete are bilipschitz equivalent.
%If two pseudometrics $d_1$, $d_2$ on $\R$ are both sound and complete, then $d_1$ and $d_2$ are bilipschitz equivalent.
\end{restatable}

We will next derive a family of pseudometrics on $\R$, which we refer to as \emph{STAndardised Reward Comparison (STARC) metrics}, and show that these pseudometrics are both sound and complete. This means that all pseudometrics in this family induce both an upper and a lower bound on worst-case regret, and that any other pseudometric with this property must be bilipschitz equivalent to our metrics. As such, STARC metrics can be considered to be canonical, in a certain sense.

STARC metrics are computed in several steps, where the first steps collapse certain equivalence classes in $\R$ to a single representative, and the last step measures a distance. Recall that Proposition~\ref{prop:sound_and_complete_means_distance_0_iff_same_order} says that if $d$ is both sound and complete, then $d(R_1,R_2) = 0$ if and only if $R_1$ and $R_2$ have the same policy order. Moreover, also recall that Theorem~\ref{thm:policy_ordering} says that $R_1$ and $R_2$ have the same policy order if and only if $R_1$ and $R_2$ differ by potential shaping, $S'$-redistribution, and positive linear scaling. 
This implies that if $d$ is sound and complete, then $d(R_1,R_2) = 0$ if and only if $R_1$ and $R_2$ differ by potential shaping, $S'$-redistribution, and positive linear scaling.
Our metrics are therefore computed by first standardising the reward functions to ensure that rewards are considered to be equivalent when they differ by 
these transformations.
%potential shaping, $S'$-redistribution, and positive linear scaling. 
After this, the distance can be measured.

The first step standardises potential shaping and $S'$-redistribution. These transformations can be characterised in terms of linear subspaces of $\R$ (c.f.\ Proposition~\ref{prop:PS_SR_linear} in \ref{appendix:reward_transformation_properties}), which means that this standardisation can be achieved by a linear transformation:

\begin{definition}\label{def:canonicalisation_function}
A function $c : \R \to \R$ is a \emph{canonicalisation function} if $c$ is linear, $c(R)$ and $R$ differ by potential shaping and $S'$-redistribution, and $c(R_1) = c(R_2)$ if and only if $R_1$ and $R_2$ only differ by potential shaping and $S'$-redistribution.
\end{definition}

A canonicalisation function is a quotient map for the subspace of $\R$ that is given by potential shaping and $S'$-redistribution. Note that we require $c$ to be linear. We will later provide examples of canonicalisation functions. Let us next introduce the functions that we use to compute a distance:

\begin{definition}\label{def:admissible_metric}
A metric $m : \R \times \R \to \mathbb{R}$ is \emph{admissible} if there exists a norm $p$ and two (positive) constants $u,\ell$ such that $\ell \cdot p(x,y) \leq m(x,y) \leq u \cdot p (x,y)$ for all $x, y \in \R$.
\end{definition}
 
%Recall that when $p$ is a norm, we use $p(x,y)$ as a shorthand for $p(x-y)$. 
A metric is admissible if it is bilipschitz equivalent to a norm.
Any norm is an admissible metric, though there are admissible metrics which are not norms.\footnote{For example, the unit ball of $m$ does not have to be convex, or symmetric around the origin, etc.}
Recall also that all norms are bilipschitz equivalent on any finite-dimensional vector space. This means that if $m$ satisfies Definition~\ref{def:admissible_metric} for one norm, then it satisfies it for all norms.
Given these two components, we can now define our class of reward metrics:

\begin{definition}\label{def:reward_metric}
A function $d : \R \times \R \to \mathbb{R}$ is a \emph{STARC metric} (STAndardised Reward Comparison) if there is a canonicalisation function $c$, a function $n$ that is a norm on $\mathrm{Im}(c)$, and a metric $m$ that is admissible on $\mathrm{Im}(s)$, such that $d(R_1, R_2) = m(s(R_1), s(R_2))$,
where $s(R) = c(R)/n(c(R))$ when $n(c(R)) \neq 0$, and $c(R)$ otherwise.
\end{definition}

Intuitively speaking, $c$ ensures that all reward functions which differ by potential shaping and $S'$-redistribution are considered to be equivalent, and division by $n$ ensures that positive linear scaling is ignored as well.
Note that if $n(c(R)) = 0$, then $c(R) = R_0$.
Note also that $\mathrm{Im}(c)$ is the image of $c$, if $c$ is applied to the entirety of $\R$, and similarly for $\mathrm{Im}(s)$. If $n$ is a norm on $\R$, then $n$ is also a norm on $\mathrm{Im}(c)$, but there are functions which are norms on $\mathrm{Im}(c)$ but not on $\R$ (c.f.\ Proposition~\ref{prop:J_norm}), and similarly for $\mathrm{Im}(s)$.

STARC metrics have a number of important properties. We first note that STARC metrics indeed are pseudometrics on $\R$, which means that they give us a well-defined notion of a \enquote{distance} between reward functions:

\begin{restatable}[]{proposition}{STARCpseudometrics}
\label{prop:STARC_pseudometrics}
    All STARC metrics are pseudometrics on $\mathcal{R}$.
\end{restatable}

Moreover, all STARC metrics have the property of being both sound and complete. This constitutes our most important results for this section:

\begin{restatable}[]{theorem}{STARCsound}
\label{thm:STARC_sound}
Any STARC metric is sound.
\end{restatable}

\begin{restatable}[]{theorem}{STARCcomplete}
\label{thm:STARC_complete}
Any STARC metric is complete.
\end{restatable}

Theorems~\ref{thm:STARC_sound} and \ref{thm:STARC_complete} together imply that, for any STARC metric $d$, we have that a small value of $d$ is both necessary and sufficient for a low regret.
This means that STARC metrics, in a certain sense, exactly capture what it means for two reward functions to be similar, and that we should not expect it to be possible to significantly improve upon them. Also recall that Proposition~\ref{prop:bilipschitz} says that if two pseudometrics $d_1$ and $d_2$ are both sound and complete, then $d_1$ and $d_2$ are bilipschitz equivalent. This means that STARC metrics are unique up to bilipschitz equivalence. In particular, all STARC metrics are bilipschitz equivalent, and any other pseudometric on $\R$ that induces both an upper and a lower bound on worst-case regret (as we define it) must also be bilipschitz equivalent to STARC metrics.

It is also worth noting that STARC metrics assign two rewards a distance of zero if and only if those rewards induce the same ordering of policies. This is implied by Proposition~\ref{prop:sound_and_complete_means_distance_0_iff_same_order}, together with Theorems~\ref{thm:STARC_sound} and \ref{thm:STARC_complete}.

\begin{restatable}[]{proposition}{STARCdistancezero}
\label{prop:STARC_distance_zero}
All STARC metrics have the property that $d(R_1, R_2) = 0$ if and only if $R_1 \eq{\ORD} R_2$.
\end{restatable}

We will next give a few concrete examples of STARC metrics. We begin by showing how to construct canonicalisation functions:

\begin{restatable}[]{proposition}{valuecanon}
\label{proposition:value_canon}
    For any policy $\pi$, the function $c : \R \to \R$ given by 
    $$
    c(R)(s,a,s') = \mathbb{E}_{S' \sim \tau(s,a)}\left[R(s,a,S') - V^\pi(s) + \gamma V^\pi(S')\right]
    $$
    is a canonicalisation function. Here $V^\pi$ is computed under the reward $R$ given as input to $c$. We call this function Value-Adjusted Levelling (VAL).
\end{restatable}

Proposition~\ref{proposition:value_canon} provides an easy way to construct canonicalisation functions. 
Moreover, while our focus in this paper is on theoretical analysis, it is worth noting that in practice VAL can be approximated as long as $V^\pi$ can be approximated - this is a well studied problem in the instance of large-scale environments.  
We next discuss a desirable property of canonicalisation functions: 

\begin{definition}\label{def:minimal}
    A canonicalisation function $c : \R \to \R$ is \emph{minimal} for a norm $n$ if $n(c(R_1)) \leq n(R_2)$    
    for all $R_1$ and $R_2$ that differ by potential shaping and $S'$-redistribution. 
\end{definition}

It is not a given that minimal canonicalisation functions 
%exist for a given norm $n$, or that they 
are unique for a given norm. For example, they are not unique for the $L_\infty$-norm.\footnote{To see this, note that for any reward $R$, the set of all rewards that differ from $R$ by $\PS \bigodot \SR$ forms an affine subspace of $\R$ (c.f.\ Proposition~\ref{prop:PS_SR_linear} in \ref{appendix:reward_transformation_properties}). A minimal canonicalisation function for a norm $n$ must map each reward in this space to an element that minimises $n$. However, due to the shape of $L_\infty$'s unit ball, this space will always contain multiple rewards with the same (minimal) $L_\infty$-length.}
%\textcolor{red}{[this is not immediately clear:] For example, the minimal canonicalisation function is not unique for the $L_1$-norm, since its unit ball is not strictly convex.} 
However, for any weighted $L_2$-norm, the minimal canonicalisation function %does exist, and 
is unique:

\begin{restatable}[]{proposition}{minimalcanonicalisation}
    For any weighted $L_2$-norm, a minimal canonicalisation function exists and is unique.
\end{restatable}

Finally, we will introduce one more canonicalisation function, which will be useful for illustrative purposes (c.f.\ \ref{appendix:understand_starc}): 
%\textcolor{red}{[not very clear for illustration or immediately useful - by elaborating on the proof, turn this into an example?]}

\begin{restatable}[]{proposition}{occupancymeasureprojectioniscanonicalisation}
\label{prop:occupancy_measure_projection_is_canonicalisation}
Let $\Omega = \{\eta^\pi : \pi \in \Pi\}$ be the set of all occupancy measures, and let $c : \R \to \R$ be the function that projects each reward onto the linear subspace of $\R$ that is parallel to $\Omega$. Then $c$ is a canonicalisation function.
\end{restatable}

%\begin{proof}
%To show that $c$ is a canonicalisation function, we must show that 
%\begin{enumerate}
%    \item $c$ is linear,
%    \item $c(R)$ and $R$ differ by potential shaping and $S'$-redistribution for all $R$, and
%    \item $c(R_1) = c(R_2)$ for all $R_1$ and $R_2$ which differ by potential shaping and $S'$-redistribution.
%\end{enumerate}
%It is straightforward that $c$ is linear, since it is a projection map. To see that $R$ and $c(R)$ differ by potential shaping and $S'$-redistribution, note that there is a constant $k$ such that $\eta^\pi \cdot R = \eta^\pi \cdot c(R) + k$ for all policies $\pi$, since $\mathrm{Im}(c)$ is parallel to $\Omega$. This means that the policy evaluation functions of $R$ and $c(R)$ differ by a constant $k$. By Proposition~\ref{prop:change_from_potentials}, this means that we can create a reward function $R'$ which has the same policy evaluation function as $c(R)$, by applying potential shaping to $R$ with a potential function such that $\Phi(s) = -k$ for all $s \in \mathrm{supp}(\init)$. By Lemma~\ref{lemma:ambiguity_of_J}, this implies that $R'$ and $c(R)$ differ by potential shaping and $S'$-redistribution. Thus $R$ and $c(R)$ differ by potential shaping and $S'$-redistribution. Finally, note that if $R_1$ and $R_2$ differ by potential shaping and $S'$-redistribution, then there is a constant $k$ such that $\eta^\pi \cdot R_2 = \eta^\pi \cdot R_1 + k$ for all policies $\pi$ (Proposition~\ref{prop:change_from_potentials}). This in turn means that $c(R_1) = c(R_2)$.
%\end{proof}

A STARC metric can use any canonicalisation function $c$. Moreover, the normalisation step can use any function $n$ that is a norm on $\mathrm{Im}(c)$. This does of course include the $L_1$-norm, $L_2$-norm, $L_\infty$-norm, and so on. We next show that $\max_\pi J(\pi) - \min_\pi J(\pi)$ also is a norm on $\mathrm{Im}(c)$:

\begin{restatable}[]{proposition}{Jnorm}
\label{prop:J_norm}
    If $c$ is a canonicalisation function, then the function $n : \R \to \R$ given by $n(R) = \max_\pi J(\pi) - \min_\pi J(\pi)$ is a norm on $\mathrm{Im}(c)$. Here $J$ is computed under the reward $R$ given as input to $c$.
\end{restatable}

For the final step we of course have that any norm is an admissible metric, though some other metrics are admissible as well. For example, if $m(R_1,R_2)$ is the \emph{angle} between $R_2$ and $R_2$ when $R_1,R_2 \neq R_0$, and we define $m(R_0,R_0) = 0$ and $m(R,R_0) = \pi/2$ for $R \neq R_0$, then $m$ is also admissible.
To obtain a STARC metric, we then pick any canonicalisation function $c$, norm $n$, and admissible metric $m$, and combine them as described in Definition~\ref{def:reward_metric}.

Some of our results will apply to \emph{any} pseudometric on $\R$, and most of our other results will apply to any pseudometric on $\R$ that is both sound and complete (including any STARC metric). However, to obtain \emph{specific} quantitative results, we will need to use a specific pseudometric. Therefore, we will use the following STARC metric as our \enquote{standard} pseudometric:

\begin{definition}\label{def:standard_starc}
    Let $\starc$ be the STARC metric for which $n$ is the $L_2$-norm, $c$ is the canonicalisation function that is minimal for the $L_2$-norm, and $m$ is the metric given by $m(x,y) = 0.5 \cdot L_2(x,y)$.  
\end{definition}

We will use $c^\mathrm{STARC}_{\tfunc, \gamma}$ to denote the canonicalisation function of $\starc$ (i.e., the minimal canonicalisation function for the $L_2$-norm). Moreover, we will use $s^\mathrm{STARC}_{\tfunc, \gamma} : \R \to \R$ to denote the function that is equal to 
$$
\left(\frac{c^\mathrm{STARC}_{\tfunc, \gamma}(R)}{L_2(c^\mathrm{STARC}_{\tfunc, \gamma}(R))}\right)
$$
when $L_2(c^\mathrm{STARC}_{\tfunc, \gamma}(R)) > 0$, and $c^\mathrm{STARC}_{\tfunc, \gamma}(R)$ otherwise.\footnote{This means that $\starc(R_1, R_2) = 0.5\cdot L_2(s^\mathrm{STARC}_{\tfunc, \gamma}(R_1), s^\mathrm{STARC}_{\tfunc, \gamma}(R_2))$.} We let $m$ be equal to half the $L_2$-distance, to ensure that $\starc$ is bounded between 0 and 1. We will use $\starc$ as our \enquote{standard} STARC metric as it is easy to work with when constructing proofs. However  this choice is not very consequential, 
since all STARC metrics are bilipschitz equivalent.
%; recall that all STARC metrics are sound and complete, and (as per Proposition~\ref{prop:bilipschitz}) all such pseudometrics are bilipschitz equivalent.

In ~\ref{appendix:understand_starc}, we provide a geometric intuition for how STARC metrics work. This will make it easier to understand STARC metrics, and may also make it easier to understand why STARC metrics induce bounds on worst-case regret.
          % DONE     
\section{Partial Identifiability}\label{sec:partial_identifiability}

In this section, we present our results about the partial identifiability of the reward function in IRL, relative to the standard behavioural models.
Specifically, we derive the ambiguity of the Boltzmann-rational model, the MCE model, and the optimality model. We also discuss the ambiguity tolerance of a number of different applications, and analyse the question of transfer learning. Note that our results in this section cover both the case where reward functions are compared using equivalence relations, and the case where they are compared using pseudometrics.
%First, we will characterise the ambiguity of several types of $Q$-functions and advantage functions, in terms of the reward transformations that preserve them. Using these results, we will then derive the ambiguity of the Boltzmann-rational model, the MCE model, and the optimality model. Finally, we will discuss the ambiguity tolerance of a number of different applications, and analyse the question of transfer learning.

\subsection{Invariances of Policies}

In this section, we derive the invariances of the three behavioural models that are most common in the current IRL literature. 
%\textcolor{red}{[please recall notions of transformation and compositions thereof, referring to relevant earlier subsection. ]}
These results will be expressed in terms the sets of reward transformations that we introduced in Section~\ref{sec:reward_transformations}.
%, namely Boltzmann-rational policies, MCE policies, and optimal policies. These results characterise the partial identifiability of the reward $R$ in the IRL problem for these behavioural models. 
We begin with the Boltzmann-rational model:

\begin{restatable}[]{theorem}{ambiguityboltzmannrational}
\label{thm:ambiguity-boltzmann-rational}
    For any transition function $\tfunc$, discount $\gamma$, and temperature $\beta$, we have that $b_{\tfunc, \gamma, \beta}$ determines $R$ up to $\PS \bigodot \SR$.
\end{restatable}

Stated differently, $\mathrm{Am}(b_{\tfunc, \gamma, \beta})$ is given by $\PS \bigodot \SR$, so two reward functions have the same Boltzmann-rational policy if and only if they differ by potential shaping and $S'$-redistribution. We next consider the MCE model:

\begin{restatable}[]{theorem}{ambiguityMCE}
\label{thm:ambiguity-MCE}
    For any transition function $\tfunc$, discount $\gamma$, and weight $\alpha$, we have that $c_{\tfunc, \gamma, \alpha}$ determines $R$ up to $\PS \bigodot \SR$.
\end{restatable}

Stated differently, $\mathrm{Am}(c_{\tfunc, \gamma, \alpha})$ is given by $\PS \bigodot \SR$, so two reward functions have the same MCE policy if and only if they differ by potential shaping and $S'$-redistribution. We next consider the optimality model:

\begin{restatable}[]{theorem}{ambiguityoptimal}
\label{thm:ambiguity-optimal}
    For any transition function $\tfunc$ and discount $\gamma$, we have that $o_{\tfunc, \gamma}^\star$ determines $R$ up to $\OP$.
\end{restatable}
%\begin{proof}
%Immediate from Theorem~\ref{thm:OPT_ambiguity}, since $o_{\tfunc, \gamma}^\star(R_1) = o_{\tfunc, \gamma}^\star(R_2)$ if and only if $R_1$ and $R_2$ have the same optimal policies.
%\end{proof}

Stated differently, $\mathrm{Am}(o_{\tfunc, \gamma}^\star)$ is given by $\OP$, so two reward functions have the same maximally supportive optimal policies if and only if they differ by an optimality-preserving transformation. 
These results exactly characterise the partial identifiability of the reward function $R$ under IRL which uses any of these three behavioural models. 

\subsection{Ambiguity Tolerance and Applications}

Now that we have derived the ambiguity of $R$ under each of the three standard behavioural models, it may be worth reflecting on the implications of these results. First of all, both $b_{\tfunc,\gamma,\beta}$ and $c_{\tfunc,\gamma,\alpha}$ determine $R$ up to $S'$-redistribution (with $\tau$) and potential shaping (with $\gamma$). From this, we can straightforwardly derive the following:

\begin{restatable}[]{corollary}{diameterofamforBRandMCE}
\label{cor:diameter_of_am_for_BR_and_MCE}
    If $f$ is $b_{\tfunc, \gamma, \beta}$ or $c_{\tfunc, \gamma, \alpha}$, then we have that:
    \begin{enumerate}
        \item $\mathrm{Am}(f) \refines \ORD$.
        \item $\mathrm{Am}(f) \refines \OPT$.
        \item If $d^\R$ is a pseudometric on $\R$ that is both sound and complete, then the upper and lower diameter of $\mathrm{Am}(f)$ under $d^\R$ is $0$.
    \end{enumerate} 
\end{restatable}

%\begin{proof}
%    As per Theorem~\ref{thm:ambiguity-boltzmann-rational} and \ref{thm:ambiguity-MCE}, if $f$ is either $b_{\tfunc, \gamma, \beta}$ or $c_{\tfunc, \gamma, \alpha}$, and $f(R_1) = f(R_2)$, then $R_1$ and $R_2$ differ by a transformation in $\PS \bigodot \SR$. As per theorem~\ref{thm:policy_ordering}, this implies that $R_1 \eq{\ORD} R_2$, and so $\mathrm{Am}(f) \refines \ORD$. Since $\ORD \refines \OPT$, we also have that $\mathrm{Am}(f) \refines \OPT$. Finally, as per Proposition~\ref{prop:sound_and_complete_means_distance_0_iff_same_order}, all sound and complete pseudometrics metrics have the property that $d^\R(R_1, R_2) = 0$ if $R_1 \eq{\ORD} R_2$. Thus the upper (and hence also the lower) diameter of $\mathrm{Am}(f)$ under $d^\R$ is $0$. 
%\end{proof}

This means that for any transition function $\tfunc$, any discount factor $\gamma$, and any true reward function $R^\star$, if an IRL algorithm $\mathcal{L}$ for Boltzmann-rational policies or MCE policies is trained on data that in fact is generated by a Boltzmann-rational policy or an MCE policy, then $\mathcal{L}$ will converge to a reward function $R_H$ such that $R^\star$ and $R_H$ have the same policy ordering (and optimal policies) under $\tfunc$ and $\gamma$, and that for any STARC metric $d^\R$, we have that $d^\R(R^\star, R_H) = 0$. This means that the ambiguity of these models is unproblematic.

Of course, there are a few caveats here that it is important to be cognisant of. First of all, this result relies on the assumption that the training data in fact comes from a Boltzmann-rational policy or MCE policy (i.e., that there is no misspecification). In reality, this assumption is unrealistic. In Sections~\ref{sec:misspecification_1} and \ref{sec:misspecification_2}, we will loosen this assumption. Moreover, we are only guaranteed that $R^\star$ and $R_H$ have the same policy order under $\tfunc$ and $\gamma$. In other words, we assume that $R_H$ will be applied in the same environment where it is learnt (or, stated differently, that there is no distributional shift after the learning process). In Section~\ref{sec:ambiguity_transfer_learning}, we will loosen this assumption, and see that we fail to obtain similar guarantees in that setting. Nonetheless, even with these caveats, Corollary~\ref{cor:diameter_of_am_for_BR_and_MCE} is still good news.

Our results also show that the invariances of $o_{\tfunc,\gamma}^\star$ preserve $\mathrm{OPT}_{\tau,\gamma}$. This is perhaps obvious --- the information that is contained in an optimal policy is of course sufficient to construct an optimal policy. Nonetheless, it is good to assimilate this result into our framework, and express it in the same terminology as our other results. Moreover, this result is of course also subject to the caveat that the training data in fact must come from an optimal policy, and the caveat that it only applies for the $\tfunc$ and $\gamma$ that were used during training. Next, we will show that the invariances of $o_{\tfunc,\gamma}^\star$ do \emph{not} preserve $\mathrm{ORD}_{\tau,\gamma}$, except in highly constrained environments:

\begin{restatable}[]{proposition}{optimalpoliciesdonotpreserveord}
\label{prop:optimal_policies_do_not_preserve_ord}
Unless $|\States| = 1$ and $|\Actions| = 2$, then for any $\tau$ and any $\gamma$ there are reward functions $R_1$, $R_2$ such that $o_{\tfunc,\gamma}^\star(R_1) = o_{\tfunc,\gamma}^\star(R_2)$ but $R_1 \not\eq{\mathrm{ORD}_{\tau,\gamma}} R_2$.
\end{restatable}

%\begin{proof}
%If $|\States| \geq 2$ or $|\Actions| \geq 3$, then there exists uncountably many reward functions that do not have the same ordering of policies (this is immediate from Theorem~\ref{thm:policy_ordering}). Moreover, $\mathrm{Im}(o_{\tau,\gamma}^\star)$ is finite. By the pigeonhole principle, this means that there must exist reward functions $R_1$, $R_2$ such that $o_{\tfunc,\gamma}^\star(R_1) = o_{\tfunc,\gamma}^\star(R_2)$ but $R_1 \not\eq{\mathrm{ORD}_{\tau,\gamma}} R_2$.
%\end{proof}

We can thus summarise our results about the ambiguity of $o_{\tfunc,\gamma}^\star$ as follows:

\begin{restatable}[]{corollary}{diameterofop}
\label{cor:diameter_of_op}
    Unless $|\States| = 1$ and $|\Actions| = 2$, we have that:
    \begin{enumerate}
        \item $\mathrm{Am}(o_{\tfunc,\gamma}^\star) \not\refines \ORD$.
        \item $\mathrm{Am}(o_{\tfunc,\gamma}^\star) \refines \OPT$.
        \item If $d^\R$ is a pseudometric on $\R$ that is both sound and complete, then the lower diameter of $\mathrm{Am}(o_{\tfunc,\gamma}^\star)$ under $d^\R$ is $0$, but the upper diameter is greater than $0$.
    \end{enumerate} 
\end{restatable}

%\begin{proof}
%    The first part follows from Proposition~\ref{prop:optimal_policies_do_not_preserve_ord}, and the second part follows from Theorem~\ref{thm:ambiguity-optimal}. For the third part, first note that Proposition~\ref{prop:sound_and_complete_means_distance_0_iff_same_order} implies that if $d^\R$ is both sound and complete, then $d^\R(R_1, R_2) = 0$ if and only if $R_1 \eq{\ORD} R_2$. Thus the fact that $\mathrm{Am}(o_{\tfunc,\gamma}^\star) \not\refines \ORD$ implies that the upper diameter of $\mathrm{Am}(o_{\tfunc,\gamma}^\star)$ under $d^\R$ is greater than $0$. To see that the lower diameter is $0$, consider the reward function $R_0$ that is $0$ everywhere. Then $o_{\tfunc,\gamma}^\star(R_0)$ must indicate that all actions are optimal in all states, which means any reward function $R$ such that $o_{\tfunc,\gamma}^\star(R) = o_{\tfunc,\gamma}^\star(R_0)$ must be trivial. All trivial reward functions have the same policy order, and so $d^\R(R,R_0) = 0$.
%\end{proof}

%Note that this is equivalent to saying that there are reward functions $R_1$, $R_2$ such that $o_{\tfunc,\gamma}^\star(R_1) = o_{\tfunc,\gamma}^\star(R_2)$ but $d^\R(R_1, R_2) > 0$, if $d^\R$ is a STARC-metric. However, 
Note that there are some special cases where $o_{\tfunc,\gamma}^\star(R)$ does allow us to infer the policy order of $R$, even if $|\States| \geq 2$ or $|\Actions| \geq 3$. As a simple example, if we have a one-state MDP with three actions $a_1,a_2,a_3$, and $o_{\tfunc,\gamma}^\star(R)$ shows that action $a_1$ and $a_2$ are optimal, then we can also infer the policy order of $R$. Alternatively, if $o_{\tfunc,\gamma}^\star(R)$ shows that all actions are optimal in all states, then all policies must have the same value --- this is why the lower diameter of $\mathrm{Am}(o_{\tfunc,\gamma}^\star)$ is $0$. Nonetheless, these cases are marginal, and in most situations, we will not be able to infer the policy order of $R$ from $o_{\tfunc,\gamma}^\star(R)$. Also note that the exact value of the upper diameter will depend on which pseudometric we use, as well as on the transition function $\tfunc$ and discount factor $\gamma$. Calculating this value exactly would be quite difficult, but we expect it to typically be quite large (since two reward functions may have the same optimal policies, and yet have wildly different policy orderings).

%Before moving on, l
Let us also briefly note that the behavioural models in $\mathcal{O}_{\tfunc,\gamma}$ other than $o_{\tfunc,\gamma}^\star$ 
%in fact 
are too ambiguous to identify even the correct equivalence class of $\OPT$:

\begin{restatable}[]{theorem}{nonmaxsupportiveoptveryambiguous}
\label{thm:non_max_supportive_opt_very_ambiguous}
If $o \in \mathcal{O}_{\tfunc,\gamma}$ but $o \neq o_{\tfunc,\gamma}^\star$, then $\mathrm{Am}(o_{\tfunc,\gamma}) \not\refines \OPT$.
\end{restatable}

As described in Section~\ref{sec:frameworks_ambiguity}, we can use the invariances of different reward objects to place them in a lattice structure, which graphically explains the relationship between their respective ambiguity and ambiguity tolerance --- see Figure~\ref{fig:ambiguity_lattice}. Other data sources can be placed in the same graph, using similar techniques to what we have used in this section.

\begin{figure}
\centering
\begin{tikzpicture}[shorten >=1pt,node distance=2.6cm,on grid,auto]
   \node[] (R)   {$R$}; 
   \node[] (Q) [below=of R, yshift=1cm] {$Q^\star$, $Q^\pi$, $\QSoft$};
   \node[] (A) [below=of Q, yshift=1cm] {$b_{\tfunc,\gamma,\beta}, c_{\tfunc,\gamma,\alpha}, A^\star, A^\pi$};
   \node[] (ORD) [below=of A, yshift=1cm] {$\ORD$};
   \node[] (OPT) [below=of ORD, yshift=1cm] {$\OPT$, $o^\star_{\tfunc,\gamma}$};
   \node[] (R_text)  [right=of R, xshift=1.5cm] {id};
   \node[] (Q_text)  [below=of R_text, yshift=1cm] {$\SR$}; 
   \node[] (A_text)  [below=of Q_text, yshift=1cm] {$\SR \bigodot \PS$}; 
   \node[] (ORD_text)  [below=of A_text, yshift=1cm] {$\SR \bigodot \PS \bigodot \LS$}; 
   \node[] (OPT_text)  [below=of ORD_text, yshift=1cm] {$\OP$}; 
    \path[->] 
    (R) edge [swap] node {} (Q)
    (Q) edge [swap] node {} (A)
    (A) edge [swap] node {} (ORD)
    (ORD) edge [swap] node {} (OPT)
    ;
\end{tikzpicture}
\caption{This figure summarises our results from Section~\ref{sec:partial_identifiability} (also incorporating some results from \ref{appendix:reward_transformation_properties}). On the left-hand side, we list several reward objects and equivalence relations on $\R$. We write $f \to g$ if $\mathrm{Am}(f) \refines \mathrm{Am}(g)$. Since ambiguity refinement is transitive and antisymmetric, this lets us place all reward objects in a lattice structure. Using this structure, we can read out several important relationships graphically: if $f \to g$, then a data source that is based on $g$ is at least as ambiguous as a data source based on $f$, the information contained in a data source based on $f$ is sufficient to derive the value of $g$ as an application, and it is in principle possible to compute $g$ based on $f$. Note that the lattice structure in this case forms a linear order --- this is a special property of the reward objects and equivalence relations we have studied, and does not hold in general. On the right-hand side of the figure we list the reward transformations that characterise the ambiguity of the reward objects to the left.}
%Arrow means computability, more ambiguity, and ambiguity tolerance. In this case linear order, but in general may not be. Transformations listed to the right.
\label{fig:ambiguity_lattice}
\end{figure}
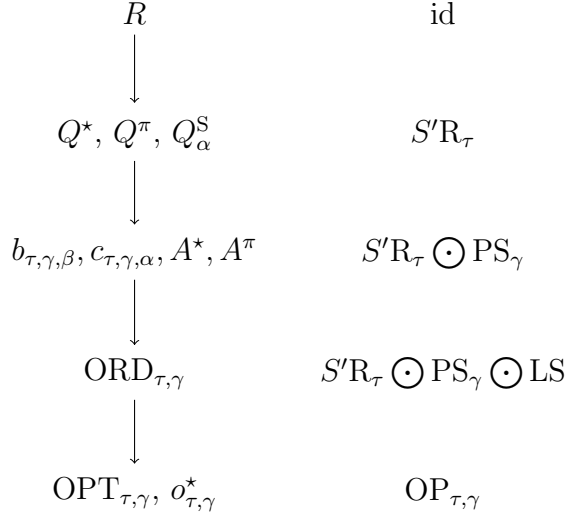

\subsection{Transfer Learning}\label{sec:ambiguity_transfer_learning}

It is interesting to consider the setting where a reward function is learnt in one MDP, but used in a different MDP.
For example, a quite natural setup is when we may learn the reward under one transition function $\tfunc_1$, but wish to use it under another transition function $\tfunc_2$. Alternatively, the observed agent may discount using one discount factor $\gamma_1$, but we wish to use the reward with a different discount factor $\gamma_2$. In this section, we will demonstrate that it is impossible to guarantee robust transfer in this setting. Our first result shows that a wide class of behavioural models are too ambiguous to ensure that the learnt reward has the same optimal policies as the true reward under transition dynamics that are different from those of the training environment:

\begin{restatable}[]{theorem}{wrongtautooambiguous}
\label{thm:wrong_tau_too_ambiguous}
If $f_{\tfunc_1}$ is invariant to $S'$-redistribution with $\tfunc_1$, and $\tfunc_1 \neq \tfunc_2$, then we have that $\mathrm{Am}(f_{\tfunc_1}) \not\refines \mathrm{OPT}_{\tfunc_2,\gamma}$.
\end{restatable}

%\begin{proof}
%Let $R_1$ be an arbitrary reward. If $\tfunc_1 \neq \tfunc_2$, then there exists some $s,a$ such that $\tfunc_1(s,a) \neq \tfunc_2(s,a)$. Using the construction in Lemma~\ref{lemma:s-redistribution_and_transfer}, we can find a reward function $R_2$ such that $R_1$ and $R_2$ differ by $S'$-redistribution with $\tfunc_1$, and such that $A^\star_2(s,a)$ has any arbitrary value when computed under $\tfunc_2$ (and any discount $\gamma$). In particular, if $a \not\in \mathrm{argmax}_{a'} A_1^\star(s,a')$ under $\tfunc_2$ and $\gamma$, then we can let $a \in \mathrm{argmax}_{a'} A_2^\star(s,a')$ under $\tfunc_2$ and $\gamma$, and vice versa. This means that $\mathrm{argmax}_{a} A_1^\star(s,a) \neq \mathrm{argmax}_{a} A_2^\star(s,a)$ under $\tfunc_2$ and $\gamma$, and so $R_1 \not\eq{\mathrm{OPT}_{\tfunc_2, \gamma}} R_2$. However, $R_1$ and $R_2$ differ by $S'$-redistribution with $\tfunc_1$, and so $f_{\tfunc_1(R_1)} = f_{\tfunc_1(R_2)}$.
%\end{proof}

Recall that if $\mathrm{Am}(f) \not\refines \OPT$ then $\mathrm{Am}(f) \not\refines \ORD$.
Also recall that $b_{\tfunc,\gamma,\beta}$, $c_{\tfunc,\gamma,\alpha}$, and $o^\star_{\tfunc,\gamma}$ all are invaraint to $S'$-redistribution, and hence subject to Theorem~\ref{thm:wrong_tau_too_ambiguous}.
We can also extend this result to a stronger statement, expressed in terms of Definition~\ref{def:ambiguity_diameter}. Recall that $\starc$ is the STARC metric described in Definition~\ref{def:standard_starc}:

\begin{restatable}[]{theorem}{wrongtaulargediameter}
\label{thm:wrong_tau_large_diameter}
    If $f_{\tfunc_1}$ is invariant to $S'$-redistribution with $\tfunc_1$, and $\tfunc_1 \neq \tfunc_2$, then the lower and upper diameter of $\mathrm{Am}(f_{\tfunc_1})$ under $d^\mathrm{STARC}_{\tfunc_2, \gamma}$ is 1.
\end{restatable}

Note that $1$ is the maximal distance that is possible under $\starc$. This result may be somewhat surprising, as if $\tfunc_1 \approx \tfunc_2$, then one might expect that a reward function that is learnt under $\tfunc_1$ should be guaranteed to be mostly accurate under $\tfunc_2$. We provide an intuitive explanation for why this is not the case in \ref{appendix:understand_transfer_learning}. Note also that Theorem~\ref{thm:wrong_tau_large_diameter} applies 
as long as
%for any two transition functions $\tau_1, \tau_2$ such that 
$\tau_1 \neq \tau_2$ \emph{in any state}, hence it is not required that $\tau_1 \neq \tau_2$ in \emph{every} state. 

We will next show that any behavioural model that is invariant to potential shaping is unable to guarantee transfer learning to a different discount factor $\gamma$. We say that $\tfunc$ is \emph{trivial} if for each $s \in \States$, $\tfunc(s,a) = \tfunc(s,a')$ for all $a,a' \in \Actions$. We can now state the following result:

\begin{restatable}[]{theorem}{wronggammatooambiguous}
\label{thm:wrong_gamma_too_ambiguous}
If $f_{\gamma_1}$ is invariant to potential shaping with $\gamma_1$, $\gamma_1 \neq \gamma_2$, and $\tfunc$ is non-trivial, then we have that $\mathrm{Am}(f_{\gamma_1}) \not\refines \mathrm{OPT}_{\tfunc,\gamma_2}$.
\end{restatable}

Note that if $\tfunc$ is trivial, then there can never be any situations where the agent has to decide between obtaining a smaller reward sooner or a greater reward later, which means that the discount factor has no impact on which policies are optimal.
This requirement is therefore necessary, although it is very mild.
We can also extend this result to a stronger statement, expressed in terms of Definition~\ref{def:ambiguity_diameter}: 

\begin{restatable}[]{theorem}{wronggammalargediameter}
\label{thm:wrong_gamma_large_diameter}
    If $f_{\gamma_1}$ is invariant to potential shaping with $\gamma_1$, $\gamma_1 \neq \gamma_2$, and $\tfunc$ is non-trivial, then the lower and upper diameter of $\mathrm{Am}(f_{\gamma_1})$ under $d^\mathrm{STARC}_{\tfunc, \gamma_2}$ is 1.
\end{restatable}

Again, recall that $1$ is the maximal distance that is possible under $\starc$. This result may also be surprising; if $\gamma_1 \approx \gamma_2$, then one might expect that a reward function that is learnt under $\gamma_1$ should be guaranteed to be mostly accurate under $\gamma_2$. We provide an intuitive explanation of this result in \ref{appendix:understand_transfer_learning}. Also recall that $b_{\tfunc,\gamma,\beta}$, $c_{\tfunc,\gamma,\alpha}$, and $o^\star_{\tfunc,\gamma}$ all are invaraint to potential shaping, and hence subject to Theorems~\ref{thm:wrong_gamma_too_ambiguous} and \ref{thm:wrong_gamma_large_diameter}.    % DONE
\section{Misspecification With Equivalence Relations}\label{sec:misspecification_1}

In this section, we present our results about how robust IRL is to misspecified behavioural models, using the formalisation provided by Definition~\ref{def:misspecification_eq}. First, we will derive necessary and sufficient conditions that describe all forms of misspecification that are tolerated by the Boltzmann-rational model, the MCE model, and the optimality model. In so doing, we will also define some broader equivalence classes of behavioural models that are internally robust to misspecification, and which can thus include more behavioural models than the standard three we are familiar with. After this, we will discuss how to generalise some of our results to even wider classes of behavioural models, and show that some of our results should be expected to apply with some universality, to be made more precise. After this, we will discuss the case where the environment model is misspecified, as well as the important issue of transfer learning. Most of our results in this section are expressed in terms of the two equivalence relations $\ORD$ and $\OPT$ on $\R$, which were introduced in Section~\ref{sec:comparing_reward_functions}.

\subsection{Necessary and Sufficient Conditions}\label{sec:misspecification_1_policies}

In this section, we will present necessary and sufficient conditions that describe all forms of misspecification that are tolerated by the Boltzmann-rational model, the MCE model, and the optimality model.

Let $\Pi^+$ be the set of all policies such that $\pi(a \mid s) > 0$ for all $s,a$, and let $F_{\tfunc,\gamma}$ be the set of all functions $f_{\tfunc,\gamma} : \mathcal{R} \rightarrow \Pi^+$ 
%such that if $\pi = f(R)$ then 
that, given $R$, return a policy $\pi$ that satisfies
\begin{equation*}
    \mathrm{argmax}_{a \in \Actions} \pi(a \mid s) = \mathrm{argmax}_{a \in \Actions} \QStar(s,a), 
\end{equation*}
%and $\pi(a \mid s) > 0$, for all $s,a$. 
where $\QStar$ is the optimal $Q$-function for $R$ under $\tfunc$ and $\gamma$.
In other words, $F_{\tfunc,\gamma}$ is the set of all functions that generate policies which take each action with positive probability, and that take the optimal actions with the highest probability. 
%This set includes Boltzmann-rational policies (with any choice of $\beta$), $\epsilon$-greedy policies (with any choice of $\epsilon$), Thompson sampling policies (from any sufficiently accurate posterior), and much more.
%Note that $F$ does \emph{not} include optimal policies (since we require $\pi(a \mid s) > 0$), nor does it include MCE policies (since they may not take the optimal action with the highest probability). It does also not include UCB policies (again since we require $\pi(a \mid s) > 0$). Nonetheless, the class is quite large.
This class is quite large, and includes e.g.\ Boltzmann-rational policies (for any $\beta$), but it does not include optimal policies (since they do not take all actions with positive probability) or MCE policies (since they may take suboptimal actions with higher probability).

\begin{restatable}[]{theorem}{boltzmannweakrobustness}
\label{thm:boltzmann_weak_robustness}
%Let $b_\beta : \mathcal{R} \rightarrow \Pi^+$ return the Boltzmann-rational policy of $R$ with temperature $\beta$. 
Let $f_{\tfunc,\gamma} \in F_{\tfunc,\gamma}$ be surjective onto $\Pi^+$. Then $f_{\tfunc,\gamma}$ is $\mathrm{OPT}_{\tfunc,\gamma}$-robust to misspecification with $g$ if and only if $g \in F_{\tfunc,\gamma}$ and $g \neq f_{\tfunc,\gamma}$.
\end{restatable}

Boltzmann-rational policies are surjective onto $\Pi^+$. To see this, note that if a policy $\pi$ takes each action with positive probability, then its action probabilities are always the softmax of some $Q$-function, and any $Q$-function corresponds to some reward function (via Equation~\ref{equation:optimal_Q_recursion}). Therefore, Theorem~\ref{thm:boltzmann_weak_robustness} exactly characterises all forms of misspecification to which the Boltzmann-rational model is $\mathrm{OPT}_{\tfunc,\gamma}$-robust. 

Let us briefly comment on the requirement that $\pi(a \mid s) > 0$, which corresponds to the condition that $\mathrm{Im}(g) \subseteq \mathrm{Im}(f)$ in Definition~\ref{def:misspecification_eq}. If a learning algorithm $\mathcal{L}$ is based on a model $f : \mathcal{R} \rightarrow \Pi^+$ then it assumes that the observed policy takes each action with positive probability in every state. What happens if such an algorithm $\mathcal{L}$ is given data from a policy that takes some action with probability $0$? This depends on $\mathcal{L}$, but for most sensible algorithms the result should simply be that $\mathcal{L}$ assumes that (or acts as if) those actions are taken with a positive but low probability. This means that it should be possible to drop the requirement that $\pi(a \mid s) > 0$ for many reasonable learning algorithms $\mathcal{L}$.

We next consider the misspecification to which the Boltzmann-rational model is $\mathrm{ORD}_{\tfunc,\gamma}$-robust. 
Let $\psi : \mathcal{R} \rightarrow \mathbb{R}^+$ be any function from reward functions to positive real numbers, and let $b_{\tfunc,\gamma,\psi} : \mathcal{R} \rightarrow \Pi^+$ be the function that, given $R$, returns the Boltzmann-rational policy with temperature $\psi(R)$ given transition function $\tfunc$ and discount $\gamma$. Moreover, let $B_{\tfunc,\gamma} = \{b_{\tfunc,\gamma,\psi} : \psi \in \mathcal{R} \rightarrow \mathbb{R}^+\}$ be the set of all such functions $b_{\tfunc,\gamma,\psi}$. %and let $b_\beta \in B$ be the function that, given $R$, returns the Boltzmann-rational policy of $R$ with the temperature $\beta$ (where $\beta \in \mathcal{R}^+$ is a constant that does not depend on $R$).
This set includes Boltzmann-rational policies; just let $\psi$ return a constant $\beta$ for all $R$.

\begin{restatable}[]{theorem}{boltzmannordrobustness}
\label{thm:boltzmann_ord_robustness}
For any $\beta > 0$, $b_{\tfunc,\gamma,\beta}$ is $\mathrm{ORD}_{\tfunc,\gamma}$-robust to misspecification with $g$ if and only if $g \in B_{\tfunc,\gamma}$ and $g \neq b_{\tfunc,\gamma,\beta}$.
%If $b_{\tfunc,\gamma,\psi} \in B_{\tfunc,\gamma}$ then $b_{\tfunc,\gamma,\psi}$ is $\mathrm{ORD}_{\tfunc,\gamma}$-robust to misspecification with $g$ if and only if $g \in B_{\tfunc,\gamma}$, given that $g \neq b_{\tfunc,\gamma,\psi}$ and $\mathrm{Im}(g) \subseteq \mathrm{Im}(b_{\tfunc,\gamma,\psi})$.
%be surjective onto $\Pi^+$
\end{restatable}

Recall that if $b_{\tfunc,\gamma,\beta}$ is the Boltzmann-rational model for temperature $\beta$ then $b_{\tfunc,\gamma,\beta}$ is surjective onto $\Pi^+$, which means that $\mathrm{Im}(g) \subseteq \mathrm{Im}(b_{\tfunc,\gamma,\beta})$ for all $g \in B_{\tfunc, \gamma}$. Theorem~\ref{thm:boltzmann_ord_robustness} thus says that the Boltzmann-rational model is $\mathrm{ORD}_{\tfunc,\gamma}$-robust to misspecification of the temperature parameter $\beta$, but not to any other form of misspecification (with the only complication being that the misspecification of $\beta$ is allowed to depend arbitrarily on the underlying reward function).%\footnote{Note that it is not quite the case that any $f \in B_{\tfunc,\gamma}$ is $\mathrm{ORD}_{\tfunc,\gamma}$-robust to misspecifcation with any $g \in B_{\tfunc,\gamma}$ with $f \neq g$, because it can in general be the case that $\mathrm{Im}(g) \not\subseteq \mathrm{Im}(f)$. However, this will never happen if the temperature of $f$ does not depend on the underlying reward function, which is the case for the Boltzmann-rational behavioural model.}
%
%misspecification with other Boltzmann-rational policies, but where the temperature parameter could be allowed to depend on $R$. %(which would effectively be the case if you normalise $R$, for example). 
%
We next turn our attention to optimal policies.

\begin{restatable}[]{theorem}{optimalpolicyrobustness}
\label{thm:optimal_policy_robustness}
For each $o \in \mathcal{O}_{\tfunc,\gamma}$, we have that $\mathrm{Am}(o) \not\refines \mathrm{ORD}_{\tfunc,\gamma}$, unless $|\States| = 1$ and $|\Actions| = 2$.
The only function $o \in \mathcal{O}_{\tfunc,\gamma}$ such that $\mathrm{Am}(o) \refines \mathrm{OPT}_{\tfunc,\gamma}$ is $o_{\tfunc,\gamma}^\star$, but there is no function $g$ such that $o_{\tfunc,\gamma}^\star$ is $\mathrm{OPT}_{\tfunc,\gamma}$-robust to misspecification with $g$.
\end{restatable}

This essentially means that the optimality model is not robust to any form of misspecification (regardless of whether that is measured using $\ORD$ or $\OPT$).
We finally turn our attention to causal entropy maximising policies. 
As before, let $\psi : \mathcal{R} \rightarrow \mathbb{R}^+$ be any function from reward functions to positive real numbers, and let $c_{\tfunc,\gamma,\psi} : \mathcal{R} \rightarrow \Pi^+$ be the function that, given $R$, returns the MCE policy with weight $\psi(R)$ given $\tfunc$ and $\gamma$. Furthermore, let $C_{\tfunc,\gamma} = \{c_{\tfunc,\gamma,\psi} : \psi \in \mathcal{R} \rightarrow \mathbb{R}^+\}$ be the set of all such functions $c_{\tfunc,\gamma,\psi}$. %and let $b_\beta \in B$ be the function that, given $R$, returns the Boltzmann-rational policy of $R$ with the temperature $\beta$ (where $\beta \in \mathcal{R}^+$ is a constant that does not depend on $R$).
Moreover, as usual, we let $c_{\tfunc,\gamma,\alpha} : \mathcal{R} \rightarrow \Pi^+$ be the function that, given $R$, returns the MCE policy with weight $\alpha$ given $\tfunc$ and $\gamma$. Also note that $c_{\tfunc,\gamma,\alpha} \in C_{\tfunc,\gamma}$ for each $\alpha$, since we may let $\psi$ return a constant $\alpha$ for all $R$.

\begin{restatable}[]{theorem}{mcestrongrobustness}
\label{thm:mce_strong_robustness}
For any $\alpha \in \mathbb{R}^+$, we have that $c_{\tfunc,\gamma,\alpha}$ is $\ORD$-robust to misspecification with $g$ if and only if $ g \in C_{\tfunc,\gamma}$ and $g \neq c_{\tfunc,\gamma,\psi}$.
\end{restatable}

In other words, the maximal causal entropy model is $\ORD$-robust to misspecification of the weight $\alpha$, but not to any other kind of misspecification (with the only complication being that the misspecification of $\alpha$ is allowed to depend arbitrarily on the underlying reward function). This is similar to what is the case for Boltzmann-rational policies.%\footnote{Note that Theorem~\ref{thm:mce_strong_robustness} only gives necessary and sufficient conditions for $\ORD$-robustness of maximal causal entropy models where $\alpha$ is constant (which is usually the case), whereas Theorem~\ref{thm:boltzmann_ord_robustness} also quantifies over behavioural models where $\beta$ may differ depending on the reward function. This restriction of Theorem~\ref{thm:mce_strong_robustness} makes the proof easier, and is unlikely to matter in a practical situation (Theorem~\ref{thm:boltzmann_ord_robustness} is arguably overly general). Nonetheless, generalising Theorem~\ref{thm:mce_strong_robustness} ought to be reasonably straightforward.}

Finally, let us briefly discuss the misspecification to which the maximal causal entropy model is $\mathrm{OPT}_{\tfunc,\gamma}$-robust. 
Lemma~\ref{lemma:P_robustness_function_composition} tells us that $c_{\tfunc,\gamma,\alpha}$ is $\mathrm{OPT}_{\tfunc,\gamma}$-robust to misspecification with $g$ if $g = c_{\tfunc,\gamma,\alpha} \circ t$ for some $t \in \OPT$.
In other words, if $g(R_1) = \pi$ then there must exist an $R_2$ such that $\pi$ maximises causal entropy with respect to $R_2$, and such that $R_1$ and $R_2$ have the same optimal policies. It seems hard to express this as an intuitive property of $g$, so we have refrained from stating this result as a theorem.

\subsection{Wider Classes of Policies}\label{sec:misspecification_1_wider_classes}

At this point, it is worth remarking on the fact that there are several noteworthy parallels between the invariances and the misspecification robustness of $b_{\tfunc,\gamma,\beta}$ and $c_{\tfunc,\gamma,\alpha}$. In particular, both are invariant to potential shaping and $S'$-redistribution, and no other transformations. Moreover, both are defined in terms of a parameter ($\beta$ or $\alpha$), and both are $\mathrm{ORD}_{\tfunc,\gamma}$-robust to misspecification of this parameter, and no other forms of misspecification. Additionally, misspecification of this parameter results in positive linear scaling of the learnt reward function. Is this a coincidence, or should we expect the same result to generalise to a wider class of behavioural models? Before moving on, we will discuss this question in some more depth.

First of all, $Q$-functions and advantage functions are invariant to $S'$-redistribution (Lemma~\ref{lemma:ambiguity_Q_function}-\ref{lemma:ambiguity_optimal_A_function} in \ref{appendix:reward_transformation_properties}). It is easy to show that value functions and policy evaluation functions, etc, also are invariant to $S'$-redistribution. This means that any behavioural model which can be computed via one of these objects also will share this invariance, as per Lemma~\ref{lemma:ambiguity_inherited}. More generally, since $S'$-redistribution does not change the expected value of any policy in any state, it is quite natural for a behavioural policy to be invariant to such transformations. It is also quite natural for a behavioural model to be invariant to potential shaping, if that behavioural model uses exponential discounting, considering the properties of potential shaping discussed in \ref{appendix:reward_transformation_properties}.

Lemma~\ref{lemma:less_ambiguity_less_robustness} tells us that for $f$ to be $\mathrm{ORD}_{\tfunc,\gamma}$-robust to some forms of misspecification, it has to be the case that $f$ is sensitive to some reward transformations which do not affect the policy order of the reward function. For example, $b_{\tfunc,\gamma,\beta}$ and $c_{\tfunc,\gamma,\alpha}$ are sensitive to positive linear scaling, even though this does not affect the policy order. Lemma~\ref{lemma:P_robustness_function_composition} then tells us that $f$ can be composed with these transformations, to produce the forms of misspecification that $f$ will tolerate. Composing $b_{\tfunc,\gamma,\beta}$ or $c_{\tfunc,\gamma,\alpha}$ with positive linear scaling is equivalent to scaling $\beta$ or $\alpha$; hence $b_{\tfunc,\gamma,\beta}$ and $c_{\tfunc,\gamma,\alpha}$ are $\mathrm{ORD}_{\tfunc,\gamma}$-robust to such misspecification. But is it reasonable to expect a behavioural model to be sensitive to some order-preserving transformations? We next show that if $f : \mathcal{R} \to \Pi^+$ is continuous, surjective onto $\Pi^+$, and satisfies $\mathrm{Am}(f) \refines \mathrm{ORD}_{\tfunc,\gamma}$, then $f$ must be sensitive to some such transformations:

\begin{restatable}[]{proposition}{continuousnotsurjective}
If $f : \mathcal{R} \to \Pi^+$ is continuous, and $f(R_1) = f(R_2)$ if and only if $R_1 \eq{\mathrm{ORD}_{\tfunc, \gamma}} R_2$, then $f$ is not surjective onto $\Pi^+$.
\end{restatable}

%\begin{proof}
%Assume for contradiction that $f : \mathcal{R} \to \Pi^+$ is continuous and surjective, and that $f(R_1) = f(R_2)$ if and only if $R_1 \eq{\mathrm{ORD}_{\tfunc, \gamma}} R_2$. Then $f$ is a continuous bijection from $\mathrm{Im}(s^\mathrm{STARC}_{\tfunc,\gamma})$ to $\Pi^+$, where $s^\mathrm{STARC}_{\tfunc,\gamma}$ is the standardisation function of $\starc$ (Definition~\ref{def:standard_starc}). Moreover, $\mathrm{Im}(s^\mathrm{STARC}_{\tfunc,\gamma})$ is compact (because it is a closed and bounded subset of a finite-dimensional Euclidean space), and $\Pi^+$ is Hausdorff. It thus follows that $f$ is a homeomorphism. This is a contradiction, since $\Pi^+$ is not homeomorphic to $\mathrm{Im}(s^\mathrm{STARC}_{\tfunc,\gamma})$. For example, $\mathrm{Im}(s^\mathrm{STARC}_{\tfunc,\gamma})$ contains an isolated point, which $\Pi^+$ does not.
%\end{proof}

Thus, if we want $f$ to be both continuous and surjective onto $\Pi^+$, then there must either be some order-preserving reward transformations to which $f$ is not invariant (i.e.\ $\mathrm{Am}(f) \refinesStrict \mathrm{ORD}_{\tfunc,\gamma}$), or $f$ must be invariant to some transformations which are not order preserving (i.e.\ $\mathrm{Am}(f) \not\refines \mathrm{ORD}_{\tfunc,\gamma}$). The latter case would imply that $f$ violates condition 3 in Definition~\ref{def:misspecification_eq}, and thus that $f$ is not robust to any misspecification. In the former case, which transformations would it be reasonable to pick? We next show that linear scaling is a natural choice:

\begin{restatable}[]{proposition}{continuousandinvarianttolinear}
If $f : \mathcal{R} \to \Pi^+$ is continuous and invariant to positive linear scaling, then $f(R_1) = f(R_2)$ for all $R_1$, $R_2$.
\end{restatable}

%\begin{proof}
%This is straightforward. Suppose $f : \mathcal{R} \to \Pi^+$ is continuous and invariant to positive linear scaling. Let $R_1, R_2$ be two arbitrary reward functions, and consider a sequence $c_t$ where $c_t > 0$, but $c_t \to 0$ as $t \to \infty$. Next, consider the sequence given by $c_t \cdot R_1$. Since $c_t \cdot R_1 \to R_0$ as $t \to \infty$, and since $f$ is continuous, we have that $f(c_t \cdot R_1) \to f(R_0)$ as $t \to \infty$. Moreover, since $f$ is invariant to positive linear scaling, we have that $f(c_t \cdot R_1) = f(R_1)$ for all $c_t$. This implies that $f(R_1) = f(R_0)$. By an analogous argument, we also have that $f(R_2) = f(R_0)$, and hence that $f(R_1) = f(R_2)$.
%\end{proof}

Of course, if $f(R_1) = f(R_2)$ for all $R_1$, $R_2$, then $f$ is completely trivial. Thus, a continuous behavioural model $f$ should not be (everywhere) invariant to positive linear scaling. Of course, it could be the case that $f$ is sensitive to positive linear scaling in the vicinity of $R_0$, but otherwise invariant to positive linear scaling, although this seems somewhat unnatural. This then suggests that if a behavioural model $f : \R \to \Pi^+$ is continuous, surjective, and satisfies $\mathrm{Am}(f) \refines \mathrm{ORD}_{\tfunc,\gamma}$, then it is natural for $f$ to be sensitive to positive linear scaling, in which case $f$ can be composed with positive linear scaling to produce forms of misspecification to which $f$ is robust.\footnote{Very roughly and informally, a set of reward functions in which no two reward functions share the same policy order will be one dimension short of being able to cover the set of all policies. Therefore, if $f$ \emph{does} cover all policies, then it must be sensitive to one \enquote{dimension} of order-preserving transformations. This, in turn, translates to one \enquote{dimension} of misspecification to which $f$ is $\mathrm{ORD}_{\tfunc, \gamma}$-robust.} This is exemplified by $b_{\tfunc,\gamma,\beta}$ and $c_{\tfunc,\gamma,\alpha}$, and the above argument suggests that we should expect a similar result to hold for many other behavioural models.

\subsection{Misspecified Parameters}\label{sec:misspecification_1_MDPs}

A behavioural model will typically be parameterised by a $\gamma$ or $\tfunc$, implicitly or explicitly. 
In this section, we explore what happens if these parameters are misspecified.
We show that a wide class of behavioural models lack robustness to this type of misspecification.

Theorems~\ref{thm:boltzmann_weak_robustness}-\ref{thm:mce_strong_robustness} already tell us that the standard behavioural models are not ($\mathrm{ORD}_{\tfunc,\gamma}$ or $\mathrm{OPT}_{\tfunc,\gamma}$) robust to misspecified $\discount$ or $\tfunc$, since the sets $F_{\tfunc,\gamma}$, $B_{\tfunc,\gamma}$, and $C_{\tfunc,\gamma}$, all are parameterised by $\discount$ and $\tfunc$.
We will generalise this further. First, we note that any behavioural model that is invariant to $S'$-redistribution will lack robustness to a misspecified $\tfunc$. Recall Theorem~\ref{thm:wrong_tau_too_ambiguous}; if $f_{\tfunc_1}$ is invariant to $S'$-redistribution with $\tfunc_1$, and $\tfunc_1 \neq \tfunc_2$, then we have that $\mathrm{Am}(f_{\tfunc_1}) \not\refines \mathrm{OPT}_{\tfunc_2,\gamma}$. Using this, we can prove the following:

\begin{restatable}[]{theorem}{notProbusttomisspecifiedt}
\label{thm:not_P_robust_to_misspecified_t}
If $f_{\tfunc}$ is invariant to $S'$-redistribution with $\tfunc$, and $\tfunc_1 \neq \tfunc_2$, then $f_{\tfunc_1}$ is not $\mathrm{OPT}_{\tfunc_3,\gamma}$-robust to misspecification with $f_{\tfunc_2}$ for any $\tfunc_3$ or $\gamma$.
\end{restatable}

%\begin{proof}
%Suppose for contradiction that $f_{\tfunc_1}$ is $\mathrm{OPT}_{\tfunc_3,\gamma}$-robust to misspecification with $f_{\tfunc_2}$. If $\tfunc_1 \neq \tfunc_2$, then $\tfunc_3 \neq \tfunc_1$, or $\tfunc_3 \neq \tfunc_2$, or both. Corollary~\ref{cor:wrong_tau_too_ambiguous} then implies that $\mathrm{Am}(f_{\tfunc_1}) \not\refines \mathrm{OPT}_{\tfunc_3,\gamma}$, or $\mathrm{Am}(f_{\tfunc_2}) \not\refines \mathrm{OPT}_{\tfunc_3,\gamma}$, or both. The former violates condition 3 in Definition~\ref{def:misspecification_eq}, and the latter violates Lemma~\ref{lemma:P_robustness_implies_refinement}. Thus $f_{\tfunc_1}$ cannot be $\mathrm{OPT}_{\tfunc_3,\gamma}$-robust to misspecification with $f_{\tfunc_2}$.
%\end{proof}

Recall that the three standard behavioural models are invariant to $S'$-redistribution, and thus subject to Theorem~\ref{thm:not_P_robust_to_misspecified_t}.
More generally, since $S'$-redistribution does not change the expected value of any policy in any state, it is quite natural for a behavioural model to be invariant to $S'$-redistribution. We should therefore expect Theorem~\ref{thm:not_P_robust_to_misspecified_t} to apply very broadly. Also recall that if $f_{\tfunc_1}$ is not $\mathrm{OPT}_{\tfunc_3,\gamma}$-robust to misspecification with $f_{\tfunc_2}$, then it is also not $\mathrm{ORD}_{\tfunc_3,\gamma}$-robust to misspecification with $f_{\tfunc_2}$.

Similarly, we can show that a behavioural model which is invariant to potential shaping will not be robust to misspecification of $\gamma$. Recall Theorem~\ref{thm:wrong_gamma_too_ambiguous}; if $f_{\gamma_1}$ is invariant to potential shaping with $\gamma_1$, $\gamma_1 \neq \gamma_2$, and $\tfunc$ is non-trivial, then $\mathrm{Am}(f_{\gamma_1}) \not\refines \mathrm{OPT}_{\tfunc,\gamma_2}$. This implies the following:

\begin{restatable}[]{theorem}{notProbusttomisspecifiedgamma}
\label{thm:not_P_robust_to_misspecified_gamma}
If $f_{\gamma}$ is invariant to potential shaping with $\gamma$, $\gamma_1 \neq \gamma_2$, and $\tfunc$ is non-trivial, then $f_{\gamma_1}$ is not $\mathrm{OPT}_{\tfunc,\gamma_3}$-robust to misspecification with $f_{\gamma_2}$ for any $\gamma_3$.
\end{restatable}

%\begin{proof}
%Suppose for contradiction that $f_{\gamma_1}$ is $\mathrm{OPT}_{\tfunc,\gamma_3}$-robust to misspecification with $f_{\gamma_2}$. If $\gamma_1 \neq \gamma_2$, then $\gamma_3 \neq \gamma_1$, or $\gamma_3 \neq \gamma_2$, or both. Theorem~\ref{thm:wrong_gamma_too_ambiguous} then implies that $\mathrm{Am}(f_{\gamma_1}) \not\refines \mathrm{OPT}_{\tfunc,\gamma_3}$, or $\mathrm{Am}(f_{\gamma_2}) \not\refines \mathrm{OPT}_{\tfunc,\gamma_3}$, or both. The former violates condition 3 in Definition~\ref{def:misspecification_eq}, and the latter violates Lemma~\ref{lemma:P_robustness_implies_refinement}. Thus $f_{\gamma_1}$ cannot be $\mathrm{OPT}_{\tfunc,\gamma_3}$-robust to misspecification with $f_{\gamma_2}$.
%\end{proof}

In other words, if a behavioural model is invariant to $S'$-redistribution, then that model is not $\mathrm{OPT}_{\tfunc,\gamma}$-robust (and therefore also not $\mathrm{ORD}_{\tfunc,\gamma}$-robust) to misspecification of the transition function $\tau$. 
Similarly, if the behavioural model is invariant to potential shaping, then that model is not $\mathrm{OPT}_{\tfunc,\gamma}$-robust (and therefore also not $\mathrm{ORD}_{\tfunc,\gamma}$-robust) to misspecification of the discount parameter $\gamma$. 
%Note that all transformations in $\SR$ and $\PS$ preserve the ordering of policies.
%This means that an IRL algorithm must specify $\tfunc$ and $\gamma$ correctly in order to guarantee that the learnt reward $R_H$ has the same optimal policies as the true underlying reward $R^\star$, unless the algorithm is based on a behavioural model which says that the observed policy depends on features of $R$ which do not affect its policy ordering. %$\mathbb{E}_{S' \sim \tau(s,a)}[R(s,a,S')]$ or the potential function.
These results should apply to most behavioural models. For example, we can derive the following corollary:

\begin{restatable}[]{corollary}{notProbustbroadcase}
\label{cor:not_P_robust_to_misspecified_gamma}
Let $f_{\tfunc, \gamma} : \R \to (\SxA \to \mathbb{R})$ be the function that, given a reward $R$, returns the optimal $Q$-function $Q^\star$ for $R$ under $\tfunc$ and $\gamma$. Suppose $g_{\tfunc,\gamma} = h \circ f_{\tfunc, \gamma}$ for some $h$, and let $\tfunc_1 \neq \tfunc_2$. Then $g_{\tfunc_1,\gamma}$ is not $\mathrm{OPT}_{\tfunc_3,\gamma}$-robust to misspecification with $g_{\tfunc_2,\gamma}$ for any $\tfunc_3$ or $\gamma$. 
\end{restatable}

In other words, any policy which can be derived from $Q^\star$ is not robust to misspecification of $\tfunc$. This is because $Q^\star$ already is invariant to $S'$-redistribution, and so this follows from Lemma~\ref{lemma:P_robustness_inheritance} and Theorem~\ref{thm:not_P_robust_to_misspecified_t}. Lemma~\ref{lemma:P_robustness_inheritance} could also be used to derive analogous results for other intermediate reward objects, such as those discussed in Appendix~\ref{appendix:intermediate_invariances}.

\subsection{Transfer Learning}\label{sec:misspecification_1_transfer_learning}

The equivalence relations we have worked with ($\mathrm{OPT}_{\tfunc,\gamma}$ and $\mathrm{ORD}_{\tfunc,\gamma}$) only guarantee that the learnt reward function $R_H$ has the same optimal policies, or ordering of policies, as the true reward $R^\star$ for a given choice of $\tfunc$ and $\gamma$. A natural question is what happens if we strengthen this requirement, and demand that $R_H$ has the same optimal policies, or ordering of policies, as $R^\star$, for any choice of $\tfunc$ or $\discount$. We briefly discuss this setting here.

In short, it is impossible to guarantee transfer to any $\tfunc$ or $\gamma$. This is already implied by the results in Section~\ref{sec:ambiguity_transfer_learning}. In particular, if $f_{\tfunc, \gamma}$ is invariant to $S'$-redistribution (with $\tfunc$) and potential shaping (with $\gamma$), then 
$$
\mathrm{Am}(f_{\tfunc_1, \gamma_1}) \not\refines \mathrm{OPT}_{\tfunc_2,\gamma_2}
$$
if either $\tfunc_1 \neq \tfunc_2$, or $\gamma_1 \neq \gamma_2$ and $\tfunc_2$ is non-trivial. Then $f_{\tfunc_1, \gamma_1}$ will violate condition 3 in Definition~\ref{def:misspecification_eq}.
          % DONE
\section{Misspecification With Metrics}\label{sec:misspecification_2}

In this section, we present our results about how robust IRL is to misspecified behavioural models, using the formalisation provided by Definition~\ref{def:misspecification_metric}. First, we will derive necessary and sufficient conditions that describe all forms of misspecification that the Boltzmann-rational model and the MCE model are robust to, and discuss the issue of how to derive similar results for the optimality model. 
After this, we analyse a particular form of misspecification, which we refer to as \emph{perturbation}, provide necessary and sufficient conditions for a behavioural model to be robust to such misspecification, and show that none of the three main behavioural models meet these conditions. After this, we will discuss the case where the environment model is misspecified, as well as the issue of transfer learning.
Section~\ref{sec:misspecification_2_nas} is quite dense, but \ref{sec:misspecification_2_perturbation} and \ref{sec:misspecification_2_parameters} both provide more intuitive takeaways. 

Our results in this section are expressed in terms of pseudometrics on $\R$. Most of these results apply for any choice of pseudometric, but when we need to select a specific pseudometric, we will use the newly introduced STARC metric $\starc$, as specified in Definition~\ref{def:standard_starc}.

%In Subsection~\ref{section:nas}, we will provide necessary and sufficient conditions that describe when a behavioural model is robust to a given type of misspecification. In Subsection~\ref{section:perturbation}, we will provide necessary and sufficient conditions that describe when a behavioural model is robust to pertubations of the policy, and in Subsection~\ref{section:misspecified_parameters}, we will discuss when a model is robust to misspecified parameters.

\subsection{Necessary and Sufficient Conditions}\label{sec:misspecification_2_nas}

%Recall that if $f : \R \to X$ is a behavioural model such that if $f(R_1) = f(R_2)$ then $d^\mathcal{R}(R_1, R_2) = 0$, then we can use Lemma~\ref{lemma:weak_epsilon_robustness_function_composition} to derive necessary and sufficient conditions for the types of misspecification that $f$ is robust to (as measured by $d^\R$). Also recall that if $d^\R$ is both sound and complete, then $d^\R(R_1,R_2) = 0$ if and only if $R_1$ and $R_2$ induce the same ordering of policies (Proposition~\ref{prop:sound_and_complete_means_distance_0_iff_same_order}).  Moreover, if $f$ is either $b_{\tfunc,\gamma,\beta}$ or $c_{\tfunc,\gamma,\alpha}$, then $f(R_1) = f(R_2)$ if and only if $R_1$ and $R_2$ differ by potential shaping with $\gamma$ and $S'$-redistribution with $\tfunc$ (Theorem~\ref{thm:ambiguity-boltzmann-rational} and \ref{thm:ambiguity-MCE}), and both of these transformations preserve the policy order under $\tfunc$ and $\gamma$ (Theorem~\ref{thm:policy_ordering}). This means that if $d^\R$ is sound and complete, then $b_{\tfunc,\gamma,\beta}$ and $c_{\tfunc,\gamma,\alpha}$ satisfy the assumptions for Lemma~\ref{lemma:weak_epsilon_robustness_function_composition}, and so we can use it to characterise the forms of misspecification that these models will tolerate.
If $f$ is either $b_{\tfunc,\gamma,\beta}$ or $c_{\tfunc,\gamma,\alpha}$, and $f(R_1) = f(R_2)$, then $\starc(R_1,R_2) = 0$. This means that $b_{\tfunc,\gamma,\beta}$ and $c_{\tfunc,\gamma,\alpha}$ satisfy the assumptions for Lemma~\ref{lemma:weak_epsilon_robustness_function_composition}, and so we can use it to characterise the forms of misspecification that these models will tolerate.
To do this, we need to find the set $T_\epsilon$ of all transformations $t : \R \to \R$ such that $\starc(R, t(R)) \leq \epsilon$ for all $R$:
%$d^\mathcal{R}(R, t(R)) \leq \epsilon$ for all $R$. We thus begin by deriving this set $T_\epsilon$, for the STARC metric $\starc$:

\begin{restatable}[]{proposition}{nasforsmallstarc}
\label{prop:nas_for_small_starc}
For any $\epsilon < 0.5$, $t : \mathcal{R} \to \mathcal{R}$ satisfies that 
$$
\starc(R,t(R)) \leq \epsilon
$$
for all $R \in \mathcal{R}$ if and only if $t$ can be expressed as $t_1 \circ t_2 \circ t_3$ where 
$$
L_2(R, t_2(R)) \leq L_2(c^\mathrm{STARC}_{\tfunc,\gamma}(R)) \cdot \sin(2 \arcsin (\epsilon))
$$
for all $R$, and where $t_1,t_3 \in \SR \bigodot \PS \bigodot \LS$.
\end{restatable}

The statement of Proposition~\ref{prop:nas_for_small_starc} is quite terse, so let us briefly unpack it. First of all, $\starc$ is invariant to any transformation that preserves the policy ordering of the reward function, and these transformations are exactly those that can be expressed as a combination of potential shaping, $S'$-redistribution, and positive linear scaling. As such, we can apply an arbitrary number of such transformations. Moreover, we can also transform $R$ in any way that does not change the standardised reward function $s^\mathrm{STARC}_{\tfunc,\gamma}(R)$ by more than $\epsilon$; this is equivalent to the stated condition on $t_2$. Note that $\sin(2 \arcsin (\epsilon)) \approx 2\epsilon$ for small $\epsilon$, so the right-hand side is approximately equal to $2\epsilon \cdot L_2(c^\mathrm{STARC}_{\tfunc,\gamma}(R))$. However, also note that $L_2(c^\mathrm{STARC}_{\tfunc,\gamma}(R)) \leq L_2(R)$.
%; this can make it easier to show that a transformation $t$ is \emph{not} in $T_\epsilon$ (but it cannot be used to show that some $t \in T_\epsilon$, because $L_2(c^\mathrm{STARC}_{\tfunc,\gamma}(R))$ can be zero even if $L_2(R) > 0$).
The requirement that $\epsilon < 0.5$ makes the calculation easier, and is included for convenience. Generalising Proposition~\ref{prop:nas_for_small_starc} by removing this requirement would be straightforward, but tedious. However, note that $\starc$ ranges between $0$ and $1$, so a $\starc$-distance greater than $0.5$ would be very large (arguably to the point of essentially being trivial).

Using this, we can now state necessary and sufficient conditions that completely characterise all types of misspecification that the Boltzmann-rational model and the MCE model will tolerate:

\begin{restatable}[]{corollary}{nasforbrmce}
\label{cor:nas_for_br_mce}
Let $\epsilon < 0.5$, and let $T_\epsilon$ be the set of all reward transformations $t : \R \to \R$ that satisfy Proposition~\ref{prop:nas_for_small_starc}. 
Let $f : \R \to \Pi$ be either $b_{\tfunc, \gamma, \beta}$ or $c_{\tfunc, \gamma, \alpha}$. 
Then $f$ is $\epsilon$-robust to misspecification with $g$ (as measured by $\starc$) if and only if $g = f \circ t$ for some $t \in \hat{T_\epsilon}$ such that $f \neq g$.
\end{restatable}
%\begin{proof}
%    Immediate from Lemma~\ref{lemma:weak_epsilon_robustness_function_composition}, Proposition~\ref{prop:nas_for_small_starc}, Theorem~\ref{thm:ambiguity-boltzmann-rational}, and Theorem~\ref{thm:ambiguity-MCE}.
%\end{proof}

In principle, Corollary~\ref{cor:nas_for_br_mce} completely describes the misspecification robustness of the Boltzmann-rational model and of the MCE model, as measured by $\starc$. However, the statement of Corollary~\ref{cor:nas_for_br_mce} is rather opaque, and difficult to interpret qualitatively. 
For this reason, we will in the subsequent sections examine a few important special types of misspecification, and derive results that are more intuitively intelligible. 

%,
%., and determine when a behavioural model is robust to such misspecification.

%\red{TODO: comment on the fact that the ambiguity of $o_{\gamma, \tfunc}$ is greater than $0$, given some assumptions about $\tfunc$. This means that for sufficiently small $\epsilon$, this model isn't robust to anything.}

%It is worth noting that, while Proposition~\ref{prop:nas_for_small_dard} and Corollary~\ref{cor:nas_for_br_mce} are stated for the case when $\R = \mathcal{R}$, it is straightforward to generalise them to arbitrary sets of reward functions. To do this, simply let $\hat{T_\epsilon} = \{t : t \in T_\epsilon, t(R) \in \R \text{ for } R \in \R\}$. Note that Theorem~\ref{thm:nas_conditions} holds for arbitrary $\R$. 

We should also briefly comment on the fact that Corollary~\ref{cor:nas_for_br_mce} does not cover $o_{\tfunc,\gamma}^\star$, i.e.\ the optimality model. The reason for this is that, unless $|\States| = 1$ and $|\Actions| = 2$, there are reward functions $R_1, R_2$ such that $o_{\tfunc,\gamma}^\star(R_1) = o_{\tfunc,\gamma}^\star(R_2)$, but $\starc(R_1, R_2) > 0$ (Corollary~\ref{cor:diameter_of_op}). This means that Lemma~\ref{lemma:weak_epsilon_robustness_function_composition} does not apply to $o_{\tfunc,\gamma}^\star$ when $d^\mathcal{R} = \starc$. Moreover:

\begin{restatable}[]{proposition}{optnotrobust}
\label{prop:opt_not_robust}
Let $d^\R$ be a pseudometric on $\R$ that is both sound and complete.
Then unless $|\States| = 1$ and $|\Actions| = 2$, there exists an $E > 0$ such that for all $\epsilon < E$, there is no behavioural model $g$ such that $o_{\tfunc,\gamma}^\star$ is $\epsilon$-robust to misspecification with $g$ as measured by $d^\R$.
\end{restatable}

%\begin{proof}
%By Corollary~\ref{cor:diameter_of_op}, if $|\States| \geq 2$ or $|\Actions| \geq 3$, and $d^\R$ is both sound and complete, then there exists reward functions $R_1, R_2$ such that $o_{\tau,\gamma}^\star(R_1) = o_{\tau,\gamma}^\star(R_2)$, and such that $d^\R(R_1, R_2) > 0$. Thus $d^\R(R_1, R_2) = E > 0$, and so $o_{\tau,\gamma}^\star$ violates condition 3 of Definition~\ref{def:misspecification_metric} for all $\epsilon < E$.
%\end{proof}

An analogous result will hold for any behavioural model $f$ and any pseudometric $d^\mathcal{R}$ for which $f(R_1) = f(R_2) \centernot\implies d^\mathcal{R}(R_1, R_2) = 0$.
Note that $E$ corresponds to the upper diameter of $\mathrm{Am}(o_{\tfunc,\gamma}^\star)$. This means that the exact value of $E$ will depend on the choice of pseudometric $d^\R$, and potentially also on the transition function $\tfunc$, discount $\gamma$, and initial state distribution $\init$.

\subsection{Perturbation Robustness}\label{sec:misspecification_2_perturbation}

It is interesting to know whether or not a behavioural model $f$ is robust to misspecification with any behavioural model $g$ that is \enquote{close} to $f$. But what does it mean for $f$ and $g$ to be \enquote{close}? One option is to say that $f$ and $g$ are close if they always produce similar policies. In this section, we will explore under what conditions $f$ is robust to such misspecification, and provide necessary and sufficient conditions. Our results are given relative to a pseudometric $d^\Pi$ on $\Pi$. For example, $d^\Pi(\pi_1,\pi_2)$ may be the $L_2$-distance between $\pi_1$ and $\pi_2$, or it may be the KL divergence between their trajectory distributions, or it may be the $L_2$-distance between their occupancy measures, and so on. As usual, our results apply for any choice of $d^\Pi$ unless otherwise stated. We can now define a notion of a \emph{perturbation} and a notion of \emph{perturbation robustness}:

\begin{definition}\label{def:perturbation}
Let $f, g : \R \to \Pi$ be two behavioural models, and let $d^\Pi$ be a pseudometric on $\Pi$. Then $g$ is a $\delta$-perturbation of $f$ if $g \neq f$ and for all $R \in \R$ we have that $d^\Pi(f(R),g(R)) \leq \delta$.
\end{definition}

\begin{definition}\label{def:perturbation_robustness}
Let $f : \R \to \Pi$ be a behavioural model, let $d^\mathcal{R}$ be a pseudometric on $\R$, and let $d^\Pi$ be a pseudometric on $\Pi$. Then $f$ is $\epsilon$-robust to $\delta$-perturbation if $f$ is $\epsilon$-robust to misspecification with $g$ (as measured by $d^\mathcal{R}$) for any behavioural model $g : \R \to \Pi$ that is a $\delta$-perturbation of $f$ (as defined by $d^\Pi$) with $\mathrm{Im}(g) \subseteq \mathrm{Im}(f)$.
\end{definition}

A $\delta$-perturbation of $f$ simply is any function that is similar to $f$ on all inputs, and $f$ is $\epsilon$-robust to $\delta$-perturbation if a small perturbation of the observed policy leads to a small error in the inferred reward function. It would be desirable for a behavioural model to be robust in this sense. To start with, this captures any form of misspecification that always leads to a small change in the final policy. Moreover, in practice, we can often not observe the exact policy of the demonstrator, and must instead approximate it from a number of samples. In this case, we should also expect to infer a policy that is a perturbation of the true policy. 
Before moving on, we need one more definition:

\begin{definition}\label{def:separating}
Let $f : \R \to \Pi$ be a behavioural model, let $d^\mathcal{R}$ be a pseudometric on $\R$, and let $d^\Pi$ be a pseudometric on $\Pi$. Then $f$ is $\epsilon/\delta$-separating if $d^\mathcal{R}(R_1, R_2) > \epsilon \implies d^\Pi(f(R_1), f(R_2)) > \delta$ for all $R_1, R_2 \in \R$.
\end{definition}

Intuitively speaking, $f$ is $\epsilon/\delta$-separating if reward functions that are far apart, are sent to policies that are far apart.\footnote{Note that this definition is \emph{not} saying that reward functions which are close must be sent to policies which are close. In other words, $f$ being $\epsilon/\delta$-separating is \emph{not} a continuity condition. It is also not a local property of $f$, but rather, a global property. It is, however, a continuity condition on the inverse of $f$.} Using this, we can now state our main result for this section:

\begin{restatable}[]{theorem}{perturbationrobustnessnas}
\label{thm:perturbation_robustness_nas}
Let $f : \R \to \Pi$ be a behavioural model, let $d^\mathcal{R}$ be a pseudometric on $\R$, and let $d^\Pi$ be a pseudometric on $\Pi$. Then $f$ is $\epsilon$-robust to $\delta$-perturbation (as defined by $d^\mathcal{R}$ and $d^\Pi$) if and only if $f$ is $\epsilon/\delta$-separating (as defined by $d^\mathcal{R}$ and $d^\Pi$).
\end{restatable}

We have thus obtained necessary and sufficient conditions that describe when a behavioural model is robust to perturbations --- namely, it has to be the case that this behavioural model sends reward functions that are far apart, to policies that are far apart. This ought to be quite intuitive; if two policies are close, then perturbations may lead us to conflate them. To be sure that the learnt reward function is close to the true reward function, we therefore need it to be the case that policies that are close always correspond to reward functions that are close (or, conversely, that reward functions which are far apart correspond to policies which are far apart). 

Our next question is, of course, whether or not the standard behavioural models are $\epsilon/\delta$-separating. Surprisingly, we will show that this is \emph{not} the case, when the distance between reward functions is measured using $\starc$, and the policy metric $d^\Pi$ is similar to Euclidean distance. Moreover, this applies to any continuous behavioural model:

%\red{[Generalise to arbitrary STARC metric?]}

\begin{restatable}[]{theorem}{notseparating}
\label{thm:not_separating}
Let $d^\mathcal{R}$ be $\starc$, and let $d^\Pi$ be a pseudometric on $\Pi$ which satisfies the condition that for all $\delta$ there exists a $\delta'$ such that if $L_2(\pi_1-\pi_2) < \delta'$ then $d^\Pi(\pi_1, \pi_2) < \delta$.
Let $f : \R \to \Pi$ be any continuous behavioural model. Then $f$ is not $\epsilon/\delta$-separating for any $\epsilon < 1$ or $\delta > 0$. 
\end{restatable}

%\begin{proof}
%Let $\delta$ be any positive constant. By assumption, there exists a $\delta'$ such that if $L_2(\pi_1 - \pi_2) < \delta'$ then $d^\Pi(\pi_1, \pi_2) < \delta$. Moreover, since $f$ is continuous, there exists an $\epsilon$ such that if $L_2(R_1 - R_2) < \epsilon$, then $L_2(f(R_1) - f(R_2)) < \delta'$. Next, let $R$ be any reward function that is non-trivial under $\tfunc$ and $\gamma$. We now have that, for any positive constant $c$, the reward functions $c \cdot R$ and $-c \cdot R$ have the opposite policy ordering, which means that $\starc(c \cdot R, - c \cdot R) = 1$. Moreover, by making $c$ sufficiently small, we can ensure that $L_2(c \cdot R + c \cdot R) < \epsilon$. Thus, for any positive $\delta$ there exist reward functions $c \cdot R$ and $-c \cdot R$ such that $d^\Pi(f(c \cdot R), f(-c \cdot R)) < \delta$, and such that $d^\mathrm{STARC}_{\tau,\gamma}(c \cdot R, -c \cdot R) = 1$. Hence $f$ is not $\epsilon/\delta$-separating for any $\delta > 0$ and any $\epsilon < 1$.
%\end{proof}

Note that the Boltzmann-rational model and the maximal causal entropy model (i.e.\ $b_{\tfunc,\gamma,\beta}$ and $c_{\tfunc,\gamma,\alpha}$) both are continuous, and hence subject to Theorem~\ref{thm:not_separating}.
The condition given on $d^\Pi$ in Theorem~\ref{thm:not_separating} is satisfied by any norm, but will also be satisfied by other metrics.\footnote{Note that while Theorem~\ref{thm:not_separating} uses a \enquote{special} pseudometric on $\mathcal{R}$, in the form of $\starc$, we do not need to use a special (pseudo)metric on $\Pi$, because for policies, $L_2$ does capture the relevant notion of similarity.} 

Intuitively, the fundamental reason for why Theorem~\ref{thm:not_separating} holds is that if $f$ is continuous, then it must send reward functions that are close under the $L_2$-norm to policies that are close under the $L_2$-norm. However, there are reward functions that are close under the $L_2$-norm but which have a large STARC distance. Hence $f$ will send some reward functions that are far apart (under $\starc$) to policies which are close, which means that $f$ is not $\epsilon/\delta$-separating. A similar result will hold for any other pseudometric $d^\R$ on $\R$ that is both sound and complete, if the upper bound on $\epsilon$ is replaced with the smallest distance between any two opposite reward functions under $d^\R$. Note that this distance is always $1$ under $\starc$.

It is worth noting that the proof of Theorem~\ref{thm:not_separating} only demonstrates that we may run into trouble for reward functions that are very close to $R_0$, and we may expect such reward functions to be unlikely (both in the sense that the observed agent is unlikely to have such a reward function, and in the sense that the learning algorithm is unlikely to generate such a hypothesis). It would therefore be natural to restrict $\R$ in some way, for example by imposing a minimum size on the $L_2$-norm of all considered reward functions, or by supposing that they are normalised. We will discuss this option further in \ref{appendix:generalising_analysis}, where we also give a generalisation of Theorem~\ref{thm:not_separating}.

\subsection{Misspecified Parameters}\label{sec:misspecification_2_parameters}

%A behavioural model is typically defined relative to some parameters. For example, the Boltzmann-rational model is defined relative to a temperature parameter $\beta$ and a discount parameter $\gamma$, as well as the transition dynamics $\tfunc$. Moreover, determining the exact values of these parameters ex post facto can often be quite difficult. For example, there is a sizeable literature that attempts to estimate the rate at which humans discount future reward, and there is a fairly large range in the estimates that this literature produces \citep[e.g.][]{Percoco2009EstimatingIR}. It is therefore interesting to know to what extent a behavioural model is robust to misspecification of its parameters.
In Section~\ref{sec:misspecification_1}, we showed that many behavioural models are not $\mathrm{OPT}_{\tfunc,\gamma}$-robust to any misspecification of $\tfunc$ or $\gamma$. However, this result says that we cannot identify the exact right optimal policies, or the exact right policy order, given misspecification of $\tfunc$ or $\gamma$. This does not rule out the possibility that a small misspecification of $\tfunc$ or $\gamma$ leads to a small (but nonzero) STARC distance between the true reward function and the learnt reward function. This is the question that we will investigate in this section.

First of all, recall that Lemma~\ref{lemma:epsilon_robustness_implies_weak_refinement} implies that if $f$ is $\epsilon$-robust to misspecification with $g$ (as measured by $d^\R$), and $g(R_1) = g(R_2)$, then we have that $d^\R(R_1, R_2) \leq 2\epsilon$. The converse of this statement is that if there are reward functions $R_1, R_2$ such that $g(R_1) = g(R_2)$ and $d^\R(R_1, R_2) > 2\epsilon$, then $f$ is \emph{not} $\epsilon$-robust to misspecification with $g$ (as measured by $d^\R$). Therefore, we can use the (upper) diameter of $\mathrm{Am}(g)$ to derive a limit on how robust any $f$ may be to misspecification with $g$.
Our results in this section will use this proof strategy.
We first consider the case when $\tfunc$ is misspecified:

%\red{[Generalise to use the lower diameter?]}

\begin{restatable}[]{theorem}{misspecifiedenv}
\label{thm:misspecified_env}
If $f_\tfunc : \mathcal{R} \to X$ is invariant to $S'$-redistribution with $\tfunc$, and $\tfunc_1 \neq \tfunc_2$, then $f_{\tfunc_1}$ is not $\epsilon$-robust to misspecification with $f_{\tfunc_2}$ under $d^\mathrm{STARC}_{\tfunc_3,\gamma}$ for any $\tfunc_3$, any $\gamma$, and any $\epsilon < 0.5$.
\end{restatable}

%\begin{proof}
%If $\tau_1 \neq \tau_2$, then either $\tau_1 \neq \tau_3$ or $\tau_2 \neq \tau_3$. If $\tau_1 \neq \tau_3$, then Theorem~\ref{thm:wrong_tau_large_diameter} implies that there exists reward functions $R_1, R_2$ such that $f_{\tau_1}(R_1) = f_{\tau_1}(R_2)$, but such that $d^\mathrm{STARC}_{\tfunc_3,\gamma}(R_1, R_2)$ is arbitrarily close to 1. Thus $f_{\tau_1}$ violates condition 2 of Definition~\ref{def:misspecification_metric} for all $\epsilon < 1$. Similarly, if $\tau_2 \neq \tau_3$, then Theorem~\ref{thm:wrong_tau_large_diameter} implies that there exists reward functions $R_1, R_2$ such that $f_{\tau_2}(R_1) = f_{\tau_2}(R_2)$, but such that $d^\mathrm{STARC}_{\tau_3,\gamma}(R_1, R_2)$ is arbitrarily close to 1. Then Lemma~\ref{lemma:epsilon_robustness_implies_weak_refinement} implies that there can be no $f$ that is $\epsilon$-robust to misspecification with $f_{\tau_2}$ (as defined by $d^\mathrm{STARC}_{\tau_3,\gamma}$) for any $\epsilon < 0.5$. 
%\end{proof}

%Note that Theorem~\ref{thm:misspecified_env} permits that $\tfunc_3 = \tfunc_1$ or $\tfunc_3 = \tfunc_2$.
%Recall that $\SR$ is the set of all $S'$-redistribution transformations. 
Theorem~\ref{thm:misspecified_env} is saying that if some behavioural model $f$ is invariant to $S'$-redistribution, then it is not robust to any degree of misspecification of $\tfunc$ (even if $\tfunc_1$ and $\tfunc_2$ are arbitrarily close).
Note that a $\starc$-distance of $0.5$ is very large; this corresponds to the case where the reward functions are nearly orthogonal. The greatest possible value of $\starc$ is 1.
Moreover, optimal policies, Boltzmann-rational policies, and maximal causal entropy policies, are all invariant to $S'$-redistribution,
and hence $o_{\tfunc, \gamma}^\star$, $b_{\tfunc, \gamma, \beta}$, and $c_{\tfunc, \gamma, \alpha}$ are subject to Theorem~\ref{thm:misspecified_env}.
%(Proposition~\ref{prop:policy_order} and \ref{prop:BR_MCE_ambiguity}). Therefore, Theorem~\ref{thm:misspecified_discounting} applies to $o_{\tfunc, \gamma}$, $b_{\tfunc, \gamma, \beta}$, and $c_{\tfunc, \gamma, \alpha}$. 
%Indeed, since $S'$-redistribution does not change the expected value of any policy, we should expect almost all sensible behavioural models to be invariant to $S'$-redistribution. As such, Theorem~\ref{thm:misspecified_env} will also apply very widely.
This means that Theorem~\ref{thm:not_P_robust_to_misspecified_t} generalises to the setting with distance metrics. Moreover, contrary to what we might expect, a small amount of misspecification of $\tfunc$ does not guarantee a small error in the learnt reward $R_H$, if this error is quantified with STARC metrics.

We next consider the case when the discount parameter, $\gamma$, is misspecified. As before, we say that a transition function $\tfunc$ is \emph{trivial} if for all states $s$ and all actions $a_1$, $a_2$, we have that $\tfunc(s,a_1) = \tfunc(s,a_2)$. 

%\red{[Generalise to use a "lower worst-case diamater"?]}

\begin{restatable}[]{theorem}{misspecifieddiscounting}
\label{thm:misspecified_discounting}
If $f_\gamma : \mathcal{R} \to \Pi$ is invariant to potential shaping with $\gamma$, and $\gamma_1 \neq \gamma_2$, then $f_{\gamma_1}$ is not $\epsilon$-robust to misspecification with $f_{\gamma_2}$ under $d^\mathrm{STARC}_{\tfunc,\gamma_3}$ for any non-trivial $\tfunc$, any $\gamma_3$, and any $\epsilon < 0.5$.
\end{restatable}

%\begin{proof}
%If $\gamma_1 \neq \gamma_2$, then either $\gamma_1 \neq \gamma_3$ or $\gamma_2 \neq \gamma_3$. If $\gamma_1 \neq \gamma_3$, then Theorem~\ref{thm:wrong_gamma_large_diameter} implies that there exists reward functions $R_1, R_2$ such that $f_{\gamma_1}(R_1) = f_{\gamma_1}(R_2)$, but such that $d^\mathrm{STARC}_{\tau,\gamma_3}(R_1, R_2)$ is arbitrarily close to 1. Thus $f_{\gamma_1}$ violates condition 2 of Definition~\ref{def:misspecification_metric} for all $\epsilon < 1$. Similarly, if $\gamma_2 \neq \gamma_3$, then Theorem~\ref{thm:wrong_gamma_large_diameter} implies that there exists reward functions $R_1, R_2$ such that $f_{\gamma_2}(R_1) = f_{\gamma_2}(R_2)$, but such that $d^\mathrm{STARC}_{\tau,\gamma_3}(R_1, R_2)$ is arbitrarily close to 1. Then Lemma~\ref{lemma:epsilon_robustness_implies_weak_refinement} implies that there can be no $f$ that is $\epsilon$-robust to misspecification with $f_{\gamma_2}$ (as defined by $d^\mathrm{STARC}_{\tau,\gamma_3}$) for any $\epsilon < 0.5$. 
%\end{proof}

%Note that Theorem~\ref{thm:misspecified_discounting} permits that $\gamma_3 = \gamma_1$ or $\gamma_3 = \gamma_2$.
Of course, any interesting environment will have a non-trivial transition function, so this requirement is very mild. %\footnote{If $\tfunc$ is trivial, then the environment is essentially a contextual Bandit problem. In such a problem, the actions that are taken in a given time step cannot affect what actions are available at a later time step. There can therefore be no trade-off between immediate reward and delayed reward, which means that the discounting is unimportant in such environments.} 
%Moreover, a $\starc$-distance of $0.5$ is very large; this corresponds to the case where the reward functions are nearly orthogonal. %, or one of them is trivial.\footnote{It is also sufficiently large to make the regret bound of $\starc$ vacuous; see Appendix~\ref{appendix:starc}.}
%Moreover, recall that $1$ is the greatest possible STARC-distance between any two reward functions, corresponding to the case when those reward functions induce the opposite ordering of policies. Setting $\epsilon$ to $1$ is therefore completely vacuous. 
%Also recall that $\PS$ is the set of all potential shaping transformations. 
This means that Theorem~\ref{thm:misspecified_discounting} is saying that if a behavioural model $f$ is invariant to potential shaping, then it is not robust to any misspecification of the discount parameter. Note that this holds even if $\gamma_1$ and $\gamma_2$ are arbitrarily close! Moreover, optimal policies, Boltzmann-rational policies, and MCE policies are all invariant to potential shaping,
and hence $o_{\tfunc, \gamma}$, $b_{\tfunc, \gamma, \beta}$, and $c_{\tfunc, \gamma, \alpha}$ are subject to  Theorem~\ref{thm:misspecified_discounting}.
%(see Proposition~\ref{prop:policy_order} and \ref{prop:BR_MCE_ambiguity}, and note that any transformation that preserves the policy order also must preserve optimal policies). Therefore, Theorem~\ref{thm:misspecified_discounting} applies to $o_{\tfunc, \gamma}$, $b_{\tfunc, \gamma, \beta}$, and $c_{\tfunc, \gamma, \alpha}$. 
%In general, we should expect any behavioural model that uses exponential discounting to be invariant to potential shaping, and so Theorem~\ref{thm:misspecified_discounting} will apply very widely.
This means that Theorem~\ref{thm:not_P_robust_to_misspecified_gamma} generalises to the setting with distance metrics. Moreover, contrary to what we might expect, a small amount of misspecification of $\gamma$ does not guarantee a small error in the learnt reward $R_H$, if this error is quantified with STARC metrics.

\subsection{Transfer Learning}\label{sec:misspecification_2_transfer_learning}

STARC metrics, such as $\starc$, are designed to be sound and complete. Moreover, our definitions of soundness and completeness for a pseudometric $d^\S$ require that $d^\R(R_1, R_2)$ is small if and only if the regret of using $R_1$ instead of $R_2$ is small, relative to a particular transition function $\tfunc$ and discount factor $\gamma$. A natural question is what happens if we strengthen this requirement, and demand that the regret is small for any choice of $\tfunc$ or any choice of $\discount$. We briefly discuss this setting here.

In short, as for Definition~\ref{def:misspecification_eq}, it is impossible to guarantee transfer to any $\tfunc$ or $\gamma$. This is already implied by the results in Section~\ref{sec:ambiguity_transfer_learning}. In particular, if $f_{\tfunc, \gamma}$ is invariant to $S'$-redistribution (with $\tfunc$) and potential shaping (with $\gamma$), then the (upper and lower) diameter of $\mathrm{Am}(f_{\tfunc_1, \gamma_1})$ under $d^\mathrm{STARC}_{\tfunc_2,\gamma_2}$ is 1, provided that either $\tfunc_1 \neq \tfunc_2$, or $\gamma_1 \neq \gamma_2$ and $\tfunc_2$ is non-trivial. Then $f_{\tfunc_1, \gamma_1}$ will violate condition 3 in Definition~\ref{def:misspecification_metric}. Moreover, note that this result is not specific to $\starc$, and that a similar result will hold for any pseudometric on $\R$ that is both sound and complete.          % DONE
\section{Discussion}

%In this section, we discuss our contributions. 
Here we  provide a discussion of the impact and significance of our results, 
their limitations, and how they may be extended.
%and of what repercussions can be drawn from them. We also offer a detailed discussion of the limitations of our results, and of they may be extended.

\subsection{Impact and Significance}

We have shown that both the partial identifiability as well as the misspecification robustness of behavioural models in IRL can be quantified and understood. Specifically, we have fully characterised the partial identifiability (or the \emph{ambiguity}) of the reward function under the three standard behavioural models, and we have fully characterised all forms of misspecification that these behavioural models are robust to. Moreover, we have shown that these results can be used to gain an intuitive insight into the practical consequences of partial identifiability and misspecification in IRL.

Our results show that the ambiguity of the reward function under the Boltzmann-rational model and the MCE model is low as long as the learnt reward function is evaluated in the training environment, whereas the ambiguity under the optimality model is larger. Moreover, and perhaps surprisingly, we have shown that each of these models can be too ambiguous to guarantee that the learnt reward function robustly leads to desirable behaviour in new environments (i.e., in environments where the transition function $\tfunc$ or the discount factor $\gamma$ differ from the training environment). 

We have shown that the optimality model lacks robustness to any kind of misspecification, whereas both the Boltzmann-rational model and the MCE model are robust to several forms of misspecification --- the exact forms of misspecification they are robust to depends on how we quantify the error in the learnt reward function. However, we have shown that none of these behavioural models are robust to even slight misspecification of the transition function $\tfunc$ or the discount function $\gamma$. Moreover, we have shown that very minimal assumptions about the behavioural model are needed to obtain this negative result, 
which means that new behavioural models are likely to also have this limitation.
%which means that this outcome is likely to generalise to new behavioural models as well. 
We find this quite surprising, as in the reinforcement learning literature the discount $\gamma$ is typically selected in a somewhat arbitrary way, and it can often be difficult to establish post-facto which $\gamma$ was used to compute a given policy. The fact that $\tfunc$ must be specified correctly is somewhat less surprising \citep[considering, for instance, the examples discussed by ][]{choicesetmisspecification}, yet important to have established. We have also shown that none of these behavioural models are robust to arbitrarily small perturbations of the observed policy. We have similarly needed very minimal assumptions about the behavioural model to obtain this result, which again means that this result also is likely to generalise to new behavioural models. 

In addition to these contributions, we have also derived several powerful mathematical tools that can be used in the analysis of reward learning algorithms. First of all, in Section~\ref{sec:comparing_reward_functions}, we have provided a wide range of useful results about the properties of reward functions. Notably, we have introduced STARC metrics, shown that these pseudometrics induce both an upper and a lower bound on worst-case regret (see Definitions~\ref{def:soundness} and \ref{def:STARC_complete}), and that they are unique in doing so. Thus, STARC metrics are an appropriate tool for analysing the properties and performance of reward learning algorithms. We have also provided necessary and sufficient conditions that describe when two reward functions have the same optimal policies, or the same ordering of policies, and we have elucidated the properties of many important forms of reward transformations. In addition to this, we have provided several powerful lemmas in Section~\ref{sec:frameworks}, that are useful for proving results about partial identifiability and misspecification robustness. We expect these results and contributions to be useful for further theoretical analysis of reward learning algorithms, beyond the analysis that we have carried out in this paper.

Our analysis provides a first step towards answering the more general question of how sensitive IRL is to misspecification of the behavioural model. Our results show that a very wide range of behavioural models --- including all the three behavioural models that are most common in the current IRL literature --- can be highly sensitive to some types of misspecification,  namely misspecification of the transition function $\tfunc$ or discount factor $\gamma$, or perturbations of the observed policy.  
These results 
%are not fully conclusive, they do 
indicate that IRL in general can be highly sensitive to misspecification of the behavioural model. 
%If this is the case, then 
This provides a cautionary lesson on the prospects of IRL as a tool for accurate preference elicitation:  
the relationship between human preferences and human behaviour is very complex, and while it is certainly possible to create increasingly accurate models of human behaviour, it will never be realistically possible to create a behavioural model that is completely free forms of misspecification. Therefore, if IRL is unable to guarantee accurate inferences under even mild misspecification of the behavioural model, then we should expect to be very difficult 
%(and perhaps even prohibitively difficult) 
to guarantee that IRL reliably will produce accurate inferences in real-world situations. 
%Further analysis is needed, but 
Our results thus suggest that IRL should be used cautiously, and that the learnt reward functions should be carefully evaluated \citep[as done by e.g.\ ][]{michaud2020understanding,jenner2022preprocessing}. This also means that we need IRL algorithms that are specifically designed to be more robust to misspecification, such as e.g.\ that proposed by \citet{viano2021robust}. It may also be fruitful to combine IRL with other data sources, as done by e.g.\ \cite{ibarz2018}, or consider policy optimisation algorithms that conservatively assume that the reward may be misspecified, as done by e.g.\ \citet{krakovnaside, krakovnaside2, turnerside, griffin2022alls}.

\subsection{Limitations and Further Work}

There are several ways to extend our work. First of all, our analysis has primarily focused on the three behavioural models that are most common in the current IRL literature (namely optimality, Boltzmann-rationality, and MCE optimality). One way to extend our work is to consider broader classes of behavioural models, or behavioural models that are more realistic. For example, there is an extensive body of work in the behavioural sciences that suggests that human behaviour (and that of many other animals) is better modelled using hyperbolic discounting, rather than exponential discounting \cite[for example, see ][]{inconsistencyempirical,mazur1987adjusting, green1996exponential, againstnormativediscounting, discounting_review}. As the three standard behavioural models are all based on exponential discounting,  
it would be interesting to extend our analysis to behavioural models that are based on hyperbolic discounting (or alternative kinds of discounting). 
Similarly, humans typically exhibit risk-averse behaviour, according to which losses are given a greater weight than gains, but this is likewise not modelled by any of the three standard behavioural models. It would thus also be interesting to extend our analysis to behavioural models that incorporate current models of human risk-aversion, such as \emph{prospect theory} \citep{prospecttheory}. 
%\textcolor{red}{Other models of human decision-making could also be considered. [cut if there is no no explicit example/citation]}
Alternatively, our analysis could also be extended by deriving results that apply to very wide classes of behavioural models, obtained by raising minimal assumptions 
%about the behavioural model 
(as we do in e.g.\ Section~\ref{sec:ambiguity_transfer_learning} and \ref{sec:misspecification_1_wider_classes}). 

Another way to extend our work is to consider other equivalence relations or other pseudometrics on $\R$. Much of our analysis has been based on the equivalence relations given by $\ORD$ and $\OPT$, as well as on STARC metrics. 
%but there may be other equivalence relations or pseudometrics that it would also be interesting to consider. 
Note that we have shown that any pseudometric on $\R$ that gives rise to both an upper and a lower bound on worst-case regret must be bilipschitz equivalent to STARC metrics, so these pseudometrics must have a degree of canonicity. However, our definition of regret is quite strong: it may thus be possible to create more permissive pseudometrics, by allowing them to induce guarantees that are weaker than a worst-case regret bound.

Next, our work has assumed that the environment 
%of the observed agent 
is described by a single-agent MDP. An interesting extension would be to consider more general classes of environments, such as multi-agent environments, environments with partial observability, environments with non-Markovian dynamics \citep{limitations_of_Markov_rewards},  
%We could also consider more exotic classes of environments, such as 
or environments where the actions of the agent may be predicted in advance \citep[as done by e.g.\ ][]{newcomb_rl}.
Note that such extensions would require the results in Section~\ref{sec:comparing_reward_functions} to be generalised as well.

Another interesting direction for future work is to extend our analysis in \ref{appendix:generalising_analysis}, by more carefully considering the consequences of imposing restrictions on the set of rewards $\R$,
or the consequences of using a \emph{probability distribution} over $\R$, and demanding that the learnt reward $R_H$ is close to the true reward $R^\star$ \emph{with high probability}. %(rather than with certainty)
%\textcolor{red}{or the consequences of making the analysis more probabilistic [rather qualitative statement]}. 
We provide a range of preliminary results regarding these settings in \ref{appendix:generalising_analysis}.  
%However, these results are somewhat preliminary, and it may be possible to draw additional interesting conclusions if this setting is explored more extensively. 

Furthermore, our analysis primarily concerns the asymptotic behaviour of IRL algorithms, in the limit of infinite data. Thus an interesting extension could study the properties of IRL algorithms in the case of finite data \citep[as done by e.g.\ ][]{towardstheoreticalunderstandingofIRL}. Finally, whilst our analysis has been theoretical, it could be insightful to study the impact of misspecification in IRL from an empirical angle \citep[as done by e.g.\ ][]{chan2021human}.                  % DONE                    



\bibliographystyle{elsarticle-harv} 
\bibliography{bibliography.bib}

\newpage
\appendix
\section{Motivating and Generalising Our Frameworks}\label{appendix:generalising_analysis}

This appendix has two purposes. The first purpose is to provide additional motivation for our core definitions provided in Section~\ref{sec:frameworks}, beyond what could be given in the main text. In particular, we will provide an extended explanation of the third condition in Definitions~\ref{def:misspecification_eq} and \ref{def:misspecification_metric}, and also discuss the assumption that behavioural models are \emph{functions}. The second purpose of this section is to discuss how to generalise our analysis, and extend the frameworks presented in Section~\ref{sec:frameworks}. In particular, the definitions that we have worked with so far quantify over all reward functions, both for the true reward function $R^\star$ and the learnt reward function $R_H$. In some cases, we may have some prior knowledge about the true reward function $R^\star$, or we may know that the inductive bias of the learning algorithm is unlikely to generate certain reward functions $R_H$, even if they are compatible with the training data. Consequently, we may wish to incorporate assumptions about the true reward or about the inductive bias of the learning algorithm into our analysis. We will discuss these extensions, and show that our analysis remains largely unchanged by such generalisations. We will also discuss some alternative equivalence relations on $\R$.

\subsection{Explaining the Third Condition For Misspecification Robustness}\label{appendix:additional_comments_on_definitions}

In this section, we provide some additional discussion regarding the third condition in Definition~\ref{def:misspecification_eq} and Definition~\ref{def:misspecification_metric}. This condition informally says that for $f$ to be robust to misspecification with $g$, it is necessary that a learning algorithm which is based on $f$ should be guaranteed to learn a reward function that is close to the true reward function when there is no misspecification. It may not be immediately obvious why this assumption is included, since we assume that the data is generated by $g$, where $f \neq g$. We will explain the motivation for this condition in more detail.

Let $R_1, R_2, R_3, R_4$ be four reward functions such that $d^\mathcal{R}(R_1, R_2) < \epsilon$, $d^\mathcal{R}(R_3, R_4) < \epsilon$, and $d^\mathcal{R}(R_2, R_3) \gg \epsilon$ (or, alternatively, $R_1 \eq{P} R_2$, $R_3 \eq{P} R_4$, and $R_2 \not\eq{P} R_3$). Moreover, let $f,g : \R \to \Pi$ be two behavioural models where $f(R_1) = \pi_1$, $f(R_2) = f(R_3) = \pi_2$, $f(R_4) = \pi_3$, and $g(R_1) = g(R_2) = \pi_1$, $g(R_3) = g(R_4) = \pi_3$. We may assume that $f = g$ for all other reward functions. This is illustrated in the diagram below:

\begin{figure}[H]
    \centering
    \includegraphics[width=\textwidth/2]{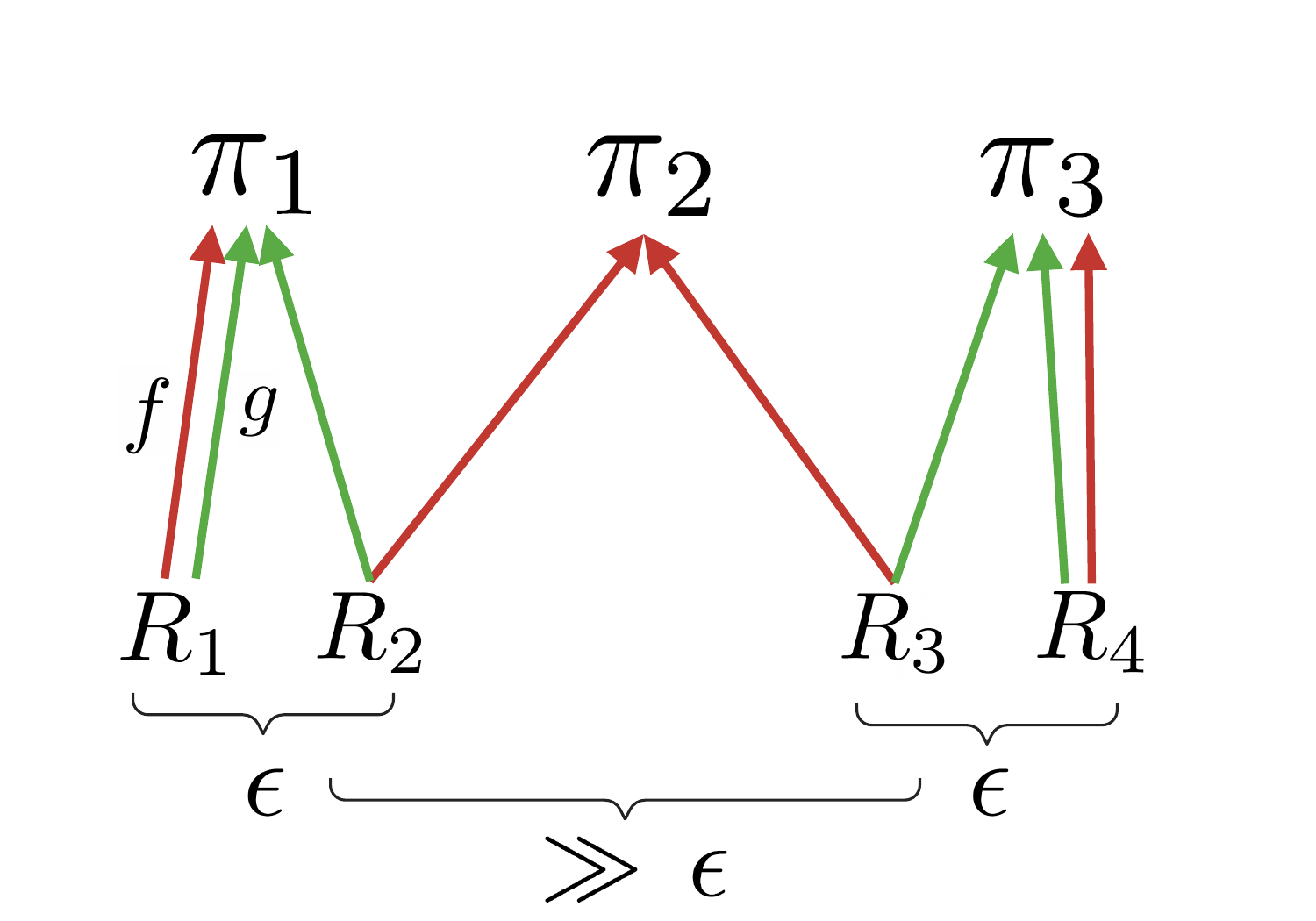}
    %\caption{Caption}
    %\label{fig:grid}
\end{figure}

In this case, we have that $f(R_2) = f(R_3)$, but $d^\mathcal{R}(R_2, R_3) \gg \epsilon$. As such, $f$ violates the third condition in Definition~\ref{def:misspecification_metric}; a learning algorithm $\mathcal{L}$ based on $f$ is \emph{not} guaranteed to learn a reward function that has distance at most $\epsilon$ to the true reward function when there is no misspecification, because $f$ cannot distinguish between $R_2$ and $R_3$, which have a large distance. However, if $f(R) = g(R')$, it does in this case follow that $d^\mathcal{R}(R,R') \leq \epsilon$. 
In other words, if the training data is coming from $g$, then a learning algorithm $\mathcal{L}$ based on $f$ \emph{is} guaranteed to learn a reward function that has distance at most $\epsilon$ to the true reward function.
%, because $g$ does not conflate any reward functions that have a large distance, and $f$ maps each reward back to the right cluster. 
As such, we could define misspecification robustness in such a way that $f$ would be considered to be robust to misspecification with $g$ in this case. However, this seems unsatisfactory, because $g$ essentially has to be carefully designed specifically to avoid certain blind spots in $f$. %We have that $d^\mathcal{R}(R,R') \leq \epsilon$ for any $R, R'$ such that $f(R) = g(R')$, but there is something very spurious and brittle about this relationship. 
In other words, while condition 1 in Definition~\ref{def:misspecification_eq}/\ref{def:misspecification_metric} is met, it is only met \emph{spuriously}.
%, and not \emph{for the right reason}. 
The third condition is included to rule out these edge cases.

\subsection{On the Functionality of Behavioural Models}\label{appendix:functional_behavioural_models}

We should also comment on the fact that behavioural models are assumed to be \emph{functions}; i.e., we assume that a behavioural model associates each reward function $R$ with a unique policy $\pi$. This is true for the Boltzmann-rational model and the maximal causal entropy model, but it may not be a natural assumption in all cases. For example, there may in general be more than one optimal policy. Thus, an optimal agent could associate some reward functions $R$ with multiple policies $\pi$. This particular example is not too problematic, because all optimal policies still form a convex set. As such, it is natural to assume that an optimal agent would take all optimal actions with equal probability, which is what we have done in the definition of $o_{\tau,\gamma}^\star$. However, we could imagine alternative criteria which would associate some rewards with multiple policies, and where there may not be any canonical way to select a single policy among them. Such criteria may then not straightforwardly translate into a functional behavioural model.

There are several ways to handle such cases within our framework. First of all, we may simply assume that the observed agent still has some fixed method for breaking ties between policies that it considers to be equivalent. In that case, we still ultimately end up with a function from $\mathcal{R}$ to $\Pi$, in which case our framework can be applied without modification. We expect this approach to be satisfactory in most cases.

It is worth noting that this approach does not necessarily require us to actually know how the observed agent breaks ties between equivalent policies. To see this, let $G : \mathcal{R} \to \mathcal{P}(\Pi)$ be a function that associates each reward function with a set of policies. We can then say that a behavioural model $g : \mathcal{R} \to \Pi$ \emph{implements} $G$ if $g(R) \in G(R)$ for all $R \in \mathcal{R}$. Using this definition, we could then say that $f : \mathcal{R} \to \Pi$ is robust to misspecification with $G : \mathcal{R} \to \mathcal{P}(\Pi)$ if $f$ is robust to misspecification with each $g$ that implements $G$, where $f$ being robust to misspecification with $g$ is defined as in Definition~\ref{def:misspecification_eq} or \ref{def:misspecification_metric}. In other words, we assume that the observed agent has a fixed method for breaking ties between policies in $G$, but without making any assumptions about what this method is. Using that definition, our framework can then be applied without modification.

An alternative approach is to generalise the definition of behavioural models to allow them to return a set of policies, i.e.\ $f : \mathcal{R} \to \mathcal{P}(\Pi)$. Most of our results can be extended to cover this case in a mostly straightforward manner, since many results do not make any assumptions about the co-domain of the reward objects. However, this approach is somewhat unsatisfactory, because we would then assume that the learning algorithm $\mathcal{L}$ gets to observe all policies in the set $f(R^\star)$. However, in reality, it seems more realistic to assume that $\mathcal{L}$ only gets to observe a single element of $f(R^\star)$, unless perhaps $\mathcal{L}$ gets data from multiple agents.

%Recall that Definition~\ref{def:misspecification_eq} says that $f$ is \emph{$P$-robust to misspecification} with $g$ if each of the following conditions are satisfied:
%    \begin{enumerate}
%        \item If $f(R_1) = g(R_2)$ then $R_1 \eq{P} R_2$.
%        \item $\mathrm{Im}(g) \subseteq \mathrm{Im}(f)$.
%        \item $\mathrm{Am}(f) \refines P$.
%        \item $f \neq g$.
%    \end{enumerate}
%Similarly, Definition~\ref{def:misspecification_metric} says that $f$ is $\epsilon$-robust to misspecification with $g$ as measured by $d^\R$ if each of the following conditions are satisfied:
%\begin{enumerate}
%    \item If $f(R_1) = g(R_2)$ then $d^\R(R_1, R_2) \leq \epsilon$.
%    \item $\mathrm{Im}(g) \subseteq \mathrm{Im}(f)$.
%    \item If $f(R_1) = f(R_2)$ then $d^\R(R_1, R_2) \leq \epsilon$.
%    \item $f \neq g$.
%\end{enumerate}
%Neither of these definitions take the \emph{inductive bias} of the learning algorithm into account. Specifically, 

\subsection{Incorporating Assumptions About Inductive Bias}\label{appendix:inductive_bias}

It is worth noting that 
none of the definitions in Section~\ref{sec:frameworks}
%neither Definition~\ref{def:misspecification_eq} nor \ref{def:misspecification_metric} (which formalise misspecification robustness) 
make any assumptions about the \emph{inductive bias} of the learning algorithm. 
For example, let the true reward function be $R^\star$, let the true data generating process be described by $g$, and let $f$ be the assumed model of the data generating process. Then both Definition~\ref{def:misspecification_eq} and \ref{def:misspecification_metric} require that \emph{every} reward function $R_H$ that is compatible with the training data (in the sense that $f(R_H) = g(R^\star)$) must be equivalent or similar to the true reward (in the sense that $R_H \eq{} R^\star$ or $d^\R(R_H, R^\star) \leq \epsilon$). This requirement may seem unnecessarily strong, because some reward functions $R_H$ such that $f(R_H) = g(R^\star)$ may be very unlikely to be generated under the inductive bias of the learning algorithm, $\mathcal{L}$. This, in turn, raises the question of whether we may be able to create weaker, more permissive formalisations of ambiguity tolerance and misspecification robustness by also making assumptions about the inductive bias of the learning algorithm. In this section, we discuss this option.

Let us first focus on Definition~\ref{def:refinement}, which formalises when a given application tolerates the ambiguity of a given data source.  In this case, it does not seem like anything can be gained from incorporating assumptions about inductive bias. To see this, consider the following modified definition: 

\begin{definition}\label{def:inductive_bias}
    Given a reward object $f : \R \to X$, we say that $I : X \to \R$ is an \emph{inductive bias} for $f$ if $f(I(x)) = x$ for all $x \in X$.
\end{definition}

\begin{definition}\label{def:refinement_with_inductive_bias}
Given two reward objects $f : \R \to X$, $g : \R \to Y$, and an \emph{inductive bias} $I : X \to \R$ for $f$, we say that $g$ tolerates the ambiguity of $f$ with $I$ if, for all $R^\star \in \R$, if $R_H = I(f(R^\star))$, then $g(R_H) = g(R^\star)$.
\end{definition}

Note that $I : X \to \R$ is an inductive bias for $f : \R \to X$ if, for any $x \in X$, $I$ maps $x$ to a reward function $R$ such that $f(R) = x$. In other words, $I$ is a function that, for any possible observable data $x$, picks a reward function $R$ that is compatible with $x$ under $f$. Definition~\ref{def:refinement_with_inductive_bias} then says that $g$ tolerates the ambiguity of $f$ with $I$ if, for any true reward function $R^\star$ and any corresponding data distribution $f(R^\star)$, $I$ always picks a reward function $R_H$ such that $g(R_H) = g(R^\star)$. This is directly analogous to Definition~\ref{def:refinement}, except that we assume that the learning algorithm uses a particular inductive bias $I$. Using these definitions, we can now derive the following result:

\begin{theorem}\label{thm:inductive_bias_refinement}
    Let $f : \R \to X$, $g : \R \to Y$ be any two reward objects, and let $I : X \to \R$ be any inductive bias for $f$. Then $g$ tolerates the ambiguity of $f$ with $I$ (in the sense of Definition~\ref{def:refinement_with_inductive_bias}) if and only if $\mathrm{Am}(f) \refines \mathrm{Am}(g)$ (in the sense of Definition~\ref{def:refinement}).
\end{theorem}
\ifshowproofs
\begin{proof}
    For the first direction, assume that $\mathrm{Am}(f) \refines \mathrm{Am}(g)$.
    Note that $f(I(x)) = x$ for all $x \in X$. This means that $I(f(R))$ is a reward function such that $f(I(f(R))) = f(R)$. Since $\mathrm{Am}(f) \refines \mathrm{Am}(g)$, this means that $g(I(f(R))) = g(R)$. This completes the first direction.
    
    For the other direction, assume that $g$ tolerates the ambiguity of $f$ with $I$, in the sense of Definition~\ref{def:refinement_with_inductive_bias}. 
    Suppose $f(R_1) = f(R_2)$. Since $g(I(f(R))) = g(R)$ for all $R$, we have that $g(I(f(R_1))) = g(R_1)$ and $g(I(f(R_2))) = g(R_2)$. Moreover, since $f(R_1) = f(R_2)$, this means that $g(I(f(R_1))) = g(I(f(R_2)))$. By transitivity, this then implies that $g(R_1) = g(R_2)$, and so $\mathrm{Am}(f) \refines P$. This completes the proof.
\end{proof}
\fi

Thus, Definition~\ref{def:refinement_with_inductive_bias} is functionally equivalent to Definition~\ref{def:refinement}. In other words, for the purposes of ambiguity tolerance, it does not make any difference which inductive bias $I$ the learning algorithm uses. Also note that, while Definition~\ref{def:inductive_bias} defines $I$ to be a function that deterministically picks a fixed $R$ for each $x \in X$, we would obtain a result analogous to Theorem~\ref{thm:inductive_bias_refinement} if we instead defined $I$ to be a set-valued function, etc. Thus, the analysis in Section~\ref{sec:partial_identifiability} which is based on Definition~\ref{def:refinement} would remain unchanged if we defined ambiguity tolerance relative to a particular choice of inductive bias for the learning algorithm.

%thm:inductive_bias_refinement

Let us next consider Definition~\ref{def:misspecification_eq}, which defines misspecification robustness relative to equivalence relations on $\R$. In this case, we similarly find that the inductive bias does not affect our results. To see this, consider the following modified definition of misspecification robustness:

\begin{definition}\label{def:misspecification_eq_with_inductive_bias}
    Given a partition $P$ of $\R$, two reward objects $f, g : \R \to X$, and an \emph{inductive bias} $I : X \to \R$ for $f$, we say that $f$ is \emph{$P$-robust to misspecification} with $g$ using $I$ if each of the following conditions are satisfied:
    \begin{enumerate}
        \item $I(g(R)) \eq{P} R$ for all $R$.
        \item $\mathrm{Im}(g) \subseteq \mathrm{Im}(f)$.
        \item $I(f(R)) \eq{P} R$ for all $R$.
        \item $f \neq g$.
    \end{enumerate}
\end{definition}

Definition~\ref{def:misspecification_eq_with_inductive_bias} is simply directly analogous to Definition~\ref{def:misspecification_eq}, except that it makes the assumption that the learning algorithm $\mathcal{L}$ uses the inductive bias described by $\mathcal{L}$. Using this definitions, we then get the following result:

\begin{theorem}\label{thm:inductive_bias_eq}
    Let $P$ be any partition of $\R$, let $f, g : \R \to X$ be any two reward objects, and let $I : X \to \R$ be any inductive bias for $f$. Then $f$ is $P$-robust to misspecification with $g$ (in the sense of Definition~\ref{def:misspecification_eq}) if and only if $f$ is \emph{$P$-robust to misspecification} with $g$ using $I$ (in the sense of Definition~\ref{def:misspecification_eq_with_inductive_bias}).
\end{theorem}
\ifshowproofs
\begin{proof}
    For the first direction, assume that $f$ is $P$-robust to misspecification in the sense of Definition~\ref{def:misspecification_eq}. We then have that:
    \begin{enumerate}
        \item If $f(R_1) = g(R_2)$ then $R_1 \eq{P} R_2$.
        \item $\mathrm{Im}(g) \subseteq \mathrm{Im}(f)$.
        \item $\mathrm{Am}(f) \refines P$.
        \item $f \neq g$.
    \end{enumerate}
    To show that $f$ is $P$-robust to misspecification with $g$ using $I$, we must show that the following two conditions hold:
    \begin{enumerate}
        \item $I(g(R)) \eq{P} R$ for all $R$.
        \item $I(f(R)) \eq{P} R$ for all $R$.
    \end{enumerate}
    Note that $f(I(x)) = x$ for all $x \in X$. This means that $I(g(R))$ is a reward function such that $f(I(g(R))) = g(R)$. Since $R_1 \eq{P} R_2$ whenever $f(R_1) = g(R_2)$, this means that $I(g(R)) \eq{P} R$. Similarly, $I(f(R))$ is a reward function such that $f(I(f(R))) = f(R)$. Since $\mathrm{Am}(f) \refines P$, this means that $I(f(R)) \eq{P} R$. This completes the first direction.
    
    For the other direction, assume that $f$ is $P$-robust to misspecification with $g$ using $I$ in the sense of Definition~\ref{def:misspecification_eq_with_inductive_bias}. We then have that 
    \begin{enumerate}
        \item $I(g(R)) \eq{P} R$ for all $R$.
        \item $\mathrm{Im}(g) \subseteq \mathrm{Im}(f)$.
        \item $I(f(R)) \eq{P} R$ for all $R$.
        \item $f \neq g$.
    \end{enumerate}
    To show that $f$ is $P$-robust to misspecification with $g$, we must show that the following two conditions hold:
    \begin{enumerate}
        \item If $f(R_1) = g(R_2)$ then $R_1 \eq{P} R_2$.
        \item $\mathrm{Am}(f) \refines P$.
    \end{enumerate}
    First, suppose $f(R_1) = f(R_2)$. Since $I(f(R)) \eq{P} R$ for all $R$, we have that $I(f(R_1)) \eq{P} R_1$ and $I(f(R_2)) \eq{P} R_2$. Moreover, since $f(R_1) = f(R_2)$, this means that $I(f(R_1)) = I(f(R_2))$. By transitivity, this then implies that $R_1 \eq{P} R_2$, and so $\mathrm{Am}(f) \refines P$. Similarly, suppose $f(R_1) = g(R_2)$. Since $I(g(R)) \eq{P} R$ for all $R$, we have that $I(g(R_2)) \eq{P} R_2$. Moreover, since $f(I(f(R_1))) = f(R_1)$, and since $\mathrm{Am}(f) \refines P$, we have that $I(f(R_1)) \eq{P} R_1$. Since $f(R_1) = g(R_2)$, this means that $I(g(R_2)) \eq{P} R_1$. By transitivity, we thus have that $R_1 \eq{P} R_2$. This completes the proof.
\end{proof}
\fi

In other words, our analysis of misspecification robustness in terms of equivalence relations on $\R$ is also not affected by the inductive bias of the learning algorithm. As such, all of our results in Section~\ref{sec:misspecification_1} would be identical if we used Definition~\ref{def:misspecification_eq_with_inductive_bias} instead of Definition~\ref{def:misspecification_eq}.

The case is less straightforward when we characterise the difference between reward functions in terms of pseudometrics on $\R$, rather than equivalence relations on $\R$ (as for Definition~\ref{def:ambiguity_diameter} and Definition~\ref{def:misspecification_metric})
For example, we can have a reward object $f : \R \to X$ for which the lower diameter of $\mathrm{Am}(f)$ is greater than $\epsilon$, but where there exists an inductive bias $I$ for $f$ such that $d^\R(I(f(R)),R) \leq \epsilon$ for all $R$. To see this, suppose the (upper and lower) diameter of $\mathrm{Am}(f)$ is $2\epsilon$, but that the inductive bias $I$ always picks a reward function that is in the \enquote{middle} of each set in $\mathrm{Am}(f)$, such that the distance between this reward function and all other reward functions in the same set of $\mathrm{Am}(f)$ always is at most $\epsilon$. In this way, the worst-case error between the learnt reward function $R_H$ and the true reward function $R^\star$ may be more than $\epsilon$ for data generated under $f$, even though it can be guaranteed to be at most $\epsilon$ under a particular inductive bias $I$. In a similar way, the conditions for when $f$ is $\epsilon$-robust to misspecification with $g$ (under Definition~\ref{def:misspecification_metric}) may also be affected by the inductive bias $I$ of the learning algorithm.

However, note that this generalisation cannot affect the derived results by a substantial amount. To see this, first note that if the upper diameter of $\mathrm{Am}(f)$ under $d^\R$ is $\delta$, then for any inductive bias $I$ for $f$, there is a reward function $R$ such that $d^\R(R, I(f(R))) \geq \delta/2$, by the triangle inequality. This means that we will still have to require that the upper diameter of $\mathrm{Am}(f)$ is small.
%Moreover, if the upper diameter of $\mathrm{Am}(f)$ is small, then the lower diameter must of course also be small. 
Furthermore, if there are reward functions $R_1, R_2$ such that $f(R_1) = g(R_2)$ and $d^\R(R_1, R_2) = \epsilon$, and if the upper diameter of $\mathrm{Am}(f)$ is $\delta$, then for any inductive bias $I$ for $f$, we have that $d^\R(R_2, I(g(R_2))) \geq \epsilon - \delta$, again by the triangle inequality.\footnote{For clarity, let us spell this out. First, note that if $g(R_2) = f(R_1)$, and $I$ is an inductive bias for $f$, then $I(g(R_2))$ is some reward function $R_3$ such that $f(R_1) = f(R_3)$. Since the upper diameter of $\mathrm{Am}(f)$ is $\delta$, we have that $d^\R(R_1, R_3) \leq \delta$. We also have that $d^\R(R_1, R_2) = \epsilon$. By the triange inequality, $d^\R(R_1, R_2) \leq d^\R(R_1, R_3) + d^\R(R_3, R_2)$, so $\epsilon \leq x + d^\R(R_3, R_2)$, where $x \leq \delta$. This implies that $\epsilon - \delta \leq d^\R(R_2, R_3)$.} Thus, if the upper diameter of $\mathrm{Am}(f)$ is small, then the inductive bias cannot have a large impact on the misspecification robustness of the algorithm. In other words, we must require the upper diameter of $\mathrm{Am}(f)$ to be small, but if this is the case, then the inductive bias cannot matter much.

More generally, the inductive bias of the learning algorithm could make a meaningful difference to the derived results primarily when the (upper) diameters of $\mathrm{Am}(f)$ and $\mathrm{Am}(g)$ are small, but greater than zero. 
However, in most of the cases we have analysed, the diameter of $\mathrm{Am}(f)$ is either zero, or very large. 
For example, the upper diameter of both $\mathrm{Am}(b_{\tfunc, \gamma, \beta})$ and $\mathrm{Am}(c_{\tfunc, \gamma, \alpha})$ is zero (Corollary~\ref{cor:diameter_of_am_for_BR_and_MCE}), and the upper diameter of $\mathrm{Am}(o^\star_{\tfunc,\gamma})$ is large (Corollary~\ref{cor:diameter_of_op}). 
Similarly, when $f$ is invariant to potential shaping with $\gamma$ or $S'$-redistribution with $\tfunc$, and $\gamma$ or $\tfunc$ is misspecified, then the upper diameter of $\mathrm{Am}(f)$ is large (Theorem~\ref{thm:wrong_tau_large_diameter} and 
\ref{thm:wrong_gamma_large_diameter}). Therefore, we should not expect any of the results we have derived using Definition~\ref{def:misspecification_metric} to change substantially if we modify Definition~\ref{def:misspecification_metric} to also incorporate assumptions about the inductive bias of the learning algorithm.
Nonetheless, carrying out this analysis in more detail may be an interesting direction for future work.

\subsection{Incorporating Assumptions About the True Reward}\label{appendix:assumptions_about_true_reward}

It is worth noting that our definitions in Section~\ref{sec:frameworks} make no assumptions about the true reward function, $R^\star$. While this is reasonable, it does raise the worry that some of our negative results (e.g., those in Sections~\ref{sec:ambiguity_transfer_learning}, \ref{sec:misspecification_1_MDPs}, and \ref{sec:misspecification_2_parameters}) may be caused by specific edge-cases that are unlikely to arise in practice. For example, in order for $f$ to be $P$-robust to misspecification with $g$, it is required that there for \emph{every} reward function $R^\star$ is no $R_H$ such that $f(R_H) = g(R^\star)$, but $R_H \not\eq{P} R^\star$. This may seem unnecessarily strong, if we have reason to believe that certain reward functions $R^\star$ are unlikely to come up in practice. A natural question is therefore if we may be able to obtain stronger results by incorporating some assumptions about $R^\star$. In this section, we will discuss this question.

In short, most of our negative results would not change if we incorporate assumptions about the true reward function $R^\star$.
This is largely ensured by the fact that we distinguish between the upper and lower diameter of $\mathrm{Am}(f)$.
For example, let us first consider the results in Section~\ref{sec:ambiguity_transfer_learning}, which show that a wide range of behavioural models are unable to guarantee robust transfer to a different transition function $\tfunc$ or discount factor $\gamma$. These results are not merely saying that there exists \emph{some} $R^\star$ and \emph{some} $R_H$ such that $R^\star$ and $R_H$ are indistinguishable by the learning algorithm, but such that $R^\star$ and $R_H$ are qualitatively different after transfer to a different $\tfunc$ or $\gamma$. Rather, they are saying that there for \emph{every} $R^\star$ is an $R_H$ such that $R^\star$ and $R_H$ are indistinguishable by the learning algorithm, but such that $R^\star$ and $R_H$ are qualitatively different after transfer. In other words, it is not just the \emph{upper} diameter of $\mathrm{Am}(f)$ that is too large, but also the \emph{lower} diameter. Every reward function $R^\star$ is indistinguishable from some reward $R_H$ such that $R^\star$ and $R_H$ are too different to guarantee robust transfer.

The negative results in Section~\ref{sec:misspecification_1_MDPs} and \ref{sec:misspecification_2_parameters}, which concern misspecified $\tfunc$ or $\gamma$, will generalise for a similar reason. 
Recall that these results are saying, roughly, that if $f_{\tfunc,\gamma}$ is invariant to $S'$-redistribution with $\tfunc$ or potential shaping with $\gamma$, and either $\tfunc$ or $\gamma$ is misspecified, then $f_{\tfunc_1,\gamma_1}$ is not robust to misspecification with $f_{\tfunc_2,\gamma_2}$. Intuitively speaking, the reason for why this is true is that at least one of $\mathrm{Am}(f_{\tfunc_1,\gamma_1})$ or $\mathrm{Am}(f_{\tfunc_2,\gamma_2})$ will be too large. This result is derived from the ambiguity results in Section~\ref{sec:ambiguity_transfer_learning}, which means that it is not just the worst-case size (i.e.\ the upper diameter) of $\mathrm{Am}(f_{\tfunc_1,\gamma_1})$ or $\mathrm{Am}(f_{\tfunc_2,\gamma_2})$ that is too large, but the best-case size (i.e.\ the lower diameter) as well.
These results will therefore also not be affected by any assumptions we could make about $R^\star$.

The only exception to this is Theorem~\ref{thm:not_separating}, which says that no continuous behavioural model is $\epsilon/\delta$-separating relative to $\starc$, which also means that no continuous behavioural model is perturbation robust relative to $\starc$. The proof of Theorem~\ref{thm:not_separating} finds a specific counterexample in the vicinity of the zero reward, $R_0$, which could be ruled out by making certain assumptions about $R^\star$. This issue will also be discussed in \ref{appendix:restricted_reward_functions}.

\subsection{Restricting the Space of Reward Functions}\label{appendix:restricted_reward_functions}

In \ref{appendix:inductive_bias}, we discussed the option of incorporating assumptions about the inductive bias of the learning algorithm, and in \ref{appendix:assumptions_about_true_reward}, we discussed the option of incorporating assumptions about the true reward function. Moreover, we argued that neither of these generalisations would make a meaningful difference to our results. But what if we do both at the same time? 
Formally, suppose that instead of quantifying over all reward functions in $\R$, we restrict the reward functions to lie in some set $\Rspace \subseteq \R$. We can then assume that the true reward function $R^\star$ is in $\Rspace$, and that the learning algorithm $\mathcal{L}$ only will generate reward functions $R_H$ that are also in $\Rspace$. In this section, we will discuss this option.
As we will see, our results are largely unaffected by this generalisation, though it does open up some avenues for further analysis.

We first need to generalise the definitions in Section~\ref{sec:frameworks}, which is straightforward; simply replace any quantifier which ranges over all of $\R$ with one that only ranges over $\Rspace$ (where $\Rspace$ is permitted to be any subset of $\R$). For example, given partitions $P$, $Q$ of $\R$, we say that $P \preceq Q$ on $\Rspace$ if $R_1 \eq{P} R_2$ implies that $R_1 \eq{Q} R_2$ for all $R_1,R_2 \in \Rspace$. Similarly, we say that $f$ is $P$-robust to misspecification with $g$ on $\Rspace$ if $f(R_1) = g(R_2)$ implies that $R_1 \eq{P} R_2$ for all $R_1, R_2 \in \Rspace$, $\mathrm{Im}(g |_\Rspace) \subseteq \mathrm{Im}(f |_\Rspace)$, $\mathrm{Am}(f) \refines P$ on $\Rspace$, and $f |_\Rspace \neq g |_\Rspace$. In a similar manner, it is straightforward to generalise all definitions in Section~\ref{sec:frameworks} to be relative to a set of permitted reward functions $\Rspace$.

We should first note that all lemmas in Section~\ref{sec:frameworks} apply with these more general definitions for any arbitrary subset $\Rspace \subseteq \R$. We will not spell out these proofs explicitly, but they are directly analogous to the proofs given in Section~\ref{sec:frameworks}, since none of these proofs rely on any assumptions about $\R$. We should also note that the propositions given in Section~\ref{sec:frameworks_properties_of_definitions} (in particular, Proposition~\ref{prop:no_epsilon_robustness_implies_refinement}, \ref{prop:no_epsilon_robustness_symmetry}, and \ref{prop:no_epsilon_robustness_function_composition}) do not hold for arbitrary sets $\Rspace$. However, none of our later results rely on these propositions, since their purpose primarily is to highlight some of the differences between Definition~\ref{def:misspecification_eq} and \ref{def:misspecification_metric}.

In addition to this, most of our results in Section~\ref{sec:partial_identifiability}, \ref{sec:misspecification_1}, and \ref{sec:misspecification_2} carry over very directly to the setting with restricted spaces of reward functions. To make this clear, we provide the following results:

\begin{proposition}\label{prop:refinement_with_subspaces}
Let $P, Q$ be any partitions on $\R$, and let $\Rspace$ be any subset of $\R$. We then have that if $P \refines Q$ on $\R$, then $P \refines Q$ on $\Rspace$.
\end{proposition}
\ifshowproofs
\begin{proof}
Note that $P \refines Q$ on $\R$ if, for all $R_1, R_2 \in \R$, if $R_1 \eq{P} R_2$ then $R_1 \eq{Q} R_2$. This directly implies that for all $R_1, R_2 \in \Rspace \subseteq \R$, if $R_1 \eq{P} R_2$ then $R_1 \eq{Q} R_2$, and so $P \refines Q$ on $\Rspace$.
\end{proof}
\fi

\begin{proposition}\label{prop:ambiguity_diameter_with_subspaces}
Let $f : \R \to X$ be any reward object, let $d^\R$ be any pseudometric on $\R$, and let $\Rspace$ be any subset of $\R$. Then the upper diameter of $\mathrm{Am}(f)$ on $\Rspace$ is no greater than the upper diameter of $\mathrm{Am}(f)$ on $\R$, and likewise for the lower diameter.
\end{proposition}
\ifshowproofs
\begin{proof}
For any two sets $S_1, S_2$, we have that $\mathrm{diam}(S_1) \leq \mathrm{diam}(S_1 \cap S_2)$. Thus, if $S \in \mathrm{Am}(f)$, then $\mathrm{diam}(S) \leq \mathrm{diam}(S \cap \Rspace)$. This implies that the upper diameter of $\mathrm{Am}(f)$ on $\Rspace$ is no greater than the upper diameter of $\mathrm{Am}(f)$ on $\R$, and likewise for the lower diameter.
\end{proof}
\fi

\begin{theorem}\label{thm:misspecification_eq_with_subspaces}
For any $\Rspace \subseteq \R$ and any partition $P$ of $\R$, if $f$ is $P$-robust to misspecification with $g$ on $\R$ then $f$ is $P$-robust to misspecification with $g$ on $\Rspace$, unless $f|_\Rspace = g|_\Rspace$. Similarly, if $f$ is $P$-robust to misspecification with $g$ on $\Rspace$ then $f$ is $P$-robust to misspecification with $g'$ on $\R$ for some $g'$ where $g'|_\Rspace = g|_\Rspace$, unless $\mathrm{Am}(f) \not\refines P$ on $\R$.
\end{theorem}
\ifshowproofs
\begin{proof}
For the first claim, suppose that $f$ is $P$-robust to misspecification with $g$ on $\R$, and that $f|_\Rspace \neq g|_\Rspace$. Since $f$ is $P$-robust to misspecification with $g$ on $\R$, we have that for all $R_1, R_2 \in \R$, if $f(R_1) = g(R_2)$ then $R_1 \eq{P} R_2$. Since $\Rspace \subseteq \R$, this directly implies that for all $R_1, R_2 \in \Rspace$, if $f(R_1) = g(R_2)$ then $R_1 \eq{P} R_2$. We also have that $\mathrm{Im}(f)  \subseteq \mathrm{Im(g)}$, which directly implies that $\mathrm{Im}(g |_\Rspace) \subseteq \mathrm{Im}(f |_\Rspace)$. Moreover, we have that $\mathrm{Am}(f) \refines P$ on $\R$, which (as shown in Proposition~\ref{prop:refinement_with_subspaces}) implies that $\mathrm{Am}(f) \refines P$ on $\Rspace$. We assume that $f|_\Rspace \neq g|_\Rspace$. Thus, $f$ is $P$-robust to misspecification with $g$ on $\Rspace$.

For the second claim, suppose $f$ is $P$-robust to misspecification with $g$ on $\Rspace$, and that $\mathrm{Am}(f) \refines P$ on $\R$. We construct a $g'$ as follows; let $g'(R) = g(R)$ for all $R \in \Rspace$, and let $g'(R) = f(R)$ for all $R \not\in \Rspace$. 
By construction, we have that $g(R) = g'(R)$ for all $R \in \Rspace$. Moreover, $f$ is $P$-robust to misspecification with $g'$ on $\R$. To see this, first note that if $f |_\Rspace \neq g |_\Rspace$, and $g |_\Rspace = g' |_\Rspace$, then $f \neq g'$. Moreover, if $\mathrm{Im}(g |_\Rspace) \subseteq \mathrm{Im}(f |_\Rspace)$, and $g'$ is equal to $g$ on $\Rspace$ and equal to $f$ outside $\Rspace$, then $\mathrm{Im}(g') \subseteq \mathrm{Im(f)}$. Moreover, we have assumed that $\mathrm{Am}(f) \refines P$ on $\R$. Finally, suppose that $f(R_1) = g'(R_2)$. If $R_2 \not\in \Rspace$, then $g'(R_2) = f(R_2)$. Since $\mathrm{Am}(f) \refines P$ on $\R$, this implies that $R_1 \eq{P} R_2$. Next, if $R_1, R_2 \in \Rspace$, then $R_1 \eq{P} R_2$, since $f$ is $P$-robust to misspecification with $g$ on $\Rspace$. Finally, if $R_2 \in \Rspace$, $R_1 \not\in \Rspace$, let $R_3$ be a reward function such that $R_3 \in \Rspace$ and $f(R_3) = g'(R_2)$. Since $\mathrm{Im}(g |_\Rspace) \subseteq \mathrm{Im}(f |_\Rspace)$, such a reward function $R_3$ does exist. Since $f$ is $P$-robust to misspecification with $g$ on $\Rspace$, we have that $R_2 \eq{P} R_3$. Moreover, since $\mathrm{Am}(f) \refines P$ on $\R$, and $f(R_1) = f(R_3)$, we have that $R_1 \eq{P} R_3$. By transitivity, this implies that $R_1 \eq{P} R_2$. This covers all cases, which means that if $f(R_1) = g'(R_2)$ then $R_1 \eq{P} R_2$. Thus $f$ is $P$-robust to misspecification with $g'$ on $\R$.
\end{proof}
\fi

\begin{theorem}\label{thm:misspecification_metric_with_subspaces}
For any $\Rspace \subseteq \R$ and any pseudometric $d^\R$ on $\R$, if $f$ is $\epsilon$-robust to misspecification with $g$ on $\R$ using $d^\R$, then $f$ is $\epsilon$-robust to misspecification with $g$ on $\Rspace$ using $d^\R$, unless $f|_\Rspace = g|_\Rspace$. Similarly, if $f$ is $\epsilon$-robust to misspecification with $g$ on $\Rspace$ using $d^\R$ then $f$ is $2\epsilon$-robust to misspecification with $g'$ on $\R$ using $d^\R$ for some $g'$ where $g'|_\Rspace = g|_\Rspace$, unless there are $R_1, R_2 \in \R$ such that $f(R_1) = f(R_2)$ but $d^\R(R_1, R_2) > 2\epsilon$. 
\end{theorem}
\ifshowproofs
\begin{proof}
For the first claim, suppose that $f$ is $\epsilon$-robust to misspecification with $g$ on $\R$, and that $f|_\Rspace \neq g|_\Rspace$. Since $f$ is $\epsilon$-robust to misspecification with $g$ on $\R$, we have that for all $R_1, R_2 \in \R$, if $f(R_1) = g(R_2)$ then $d^\R(R_1, R_2) \leq \epsilon$. Since $\Rspace \subseteq \R$, this directly implies that for all $R_1, R_2 \in \Rspace$, if $f(R_1) = g(R_2)$ then $d^\R(R_1, R_2) \leq \epsilon$. We also have that $\mathrm{Im}(f)  \subseteq \mathrm{Im(g)}$, which directly implies that $\mathrm{Im}(g |_\Rspace) \subseteq \mathrm{Im}(f |_\Rspace)$. Moreover, we have that for all $R_1, R_2 \in \R$, if $f(R_1) = g(R_2)$ then $d^\R(R_1, R_2) \leq \epsilon$. Since $\Rspace \subseteq \R$, this directly implies that for all $R_1, R_2 \in \Rspace$, if $f(R_1) = g(R_2)$ then $d^\R(R_1, R_2) \leq \epsilon$. We assume that $f|_\Rspace \neq g|_\Rspace$. Thus, $f$ is $\epsilon$-robust to misspecification with $g$ on $\Rspace$.

For the second claim, suppose $f$ is $\epsilon$-robust to misspecification with $g$ on $\Rspace$, and that for all $R_1, R_2 \in \R$, if $f(R_1) = f(R_2)$ then $d^\R(R_1, R_2) \leq 2\epsilon$. We construct a $g'$ as follows; let $g'(R) = g(R)$ for all $R \in \Rspace$, and let $g'(R) = f(R)$ for all $R \not\in \Rspace$. 
By construction, we have that $g(R) = g'(R)$ for all $R \in \Rspace$. Moreover, $f$ is $2\epsilon$-robust to misspecification with $g'$ on $\R$. To see this, first note that if $f |_\Rspace \neq g |_\Rspace$, and $g |_\Rspace = g' |_\Rspace$, then $f \neq g'$. Moreover, if $\mathrm{Im}(g |_\Rspace) \subseteq \mathrm{Im}(f |_\Rspace)$, and $g'$ is equal to $g$ on $\Rspace$ and equal to $f$ outside $\Rspace$, then $\mathrm{Im}(g') \subseteq \mathrm{Im(f)}$. Moreover, we have assumed that for all $R_1, R_2 \in \R$, if $f(R_1) = f(R_2)$ then $d^\R(R_1, R_2) \leq \epsilon \leq 2\epsilon$. Finally, suppose that $f(R_1) = g'(R_2)$. If $R_2 \not\in \Rspace$, then $g'(R_2) = f(R_2)$. Since the upper diameter of $\mathrm{Am}(f)$ on $\R$ is assumed to be at most $\epsilon$, this implies that $d^\R(R_1, R_2) \leq \epsilon \leq 2\epsilon$. Next, if $R_1, R_2 \in \Rspace$, then $d^\R(R_1, R_2) \leq \epsilon \leq 2\epsilon$, since $f$ is $\epsilon$-robust to misspecification with $g$ on $\Rspace$. Finally, if $R_2 \in \Rspace$, $R_1 \not\in \Rspace$, let $R_3$ be a reward function such that $R_3 \in \Rspace$ and $f(R_3) = g'(R_2)$. Since $\mathrm{Im}(g |_\Rspace) \subseteq \mathrm{Im}(f |_\Rspace)$, such a reward function $R_3$ does exist. Since $f$ is $\epsilon$-robust to misspecification with $g$ on $\Rspace$, we have that $d^\R(R_2, R_3) \leq \epsilon$. Moreover, since the upper diameter of $\mathrm{Am}(f)$ on $\R$ is at most $\epsilon$, and $f(R_1) = f(R_3)$, we have that $d^\R(R_1, R_3) \leq \epsilon$. By the triangle inequality, this implies that $d^\R(R_1, R_2) \leq 2\epsilon$. This covers all cases, which means that if $f(R_1) = g'(R_2)$ then $d^\R(R_1, R_2) \leq 2\epsilon$. Thus $f$ is $2\epsilon$-robust to misspecification with $g'$ on $\R$ relative to the pseudometric $d^\R$.
\end{proof}
\fi

To understand these results, first note that Proposition~\ref{prop:refinement_with_subspaces} implies that if $\mathrm{Am}(f) \refines \mathrm{Am}(g)$ on $\R$, then $\mathrm{Am}(f) \refines \mathrm{Am}(g)$ on $\Rspace$ for any $\Rspace \subseteq \R$. In other words, if $g$ tolerates the ambiguity of $f$ relative to the space of all reward functions $\R$, then this is (of course) also the case for any restricted space of reward functions $\Rspace$. Stated differently, ambiguity tolerance cannot decrease if the space of all reward functions is restricted. Similarly,  Proposition~\ref{prop:ambiguity_diameter_with_subspaces} also says that the ambiguity of $f$ cannot increase if the space of all reward functions is restricted. Intuitively speaking, this means that all positive results about ambiguity and ambiguity tolerance generalise to arbitrary restrictions on the space of considered reward functions. 

However, the converse does \emph{not} hold; it is possible that $\mathrm{Am}(f) \refines \mathrm{Am}(g)$ on $\Rspace$ for some $\Rspace \subseteq \R$, even though $\mathrm{Am}(f) \not\refines \mathrm{Am}(g)$ on $\R$. A simple way to see this is by supposing that $\Rspace$ is a singleton set, containing only one reward function. More generally, $\mathrm{Am}(f) \not\refines \mathrm{Am}(g)$ on $\R$ if there are reward functions $R_1, R_2 \in \R$ such that $f(R_1) = f(R_2)$ but $g(R_1) \neq g(R_2)$. It may be that $\R$ contains such a counterexample, even though $\Rspace$ does not. In the same way, it may be that the upper diameter of $\mathrm{Am}(f)$ on $\Rspace$ is $\epsilon$, even though the upper diameter of $\mathrm{Am}(f)$ on $\R$ is greater than $\epsilon$.
Intuitively speaking, this means that a \emph{negative} result about ambiguity or ambiguity tolerance may not generalise to a given restricted space of reward functions. Stated differently, if we impose restrictions on the reward functions we consider, then the ambiguity of a reward object $f$ may decrease.

Theorem~\ref{thm:misspecification_eq_with_subspaces} says, roughly, that that if $f$ is $P$-robust to misspecification with $g$ if and only if $g \in G$, then $f$ is $P$-robust to misspecification with $g'$ on $\Rspace$ if and only if $g'$ behaves like some $g \in G$ for all $R \in \Rspace$. 
This means that the issue of misspecification robustness is largely unaffected if $\R$ is restricted.
Specifically, if $f$ is $P$-robust to misspecification with $g$ on $\R$, then $f$ is $P$-robust to misspecification with $g$ on $\Rspace$ (unless $f = g$ on $\Rspace$). In other words, if $f$ is robust to some form of misspecification, then this is still the case if we impose restrictions on the reward functions we consider, and so all \emph{positive} results generalise to arbitrary subsets $\Rspace$ of $\R$. The converse case is similar, but slightly more complicated; if $f$ is $P$-robust to misspecification with $g$ on $\Rspace$, then $f$ is $P$-robust to misspecification with $g'$ on $\R$ for some $g'$ where $g'|_\Rspace = g|_\Rspace$, unless $\mathrm{Am}(f) \not\refines P$ on $\R$. In other words, if $f$ is not robust to a given form of misspecification relative to $\R$, then this is still the case if we impose restrictions on the space of all reward functions, unless this restriction ensures that this misspecification has no impact on the generated data (or $\mathrm{Am}(f) \not\refines P$ on $\R$, but $\mathrm{Am}(f) \refines P$ on $\Rspace$).
Theorem~\ref{thm:misspecification_metric_with_subspaces} is analogous to Theorem~\ref{thm:misspecification_eq_with_subspaces}, but considers the case when we use pseudometrics on $\R$ instead of equivalence relations on $\R$.

It may now be interesting to consider whether some of our negative results can be mitigated by imposing restrictions on the reward functions that we consider. In particular, in Section~\ref{sec:ambiguity_transfer_learning}, we showed that a wide range of behavioural models are too ambiguous to guarantee robust transfer learning if the transition function $\tfunc$ or the discount factor $\gamma$ is different in the training environment and deployment environment, even if this difference is arbitrarily small.
These results then further imply that a wide range of behavioural models lack robustness to misspecification of $\tfunc$ or $\gamma$, even if this misspecification is arbitrarily small (Section~\ref{sec:misspecification_1_MDPs} and ~\ref{sec:misspecification_2_parameters}).
Can these results be mitigated, if we restrict $\R$ to some subset $\Rspace$?

In some cases, this can be done. For example, recall that Theorem~\ref{thm:wrong_tau_too_ambiguous} and Theorem~\ref{thm:wrong_tau_large_diameter} say that if $f_{\tfunc_1}$ is invariant to $S'$-redistribution with $\tfunc_1$, and $\tfunc_1 \neq \tfunc_2$, then we have that $\mathrm{Am}(f_{\tfunc_1}) \not\refines \mathrm{OPT}_{\tfunc_2,\gamma}$, and that the lower and upper diameter of $\mathrm{Am}(f_{\tfunc_1})$ under $d^\mathrm{STARC}_{\tfunc_2, \gamma}$ is 1. The reason for this, intuitively, is that we for any reward function $R_1$ can find another reward function $R_2$ such that $R_1$ and $R_2$ differ by $S'$-redistribution with $\tfunc_1$, but such that $R_1$ and $R_2$ induce very different policy orderings under $\tfunc_2$.
However, suppose we restrict $\R$ to the set $\Rspace$ of reward functions such that if $R \in \Rspace$, then for all $s,a,s_1,s_2$, we have that $R(s,a,s_1) = R(s,a,s_2)$ (or, in other words, we let $\Rspace$ be the set of all reward functions that can be defined over $\SxA$). 
In that case, there are no reward functions in $\Rspace$ which differ by $S'$-redistribution (for any $\tfunc$).
This means that Theorem~\ref{thm:wrong_tau_too_ambiguous} and Theorem~\ref{thm:wrong_tau_large_diameter} no longer apply, which in turn also means that the results derived from them (i.e.\ Theorem~\ref{thm:not_P_robust_to_misspecified_t} and \ref{thm:misspecified_env}) may not apply either. In other words, relative to this space of reward functions, a behavioural model $f$ may be robust to some misspecification of $\tfunc$, even if $f$ is invariant to $S'$-redistribution.

Similarly, recall that Theorem~\ref{thm:wrong_gamma_too_ambiguous} and \ref{thm:wrong_gamma_large_diameter} say that if $f_{\gamma_1}$ is invariant to potential shaping with $\gamma_1$, $\gamma_1 \neq \gamma_2$, and $\tfunc$ is non-trivial, then we have that $\mathrm{Am}(f_{\gamma_1}) \not\refines \mathrm{OPT}_{\tfunc,\gamma_2}$, and that the lower and upper diameter of $\mathrm{Am}(f_{\gamma_1})$ under $d^\mathrm{STARC}_{\tfunc, \gamma_2}$ is 1. The reason for this, intuitively, is that we for any reward function $R_1$ can find another reward function $R_2$ such that $R_1$ and $R_2$ differ by potential shaping with $\gamma_1$, but such that $R_1$ and $R_2$ induce very different policy orderings under $\gamma_2$. However, suppose if we restrict $\R$ to e.g.\ the set $\Rspace$ of reward functions that only reward a single transition. In that case, there are no reward functions in $\Rspace$ that differ by potential shaping (for any $\gamma$).
This means that Theorem~\ref{thm:wrong_gamma_too_ambiguous} and \ref{thm:wrong_gamma_large_diameter} no longer apply, which in turn also means that the results derived from them (i.e.\ Theorem~\ref{thm:not_P_robust_to_misspecified_gamma} and \ref{thm:misspecified_discounting}) may not apply either. In other words, relative to this space of reward functions, a behavioural model $f$ may be robust to some misspecification of $\gamma$, even if $f$ is invariant to potential shaping.

We can thus see that some of our negative results can be avoided, if we impose restrictions on the set of considered reward functions $\Rspace$. However, there are a few things to be mindful of. First of all, these restrictions are only applicable if we can ensure that the true reward function $R^\star$ in fact is in $\Rspace$, and that the learning algorithm $\mathcal{L}$ always will return a reward function $R_H$ that is also in $\Rspace$.
Moreover, even if e.g.\ Theorem~\ref{thm:wrong_tau_too_ambiguous} and Theorem~\ref{thm:wrong_tau_large_diameter} do not apply to a given set of reward functions $\Rspace$, it may still be the case that a given behavioural model $f$ lacks robustness to misspecification of $\tfunc$ relative to this set $\Rspace$. 
For example, recall that the Boltzmann-rational model $b_{\tfunc, \gamma, \beta}$ is $\ORD$-robust to misspecification of the temperature parameter $\beta$, and no other misspecification (Theorem~\ref{thm:boltzmann_ord_robustness}).
Theorem~\ref{thm:misspecification_eq_with_subspaces} then implies that if $b_{\tfunc_1, \gamma, \beta}$ is $\ORD$-robust to misspecification with $b_{\tfunc_2, \gamma, \beta}$ on $\Rspace$, then it must be the case that there for each $R \in \Rspace$ is a $\beta_2$ such that $b_{\tfunc_2, \gamma, \beta}(R) = b_{\tfunc_1, \gamma, \beta_2}(R)$. The overall picture is therefore still mostly unchanged.

Next, recall that Theorem~\ref{thm:not_separating} says that no continuous behavioural model is $\epsilon/\delta$-separating relative to $\starc$, which also means that no continuous behavioural model is perturbation robust relative to $\starc$. The proof of Theorem~\ref{thm:not_separating} finds a specific counterexample in the vicinity of the zero reward, $R_0$, and these kinds of counterexamples could be ruled out by imposing certain restrictions on $\Rspace$. A natural suggestion might be to require that each reward function in $\Rspace$ should be normalised, or perhaps that each reward function in $\Rspace$ should have a certain minimum $L_2$-norm, etc. Our next result shows that such restrictions are insufficient to mitigate Theorem~\ref{thm:not_separating}:

\begin{theorem}\label{thm:restricted_subspace_not_separating}
Let $d^\R$ be $\starc$, and let $d^\Pi$ be a pseudometric on $\Pi$ which satisfies the condition that for all $\delta_1$ there exists a $\delta_2$ such that if $L_2(\pi_1, \pi_2) < \delta_2$ then $d^\Pi(\pi_1, \pi_2) < \delta_1$.
Let $c$ be any positive constant, and let $\Rspace$ be a set of reward functions such that if $L_2(R) = c$ then $R \in \Rspace$.
Let $f : \Rspace \to \Pi$ be continuous. Then $f$ is not $\epsilon/\delta$-separating for any $\epsilon < 1$ or $\delta > 0$. 
\end{theorem}
\ifshowproofs
\begin{proof}
Let $R$ be a non-trivial reward function that is orthogonal to all trivial reward functions. Since the set of all trivial reward functions form a linear subspace (Proposition~\ref{prop:PS_SR_linear} and Theorem~\ref{thm:policy_ordering}), such a reward function $R$ exists. Note that $R$ must not necessarily be contained in $\hat{\R}$.

Since $R$ is non-trivial, we have that $\starc(\epsilon R, -\epsilon R) = 1$ for any positive constant $\epsilon$. Next, let $R_\Phi$ be some potential-shaping reward function such that $L_2(\epsilon R + R_\Phi) = c$. Since potential shaping does not change the ordering of policies, we have that $\starc(\epsilon R + R_\Phi, -\epsilon R + R_\Phi) = 1$.
Moreover, since $R_\Phi$ is trivial, we have that both $\epsilon R$ and $-\epsilon R$ are orthogonal to $R_\Phi$, and so $L_2(-\epsilon R + R_\Phi) = c$ as well. Since $L_2(\epsilon R + R_\Phi) = L_2(\epsilon R + R_\Phi) = c$, we have that both $\epsilon R + R_\Phi$ and $-\epsilon R + R_\Phi$ are contained in $\Rspace$.

Let $\delta_1$ be any positive constant. By assumption, there exists a $\delta_2$ such that if $L_2(\pi_1 - \pi_2) < \delta_2$ then $d^\Pi(\pi_1, \pi_2) < \delta_1$. Moreover, since $f$ is continuous, there exists an $\epsilon_1$ such that if $L_2(R_1, R_2) < \epsilon_1$, then $L_2(f(R_1), f(R_2)) < \delta_2$. Next, note that by making $\epsilon$ sufficiently small, we can ensure that the $L_2$-distance between $\epsilon R + R_\Phi$ and $\epsilon R + R_\Phi$ is arbitrarily small (and, in particular, less than $\epsilon_1$). 

Thus, for any positive $\delta$ there exist reward functions $\epsilon R + R_\Phi$ and $-\epsilon R + R_\Phi$ that are both contained in $\Rspace$, such that $d^\Pi(f(\epsilon R + R_\Phi), f(-\epsilon R + R_\Phi)) < \delta$, and such that $\starc(\epsilon R + R_\Phi, -\epsilon R + R_\Phi) = 1$. Thus $f$ is not $\epsilon/\delta$-separating for any $\delta > 0$ and any $\epsilon < 1$.
\end{proof}
\fi

Before moving on, we want to quickly note that most of our other results also are straightforward to generalise to the setting where $\R$ is restricted to some subset $\Rspace$.
First, note that if an equivalence relation $P$ of $\R$ is characterised by a set of reward transformations $T$, then the corresponding equivalence relation on $\Rspace$ is characterised by the set of reward transformations $\{t \in T : \mathrm{Im}(t|_\Rspace) \subseteq \Rspace\}$; this can be used to generalise results such as Theorem~\ref{thm:policy_ordering} or Proposition~\ref{prop:nas_for_small_starc}.
However, here there is a minor subtlety to be mindful of: if $A,B,C$ are sets of reward transformations, then $(A \bigodot B) \setminus C$ is not necessarily equal to $(A \setminus C) \bigodot (B \setminus C)$. 
%This means that if we wish to remove all transformations $t$ where $\mathrm{Im}(t(X)) \not\subseteq X$ from $A \circ B$, we cannot do this by simply removing these transformations from each of $A$ and $B$. 
This means that if we wish to specify $\{t \in A \bigodot B : \mathrm{Im}(t|_\Rspace) \subseteq \Rspace\}$, then we cannot do this by simply removing the transformations where $\mathrm{Im}(t|_\Rspace) \not\subseteq \Rspace$ from each of $A$ and $B$. 
For example, consider the transformations $\SR \bigodot \PS$ restricted to the space $\Rspace$ of reward functions where $R(s,a,s') = R(s,a,s'')$, i.e.\ to reward functions over the domain $\SxA$. The only transformation in $\SR$ on $\Rspace$ is the identity mapping, and the only transformations in $\PS$ on $\Rspace$ are those where $\Phi$ is constant over all states.
However, $\SR \bigodot \PS$ on $\Rspace$ contains all transformations where $\Phi$ is selected arbitrarily, and $t(R)(s,a,s')$ is set to $R(s,a,s') + \gamma\Expect{'\sim\tfunc(s,a)}{\Phi(S')} - \Phi(s)$.
This means that there probably are no general shortcuts for deriving $\{t \in T : \mathrm{Im}(t|_\Rspace) \subseteq \Rspace\}$ for arbitrary $\Rspace$.

\subsection{Making the Analysis More Probabilistic}\label{appendix:make_more_probabilistic}

The definitions we have given in Section~\ref{sec:frameworks} provide what is essentially a worst-case analysis, in the sense that they require each condition to hold for \emph{all} reward functions.
However, in certain cases, we may know that $R^\star$ is sampled from a particular distribution $\mathcal{D}$ over $\R$. In those cases, it may be more relevant to know whether the learnt reward function $R_H$ is similar (in a relevant sense) to the true reward $R^\star$ with \emph{high probability}. In this section, we will discuss this generalisation.

To make this setting more formal, we may assume that we have two reward objects $f, g : \R \to X$ and a distribution $\mathcal{D}$ over $\R$, that $R^\star$ is sampled from $\mathcal{D}$, and that the learning algorithm $\mathcal{L}$ observes the data given by $g(R^\star)$. We then assume that $\mathcal{L}$ returns a reward function $R_H$ such that $f(R_H) = g(R^\star)$, and that $\mathcal{L}$ selects among all such reward functions using some (potentially nondeterministic) inductive bias. We can then ask whether or not $R_H \eq{P} R^\star$ with probability at least $1-\delta$, or $d^\R(R^\star,R_h) \leq \epsilon$ with probability at least $1-\delta$, for some $\delta$ and $\epsilon$, and some partition $P$ or pseudometric $d^\R$ on $\R$. As usual, if $f \neq g$, then $f$ is misspecified, and otherwise $f$ is correctly specified. Moreover, if the learnt reward $R_H$ will be used to compute some object $h : \R \to Y$, then we can set $P = \mathrm{Am}(h)$.

To some extent, our analysis in \ref{appendix:restricted_reward_functions} can be used to understand this setting as well. In particular, suppose we pick a set $\Rspace$ of \enquote{likely} reward functions such that $\mathbb{P}_{R \sim \mathcal{D}}(R \in \Rspace) \geq 1-\delta$, and such that the learning algorithm $\mathcal{L}$ will return a reward function $R_h \in\Rspace$ if there exists a reward function $R_h \in \Rspace$ such that $f(R_h) = g(R^\star)$. 
Then if $f$ is $P$-robust to misspecification with $g$ on $\Rspace$, we have that $\mathcal{L}$ will learn a reward function $R_H$ such that $R_H \eq{P} R^\star$ with probability at least $1-\delta$. Similarly, if $f$ is $\epsilon$-robust to misspecification with $g$ on $\Rspace$, then $\mathcal{L}$ will learn a reward function $R_H$ such that $d^\R(R^\star,R_H) \leq \epsilon$ with probability at least $1-\delta$.

So, for example, suppose $\Rspace$ is the set of all reward functions that have \enquote{low complexity}, for some complexity measure and complexity threshold. The above argument then informally tells us that if the true reward function is likely to have low complexity, and if $\mathcal{L}$ will attempt to fit a low-complexity reward function to its training data, then the learnt reward function will be close to the true reward function with high probability, as long as $f$ is robust to misspecification with $g$ on the set of all low-complexity reward functions. Similarly, if the true reward function is likely to be sparse, and $\mathcal{L}$ will attempt to fit a sparse reward function to its training data, then we may let $\Rspace$ be equal to the set of all sufficiently sparse reward functions, and so on.

Thus, while our definitions in Section~\ref{sec:frameworks} give us a worst-case formalisation of ambiguity and misspecification robustness, it is relatively straightforward to carry out a more probabilistic analysis within the same framework. In particular, our results in \ref{appendix:restricted_reward_functions} directly provide us with results about the probabilistic setting as well. However, this analysis could probably be extended with more specific results. Doing so is out of scope for this paper, but it may be an interesting direction for future work.

\subsection{Stronger Equivalence Conditions}\label{appendix:stronger_equivalence_conditions_for_rewards}

We have introduced three primary methods for characterising the difference between two reward functions; the equivalence relations given by $\ORD$ and $\OPT$, and the pseudometrics given by STARC. Intuitively, $\ORD$ considers $R_1$ and $R_2$ to be equivalent if they have the same ordering of policies under $\tfunc$ and $\gamma$, and $\OPT$ considers $R_1$ and $R_2$ to be equivalent if they have the same optimal policies under $\tfunc$ and $\gamma$. Similarly, STARC metrics consider $R_1$ and $R_2$ to be similar if they have a similar ordering of policies under $\tfunc$ and $\gamma$. %Our analysis will primarily be carried out in terms of these three methods.

It is worth noting that all of these are parameterised in terms of $\tfunc$ and $\gamma$; this means that $\ORD$ only requires $R_1$ and $R_2$ to have the same ordering of policies in one particular environment, and so on. Moreover, two reward functions can have the same policy order in one environment, without having the same policy order in a different environment. To see this, note that Theorem~\ref{thm:policy_ordering} says that $R_1$ and $R_2$ have the same ordering of policies if and only if $R_1$ and $R_2$ differ by potential shaping, $S'$-redistribution, and positive linear scaling, and these transformations are parameterised by $\tfunc$ and $\gamma$. Thus, while $R_1 \eq{\ORD} R_2$ ensures that $R_1$ and $R_2$ are equivalent under $\tfunc$ and $\gamma$, it does not guarantee this for other transition functions or discounts.

To guarantee robust transfer learning, one may therefore wish to consider stronger equivalence relations. Unfortunately, the reward learning methods that we consider are unable to ensure that the learnt reward function has such guarantees (see Section~\ref{sec:ambiguity_transfer_learning}), which makes it redundant to consider stronger equivalence relations or metrics than those we have already presented. Nonetheless, for the sake of completeness, we will briefly discuss how to create stronger equivalence relations on $\R$. First, let $\ORDstar$ be the equivalence relation according to which $R_1 \eq{\ORDstar} R_2$ if and only if, for any state $s$, we have that
$$
\Expect{\xi \sim D_1}{G_1(\xi)} \geq \Expect{\xi \sim D_2}{G_1(\xi)} \iff \Expect{\xi \sim D_1}{G_2(\xi)} \geq \Expect{\xi \sim D_2}{G_2(\xi)}
$$
for all distributions $D_1, D_2$ over trajectories that start in $s$. This means that if $R_1 \eq{\ORDstar} R_2$, then $R_1$ and $R_2$ induce the same preferences over all policies for all transition functions $\tfunc$ --- indeed, this will hold even if we permit $\tfunc$ to be non-Markovian, etc. To characterise $\ORDstar$, we will make use of the following lemma:

\begin{lemma}\label{lemma:trajectories_compact_space}
Let $\Xi = (\SxA)^\omega$ be the set of all trajectories, and let $d : \Xi \times \Xi \to \mathbb{R}$ be the function given by $d(\xi_1,\xi_2) = \frac{1}{e^t}$, where $t$ is the smallest index on which $\xi_1$ and $\xi_2$ differ, or $0$ if $\xi_1 = \xi_2$. Then $(\Xi, d)$ is a compact metric space.    
\end{lemma}
\ifshowproofs
\begin{proof}
We must first show that $d$ is a metric, which requires showing that it satisfies the following:
\begin{enumerate}
    \item Indiscernibility of Identicals: $d(\xi_1,\xi_2) = 0$ if $\xi_1=\xi_2$.
    \item Identity of Indiscernibles: $d(\xi_1,\xi_2) = 0$ only if $\xi_1=\xi_2$.
    \item Positivity: $d(\xi_1,\xi_2) \geq 0$.
    \item Symmetry: $d(\xi_1,\xi_2) = d(\xi_2,\xi_1)$.
    \item Triangle Inequality: $d(\xi_1,\xi_3) \leq d(\xi_1,\xi_2) + d(\xi_2,\xi_3)$.
\end{enumerate}
It is straightforward to see that 1-4 hold. For 5, let $t$ be the smallest index on which $\xi_1$ and $\xi_3$ differ. Note that if $d(\xi_1,\xi_3) > d(\xi_1,\xi_2)$ and $d(\xi_1,\xi_3) > d(\xi_2,\xi_3)$, then it must be the case that $\xi_1[i] = \xi_2[i]$ for all $i \leq t$, and that $\xi_1[i] = \xi_2[i]$ for all $i\leq t$. However, this is a contradiction, since it would imply that $\xi_1[t] = \xi_3[t]$. Thus either $d(\xi_1,\xi_3) \leq d(\xi_1,\xi_2)$ or $d(\xi_1,\xi_3) \leq d(\xi_2,\xi_3)$, which in turn implies that $d(\xi_1,\xi_3) \leq d(\xi_1,\xi_2) + d(\xi_2,\xi_3)$.

Thus $d$ is a metric, which means that $(\Xi, d)$ is a metric space. Next, we will prove that $(\Xi, d)$ is compact. We will do this by showing that $(\Xi, d)$ is totally bounded and complete.

To see that $(\Xi, d)$ is totally bounded, let $\epsilon$ be an arbitrary positive real number, and let $t = \ln(1/\epsilon)$, so that $\epsilon = 1/e^t$. Moreover, let $s$ and $a$ be an arbitrary state and action, and let $\hat{\Xi}$ be the set of all trajectories $\xi$ such that $\xi[i] = \langle s,a \rangle$ for all $i > t$ (but which may include arbitrary transitions before time $t$). Now $\hat{\Xi}$ is finite, and for every trajectory $\xi_1$ there is a trajectory $\xi_2 \in \hat{\Xi}$ such that $d(\xi_1,\xi_2) \leq \epsilon$ (given by letting $\xi_2[i] = \xi_1[i]$ for all $i \leq t$). Thus, for every $\epsilon > 0$, we have that $(\Xi, d)$ has a finite cover. This means that $(\Xi, d)$ is totally bounded.

To see that $(\Xi, d)$ is complete, let $\{\xi_i\}_{i=0}^\infty$ be a Cauchy sequence. This implies that for every $\epsilon > 0$ there is a positive integer $N$ such that for all $n,m \geq N$ we have $d(\xi_n, \xi_m) < \epsilon$. In our case, this means that there, for each time $t$ is a positive integer $N$ such that for all $n,m \geq N$, we have that $\xi_n[i] = \xi_m[i]$ for all $i \leq t$. We can thus define a trajectory $\xi_\infty$ by letting $\xi_\infty[i] = \langle s,a \rangle$ if there is an $N$ such that, for all $n \geq N$, we have that $\pi_n[i] = \langle s,a \rangle$. Now $\lim_{i \to \infty} \{\xi_i\}_{i=0}^\infty = \xi_\infty$, and $\xi_\infty \in (\Xi, d)$. Thus every Cauchy sequence in $(\Xi, d)$ has a limit that is also in $(\Xi, d)$, and so $(\Xi, d)$ is complete.

Every metric space which is totally bounded and complete is compact. Thus, $(\Xi, d)$ is a compact metric space.
\end{proof}
\fi

We can now characterise $\ORDstar$ as follows:

\begin{proposition}
$R_1 \eq{\ORDstar} R_2$ if and only if $R_1$ and $R_2$ differ by potential shaping and positive linear scaling. 
\end{proposition}
\ifshowproofs
\begin{proof}
Fix a state $s$. It is now straightforward that
$$
\Expect{\xi \sim D_1}{G_1(\xi)} \geq \Expect{\xi \sim D_2}{G_1(\xi)} \iff \Expect{\xi \sim D_1}{G_2(\xi)} \geq \Expect{\xi \sim D_2}{G_2(\xi)}
$$
for all distributions $D_1, D_2$ over trajectories that start in $s$, if and only if $G_1$ and $G_2$ differ by an affine transformation for all trajectories starting in $s$. Moreover, this corresponds exactly to potential shaping and positive linear scaling of $R$, as per Proposition~\ref{prop:potentials_and_episodes} and \ref{prop:linear_scaling_of_G}. Furthermore, since this condition holds for all $s \in \States$ whenever it holds for one $s \in \States$, this means that $R_1 \eq{\ORDstar} R_2$ if and only if $R_1$ and $R_2$ differ by potential shaping and positive linear scaling. 

We will next show rigorously that
$$
\Expect{\xi \sim D_1}{G_1(\xi)} \geq \Expect{\xi \sim D_2}{G_1(\xi)} \iff \Expect{\xi \sim D_1}{G_2(\xi)} \geq \Expect{\xi \sim D_2}{G_2(\xi)}
$$
for all distributions $D_1, D_2$ over trajectories that start in $s$, if and only if $G_1$ and $G_2$ differ by an affine transformation for all trajectories starting in $s$. For the first direction, suppose there is an $a \in \mathbb{R}^+$ and a $b \in \mathbb{R}$ such that $G_2(\xi) = a \cdot G_1(\xi) + b$ for all trajectories $\xi$ that start in $s$. Then $\Expect{\xi \sim D}{G_2(\xi)} = a \cdot \Expect{\xi \sim D}{G_1(\xi)} + b$ for all distributions $D$ over trajectories that start in $s$, by the linearity of expectation. This in turn implies that 
$$
\Expect{\xi \sim D_1}{G_1(\xi)} \geq \Expect{\xi \sim D_2}{G_1(\xi)} \iff \Expect{\xi \sim D_1}{G_2(\xi)} \geq \Expect{\xi \sim D_2}{G_2(\xi)}
$$
for all distributions $D_1, D_2$ over trajectories that start in $s$, since 
$$
\Expect{\xi \sim D_1}{G_1(\xi)} \geq \Expect{\xi \sim D_2}{G_1(\xi)} \iff a \cdot \Expect{\xi \sim D_1}{G_1(\xi)} + b \geq a \cdot \Expect{\xi \sim D_2}{G_1(\xi)} + b.
$$
For the other direction, suppose 
$$
\Expect{\xi \sim D_1}{G_1(\xi)} \geq \Expect{\xi \sim D_2}{G_1(\xi)} \iff \Expect{\xi \sim D_1}{G_2(\xi)} \geq \Expect{\xi \sim D_2}{G_2(\xi)}
$$
for all distributions $D_1, D_2$ over trajectories that start in $s$. Next, let $\Xi = (\SxA)^\omega$ be the set of all trajectories, and let $d : \Xi \times \Xi \to \mathbb{R}$ be the function given by $d(\xi_1,\xi_2) = \frac{1}{e^t}$, where $t$ is the smallest index on which $\xi_1$ and $\xi_2$ differ, or $0$ if $\xi_1 = \xi_2$. As per Lemma~\ref{lemma:trajectories_compact_space}, $(d,\Xi)$ is a compact metric space. Moreover, it is easy to see that this still holds if we restrict $\Xi$ to the set of trajectories $\Xi_s$ that start in $s$, and that $G_1$ is continuous with respect to the metric $d$. As per the extreme value theorem, this implies that there are two trajectories $\xi_1, \xi_2 \in \Xi_s$ such that $G_1(\xi_1) \leq G_1(\xi) \leq G_1(\xi_2)$ for all $\xi \in \Xi_s$. Next, if $G_1(\xi_1) = G_1(\xi_2)$, then $R_1$ is trivial, in which case $R_2$ is trivial as well. In this case, simply let $a = 1$ and $b = G_2(\xi_1) - G_1(\xi_1)$. Otherwise, let 
$$
a = \frac{G_2(\xi_2) - G_2(\xi_1)}{G_1(\xi_2) - G_1(\xi_1)},
$$
and let $b = G_2(\xi_1) - a \cdot G_1(\xi_1)$. By rearranging, it is easy to see that $G_2(\xi_1) = a \cdot G_1(\xi_1) + b$ and $G_2(\xi_2) = a \cdot G_1(\xi_2) + b$. Moreover, let $\xi$ be an arbitrary trajectory in $\Xi_s$. The intermediate value theorem now implies that there is a $p \in [0,1]$ such that
$$
(1-p) \cdot G_1(\xi_1) + p \cdot G_1(\xi_2) = G_1(\xi).
$$
Moreover, since $(1-p) \cdot G_1(\xi_1) + p \cdot G_1(\xi_2)$ grows monotonically in $p$, this value must be unique. We can now consider the distribution $D_1$ over $\Xi_s$ that returns $\xi_2$ with probability $p$, and $\xi_1$ otherwise, and the distribution $D_2$ that returns $\xi$ with probability 1. This means that $\Expect{\xi \sim D_1}{G_1(\xi)} = \Expect{\xi \sim D_2}{G_1(\xi)}$, and so $\Expect{\xi \sim D_1}{G_2(\xi)} = \Expect{\xi \sim D_2}{G_2(\xi)}$. In other words,
$$
(1-p) \cdot G_2(\xi_1) + p \cdot G_2(\xi_2) = G_2(\xi),
$$
which in turn implies that
$$
(1-p) \cdot (a \cdot G_1(\xi_1) + b) + p \cdot (a \cdot G_1(\xi_2) + b) = G_2(\xi).
$$
By rearranging, and simplifying, we get that
$$
a \cdot ((1-p) \cdot G_1(\xi_1) + p \cdot G_1(\xi_2)) + b = G_2(\xi)
$$
and so $G_2(\xi) = a \cdot G_1(\xi) + b$. Since $\xi$ was chosen arbitrarily, this holds for all $\xi \in \Xi_s$. This completes the other direction.
\end{proof}
\fi

Moreover, we would like to remark on a minor subtlety. One might expect that the equivalence relation $\eq{\Omega}$ according to which $R_1 \eq{\Omega} R_2$ if and only if $R_1 \eq{\ORD} R_2$ for all $\tfunc$, is the same as $\ORDstar$. However, this is not the case.
To see this, consider the rewards $R_1$, $R_2$ where $R_1(s_1,a_1,s_1) = 1$, $R_1(s_1,a_1,s_2) = 0.5$, $R_2(s_1,a_1,s_1) = 0.5$, and $R_2(s_1,a_1,s_2) = 1$, and where $R_1$ and $R_2$ are $0$ for all other transitions. Now $R_1$ and $R_2$ do not differ by potential shaping and linear scaling, yet they have the same policy order for all $\tfunc$. 
The reason for this is that $\ORDstar$ considers all distributions over $\xi$, and that not all of these distributions can be realised by some policy $\pi$ and some transition function $\tfunc$. Characterising $\eq{\Omega}$ would therefore require a more careful analysis of the impact of changing the transition function $\tfunc$.
\section{Explanations and Examples}\label{appendix:explanations_and_examples}

The purpose of this appendix is to make some of our results more intuitive. In particular, we provide a more extended discussion of STARC metrics, including several geometric intuitions, that may make them easier to understand. We also illustrate Theorem~\ref{thm:wrong_tau_large_diameter} and \ref{thm:wrong_gamma_large_diameter} with two examples, that may help to explain why robust transfer learning is difficult to guarantee.

\subsection{Understanding STARC Metrics}\label{appendix:understand_starc}

In this Appendix, we provide a geometric intuition for how STARC metrics work. This will make it easier to understand STARC metrics, and may also make it easier to understand our proofs.

First of all, note that the space of all reward functions $\mathcal{R}$ forms an $|\States||\Actions||\States|$-dimensional vector space. Next, recall that if two reward functions $R_1$ and $R_2$ differ by (some combination of) potential shaping and $S'$-redistribution, then $R_1$ and $R_2$ induce the same ordering of policies. Moreover, these transformations correspond to a linear subspace of $\R$ (see Proposition~\ref{prop:PS_SR_linear} in \ref{appendix:reward_transformation_properties}). A canonicalisation function is simply a linear map that removes the dimensions that are associated with potential shaping and $S'$-redistribution. In other words, they map $\mathcal{R}$ to an $|\States|(|\Actions|-1)$-dimensional subspace of $\mathcal{R}$ in which no reward functions differ by potential shaping or $S'$-redistribution. The canonicalisation function that is minimal for the $L_2$-norm is the orthogonal map that satisfies these properties, whereas other canonicalisation functions are non-orthogonal.
When we normalise the resulting reward functions by dividing by a norm $n$, we project the entire vector space onto the unit ball of $n$ (except the zero reward, which remains at the origin). The metric $m$ then measures the distance between the resulting reward functions on the surface of this sphere:

\begin{figure}[H]
    \centering
    \includegraphics[width=\textwidth/3]{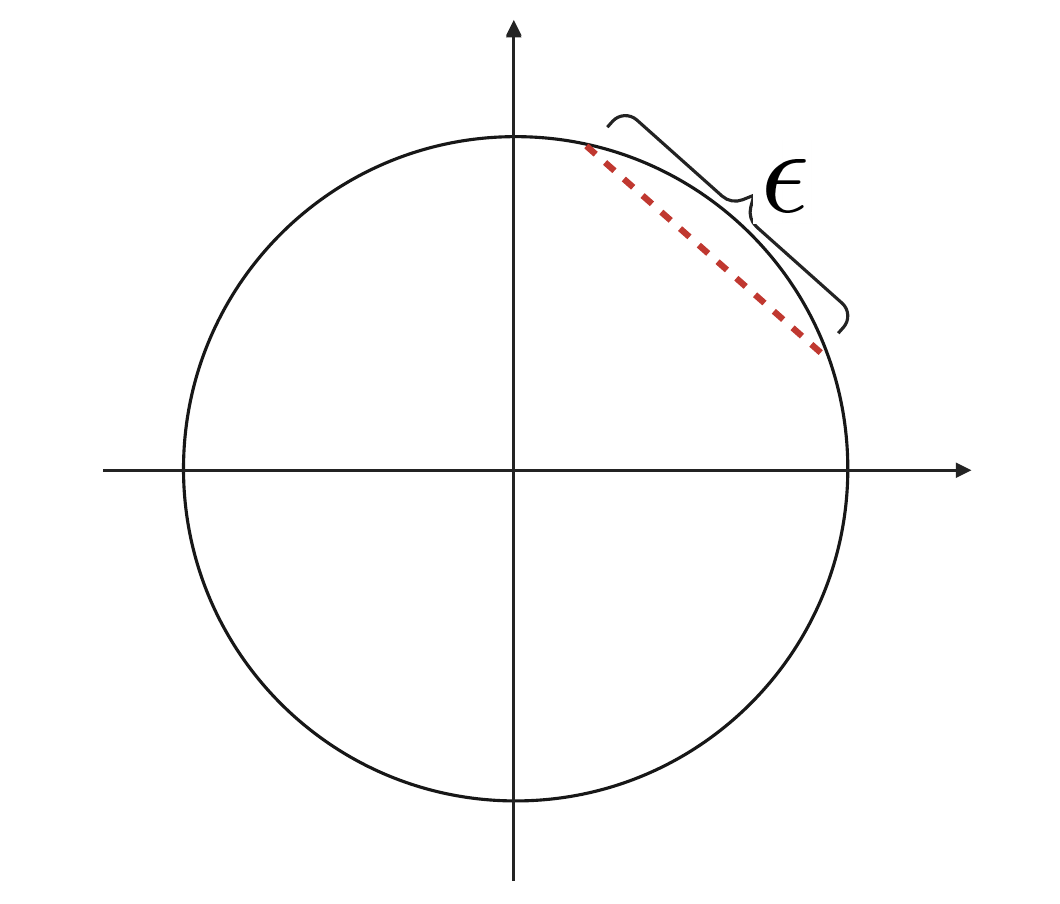}
    %\caption{Caption}
    \label{fig:circle_diagram}
\end{figure}

To make this more clear, it may be worth considering the case of non-sequential decision making. Suppose we have a finite set of \emph{choices} $C$, and a utility function $U : C \to \mathbb{R}$. Given two distributions $D_1$, $D_2$ over $C$, we say that we prefer $D_1$ over $D_2$ if $\mathbb{E}_{c \sim D_1}[U(c)] > \mathbb{E}_{c \sim D_2}[U(c)]$. The set of all utility functions over $C$ forms a $|C|$-dimensional vector space. Moreover, in this setting, it is well-known that two utility functions $U_1$, $U_2$ induce the same preferences between all possible distributions over $C$ if and only if they differ by an affine transformation. Therefore, if we wanted to represent the set of all \emph{non-equivalent} utility functions over $C$, we may consider requiring that $U(c_0) = 0$ for some $c_0 \in C$, and that $L_2(U) = 1$ unless $U(c) = 0$ for all $c \in C$. Any utility function over $C$ is equivalent to some utility function in this set, and this set can in turn be represented as the surface of a $(|C|-1)$-dimensional sphere, together with the origin. 
This is essentially analogous to the standardisation that the canonicalisation function $c$ and the normalisation function $n$ perform for STARC metrics. Here $C$ is analogous to the set of all trajectories, the trajectory return function $G$ is analogous to $U$, and a policy $\pi$ induces a distribution over trajectories. Affine transformations of the trajectory return function, $G$, correspond exactly to potential shaping and positive linear scaling of $R$ (Proposition~\ref{prop:potentials_and_episodes} and \ref{prop:linear_scaling_of_G} in \ref{appendix:reward_transformation_properties}). However, it is also important to note that while the cases are analogous, it is not a direct correspondence, because not all distributions over trajectories can be realised as a policy in a given MDP.

Another perspective that may help with understanding STARC metrics comes from considering the occupancy measures of policies. Recall that the occupancy measure $\eta^\pi$ of a policy $\pi$ is the $|\States||\Actions||\States|$-dimensional vector in which the value of the $(s,a,s')$'th dimension is
$$
\sum_{t=0}^\infty \gamma^t \mathbb{P}_{\xi \sim \pi}(S_t = s, A_t = a, S_{t+1} = s').
$$
Also recall that $J(\pi) = \eta^\pi \cdot R$. Therefore, by computing occupancy measures, we can divide the computation of $J$ into two parts, the first of which is independent of $R$, and the second of which is a linear function. Moreover, let $\Omega = \{\eta^\pi : \pi \in \Pi\}$ be the set of all occupancy measures. We now have that the policy value function $J$ of a reward function $R$ can be visualised as a linear function on this set. Moreover, if we have two reward functions $R_1$, $R_2$, then they can be visualised as two different linear functions on this set:

\begin{figure}[H]
    \centering
    \includegraphics[width=\textwidth/3]{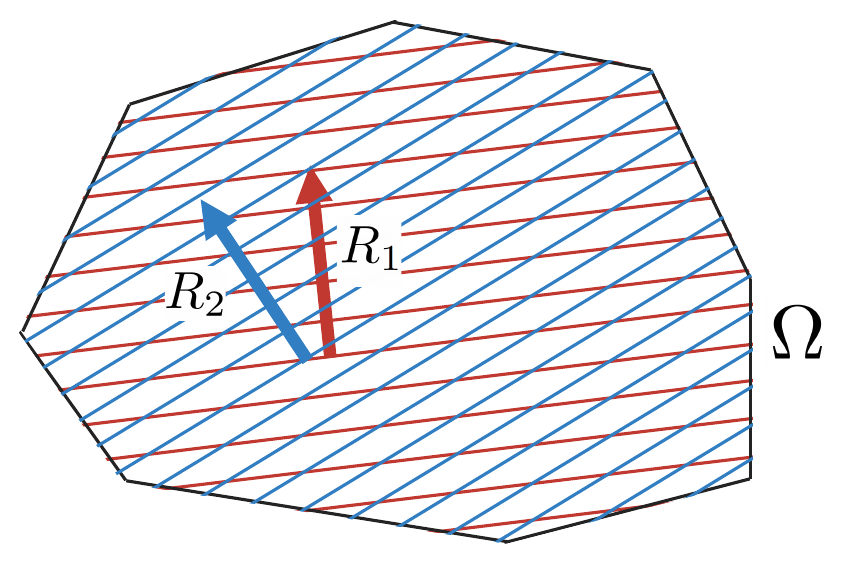}
    %\caption{Caption}
    \label{fig:occupancy_measure}
\end{figure}

%$\Omega$ is located in a $|\States|(|\Actions|-1)$-dimensional affine subspace of $\mathbb{R}^{|\States||\Actions||\States|}$, and contains a set which is open in this space (Lemma~\ref{lemma:homeomorphism}). Moreover, it can be represented as the convex hull of a finite set of points \citep{splitting_randomised_policies}. It is thus a polytope.

From this image, it is visually clear that the worst-case regret of maximising $R_1$ instead of $R_2$, should be proportional to the angle between the linear functions that $R_1$ and $R_2$ induce on $\Omega$. Moreover, this is what STARC metrics measure.
%; any STARC metric is bilipschitz equivalent to the angle between reward functions projected onto $\Omega$. 
In particular, the function $c$ that projects each $R$ onto the linear subspace of $\R$ that is parallel to $\Omega$ is a canonicalisation function (Proposition~\ref{prop:occupancy_measure_projection_is_canonicalisation}). Normalising these reward functions, and measuring their distance using a metric that is bilipschitz equivalent to a norm, is bilipschitz equivalent to measuring their angle.
This should in turn give an intuition for why the STARC distance between two rewards provide both an upper and lower bound on their worst-case regret. %(c.f.\ Section~\ref{section:properties_of_STARC}).

\subsection{Understanding the Difficulty of Transfer Learning}\label{appendix:understand_transfer_learning}

In this appendix, we will provide an intuitive explanation for Theorem~\ref{thm:wrong_tau_large_diameter} and \ref{thm:wrong_gamma_large_diameter}. First, recall Theorem~\ref{thm:wrong_tau_large_diameter}:

\wrongtaulargediameter*

Note that $1$ is the maximal distance that is possible under $\starc$. This result may be surprising; if $\tfunc_1 \approx \tfunc_2$, then one might expect that a reward function that is learnt under $\tfunc_1$ should be guaranteed to be mostly accurate under $\tfunc_2$.
To see more intuitively why this is not the case, consider the following example. Suppose we have a simple $N \times N$ gridworld environment, as illustrated in Figure~\ref{fig:gridworld}. We assume that the agent has four actions, \texttt{up}, \texttt{down}, \texttt{left}, and \texttt{right}. We assume that $\tau_2$ is deterministic, so that if the agent takes action \texttt{up}, then it moves one step up, etc. Moreover, we assume that $\tau_1$ is \enquote{slippery}, so that if the agent takes action \texttt{up}, then it moves up, up-left, and up-right with equal probability, and that if it takes action \texttt{right}, then it moves right, up-right, and down-right with equal probability, etc. For simplicity, we will also assume that the environment wraps around itself (like a torus), so that if the agent moves up from the top of the environment, then it ends up at the bottom, and so on.

\begin{figure}[H]
    \centering
    \includegraphics[angle=180,width=\textwidth/3]{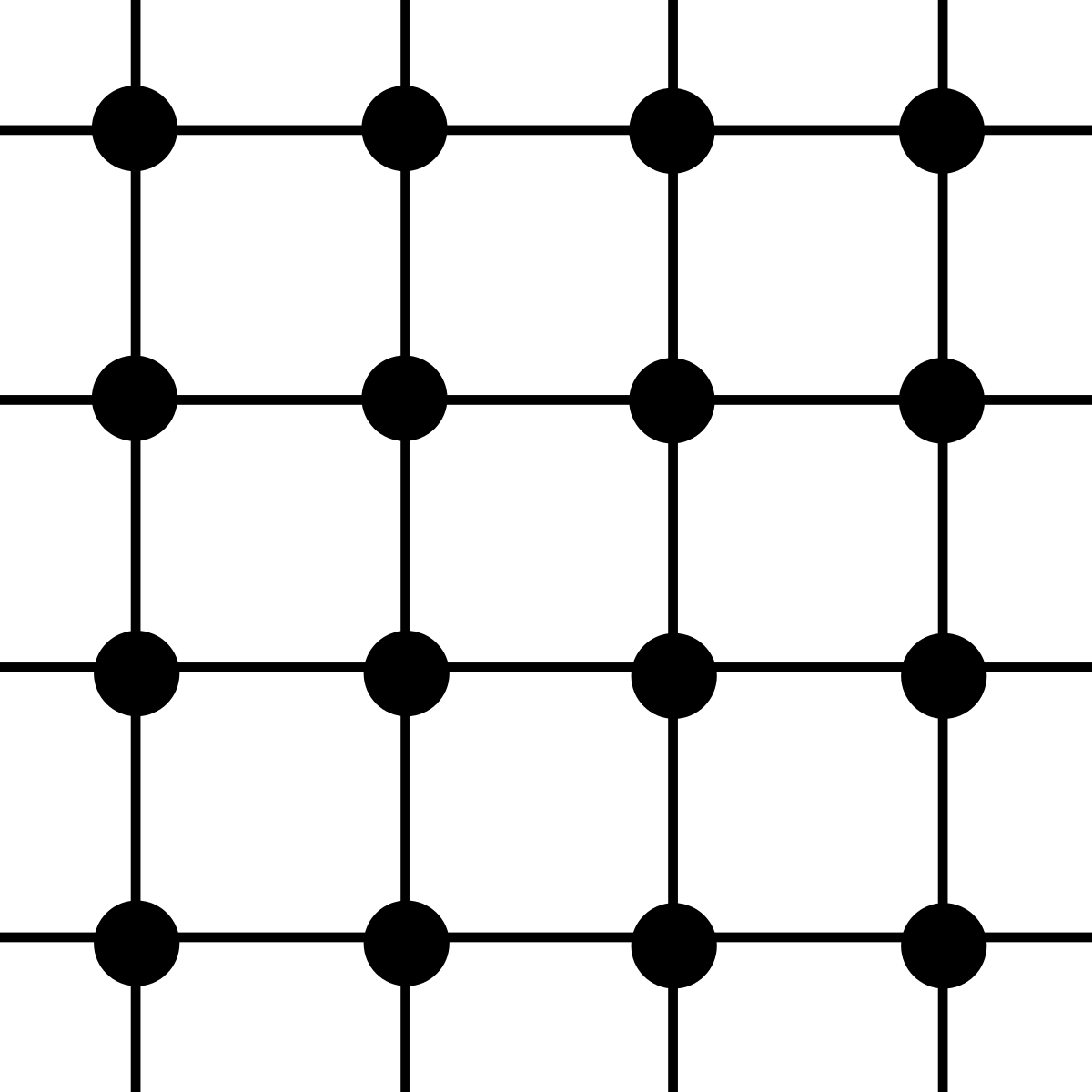}
    \caption{A simple illustration of a gridworld environment.}
    \label{fig:gridworld}
\end{figure}

Now suppose that $R_1$ and $R_2$ reward each transition $(s,a,s')$ depending on the relative location of $s$ and $s'$, according to the following schemas:

\begin{center}
\begin{tikzpicture}[shorten >=1pt,node distance=2.6cm,on grid,auto]
   \node[state] (r1_center)   {}; 
   \node[] (r1_label) [below=of r1_center, yshift=-30]  {$R_1$};
   \node[]         (1_u) [above=of r1_center] {$0$};
   \node[]         (1_ur) [above right=of r1_center] {$1$};
   \node[]         (1_r) [right=of r1_center] {$1$};
   \node[]         (1_dr) [below right=of r1_center] {$1$};
   \node[]         (1_d) [below=of r1_center] {$0$};
   \node[]         (1_dl) [below left=of r1_center] {$-1$};
   \node[]         (1_l) [left=of r1_center] {$-1$};
   \node[]         (1_ul) [above left=of r1_center] {$-1$};
   \node[state] (r2_center) [right=of r1_center, xshift=130]  {}; 
   \node[] (r2_label) [below=of r2_center, yshift=-30]  {$R_2$};
   \node[]         (2_u) [above=of r2_center] {$0$};
   \node[]         (2_ur) [above right=of r2_center] {$3$};
   \node[]         (2_r) [right=of r2_center] {$-1$};
   \node[]         (2_dr) [below right=of r2_center] {$-1$};
   \node[]         (2_d) [below=of r2_center] {$0$};
   \node[]         (2_dl) [below left=of r2_center] {$1$};
   \node[]         (2_l) [left=of r2_center] {$1$};
   \node[]         (2_ul) [above left=of r2_center] {$-3$};
    \path[->] 
    (r1_center) edge [] node {} (1_u)
    (r1_center) edge [] node {} (1_ur)
    (r1_center) edge [] node {} (1_r)
    (r1_center) edge [] node {} (1_dr)
    (r1_center) edge [] node {} (1_d)
    (r1_center) edge [] node {} (1_dl)
    (r1_center) edge [] node {} (1_l)
    (r1_center) edge [] node {} (1_ul)
    (r2_center) edge [] node {} (2_u)
    (r2_center) edge [] node {} (2_ur)
    (r2_center) edge [] node {} (2_r)
    (r2_center) edge [] node {} (2_dr)
    (r2_center) edge [] node {} (2_d)
    (r2_center) edge [] node {} (2_dl)
    (r2_center) edge [] node {} (2_l)
    (r2_center) edge [] node {} (2_ul)
    ;
\end{tikzpicture}
\end{center}

%\begin{figure}[H]
%    \centering
%    \includegraphics[width=\textwidth/2]{images/windy_rewards.jpg}
    %\caption{Caption}
%    \label{fig:windy_rewards}
%\end{figure}

These two reward functions are equivalent under $\tau_1$, and give the agent $1$ reward for going right, $-1$ for going left, and $0$ for going up or down. However, under $\tau_2$, they are opposites; $R_1$ rewards the agent for going right, and $R_2$ rewards the agent for going left. Thus, if we observe a policy computed under $\tau_1$, then we will not be able to distinguish between $R_1$ and $R_2$, even though they induce very different behaviour under $\tau_2$. For this reason, it is difficult to obtain guarantees for transfer learning in IRL.

Next, recall Theorem~\ref{thm:wrong_gamma_large_diameter}:

\wronggammalargediameter*

Again, recall that $1$ is the maximal distance that is possible under $\starc$. This result may also be surprising; if $\gamma_1 \approx \gamma_2$, then one might expect that a reward function that is learnt under $\gamma_1$ should be guaranteed to be mostly accurate under $\gamma_2$. To see more intuitively why this is not the case, consider a simple environment with three states $s_0, s_1, s_2$, where $s_0$ is the initial state, and where the agent can choose to either go directly from $s_0$ to $s_2$, or choose to first visit state $s_1$: %\red{[TODO: streamline?]}

\begin{center}
\begin{tikzpicture}[shorten >=1pt,node distance=2.6cm,on grid,auto]
   \node[state, initial] (s_0)   {$s_0$}; 
   \node[state]         (s_1) [above right=of s_0] {$s_1$};
   \node[state]         (s_2) [above left=of s_1] {$s_2$};
    \path[->] 
    (s_0) edge [swap] node {} (s_1) %{$\gamma_1 \cdot x$}
          edge [swap, sloped] node {} (s_2)
    %(s_1) edge [] node {$a_1$} (s_2)
    %      edge [swap] node {$a_2$} (s_3)
    (s_1) edge [swap] node {} (s_2) %{$-x$}
    (s_2) edge [loop] node {} (s_2)
    %(s_3) edge [sloped] node {} (invis)
    %(invis) edge [sloped] node {} (s_n)
    ;
\end{tikzpicture}
\end{center}

Let $R_1$ be any reward function over this environment, and let $R_2$ be the reward function that we get if we take $R_1$ and \emph{increase} the reward of going from $s_0$ to $s_1$ by $\gamma_1 \cdot X$, and \emph{decrease} the reward of going from $s_1$ to $s_2$ by $X$. Now, the policy order under discounting with $\gamma_1$ is completely unchanged. 
%At $s_1$, the value of every action is changed by the same amount, and so there is no reason to change action. Similarly, at $s_0$, the value of going to $s_1$ is changed by $\gamma_1 \cdot X - \gamma_1 \cdot X = 0$, and so there is likewise no reason to change action. 
This transformation corresponds to potential shaping where $\Phi(s_1) = X$ and $\Phi(s_0) = \Phi(s_2) = 0$. Therefore, if $f : \mathcal{R} \to \Pi$ is invariant to potential shaping with $\gamma_1$, then $f(R_1) = f(R_2)$.
However, if we discount with $\gamma_2$, then $R_1$ and $R_2$ have a different policy order.
In particular, the value of going from $s_0$ to $s_1$ is changed by $\gamma_1 \cdot X - \gamma_2 \cdot X = (\gamma_1-\gamma_2) \cdot X \neq 0$. Thus, if the optimal action under $R_1$ at $s_0$ is to go to $s_1$, then by making $X$ sufficiently large or sufficiently small (depending on whether $\gamma_1 > \gamma_2$, or vice versa), then we can create a reward function $R_2$ for which the optimal action instead is to go to $s_2$, and vice versa. Thus, if we observe a policy computed under $\gamma_1$, then we will not be able to distinguish between $R_1$ and $R_2$, even though they induce different behaviours when discounting with $\gamma_2$. This makes it difficult to ensure robust transfer to a new $\gamma$.

Intuitively speaking, we can use potential shaping to move reward around in the MDP (so that the agent receives a larger immediate reward at the cost of a lower reward later, or vice versa). However, to cancel out the effect of the discounting, later rewards must be made larger than immediate rewards. If the discount values do not match, then this \enquote{compensation} will also not match, leading to a distortion of the policy ordering. Indeed, we can make it so that this distortion dominates the rest of the reward function. 

\section{Proofs}\label{appendix:proofs}

This appendix contains the proofs of all theorems provided in the main text. Its organisation largely mirrors the sections of the main text, but with additional sections for supporting lemmas and auxiliary results.

\subsection{Properties of Our Frameworks}

In this section we will provide proofs of the basic results concerning our theoretical frameworks, first stated in Sections~\ref{sec:frameworks_ambiguity} and \ref{sec:frameworks_properties_of_definitions}. 

\subsubsection{Partial Identifiability}

\ambiguityinherited*

\begin{proof}
If $f(R_1) = f(R_2)$, then $h \circ f(R_1) = h \circ f(R_2)$, so $g(R_1) = g(R_2)$. Thus $f(R_1) = f(R_2) \implies g(R_1) = g(R_2)$, so $\mathrm{Am}(f) \refines \mathrm{Am}(g)$.
\end{proof}

\subsubsection{Misspecification With Equivalence Relations}

\Probustnessinheritance*

\begin{proof}
If $f$ is not $P$-robust to misspecification with $g$, and $\mathrm{Im}(g) \subseteq \mathrm{Im}(f)$, then either $f \not \refines P$, or $f = g$, or $f(R_1) = g(R_2)$ but $R_1 \not\equiv_P R_2$ for some $R_1, R_2$.
In the first case, if $f \not \refines P$ then $h \circ f \not \refines P$, as per Lemma~\ref{lemma:ambiguity_inherited}. This implies that $h \circ f$ is not $P$-robust to misspecification with any reward object (including with $h \circ g$).
In the second case, if $f = g$ then $h \circ f = h \circ g$. This implies that $h \circ f$ is not $P$-robust to misspecification with $h \circ g$.
In the last case, suppose $f(R_1) = g(R_2)$ but $R_1 \not\equiv_P R_2$ for some $R_1, R_2$. If $f(R_1) = g(R_2)$ then $h \circ f(R_1) = h \circ g(R_2)$, so there are $R_1, R_2$ such that $h \circ f(R_1) = h \circ g(R_2)$, but $R_1 \not\equiv_P R_2$. This implies that $h \circ f$ is not $P$-robust to misspecification with $h \circ g$.
\end{proof}

\Probustnessimpliesrefinement*

\begin{proof}
Suppose $f$ is $P$-robust to misspecification with $g$, and let $R_1,R_2$ be any two reward functions such that $g(R_1) = g(R_2)$.
Since $\mathrm{Im}(g) \subseteq \mathrm{Im}(f)$ there is an $R_3$ such that $f(R_3) = g(R_1) = g(R_2)$. 
Since $f$ is $P$-robust to misspecification with $g$, it must be the case that $R_3 \equiv_P R_1$ and $R_3 \equiv_P R_2$. By transitivity, we thus have that $R_1 \equiv_P R_2$. Since $R_1$ and $R_2$ were chosen arbitrarily, it must be that $R_1 \equiv_P R_2$ whenever $g(R_1) = g(R_2)$.
\end{proof}

\Probustnesssymmetry*

\begin{proof}
If $f$ is $P$-robust to misspecification with $g$ then this immediately implies that $f \neq g$, and that if $f(R_1) = g(R_2)$ for some $R_1,R_2$ then $R_1 \equiv_P R_2$. Lemma~\ref{lemma:P_robustness_implies_refinement} implies that $\mathrm{Am}(g) \refines P$, and if $\mathrm{Im}(f) = \mathrm{Im}(g)$ then $\mathrm{Im}(f) \subseteq \mathrm{Im}(g)$. This means that $g$ is $P$-robust to misspecification with $f$.
\end{proof}

\lessambiguitylessrobustness*

\begin{proof}
First consider the case when $\mathrm{Am}(f) = P$, and assume for contradiction that $f$ is $P$-robust to misspecification with $g$. Let $R_1$ be any reward function, and consider $g(R_1)$. Since $\mathrm{Im}(g) \subseteq \mathrm{Im}(f)$, there is an $R_2$ such that $f(R_2) = g(R_1)$. Since $f$ is $P$-robust to misspecification with $g$, this implies that $R_2 \equiv_P R_1$. Moreover, if $\mathrm{Am}(f) = P$ then $R_2 \equiv_P R_1$ if and only if $f(R_2) = f(R_1)$, so it must be the case that $f(R_2) = f(R_1)$. Now, since $f(R_2) = f(R_1)$ and $f(R_2) = g(R_1)$, we have that $g(R_1) = f(R_1)$. Since $R_1$ was chosen arbitrarily, this implies that $f = g$, which is a contradiction.
Hence, if $\mathrm{Am}(f) = P$ then there is no $g$ such that $f$ is $P$-robust to misspecification with $g$.

Next, consider the case when $\mathrm{Am}(f) \refines P$ and $\mathrm{Am}(f) \neq P$. This implies that there are $R_1, R_2$ such that $R_1 \equiv_P R_2$ but $f(R_1) \neq f(R_2)$. We can then construct a $g$ as follows; let $g(R_1) = f(R_2)$, $g(R_2) = f(R_1)$, and $g(R) = f(R)$ for all $R \notin \{R_1,R_2\}$. Now $f$ is $P$-robust to misspecification with $g$. Hence, if $\mathrm{Am}(f) \refines P$ and $\mathrm{Am}(f) \neq P$ then there is a $g$ such that $f$ is $P$-robust to misspecification with $g$, which in turn implies that if $\mathrm{Am}(f) \refines P$ and there is no $g$ such that $f$ is $P$-robust to misspecification with $g$ then $\mathrm{Am}(f) = P$.
\end{proof}

\Probustnessfunctioncomposition*

\begin{proof}
First suppose that $f$ is $P$-robust to misspecification with $g$ --- we will construct a $t$ that fits our description.
Since $\mathrm{Im}(g) \subseteq \mathrm{Im}(f)$, we have that there for each $R$ exists an $R'$ such that $g(R) = f(R')$. Let $t : \R \to \R$ be a function that maps each $R$ to one such $R'$.
Since by construction $g(R) = f(t(R))$ for each $R$, we have that $g = f \circ t$.
Moreover, since $f$ is $P$-robust to misspecification with $g$, we have that $R \eq{P} t(R)$. This completes the first direction.

For the other direction, suppose $f \neq g$ and $g = f \circ t$ for some $t : \R \to \R$ such that $R \eq{P} t(R)$ for all $R$. 
By assumption we have that $\mathrm{Am}(f) \refines P$.
Moreover, we clearly have that $\mathrm{Im}(g) \subseteq \mathrm{Im}(f)$. 
Finally, if $g(R_1) = f(R_2)$ then $f(t(R_1)) = f(R_2)$. Since $\mathrm{Am}(f) \refines P$, this means that $t(R_1) \equiv_P R_2$. Moreover, since $R \eq{P} t(R)$ for all $R$, we have that $R_1 \eq{P} t(R_1)$. By transitivity, this means that $R_1 \eq{P} R_2$.
Thus $f$ is $P$-robust to misspecification with $g$, and we are done.
\end{proof}

It is worth noting that Lemma~\ref{lemma:less_ambiguity_less_robustness} also can be derived as a corollary of Lemma~\ref{lemma:P_robustness_function_composition}, by noting that if $\mathrm{Am}(f) = P$, then $f = f \circ t$ for all $t : \R \to \R$ such that $R \eq{P} t(R)$, and vice versa.

\subsubsection{Misspecification With Distance Metrics}

\metricequivalencerelationconnection*

\begin{proof}
    Immediate from Definition~\ref{def:misspecification_eq} and \ref{def:misspecification_metric}.
\end{proof}

\metricequivalencerelationconnectiontwo*

\begin{proof}
    Immediate from Definition~\ref{def:misspecification_metric}.
\end{proof}

\noepsilonrobustnessimpliesrefinement*

\begin{proof}
For example, let $d^\R$ be the metric induced by the $L_2$-norm, let $X$ be any set such that $|X| \geq |\R|$, and let $f : \R \to X$ be any injective function.
Pick two reward functions $R_1, R_2$ such that $d^\R(R_1,R_2) = 2\epsilon$, let $g(R_1) = g(R_2) = f((R_1 + R_2)/2)$, and let $g(R) = f(R)$ for $R \neq R_1,R_2$.
\begin{figure}[H]
    \centering
    \includegraphics[width=3\textwidth/4]{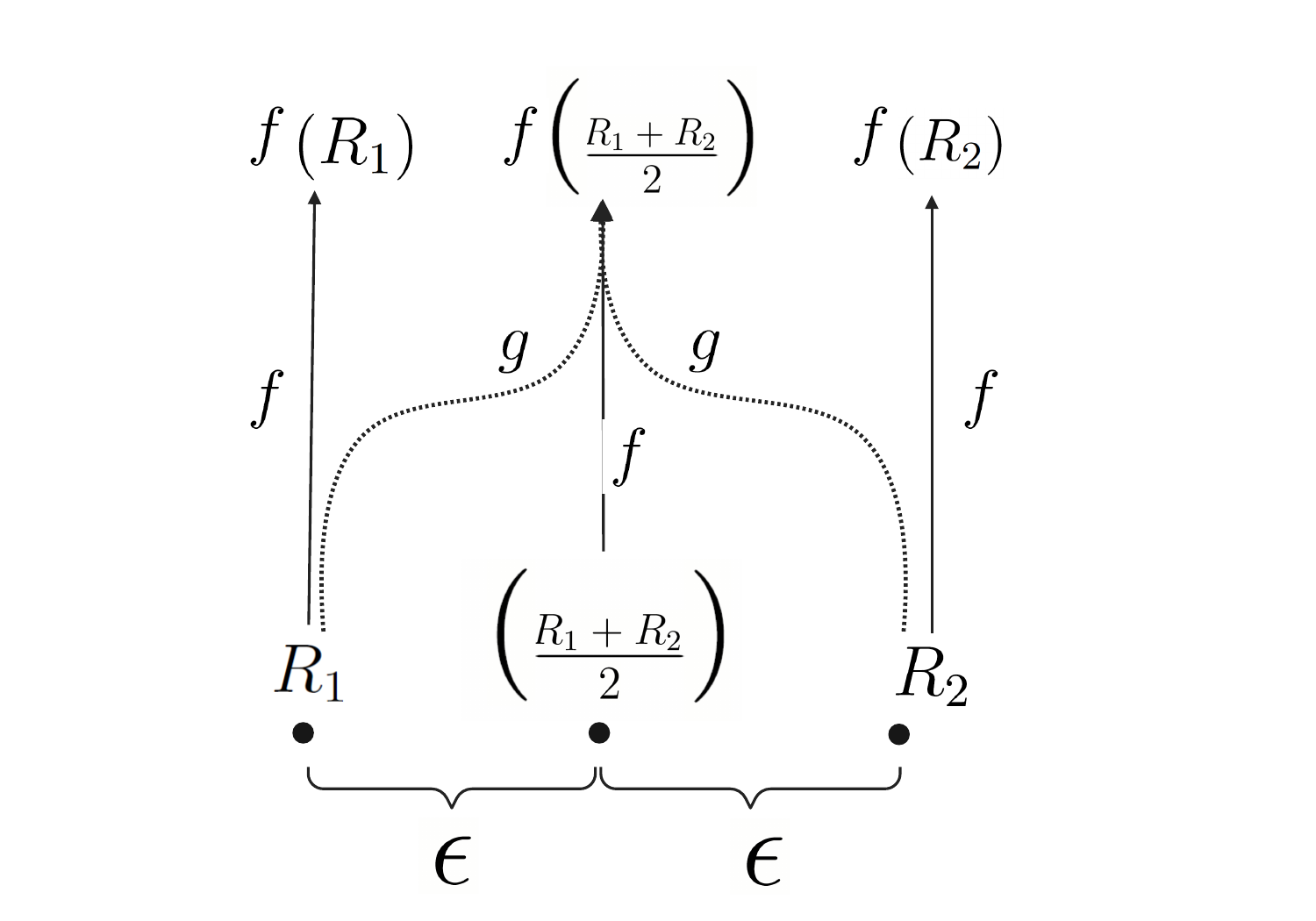}
    %\caption{\red{TODO}}
    \label{fig:frameworks_prop17}
\end{figure}
\end{proof}

\epsilonrobustnessimpliesweakrefinement*

\begin{proof}
Let $R_1, R_2$ be any two reward functions such that $g(R_1) = g(R_2)$. From condition 2 in Definition~\ref{def:misspecification_metric}, we have that there is a reward $R_3$ such that $f(R_3) = g(R_1) = g(R_2)$. From condition 1 in Definition~\ref{def:misspecification_metric}, we have that $d^\mathcal{R}(R_3, R_1) \leq \epsilon$ and that $d^\mathcal{R}(R_3, R_2) \leq \epsilon$. The triangle inequality then implies that $d^\mathcal{R}(R_1, R_2) \leq 2\epsilon$. 
\end{proof}

\noepsilonrobustnesssymmetry*

\begin{proof}
For example, let $d^\R$ be the metric induced by the $L_2$-norm, let $X$ be any set such that $|X| \geq |\R|$, and let $h : \R \to X$ be any injective function. Pick four reward functions $R_1,R_2,R_a,R_b$ such that $d^\R(R_1,R_2) = 2\epsilon$, $d^\R(R_1,R_a) < \epsilon$, and $d^\R(R_2,R_b) < \epsilon$. Let $g(R_1) = g(R_2) = h((R_1 + R_2)/2)$, and let $g(R) = h(R)$ for $R \neq R_1,R_2$. Let $f(R_1) = h(R_a)$, $f(R_2) = h(R_b)$, and $f(R) = h(R)$ for $R \neq R_1,R_2$.

\begin{figure}[H]
    \centering
    \includegraphics[width=3\textwidth/4]{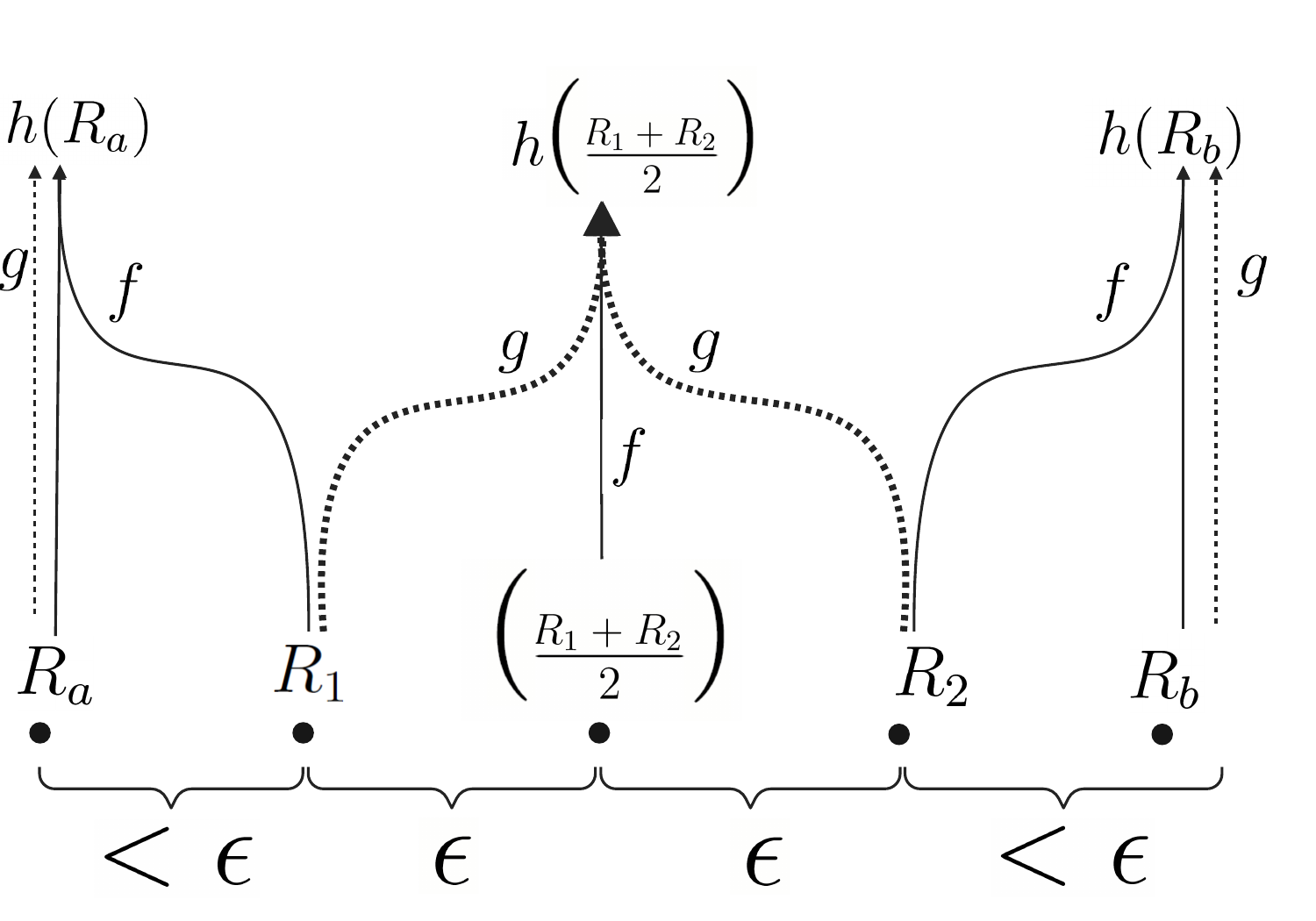}
    %\caption{\red{TODO}}
    \label{fig:frameworks_prop19}
\end{figure}

Now $g$ is not $\epsilon$-robust to misspecification with $f$, since $g(R_1) = g(R_2)$ even though $d^\R(R_1, R_2) = 2\epsilon$. 
However, $f$ is $\epsilon$-robust to misspecification with $g$. 
First, if $f(R) = g(R')$, then either $R = R'$, or $R' = (R_1 + R_2)/2$ and $R$ is either $R_1$ or $R_2$. In the former case $d^\R(R, R') = 0$, and in the latter $d^\R(R_1, R_2) = \epsilon$. Moreover, if $f(R) = f(R')$, then either $R = R'$, or $R = R_1$ and $R' = R_a$ (or vice versa), or $R = R_2$ and $R' = R_b$ (or vice versa). In the first case $d^\R(R, R') = 0$, and in the latter two cases $d^\R(R, R') < \epsilon$. Next, $f \neq g$, since $f(R_1) \neq g(R_1)$ and $f(R_2) \neq g(R_2)$. Finally, $\mathrm{Im}(f) = \mathrm{Im}(g)$, since both $\mathrm{Im}(f)$ and $\mathrm{Im}(g)$ are equal to $\mathrm{Im}(h) \setminus \{h(R_1), h(R_2)\}$.
\end{proof}

\noepsilonrobustnessfunctioncomposition*

\begin{proof}
For example, let $d^\R$ be the metric induced by the $L_2$-norm, and let $X$ be any set such that $|X| \geq |\R|$. Pick three reward functions $R_1, R_2, R_3$ such that $d^\R(R_1, R_2) = \epsilon$, $d^\R(R_2, R_3) = \epsilon$, and $d^\R(R_1, R_3) = 2\epsilon$. Let $f$ be injective, except that $f(R_1) = f(R_2)$, and let $t(R) = R$ for all $R$, except that $t(R_3) = R_2$. Now $f(R_1) = f \circ t(R_3)$, even though $d^\R(R_1,R_2) = 2\epsilon > \epsilon$, and so $f$ is not $\epsilon$-robust to misspecification with $f \circ t$ (as measured by $d^\R$).
\begin{figure}[H]
    \centering
    \includegraphics[width=3\textwidth/4]{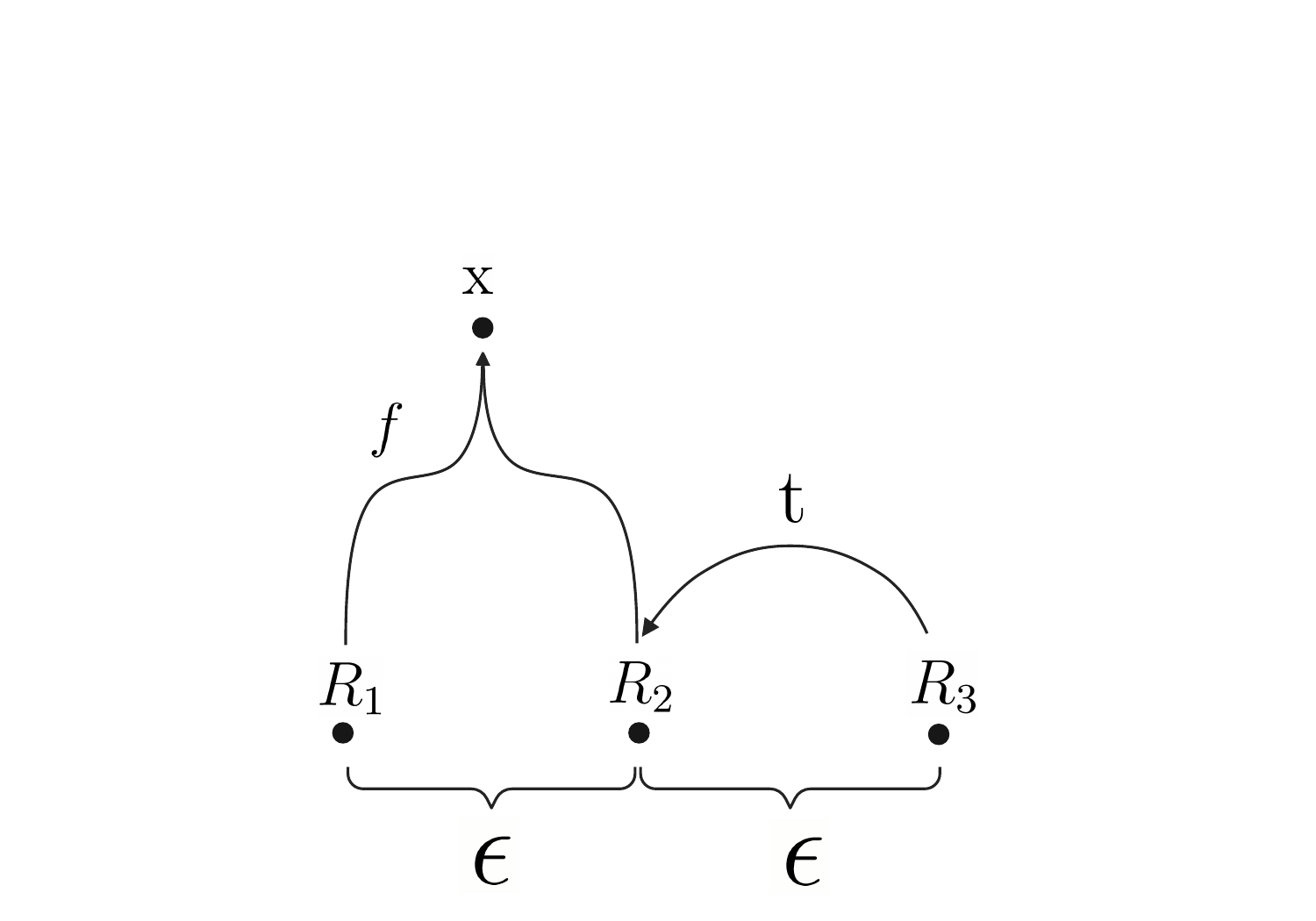}
    %\caption{\red{TODO}}
    \label{fig:frameworks_prop20}
\end{figure}
\end{proof}

\weakepsilonrobustnessfunctioncomposition*

\begin{proof}
For the first direction, let $t : \R \to \R$ be a transformation such that $d^\mathcal{R}(R, t(R)) \leq \epsilon$ for all $R$, and let $g = f \circ t$. Suppose $f \neq g$. To show that $f$ is $\epsilon$-robust to misspecification with $g$, we need to show that:
\begin{enumerate}
    \item If $f(R_1) = g(R_2)$ then $d^\mathcal{R}(R_1, R_2) \leq \epsilon$.
    \item $\mathrm{Im}(g) \subseteq \mathrm{Im}(f)$.
    \item If $f(R_1) = f(R_2)$ then $d^\mathcal{R}(R_1, R_2) \leq \epsilon$.
    \item $f \neq g$.
\end{enumerate}
For the first condition, suppose $f(R_1) = g(R_2)$, which implies that $f(R_1) = f \circ t(R_2)$. By assumption, we have that if $f(R) = f(R')$, then $d^\mathcal{R}(R, R') = 0$. This implies that $d^\mathcal{R}(R_1, t(R_2)) = 0$. Moreover, we have that $d^\mathcal{R}(R, t(R)) \leq \epsilon$ for all $R$; this implies that $d^\mathcal{R}(R_2, t(R_2)) \leq \epsilon$. By the triangle inequality, we then have that $d^\mathcal{R}(R_1, R_2) \leq 0 + \epsilon = \epsilon$. Since $R_1$ and $R_2$ were chosen arbitrarily, this means that condition 1 holds. Condition 2 holds straightforwardly, from the construction of $g$. For condition 3, note that we by assumption have that if $f(R_1) = f(R_2)$, then $d^\mathcal{R}(R_1, R_2) = 0 < \epsilon$. Condition 4 is satisfied by direct assumption. This proves the first direction.

\begin{figure}[H]
    \centering
    \includegraphics[width=3\textwidth/4]{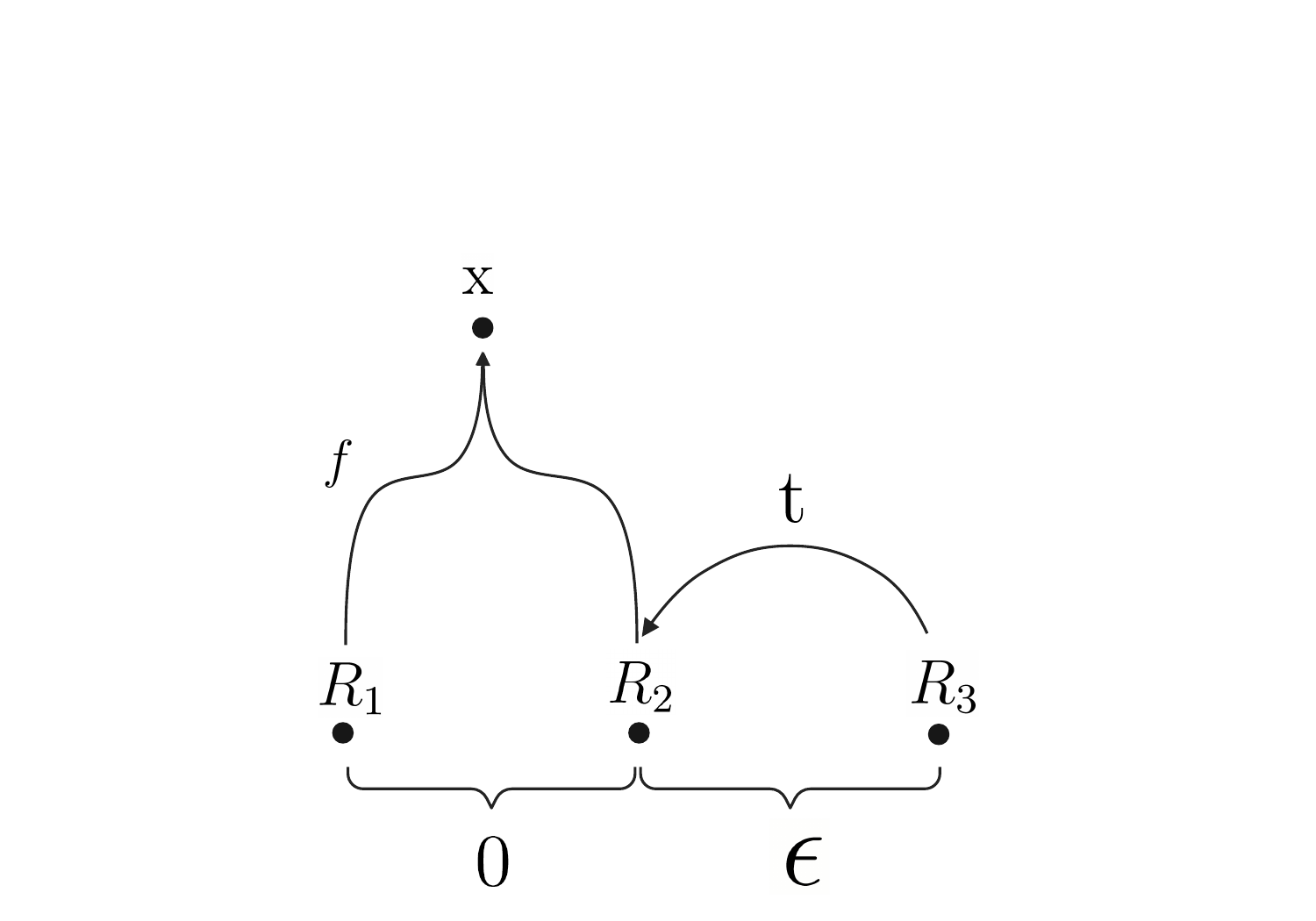}
    %\caption{\red{TODO}}
    \label{fig:frameworks_lemma21}
\end{figure}

For the other direction, let $f$ be $\epsilon$-robust to misspecification with $g$ (as measured by $d^\mathcal{R}$).
Since $\mathrm{Im}(g) \subseteq \mathrm{Im}(f)$, we have that there for each $R$ exists an $R'$ such that $g(R) = f(R')$. Let $t : \R \to \R$ be a function that maps each $R$ to one such $R'$.
Since by construction $g(R) = f(t(R))$ for each $R$, we have that $g = f \circ t$.
Moreover, since $f$ is $\epsilon$-robust to misspecification with $g$ as measured by $d^\R$, we have that $d^\R(R,t(R)) \leq \epsilon$. This completes the proof of the other direction, which means that we are done.
\end{proof}

\epsilonrobustnessinheritance*

\begin{proof}
If $f$ is not $\epsilon$-robust to misspecification with $g$ as measured by $d^\R$, and $\mathrm{Im}(g) \subseteq \mathrm{Im}(f)$, then either there are $R_1, R_2$ such that $f(R_1) = g(R_2)$ but $d^\R(R_1, R_2) > \epsilon$, or there are $R_1, R_2$ such that $f(R_1) = f(R_2)$ but $d^\R(R_1, R_2) > \epsilon$, or $f = g$. 
In the first case, if $f(R_1) = g(R_2)$ but $d^\R(R_1, R_2) > \epsilon$ then $h \circ f(R_1) = h \circ g(R_2)$ but $d^\R(R_1, R_2) > \epsilon$. In the second case, if $f(R_1) = f(R_2)$ but $d^\R(R_1, R_2) > \epsilon$ then $h \circ f(R_1) = h \circ f(R_2)$ but $d^\R(R_1, R_2) > \epsilon$. In the third case, if $f = g$ then $h \circ f = h \circ g$. In each case, we thus have that $h \circ f$ is not $\epsilon$-robust to misspecification with $h \circ g$ as measured by $d^\R$.
\end{proof}

\subsection{Key Properties of Reward Transformations}\label{appendix:reward_transformation_properties}

In this section, we will provide a few key properties of the reward transformations we introduced in Section~\ref{sec:reward_transformations}. These properties will help to provide some intuition for how these transformations work, and will also be used to prove some of our later results.

First, it is worth noting that $\CS \subseteq \PS$; to see this, note that we for any constant $c$ and any discount factor $\gamma$ can define a potential function $\Phi$ such that $\Phi(s) = c/(\gamma-1)$ for all states $s$. Moreover, each of $\PS$, $\SR$, $\LS$ and $\CS$ are subsets of $\OP$. First note that $\OP$ is exactly the set of all reward transformations that preserve optimal policies  (c.f.\ Theorem~\ref{thm:OPT_ambiguity}). Next, it should be clear that positive linear scaling of the reward preserves the set of optimal policies, which means that $\LS \subseteq \OP$. Moreover, using the linearity of expectation, it is also easy to see that $S'$-redistribution preserves optimal policies, which means that $\SR \subseteq \OP$. Finally, \cite{ng1999} show that potential shaping preserves optimal policies, which means that $\CS \subseteq \PS \subseteq \OP$ (c.f.\ also Proposition~\ref{prop:change_from_potentials}).

%Specifically, for any reward function $R$, the set of all reward transformations that differ from $R$ by potential shaping forms an $|\States|$-dimensional affine subspace of $\R$.

We will next prove a number of important properties of potential shaping, which are not explicitly discussed in \cite{ng1999}. These properties will be important for our later results, and will also help with providing more intuition for what potential shaping does and how it behaves.

\begin{proposition}\label{prop:change_from_potentials}
Let $R_1$ and $R_2$ be any two reward functions.
If $R_2$ is produced by potential shaping of $R_1$ with a potential function $\Phi$, then
\begin{enumerate}
%\item $G_2(\zeta) = G_1(\zeta) + \gamma^{|\zeta|} \cdot \Phi(s_n) - \Phi(s_0)$,
\item $G_2(\xi) = G_1(\xi) - \Phi(s_0)$,
\item $Q^\pi_2(s, a) = Q^\pi_1(s, a) - \Phi(s)$,
\item $V^\pi_2(s) = V^\pi_1(s) - \Phi(s)$,
\item $A^\pi_2(s, a) = A^\pi_1(s, a)$, and
\item $\Evaluation_2(\policy) = \Evaluation_1(\policy) - \Expect{S_0 \sim \init}{\Phi(S_0)}$
\end{enumerate}
for all trajectories $\xi$, policies $\pi$, states $s$, actions $a$, transition functions $\tfunc$, and initial state distributions $\init$. In (1), $s_0$ is the first state of $\xi$.
\end{proposition}
\ifshowproofs
\begin{proof}
To prove (1), first consider a finite trajectory fragment $\zeta$ with $n$ transitions. It is then easy to prove via induction on $n$ that $G_2(\zeta) = G_1(\zeta) + \gamma^n \cdot \Phi(s_n) - \Phi(s_0)$, where $s_0$ is the first state of $\zeta$, and $s_n$ is the last state. Moreover, $\Phi$ is bounded (since $\States$ is finite) and $\gamma \in (0,1)$. This means that $\gamma^n \cdot \Phi(s_n)$ goes to $0$ as $n$ goes to infinity.
(2) and (3) follow immediately from (1). For (4), note that $A^\pi(s,a) = Q^\pi(s,a) - V^\pi(s)$. This, together with (2) and (3), give us (4). (5) is immediate from (3).
\end{proof}
\fi

This means that we can think of $\Phi$ as assigning \enquote{credit} to each state $s$, such that the total reward of any policy or trajectory which starts in that state $s$ will lose a total of $\Phi(s)$ reward. Note that this directly implies that potential shaping preserves optimal policies (and the ordering of policies), since the value of every policy is shifted by the same amount. Moreover, this property also extends to the soft $Q$-function:

\begin{proposition}\label{prop:potentials_to_soft_Q}
Let $R_1$ and $R_2$ be two reward functions, where $R_2$ is given by potential shaping of $R_1$ with $\Phi$. Let $\QSoftN{1}$ and $\QSoftN{2}$ be their soft $Q$-functions (for some $\tfunc$, $\gamma$, and $\alpha$).
Then $\QSoftN{2}(s, a) = \QSoftN{1}(s, a) - \Phi(s)$.
\end{proposition}
\ifshowproofs
\begin{proof}
%We will appeal to to uniqueness of the soft $Q$-function.
%By definition, $\reward_2(s, a, s') = \reward_1(s, a, s') + \gamma\cdot\Phi(s') - \Phi(s)$.
Recall that $\QSoftN{1}$ is the unique function which satisfies
$$
  \QSoftN{1}(s,a) = \Expect{S' \sim \tfunc(s,a)}{ R_1(s,a,S') + \gamma \alpha \log
        \sum_{a' \in \Actions} \exp\left(\left(\frac{1}{\alpha}\right)\QSoftN{1}(S', a')\right)}
$$
for all $s,a$. Since $R_2(s, a, s') = R_1(s, a, s') + \gamma\cdot\Phi(s') - \Phi(s)$, we can rewrite the right-hand side of this equation as
\begin{align*}
    &\Expect{}{
            R_1(s,a,S')
            + \gamma \alpha \log
            \sum_{a' \in \Actions} \exp\left(\left(\frac{1}{\alpha}\right)\QSoftN{1}(S', a')\right)
        }\\
    = &\Expect{}{
            R_2(s, a, S') - \gamma\cdot\Phi(S') + \Phi(s)
            + \gamma \alpha \log
            \sum_{a' \in \Actions} \exp\left(\left(\frac{1}{\alpha}\right)\QSoftN{1}(S', a')\right)
        }
    \\
    = &\Expect{}{
            R_2(s, a, S') + \Phi(s)
            + \gamma \alpha \log
            \sum_{a' \in \Actions} \exp\left(\left(\frac{1}{\alpha}\right)\QSoftN{1}(S', a') - \Phi(S')\right)
        }
    \\
    % % skip a few steps to take the \Phi(S') inside the logsumexp, which
    % % works because it's independent of the sum variable a
    % &= \Expect
    %     {S' \sim \tfunc(s, a)}{
    %         \reward_2(s, a, S')
    %         + \gamma \frac1\beta
    %             \log \sum_{a'\in\Actions} \exp\beta
    %                 \left(\QSoftN{1}(S', a') - \Phi(S')\right)
    %     }
    %     + \Phi(s)
    % \\
\end{align*}
By now moving $\Phi(s)$ to the left-hand side of the equation, we get that $\QSoftN{1}(s,a) - \Phi(s)$ is equal to
$$
\Expect{}{R_2(s, a, S') + \gamma \alpha \log
            \sum_{a' \in \Actions} \exp\left(\left(\frac{1}{\alpha}\right)\QSoftN{1}(S', a') - \Phi(S')\right)}.
$$
This means that $\QSoftN{1}(s,a) - \Phi(s)$ satisfies the soft $Q$-function recursion (Equation~\ref{equation:soft_Q_recursion}) for $R_2$. Since the soft $Q$-function is the unique solution to this equation, we conclude that $\QSoftN{2}(s, a) = \QSoftN{1}(s, a) - \Phi(s)$.   
\end{proof}
\fi

We next show that potential shaping of the reward function $R$ correspond \emph{exactly} to constant shift of the return function, $G$:
\begin{proposition}\label{prop:potentials_and_episodes}
Let $R_1$ be any reward function and $k$ any constant. Then we have that $G_2(\xi) = G_1(\xi) + k$ for all trajectories that start in a state $s$ if and only if $R_2$ is given by potential shaping of $R_1$ with a potential function $\Phi$ such that $\Phi(s) = -k$.
\end{proposition}
\ifshowproofs
\begin{proof}
The first direction follows from part (1) of Proposition~\ref{prop:change_from_potentials}. 
For the other direction, suppose $G_2(\xi) = G_1(\xi) + k$ for all trajectories that start in state $s$. We will show that this implies a constant difference between $G_1$ and $G_2$ for all trajectories starting in \emph{any} state, and then use this difference to define a potential function that transforms $R_1$ into $R_2$.

Consider an arbitrary state $s' \in \States$.
Given a trajectory $\xi$ starting in $s'$, let $\Delta_{\xi} = G_2(\xi) - G_1(\xi)$.
We will show that for any two trajectories $\xi_1$, $\xi_2$ starting in $s'$, we have that $\Delta_{\xi_1} = \Delta_{\xi_2}$.
Let $\zeta$ be a finite trajectory fragment that starts in $s$ and ends in $s'$, and let $n = |\zeta|$.
%Note that if $s' = s$, then $\zeta$ is empty, and $n = 0$.
Let $\zeta + \xi$ denote the concatenation of $\zeta$ and $\xi$. 
Then,
\begin{align*}
\Delta_{\xi}
    &= G_2(\xi) - G_1(\xi)                                \\
    &= \frac{G_2(\zeta + \xi) - G_2(\zeta)}{\gamma^n}
       - \frac{G_1(\zeta + \xi) - G_1(\zeta)}{\gamma^n} \\
    &= \frac{k - G_2(\zeta) + G_1(\zeta)}{\discount^n}\,.
\end{align*}
The first line follows from the fact that $G(\zeta + \xi) = G(\zeta) + \gamma^{|\zeta|} G(\xi)$. For the second line, note that $\zeta + \xi$ is a trajectory starting in $s$. Thus, by assumption, we have that $G_2(\zeta + \xi) - G_1(\zeta + \xi) = k$.
Since this expression is independent of $\xi$, this means that $\Delta_{\xi_1} = \Delta_{\xi_2}$ for any two trajectories $\xi_1$, $\xi_2$ starting in $s'$. Since $s'$ was picked arbitrarily, this holds for all states $s'$.

Let $\Phi : \States \to \mathbb{R}$ be the potential function where $\Phi(s')$ is the value of $-\Delta_{\xi}$ for all trajectories $\xi$ which start in $s'$.
In other words, $\Phi(s') = G_1(\xi) - G_2(\xi)$ for all trajectories $\xi$ which start in $s'$.
We will show that $R_2$ is given by potential shaping of $R_1$ with $\Phi$. Let $(s,a,s')$ be any transition, let $\xi'$ be any trajectory starting in $s'$, and let $\xi = (s, a, s') + \xi'$. Then:
\begin{align*}
    &R_1(s,a,s') + \discount \Phi(s') - \Phi(s)\\
    = &R_1(s,a,s') + \discount(G_1(\xi') - G_2(\xi'))
                - (G_1(\xi) - G_2(\xi)) \\
    = &G_2(\xi) - \discount G_2(\xi') + R_1(s,a,s') + \discount G_1(\xi') - G_1(\xi) \\
    = &G_2(\xi) - \discount G_2(\xi') + G_1(\xi) - G_1(\xi) \\
    = &R_2(s,a,s')\,.
\end{align*}
Thus, $R_2$ is given by potential shaping of $R_1$ with $\Phi$. Finally, note that $\Phi(s) = -k$, since $\Phi(s) = G_1(\xi) - G_2(\xi)$ for all trajectories starting in $s$, and since $G_2(\xi) = G_1(\xi) + k$ for all trajectories that start in $s$. This completes the proof.
\end{proof}
\fi

Note that Proposition~\ref{prop:potentials_and_episodes} quantifies over all trajectories in $(\SxA)^\omega$, rather than all trajectories which are possible under some transition function $\tau$. However, it should be clear from Proposition~\ref{prop:potentials_and_episodes} that $G_2(\xi) = G_1(\xi) + k$ for all \emph{possible} trajectories that start in a state $s$ if and only if $R_2$ is given by potential shaping of $R_1$ with a potential function $\Phi$ such that $\Phi(s) = -k$, and an arbitrary change of all transitions that are unreachable from $s$. 
Next, we show that positive linear scaling of $G$ corresponds to a combination of potential shaping and positive linear scaling of $R$:

\begin{proposition}\label{prop:linear_scaling_of_G}
$G_2(\xi) = c \cdot G_1(\xi)$ for all trajectories $\xi$ that start in a state $s$ if and only if $R_2$ is given by potential shaping of $R_1$ with a potential function $\Phi$ such that $\Phi(s) = 0$, and positive linear scaling by a factor of $c$.
\end{proposition}

\ifshowproofs
\begin{proof}
For the first direction, suppose $G_2(\xi) = c \cdot G_1(\xi)$ for all trajectories $\xi$ that start in a state $s$. Let $R_c$ be the reward function given by $c \cdot R_1$. It is clear that $G_c(\xi) = c \cdot G_1(\xi)$. Thus, $G_2(\xi) = c \cdot G_1(\xi)$ for all trajectories $\xi$ that start in a state $s$, if and only if $G_2(\xi) = G_c(\xi)$ for all trajectories $\xi$ that start in a state $s$. As per Proposition~\ref{prop:potentials_and_episodes}, this is equivalent to $\reward_2$ being produced from $R_c$ by potential shaping with a potential function $\Phi$ such that $\Phi(s) = 0$. This means that we can produce $R_2$ from $R_1$ by first applying positive linear scaling by a factor of $c$, and then applying potential shaping.
The other direction can be proven analogously (also c.f.\ Proposition~\ref{prop:reward_transformation_commutativity}).
\end{proof}
\fi

Together, Proposition~\ref{prop:potentials_and_episodes} and \ref{prop:linear_scaling_of_G} imply that potential shaping and positive linear scaling of $R$ correspond exactly to affine transformations of $G$. This may help with providing some intuition for what potential shaping does, and how it behaves.
Next, it is worth noting that $\PS$ and $\SR$ correspond to linear subspaces of $\R$. Specifically:
\begin{proposition}\label{prop:PS_SR_linear}
Let $S$ be the set of all reward functions which can be expressed as
$$
R(s,a,s') = \gamma \cdot \Phi(s) - \Phi(s')
$$
for some potential function $\Phi$, and let $Z$ be the set of all reward functions such that
$$
\Expect{S' \sim \tfunc(s,a)}{R(s,a,S')} = 0.
$$
Then $S$ and $Z$ are linear subspaces of $\R$, where $S$ is $|\States|$-dimensional and $Z$ is $|\States||\Actions|(|\States|-1)$-dimensional, and where $S \cap Z = R_0$.

Moreover, $R_1$ and $R_2$ differ by potential shaping if and only if $R_2 = R_1 + R'$ for some $R' \in S$, and $R_1$ and $R_2$ differ by $S'$-redistribution if and only if $R_2 = R_1 + R'$ for some $R' \in Z$.
\end{proposition}
\ifshowproofs
\begin{proof}
To show that $S$ and $Z$ are linear subspaces of $\R$, we must show that $R_0 \in S$, that if $R_1, R_2 \in S$ then $R_1 + R_2 \in S$, and that if $R \in S$ then $c \cdot R \in S$ for all scalars $c$, and likewise for $Z$. Each of these properties are straightforward in both cases. 

To see that $S$ is $|\States|$-dimensional, for each state $s$, let $R_s$ be the reward function in $S$ which corresponds to the potential function $\Phi$ such that $\Phi(s) = 1$, and $\Phi(s') = 0$ for $s' \neq s$. Now the vectors $\{R_s : s \in \States\}$ form a basis for $S$. They are also linearly independent. To see this, recall Proposition~\ref{prop:change_from_potentials}. In particular, we have that $V^\pi(s) = -1$ for the reward function $R_s$, where $\pi$ is any policy. However, for any reward function that can be expressed as a linear combination of reward functions in $\{R_s : s \in \States\} \setminus \{R_s\}$, we have that $V^\pi(s) = 0$. This means that $\{R_s : s \in \States\}$ is a minimal basis. Since $|\{R_s : s \in \States\}| = |\States|$, this means that $S$ is $|\States|$-dimensional.

It is straightforward that $Z$ is $|\States||\Actions|(|\States|-1)$-dimensional. To see that $S \cap Z = R_0$, note that if $R \in Z$, then $V^\pi(s) = 0$ for every policy $\pi$ and every state $s$. Then Proposition~\ref{prop:change_from_potentials} implies that $\Phi(s) = 0$ for all $s$, and so $R = R_0$.

It can be shown from straightforward algebra that $R_1$ and $R_2$ differ by potential shaping if and only if $R_2 = R_1 + R'$ for some $R' \in S$, and likweise that $R_1$ and $R_2$ differ by $S'$-redistribution if and only if $R_2 = R_1 + R'$ for some $R' \in Z$. This completes the proof.
\end{proof}
\fi

Note that Proposition~\ref{prop:PS_SR_linear} implies that for any reward function $R$, the set of all reward functions that differ from $R$ by potential shaping forms a $|\States|$-dimensional affine subspace of $\R$, and similarly for $S'$-redistribution. Moreover, since $S \cap Z = R_0$, we have that the set of all reward functions that differ from $R$ by potential shaping and $S'$-redistribution forms a $(|\States||\Actions|(|\States|-1) +|\States|)$-dimensional affine space.

We next note a few basic algebraic properties of our transformations:

\begin{proposition}\label{prop:reward_transformation_group_operations}
If $T$ is $\PS$, $\SR$, $\LS$, $\CS$, or $\OP$, then
\begin{enumerate}
    \item The identity transformation, $\mathrm{id}$, is in $T$.
    \item For all $t \in T$ there is a $t^- \in T$ such that $t \circ t^- = \mathrm{id}$.
    \item For all $t, t' \in T$, we have that $t \circ t' \in T$.
\end{enumerate}
\end{proposition}
\ifshowproofs
\begin{proof}
For (1), first note that $\mathrm{id}$ satisfies the conditions for potential shaping with the function $\Phi$ such that $\Phi(s) = 0$ for all $s$; hence $\mathrm{id} \in \PS$. Next, since trivially $\Expect{S' \sim \tfunc(s,a)}{R(s,a,S')} = \Expect{S' \sim \tfunc(s,a)}{R(s,a,S')}$, we have that $\mathrm{id} \in \SR$. Moreover, $\mathrm{id}$ satisfies the conditions for positive linear scaling with a factor $c = 1$, and so $\mathrm{id} \in \LS$. Furthermore, $\mathrm{id}$ satisfies the conditions for constant shift with a factor $c = 0$, and so $\mathrm{id} \in \CS$. Finally, $\mathrm{id}$ satisfies the conditions for optimality-preserving transformations, where $\Psi = V^\star$ (this is precisely the Bellman optimality equation for $V^\star$, see Equation~\ref{equation:optimal_V_recursion}).

For (2), first note that if $R_1$ and $R_2$ differ by potential shaping with $\Phi$, then $R_2$ and $R_1$ differ by potential shaping with $-\Phi$. Furthermore, we trivially have that if $\Expect{S' \sim \tfunc(s,a)}{R_1(s,a,S')} = \Expect{S' \sim \tfunc(s,a)}{R_2(s,a,S')}$ then $\Expect{S' \sim \tfunc(s,a)}{R_2(s,a,S')} = \Expect{S' \sim \tfunc(s,a)}{R_1(s,a,S')}$. Moreover, if $R_1$ and $R_2$ differ by positive linear scaling by $c$, then $R_2$ and $R_1$ differ by positive linear scaling by $(1/c)$. Similarly, if $R_1$ and $R_2$ differ by constant shift with $c$, then $R_2$ and $R_1$ differ by constant shift with $-c$. Finally, if $R_1$ and $R_2$ differ by an optimality-preserving transformation, then $A^\star_1 = A^\star_2$, and so $R_2$ and $R_1$ also differ by an optimality-preserving transformation.

For (3), note that if $R_2$ is given by potential shaping of $R_1$ with $\Phi_1$, and $R_3$ is given by potential shaping of $R_2$ with $\Phi_2$, then $R_3$ is given by potential shaping of $R_1$ with $\Phi_1 + \Phi_2$. Moreover, we trivially have that 
$$
\Expect{S' \sim \tfunc(s,a)}{R_1(s,a,S')} = \Expect{S' \sim \tfunc(s,a)}{R_3(s,a,S')}
$$
if 
$$
\Expect{S' \sim \tfunc(s,a)}{R_1(s,a,S')} = \Expect{S' \sim \tfunc(s,a)}{R_2(s,a,S')}
$$
and 
$$
\Expect{S' \sim \tfunc(s,a)}{R_2(s,a,S')} = \Expect{S' \sim \tfunc(s,a)}{R_3(s,a,S')}.
$$
Next, if $R_2$ is given by positive linear scaling of $R_1$ with $c_1$, and $R_3$ is given by positive linear scaling of $R_2$ with $c_2$, then $R_3$ is given by positive linear scaling of $R_1$ with $c_1 \cdot c_2$. Similarly, if $R_2$ is given by constant shift of $R_1$ with $c_1$, and $R_3$ is given by constant shift of $R_2$ with $c_2$, then $R_3$ is given by constant shift of $R_1$ with $c_1 + c_2$. Finally, suppose that $R_1$ and $R_2$ differ by an optimality-preserving transformation, and that $R_2$ and $R_3$ differ by an optimality-preserving transformation. We then have that $A^\star_1 = A^\star_2$, which means that $R_1$ and $R_3$ differ by an optimality-preserving transformation.
\end{proof}
\fi

%First of all, the identity transformation, $\mathrm{id}$, is contained in each of $\PS$, $\SR$, $\LS$, $\CS$, and $\OP$. To see this, note that the potential function $\Phi$ used in $\PS$ may be 0, that the identity transformation satisfies the conditions for $S'$-redistribution, that the positive linear scaling factor used in $\LS$ may be 1, that the shift induced by a transformation in $\CS$ may be 0, and that the identity transformation is an optimality preserving transformation, where $\Psi = V^\star_1$. Moreover, each of these sets contains inverse transformations. For example, if $R_2$ is given by potential shaping of $R_1$ with $\Phi$, then $R_1$ is given by potential shaping of $R_2$ with $-\Phi$. Similarly, if $R_2$ is given by positive linear scaling of $R_1$ by $c$, then $R_1$ is given by positive linear scaling of $R_2$ with $(1/c)$, and so on. Moreover, if $T$ is one of these sets of transformations, and $t_1, t_2 \in T$, then $t_1 \circ t_2 \in T$. For example, if $R_2$ is given by potential shaping of $R_1$ with $\Phi_1$, and $R_3$ is given by potential shaping of $R_2$ with $\Phi_2$, then $R_3$ is given by potential shaping of $R_1$ with $\Phi_1 + \Phi_2$.

Proposition~\ref{prop:reward_transformation_group_operations} implies that each of the sets $\PS$, $\SR$, $\LS$, $\CS$, and $\OP$ form groups. It also implies that each of these sets partitions $\R$ into equivalence classes. Note that these properties do not hold for arbitrary sets of reward transformations, so they are special properties of $\PS$, $\SR$, $\LS$, $\CS$, and $\OP$. The following is also worth noting:

\begin{proposition}\label{prop:reward_transformation_group_operations_composition}
Let $T_1$ and $T_2$ be sets of reward transformations such that if $T$ is $T_1$ or $T_2$, then
\begin{enumerate}
    \item The identity transformation, $\mathrm{id}$, is in $T$.
    \item For all $t \in T$ there is a $t^- \in T$ such that $t \circ t^- = \mathrm{id}$.
\end{enumerate}
We then have that
\begin{enumerate}
    \item The identity transformation, $\mathrm{id}$, is in $T_1 \bigodot T_2$.
    \item For all $t \in T_1 \bigodot T_2$ there is a $t^- \in T_1 \bigodot T_2$ such that $t \circ t^- = \mathrm{id}$.
    \item For all $t, t' \in T_1 \bigodot T_2$, we have that $t \circ t' \in T_1 \bigodot T_2$.
\end{enumerate}
\end{proposition}
\ifshowproofs
\begin{proof}
For (1), note that $T_1, T_2 \subset T_1 \bigodot T_2$. For (2), note that if $t \in T_1 \bigodot T_2$, then $t = t_1 \circ \dots \circ t_n$, where each transformation $t_i$ is in either $T_1$ or $T_2$. Let $t^- = t_n^- \circ \dots \circ t_1^-$. Now $t \circ t^- = \mathrm{id}$, and $t^- \in T_1 \bigodot T_2$. It is immediate from the definition of the $\bigodot$-operator that (3) is satisfied.
\end{proof}
\fi

This means that the properties described in Proposition~\ref{prop:reward_transformation_group_operations} also hold for any set of reward transformations which can be constructed from $\PS$, $\SR$, $\LS$, $\CS$, and $\OP$ using the $\bigodot$-operator. The following is also useful:

\begin{proposition}\label{prop:reward_transformation_commutativity}
If both $T_1$ and $T_2$ are $\PS$, $\SR$, $\LS$, $\CS$, or $\OP$, then for each $t_1 \in T_1$ and $t_2 \in T_2$, there is a $t_1' \in T_1$ and a $t_2' \in T_2$ such that $t_1 \circ t_2 = t_2' \circ t_1'$.
\end{proposition}
\ifshowproofs
\begin{proof}
If $T_1 = T_2$, the proposition is trivial. Next, recall that each of $\PS$, $\SR$, $\LS$, and $\CS$ is a subset of $\OP$. This means that if one of $T_1$ or $T_2$ is $\OP$, then we can set $t_2' = t_1 \circ t_2$ and $t_1' = \mathrm{id}$, or vice versa (recalling also the properties listed in Proposition~\ref{prop:reward_transformation_group_operations}).

For the remaining cases, let $S$ be the set of all reward functions that can be expressed as $R(s,a,s') = \gamma \cdot \Phi(s') - \Phi(s)$ for some potential function $\Phi$, and let $Z$ be the set of all reward functions that satisfy $\Expect{S' \sim \tfunc(s,a)}{R(s,a,S')} = 0$. Note that $R_1$ and $R_2$ differ by potential shaping if and only if $R_1 = R_2 + R_S$ for some $R_S \in S$, and that $R_1$ and $R_2$ differ by $S'$-redistribution if and only if $R_1 = R_2 + R_Z$ for some $R_Z \in Z$.

Now, let $T_1$ be $\PS$ and $T_2$ be $\SR$. We now have that for all $R$, there is an $R_S \in S$ and an $R_Z \in Z$ such that $t_1 \circ t_2(R) = R + R_S + R_Z$. Since vector addition is commutative, this means that we can find the desired $t_1'$ and $t_2'$. The case where $T_1$ is $\SR$ and $T_2$ is $\PS$ is analogous.

Next, let $T_1 = \PS$ and $T_2 = \LS$. We now have that there for all $R$ is an $R_S \in S$ and a $c \in \mathbb{R}^+$ such that $t_1 \circ t_2(R) = c \cdot R + R_S$. This also means that $t_1 \circ t_2(R) = c \cdot (R + \frac{1}{c} R_S)$. Since $\frac{1}{c} R_S \in S$, this means that we can find the desired $t_1'$ and $t_2'$. The case where $T_1 = \LS$ and $T_2 = \PS$ is analogous, and likewise for the case where $T_1$ and $T_2$ are $\SR$ and $\LS$.

The case where $T_1$ or $T_2$ is $\CS$ is covered by the above cases, since $\CS \subseteq \PS$. This completes the proof.
\end{proof}
\fi

Proposition~\ref{prop:reward_transformation_commutativity} means that we do not have to be very careful about the order in which transformations from $\PS$, $\SR$, $\LS$, $\CS$, or $\OP$ are applied. For example, if we can produce $R_1$ from $R_2$ by first applying potential shaping, and then applying $S'$-redistribution, then we can also do this by first applying $S'$-redistribution, and then applying potential shaping, and so on. This fact, combined with the properties listed in Proposition~\ref{prop:reward_transformation_group_operations}, will substantially reduce the number of cases that we have to consider in some proofs. For example, if $t \in \PS \bigodot \LS \bigodot \SR$, then it can always be expressed as $t_1 \circ t_2 \circ t_3$, where $t_1 \in \PS$, $t_2 \in \LS$, and $t_3 \in \SR$, etc.

\subsection{Comparing Reward Functions}

In this section we provide the proofs of all results concerning the equivalence relations $\OPT$ and $\ORD$ on $\R$, as well as our results concerning STARC metrics. In the main text, these results are in Section~\ref{sec:comparing_reward_functions}.

\subsubsection{Equivalent Reward Functions}

Before we can provide our results about equivalent reward functions, we must first derive a few lemmas about the topological structure of MDPs. Recall that the \emph{occupancy measure} $\eta^\pi$ of a policy $\pi$ is the $(|\States||\Actions||\States|)$-dimensional vector in which the value of the $(s,a,s')$'th dimension is given by
$$
\sum_{t=0}^\infty \gamma^t \mathbb{P}_{\xi \sim \pi}\left(S_t, A_t, S_{t+1} = s,a,s' \right). 
$$
Let $m_{\tfunc,\init,\gamma} : \Pi \rightarrow \mathbb{R}^{|S||A||S|}$ be the map that sends each policy $\pi$ to its occupancy measure, $\eta^\pi$. Recall also that $J(\pi) = \eta^\pi \cdot R$. This means that we can use $m_{\tfunc,\init,\gamma}$ to decompose $\Evaluation$ into two separate steps, the first of which is independent of the reward function, and the second of which is linear. We will first show that $m_{\tfunc,\init,\gamma}$ is a continuous function. Throughout this section, we will assume that $\Pi$ is equipped with the topological structure that is induced by the $L_2$-norm, when each policy is represented as an $(|\States||\Actions|)$-dimensional vector (where the value of the $(s,a)$'th dimension is $\pi(a \mid s)$).

\begin{lemma}\label{lemma:m_continuous}
$m_{\tfunc,\init,\gamma} : \Pi \rightarrow \mathbb{R}^{|S||A||S|}$ is a continuous function.
\end{lemma}
\ifshowproofs
\begin{proof}
Recall that a uniformly convergent series of continuous functions is continuous. Specifically, if $X$ is a topological space and $Y$ is a metric space, and $\{f_n : X \to Y\}_{n = 1}^\infty$ is a sequence of functions that converge uniformly to a function $f : X \to Y$, and each function $f_i \in \{f_n\}_{n = 1}^\infty$ is continuous, then $f$ is continuous. Moreover, $\{f_n\}_{n = 1}^\infty$ converges uniformly to $f$ if there for each $\epsilon$ exists an $i$ such that for all $j \geq i$, $|f(x) - f_j(x)| < \epsilon$ for all $x \in X$. We will show that $m_{\tfunc,\init,\gamma}$ can be expressed in this way.

Let $f_n : \Pi \rightarrow \mathbb{R}^{|S||A||S|}$ be the function that maps each policy $\pi$ to its occupancy measure, when only the first $n$ time steps are considered. That is, $f_n(\pi)$ is the vector in which the value of the $(s,a,s')$'th dimension is
$$
\sum_{t=0}^n \gamma^t \mathbb{P}_{\xi \sim \pi}\left(S_t, A_t, S_{t+1} = s,a,s' \right). 
$$
Note that $\mathbb{P}_{\xi \sim \pi}\left(S_t, A_t, S_{t+1} = s,a,s' \right) \in [0,1]$, and that $\gamma \in (0,1)$. This means that $m_{\tfunc,\init,\gamma}(\pi) \geq f_n(\pi)$, and that 
$$
m_{\tfunc,\init,\gamma}(\pi) - f_n(\pi) = \sum_{t=n+1}^\infty \gamma^t \mathbb{P}_{\xi \sim \pi}\left(S_t, A_t, S_{t+1} = s,a,s' \right) \leq \left(\frac{\gamma^{n+1}}{1-\gamma}\right).
$$
As $n$ goes to $\infty$, we have that $\left(\frac{\gamma^{n+1}}{1-\gamma}\right)$ goes to $0$. Thus, for any $n$ that is sufficiently large, we have that $|m_{\tfunc,\init,\gamma}(\pi) - f_n(\pi)| < \epsilon$. This means that $\{f_n\}_{n = 1}^\infty$ converges uniformly to $m_{\tfunc,\init,\gamma}$. Moreover, each function $f_i \in \{f_n\}_{n = 1}^\infty$ is continuous, since it can be expressed as a finite sum of terms in which each term is given by a finite number of matrix multiplications.
\end{proof}
\fi

Next, let $\bar{\Pi} \subset \Pi$ be the set of all policies that visit each state with positive probability. We then have that:

\begin{lemma}\label{lemma:m_injective}
$m_{\tfunc,\init,\gamma}$ is injective on $\bar{\Pi}$.
\end{lemma}
\ifshowproofs
\begin{proof}
Suppose $m_{\tfunc,\init,\gamma}(\pi) = m_{\tfunc,\init,\gamma}(\pi')$ for some $\pi,\pi' \in \bar{\Pi}$. Next, given $\tfunc,\init$, define $w_\pi$ as
$$
w_\pi(s) = \sum_{t=0}^\infty \gamma^t \mathbb{P}_{\xi \sim \pi}(S_t = s).
$$
Note that if $m_{\tfunc,\init,\gamma}(\pi) = m_{\tfunc,\init,\gamma}(\pi')$ then $w_\pi = w_{\pi'}$, and moreover that
$$
\sum_{s' \in \States} m_{\tfunc,\init,\gamma}(\pi)[s,a,s'] = w_\pi(s)\pi(a \mid s).
$$
This means that if $w_{\pi}(s) \neq 0$ for all $s$, which is the case for all $\pi \in \bar{\Pi}$, then we can express $\pi$ as
$$
\pi(a \mid s) = \frac{\sum_{s' \in \States} m_{\tfunc,\init,\gamma}(\pi)[s,a,s']}{w_\pi(s)}.
$$
This means that if $m_{\tfunc,\init,\gamma}(\pi) = m_{\tfunc,\init,\gamma}(\pi')$ for some $\pi,\pi' \in \bar{\Pi}$ then $\pi = \pi'$.
\end{proof}
\fi

Note that $m_{\tfunc,\init,\gamma}$ is \emph{not} injective on $\Pi$; if there is some state $s$ that $\pi$ reaches with probability $0$, then we can alter the behaviour of $\pi$ at $s$ without changing $m_{\tfunc,\init,\gamma}(\pi)$. 
Note also that Lemma~\ref{lemma:m_injective} holds for all $\tfunc$ and $\init$, assuming that all states are reachable (which we assume throughout the paper). We will also need the following lemma:

%But every policy in $\tilde{\Pi}$ visits every state with positive probability, which then makes $m$ injective. 
%In fact, Proposition~\ref{lemma:injectivity} straightforwardly generalises to the set of all policies that visit all states with positive probability (which in general can be bigger than $\tilde{\Pi}$). However, we will only use $\tilde{\Pi}$ in these proofs.

\begin{lemma}\label{lemma:image_dimension}
$\mathrm{Im}(m_{\tfunc,\init,\gamma})$ is located in an affine space with no more than $|\States|(|\Actions|-1)$ dimensions.
\end{lemma}
\ifshowproofs
\begin{proof}
We wish to establish an \emph{upper} bound on the number of linearly independent vectors in $\mathrm{Im}(m_{\tfunc,\init,\gamma})$.
We can do this by establishing a \emph{lower} bound on the size of the space of all reward functions that share the same policy evaluation function, $\Evaluation$.
To see this, consider the fact that $\Evaluation(\pi) = m_{\tfunc,\init,\gamma}(\pi)\cdot R$, and note that $R$ is an $|\States||\Actions||\States|$-dimensional vector.
Let $R_1$ be a reward function, and let $X$ be the space of all reward functions $R_2$ such that $R_1 \cdot \eta = R_2 \cdot \eta$ for all $\eta \in \mathrm{Im}(m_{\tfunc,\init,\gamma})$.
It is then a straightforward consequence of linear algebra that if $\mathrm{Im}(m_{\tfunc,\init,\gamma})$ contains $n$ linearly independent vectors, then $X$ forms an affine space with $|\States||\Actions||\States| - n$ dimensions. We can thus obtain an upper bound on the number of linearly independent vectors in $\mathrm{Im}(m_{\tfunc,\init,\gamma})$ from a lower bound on the dimensionality of $X$.

Next, recall that if $R_1$ and $R_2$ differ by potential shaping with $\Phi$, and $\mathbb{E}_{S_0 \sim \init}\left[\Phi(S_0)\right] = 0$, then $\Evaluation_1(\pi) = \Evaluation_2(\pi)$ for all $\pi$ (Proposition~\ref{prop:change_from_potentials}). 
Also recall that if $R_1$ and $R_2$ differ by $S'$-redistribution, then $\Evaluation_1(\pi) = \Evaluation_2(\pi)$ for all $\pi$.
This means that for any $R_1$, we have that $X$ contains all reward functions $R_2$ that differ from $R_1$ by $S'$-redistribution and potential shaping with a potential function $\Phi$ such that $\mathbb{E}_{S_0 \sim \init}\left[\Phi(S_0)\right] = 0$. The space of all such reward vectors is an affine space with $|\States||\Actions|(|\States|-1) + |\States| - 1$ dimensions (Proposition~\ref{prop:PS_SR_linear}). This means that $\mathrm{Im}(m_{\tfunc,\init,\gamma})$ contains at most $|\States|(|\Actions|-1) + 1$ linearly independent vectors.

Next, note that there is no $\pi$ such that $m_{\tfunc,\init,\gamma}(\pi)$ is the zero vector. In fact, $\sum m_{\tfunc,\init,\gamma}(\pi) = 1/(1-\gamma)$ for all $\pi$. This means that the smallest affine space which contains $\mathrm{Im}(m_{\tfunc,\init,\gamma})$ does not contain the origin. Therefore, $\mathrm{Im}(m_{\tfunc,\init,\gamma})$ is located in an affine space with no more than $|\States|(|\Actions|-1)$ dimensions.
\end{proof}
\fi

For the next lemma, let $\Pi^+ \subset \Pi$ be the set of all policies that take all actions with positive probability in each state. Note that $\Pi^+ \subset \bar{\Pi}$ (i.e., a policy that takes every action with positive probability in each state visits every state with positive probability), since we assume that all states are reachable under $\init$ and $\tfunc$. We then have that:

\begin{lemma}\label{lemma:homeomorphism}
$\mathrm{Im}(m_{\tfunc,\init,\gamma})$ is located in an affine space with $|\States|(|\Actions|-1)$ dimensions, in which $m_{\tfunc,\init,\gamma}(\Pi^+)$ is an open set, and $m_{\tfunc,\init,\gamma}$ is a homeomorphism between $\Pi^+$ and $m_{\tfunc,\init,\gamma}(\Pi^+)$.
\end{lemma}
\ifshowproofs
\begin{proof}
By the Invariance of Domain theorem, if %Theorem (Brouwer, 1912), if 
\begin{enumerate}
    \item $U$ is an open subset of $\mathbb{R}^n$, and
    \item $f : U \rightarrow \mathbb{R}^n$ is an injective continuous map,
\end{enumerate}
then $f(U)$ is open in $\mathbb{R}^n$, and $f$ is a homeomorphism between $U$ and $f(U)$. We will show that $m$ and $\Pi^+$ satisfy the requirements of this theorem.

We begin by noting that $\Pi$ can be represented as a set of points in $\mathbb{R}^{|\States|(|\Actions|-1)}$.
%This should be clear, but to spell it out, first note that we can t hink of 
We do this by considering 
each policy $\pi$ as a vector $\vec{\pi}$ of length $|\States||\Actions|$, where $\vec{\pi}[s,a] = \pi(a \mid s)$. 
% Moreover, since $\sum_{a \in A} \pi(a \mid s) = 1$ for all $s$, we do not need $\vec{\pi}$ to contain the value of $\pi(a \mid s)$ for all $a$; we can take an arbitrary action $a$, and let $\vec{\pi}$ be a vector of length $|S|(|A|-1)$ where $\vec{\pi}[s,a'] = \pi(a' \mid s)$ for $a' \neq a$. 
Moreover, since $\sum_{a \in A} \pi(a \mid s) = 1$ for all $s$, 
we can remove $|\Actions|$ dimensions, and embed $\Pi$ in $\mathbb{R}^{|\States|(|\Actions|-1)}$.

$\Pi^+$ is an open set in $\mathbb{R}^{|\States|(|\Actions|-1)}$. By Lemma~\ref{lemma:image_dimension}, we have that $m_{\tfunc,\init,\gamma}$ is a mapping 
from $\Pi^+$ to an affine space with no more than $|\States|(|\Actions|-1)$ dimensions.
%$m_{\tfunc,\init,\gamma} : \tilde{\Pi} \rightarrow \mathbb{R}^{|S|(|A|-1)}$. 
By Lemma~\ref{lemma:m_injective}, we have that $m_{\tfunc,\init,\gamma}$ is injective on $\Pi^+$. Finally, by Lemma~\ref{lemma:m_continuous}, we have that $m_{\tfunc,\init,\gamma}$ is continuous. We can therefore apply the Invariance of Domain theorem, and conclude that $m_{\tfunc,\init,\gamma}(\Pi^+)$ is open in this $|\States|(|\Actions|-1)$-dimensional affine space, and that $m_{\tfunc,\init,\gamma}$ is a homeomorphism between $\Pi^+$ and $m_{\tfunc,\init,\gamma}(\Pi^+)$.
\end{proof}
\fi

Note that lemma~\ref{lemma:homeomorphism} holds for all $\tfunc$ and $\init$ (for which all states are reachable). Using these results, we can now state necessary and sufficient conditions that describe when $\Evaluation_1(\pi) = \Evaluation_2(\pi)$ for all policies $\pi$:

\begin{lemma}\label{lemma:ambiguity_of_J}
$\Evaluation_1 = \Evaluation_2$ if and only if $R_1$ and $R_2$ differ by $S'$-redistribution and potential shaping with a potential $\Phi$ such that $\Expect{S_0 \sim \init}{\Phi(S_0)} = 0$.
\end{lemma}
\ifshowproofs
\begin{proof}
For the first direction, suppose $R_1$ and $R_2$ differ by $S'$-redistribution and potential shaping with a potential $\Phi$ such that $\Expect{S_0 \sim \init}{\Phi(S_0)} = 0$. Then $V_2^\pi(s) = V_1^\pi(s) - \Phi(s)$, as per Proposition~\ref{prop:change_from_potentials}. Hence $\Evaluation_1(\pi) = \Evaluation_2(\pi) - \mathbb{E}_{s_0 \sim \init}[\Phi(s_0)] = \Evaluation_2(\pi)$, and so we have proven the first direction.

For the other direction, first recall that $\Evaluation(\pi) = m_{\tfunc,\init,\gamma}(\pi) \cdot R$.
Next, Lemma~\ref{lemma:homeomorphism} implies that $\mathrm{Im}(m_{\tfunc,\init,\gamma})$ contains $|S|(|A|-1) + 1$ linearly independent vectors. 
It is then a straightforward fact of linear algebra that, for any reward function $R_1$, the space $X$ of all reward functions $R_2$ such that $R_1 \cdot \eta = R_2 \cdot \eta$ for all $\eta \in \mathrm{Im}(m_{\tfunc,\init,\gamma})$, forms an affine space with $|\States||\Actions||\States| - (|S|(|A|-1) + 1) = |\States||\Actions|(|\States|-1) + |\States|-1$ dimensions. 

We have already shown that $\Evaluation_1(\pi) = \Evaluation_2(\pi)$ for all policies $\pi$ if $R_1$ and $R_2$ differ by $S'$-redistribution and potential shaping with a potential $\Phi$ such that $\Expect{S_0 \sim \init}{\Phi(S_0)} = 0$. Next, given $R_1$, the space of all reward functions $R_2$ such that $R_1$ and $R_2$ differ by $S'$-redistribution and potential shaping with a potential $\Phi$ such that $\Expect{S_0 \sim \init}{\Phi(S_0)} = 0$, forms an affine space with $|\States||\Actions|(|\States|-1) + |\States|-1$ dimensions. Since this space is contained in $X$, and since they have the same number of dimensions, they must be one and the same. Therefore, if $\Evaluation_1 = \Evaluation_2$, then it must be the case that $R_1$ and $R_2$ differ by $S'$-redistribution and potential shaping with a potential $\Phi$ such that $\Expect{S_0 \sim \init}{\Phi(S_0)} = 0$. We have thus proven the other direction, which completes the proof.
\end{proof}
\fi

%Since any transformation which does not change the reward vector is in $\SR$, we therefore have that 
%and each such transformation is specified by $|\States|-1$ variables. 
%We can identify these variables with the dimensions of the space of all reward vectors $\Vec{R}_{2,\tau}$ for which $\Evaluation_1 = \Evaluation_2$.
%(corresponding to the value of $\Phi$ in each state, which is determined for one of the initial states). 
%This means that $\mathrm{dim}(\mathrm{Im}(m))^\bot \geq |\States|-1$. Moreover, Lemma~\ref{lemma:VDS_dimensionality} implies that $\mathrm{dim}(\mathrm{Im}(m))^\bot \leq |\States|-1$, and so $\mathrm{dim}(\mathrm{Im}(m))^\bot = |\States|-1$. 
%This means that $\mathrm{Im}(m_{\tfunc,\init,\gamma}) \cdot \Vec{R}_\tfunc$ determines $\Vec{R}_\tfunc$ modulo exactly $|\States|-1$ degrees of freedom, which we can identify with the values of $\Phi$ for the transformations in $\kPS{0}$.
%Hence $\mathrm{Im}(m_{\tfunc,\init,\gamma}) \cdot \Vec{R}_\tfunc$ (and therefore $\Evaluation$) are preserved by transformations in $\kPS{0}$ and transformations that preserve $\Vec{R}_\tfunc$, and no other transformations.
%$\Vec{R}_\tfunc$ is of course preserved by $S'$-redistribution, and no other transformations. We have hence proven the other direction.

\policyordering*

\begin{proof}
The first direction is straightforward. First, if $R_1 = t(R_2)$ for some $t \in \SR$ then $\Evaluation_1 = \Evaluation_2$. Next, if $R_1 = t(R_2)$ for some $t \in \PS$ then $\Evaluation_1 = \Evaluation_2 - \mathbb{E}_{S_0 \sim \init}[\Phi_t(S_0)]$ (Proposition~\ref{prop:change_from_potentials}). Finally, if $R_1 = t(R_2)$ for some $t \in \LS$ then $\Evaluation_1 = c \cdot \Evaluation_2$ for some $c \in \mathbb{R}^+$. Hence if $R_1 = t(R_2)$ for some $t \in \SR \bigodot \PS \bigodot \LS$, then $\Evaluation_1 = a \cdot \Evaluation_2 + b$ for some $a \in \mathbb{R}^+, b \in \mathbb{R}$. This means that $\Evaluation_1$ and $\Evaluation_2$ differ by a strictly monotonic transformation, and so $R_1$ and $R_2$ have the same ordering of policies.

For the other direction, first note that $R_1$ and $R_2$ have the same ordering of policies only if $\Evaluation_1$ is a monotonic transformation of $\Evaluation_2$. Moreover, since $\Evaluation(\pi) = m_{\tfunc,\init,\gamma}(\pi) \cdot R$, we have that all possible monotonic transformations of $\Evaluation$ are affine. Hence $R_1$ and $R_2$ have the same ordering of policies only if $\Evaluation_1 = a \cdot \Evaluation_2 + b$ for some $a \in \mathbb{R}^+, b \in \mathbb{R}$.

Now suppose $\Evaluation_1 = a \cdot \Evaluation_2 + b$ for some $a \in \mathbb{R}^+, b \in \mathbb{R}$. Consider the reward function $R_3$ given by first scaling $R_2$ by $a$, and then shaping the resulting reward with the potential function $\Phi$ that is equal to $-b$ for all initial states, and equal to $0$ elsewhere. Now $\Evaluation_3 = \Evaluation_1$, so (by Lemma~\ref{lemma:ambiguity_of_J}) we can produce $R_1$ from $R_3$ by $S'$-redistribution and potential shaping with some potential function $\Phi$ such that $\Expect{S_0 \sim \init}{\Phi(S_0)} = 0$. By composing these transformations with the transformation that produced $R_3$ from $R_2$, we obtain a $t \in \SR \bigodot \PS \bigodot \LS$ such that $R_1 = t(R_2)$. Hence if $R_1$ and $R_2$ have the same ordering of policies then $R_1 = t(R_2)$ for some $t \in \SR \bigodot \PS \bigodot \LS$. We have thus proven both directions.
\end{proof}

\OPTambiguity*

\begin{proof}
Suppose $R_1$ and $R_2$ differ by an optimality-preserving transformation.
Let $\Psi$ be the corresponding value-bounding function, that is, a
function $\Psi : \States \to \mathbb{R}$ satisfying, for all $s \in \States$ and
$a\in\Actions$,
\begin{equation*}
%\label{eq:condition-on-psi}
    \Expect{S' \sim \tfunc(s,a)}{R_2(s, a, S') + \gamma \cdot \Psi(S')}
    \leq
    \Psi(s)
\,,
\end{equation*}
with equality if and only if $a \in \mathrm{argmax} A_1^\star(s,\_)$. This gives us that
\begin{equation*}
    \Psi(s)
    =
    \max_{a\in\Actions}
    \left(
        \Expect
            {S' \sim \tfunc(s,a)}
            {R_2(s, a, S') + \gamma \cdot \Psi(S')}    
    \right)
\,.
\end{equation*}
This recursive condition on $\Psi$ is the Bellman optimality equation for the
unique optimal value function $V^\star_2$ for $R_2$.
Therefore, $\Psi(s) = V^\star_2(s)$ for all $s \in \States$, and we can
rewrite the above as
\begin{equation*}
%\label{eq:condition-as-vstar}
    \Expect{S' \sim \tfunc(s,a)}{R_2(s, a, S') + \gamma \cdot V^\star_2(S')}
    \leq
    V^\star_2(s)
\,,
\end{equation*}
with equality if and only if $a \in \mathrm{argmax} A_1^\star(s,\_)$. This means that the actions which are optimal under $R_1$ are optimal under $R_2$, and vice versa, which in turn means that $R_1 \eq{\OPT} R_2$.

Conversely, let $R_1$ and $R_2$ be any rewards such that $R_1 \eq{\OPT} R_2$.
This means that $R_1$ and $R_2$ share the same optimal actions.
Let $V^\star_2$ and $A^\star_2$ denote the optimal value and advantage
functions for $R_2$.
The Bellman optimality equation for $R_2$ ensures that, for $s\in\States$, 
\begin{equation*}
%\label{eq:bellman-gives-condition}
    V^\star_2(s)
    =
    \max_{a\in\Actions}
    \left(
        \Expect
            {S' \sim \tfunc(s,a)}
            {R_2(s, a, S') + \gamma \cdot V^\star_2(S')}
    \right)
\end{equation*}
with the maximum attained precisely when
$a \in \mathrm{argmax}_{a\in\Actions}(A^\star_2(s, a))$. We thus have
\begin{equation*}
%\label{eq:vstar-is-psi}
    \Expect
        {S' \sim \tfunc(s,a)}
        {R_2(s, a, S') + \gamma \cdot V^\star_2(S')}
    \leq
    \V^\star_2(s)
\end{equation*}
for all $s\in\States$ and $a\in\Actions$, with equality if and only if $a \in \mathrm{argmax}_{a\in\Actions}(A^\star_2(s, a))$. Note also that $\mathrm{argmax}_{a\in\Actions}(A^\star_2(s, a)) = \mathrm{argmax}_{a\in\Actions}(A^\star_1(s, a))$.
This means that $R_2$ is produced from $R_1$ by an optimality-preserving transformation (with $\Psi(s) = V^\star_2(s)$), which completes the proof.
\end{proof}

\subsubsection{STARC Metrics}

\soundandcompletemeansdistancezeroiffsameorder*

\begin{proof}
For the first direction, assume that $R_1$ and $R_2$ have the same ordering of policies. If both $R_1$ and $R_2$ are trivial, then the definition of completeness directly implies that $d(R_1, R_2) = 0$. Next, assume that $R_1$ and $R_2$ are not trivial, and assume for contradiction that $d(R_1, R_2) > 0$. Since $d$ is complete, there exists two policies $\pi_1$ and $\pi_2$ such that $J_2(\pi_2) \geq J_2(\pi_1)$ and
$$
J_1(\pi_1) - J_1(\pi_2) \geq L \cdot (\max_\pi J_1(\pi) - \min_\pi J_1(\pi)) \cdot d(R_1, R_2)
$$
for some $L > 0$. Since $R_1$ is non-trivial, we have that $(\max_\pi J_1(\pi) - \min_\pi J_1(\pi) > 0$, which means that the right-hand side of this expression is positive. This implies that there are policies $\pi_1$ and $\pi_2$ such that $J_2(\pi_2) \geq J_2(\pi_1)$ but $J_1(\pi_2) < J_1(\pi_1)$, which is a contradiction, since $R_1$ and $R_2$ have the same policy order. Thus, if $R_1$ and $R_2$ have the same policy order, then $d(R_1, R_2) = 0$.

For the other direction, assume that $d(R_1, R_2) = 0$. Since $d$ is sound, we have that there exists a positive constant $U$ such that if two policies $\pi_1$ and $\pi_2$ satisfy that $J_2(\pi_2) \geq J_2(\pi_1)$, then
$$
J_1(\pi_1) - J_1(\pi_2) \leq U \cdot (\max_\pi J_1(\pi) - \min_\pi J_1(\pi)) \cdot d(R_1, R_2).
$$
Since $d(R_1, R_2) = 0$, this means that $J_1(\pi_1) - J_1(\pi_2) \leq 0$, which means that $J_1(\pi_2) \geq J_1(\pi_1)$. Since $\pi_1$ and $\pi_2$ were chosen arbitrarily, this means that if $J_2(\pi_2) \geq J_2(\pi_1)$, then $J_1(\pi_2) \geq J_1(\pi_1)$. As such, $R_1$ and $R_2$ have the same policy order.
\end{proof}

\bilipschitz*

\begin{proof}
Assume that $d_1$ and $d_2$ are pseudometrics on $\R$ that are both sound and complete. Since $d_1$ is complete, we have that
$$
    L_1 \cdot d_1(R_1, R_2) \cdot (\max_\pi J_1(\pi) - \min_\pi J_1(\pi)) \leq \max_{\pi_1,\pi_2 : J_2(\pi_2) \geq J_2(\pi_1)} J_1(\pi_1) - J_1(\pi_2).
$$
Similarly, since $d_2$ is sound, we also have that
$$
\max_{\pi_1,\pi_2 : J_2(\pi_2) \geq J_2(\pi_1)} J_1(\pi_1) - J_1(\pi_2) \leq U_2 \cdot d_2(R_1, R_2) \cdot (\max_\pi J_1(\pi) - \min_\pi J_1(\pi)).
$$
This implies that
\begin{align*}
    &L_1 \cdot d_1(R_1, R_2) \cdot (\max_\pi J_1(\pi) - \min_\pi J_1(\pi))\\
    \leq &U_2 \cdot d_2(R_1, R_2) \cdot (\max_\pi J_1(\pi) - \min_\pi J_1(\pi)).
\end{align*}
First suppose that $(\max_\pi J_1(\pi) - \min_\pi J_1(\pi)) > 0$. We can then divide both sides, and obtain that
$$
    d_1(R_1, R_2) \leq \left(\frac{U_2}{L_1}\right) d_2(R_1, R_2).
$$
Similarly, we also have that 
$$
    \left(\frac{L_2}{U_1}\right) d_2(R_1, R_2) \leq d_1(R_1, R_2).
$$
This means that we have constants $\left(\frac{U_2}{L_1}\right)$ and $\left(\frac{L_2}{U_1}\right)$ not depending on $R_1$ or $R_2$, such that
$$
\left(\frac{L_2}{U_1}\right) d_2(R_1, R_2) \leq d_1(R_1, R_2) \leq \left(\frac{U_2}{L_1}\right) d_2(R_1, R_2)
$$
for all $R_1$ and $R_2$ such that $(\max_\pi J_1(\pi) - \min_\pi J_1(\pi)) > 0$.

Next, assume that we have that $(\max_\pi J_1(\pi) - \min_\pi J_1(\pi)) = 0$, but $(\max_\pi J_2(\pi) - \min_\pi J_2(\pi)) > 0$. Since $d_1$ and $d_2$ are pseudometrics, we have that $d_1(R_1, R_2) = d_1(R_2, R_1)$ and $d_2(R_1, R_2) = d_2(R_2, R_1)$. Therefore, $\left(\frac{L_2}{U_1}\right) d_2(R_1, R_2) \leq d_1(R_1, R_2) \leq \left(\frac{U_2}{L_1}\right) d_2(R_1, R_2)$ in this case as well, as already shown above.

Finally, assume that $(\max_\pi J_1(\pi) - \min_\pi J_1(\pi)) = 0$ and $(\max_\pi J_2(\pi) - \min_\pi J_2(\pi)) = 0$. In this case, $R_1$ and $R_2$ induce the same policy order. This in turn means that $d_1(R_1, R_2) = d_2(R_1, R_2) = 0$, and so 
$$
\left(\frac{L_2}{U_1}\right) d_2(R_1, R_2) \leq d_1(R_1, R_2) \leq \left(\frac{U_2}{L_1}\right) d_2(R_1, R_2)
$$
in this case as well. This completes the proof.
\end{proof}

\valuecanon*

\begin{proof}
To prove that $c$ is a canonicalisation function, we must show
\begin{enumerate}
    \item that $c$ is linear,
    \item that $c(R)$ and $R$ differ by potential shaping and $S'$-redistribution, and
    \item that $c(R_1) = c(R_2)$ if $R_1$ and $R_2$ differ by potential shaping and $S'$-redistribution.
\end{enumerate}
We first show that $c$ is linear. Given a state $s$, let $v_s$ be the $|\States||\Actions||\States|$-dimensional vector where the $(s',a,s'')$'th dimension is given by
$$
\sum_{i=0}^\infty \gamma^i \cdot \mathbb{P}(S_i = s', A_i = a, S_{i+1} = s''),
$$
where the probability is given for a trajectory that is generated from $\pi$ and $\tau$, starting in $s$. Now note that $V^\pi(s) = v_s \cdot R$, where $R$ is represented as a vector. Using these vectors $\{v_s\}$, it is possible to express $c$ as a linear transformation.

%Next, using the vectors $v_s$, we can create two matrices $M$, $M'$ such that $c(R) = R + \gamma MR - M'R$. This means that $c$ is linear in $R$.

To see that $c(R)$ and $R$ differ by potential shaping and $S'$-redistribution, it is sufficient to note that $V^\pi$ acts as a potential function, and that setting $R_2(s,a,s') = \mathbb{E}_{S' \sim \tau(s,a)}[R_1(s,a,S')]$ is a form of $S'$-redistribution.

To see that $c(R_1) = c(R_2)$ if $R_1$ and $R_2$ differ by potential shaping and $S'$-redistribution, first note that if $R_1$ and $R_2$ differ by potential shaping, so that $R_2(s,a,s') = R_1(s,a,s') + \gamma \Phi(s') - \Phi(s)$ for some $\Phi$, then $V^\pi_2(s) = V^\pi_1(s) - \Phi(s)$ (Proposition~\ref{prop:change_from_potentials}). This means that 
\begin{align*}
    c(R_2)(s,a,s') = &\mathbb{E}[R_2(s,a,S') + \gamma \cdot V^\pi_2(S') - V^\pi_2(s)]\\
    = &\mathbb{E}[R_1(s,a,S') + \gamma \cdot \Phi(S') - \Phi(s)\\ 
    &+ \gamma \cdot (V^\pi_1(S') - \Phi(S')) - (V^\pi_1(s) - \Phi(s))]\\
    = &\mathbb{E}[R_1(s,a,S') + \gamma \cdot V^\pi_1(S') - V^\pi_1(s)]\\
    = &c(R_1)(s,a,s').
\end{align*}
To see that $c(R_1) = c(R_2)$ only if $R_1$ and $R_2$ differ by potential shaping and $S'$-redistribution, first note that we have already shown that $R$ and $c(R)$ differ by potential shaping and $S'$-redistribution for all $R$. This implies that $R_1$ and $c(R_1)$ differ by potential shaping and $S'$-redistribution, and likewise for $R_2$ and $c(R_2)$. Then if $c(R_1) = c(R_2)$, we can combine these transformations, and obtain that $R_1$ and $R_2$ also differ by potential shaping and $S'$-redistribution.
This completes the proof.
\end{proof}

\minimalcanonicalisation*

\begin{proof}
Let $R_0$ be the reward function that is $0$ for all transitions.
First recall that the set of all reward functions that differ from $R_0$ by potential shaping and $S'$-redistribution form a linear subspace of $\R$ (Proposition~\ref{prop:PS_SR_linear}).
%First note that the set of all potential shaping functions forms an $|\States|$-dimensional linear subspace of $\R$. 
Let this space be denoted by $\mathcal{Y}$, and let $\mathcal{X}$ denote the orthogonal complement of $\mathcal{Y}$ in $\R$. Now any reward function $R \in \R$ can be uniquely expressed in the form $R_\mathcal{X} + R_\mathcal{Y}$, where $R_\mathcal{X} \in \mathcal{X}$ and $R_\mathcal{Y} \in \mathcal{Y}$.
Consider the function $c : \R \to \R$ where $c(R) = R_\mathcal{X}$. Now this function is a canonicalisation function such that $n(c(R)) \leq R'$ for all $R'$ such that $c(R) = c(R')$, assuming that $n$ is a weighted $L_2$-norm.
To see this, we must show that
\begin{enumerate}
    \item $c$ is linear,
    \item $c(R)$ and $R$ differ by potential shaping and $S'$-redistribution for all $R$,
    \item $c(R_1) = c(R_2)$ for all $R_1$ and $R_2$ which differ by potential shaping and $S'$-redistribution, and
    \item $n(c(R)) \leq n(R')$ for all $R'$ such that $c(R) = c(R')$.
\end{enumerate}
It follows directly from the construction that $c$ is linear. 
To see that $c(R)$ and $R$ differ by potential shaping and $S'$-redistribution, simply note that $c(R) = R - R_\mathcal{Y}$, where $R_\mathcal{Y}$ is given by a combination of potential shaping and $S'$-redistribution of $R_0$. 
To see that $c(R_1) = c(R_2)$ if $R_1$ and $R_2$ differ by potential shaping and $S'$-redistribution, let $R_2 = R_1 + R'$, where $R'$ is given by potential shaping and $S'$-redistribution of $R_0$, and let $R_1 = R_\mathcal{X} + R_\mathcal{Y}$, where $R_\mathcal{X} \in \mathcal{X}$ and $R_\mathcal{Y} \in \mathcal{Y}$. Now $c(R_1) = R_\mathcal{X}$. 
Moreover, $R_2 = R_\mathcal{X} + R_\mathcal{Y} + R'$. We also have that $R' \in \mathcal{Y}$, which means that $R_2$ can be expressed as $R_\mathcal{X} + (R_\mathcal{Y} + R')$, where $R_\mathcal{X} \in \mathcal{X}$ and $(R_\mathcal{Y} + R') \in \mathcal{Y}$. This implies that $c(R_2) = R_\mathcal{X}$, so if $R_1$ and $R_2$ differ by potential shaping and $S'$-redistribution, then $c(R_1) = c(R_2)$. To see that $c(R_1) = c(R_2)$ only if $R_1$ and $R_2$ differ by potential shaping and $S'$-redistribution, first note that we have already shown that $R$ and $c(R)$ differ by potential shaping and $S'$-redistribution for all $R$. This implies that $R_1$ and $c(R_1)$ differ by potential shaping and $S'$-redistribution, and likewise for $R_2$ and $c(R_2)$. Then if $c(R_1) = c(R_2)$, we can combine these transformations, and obtain that $R_1$ and $R_2$ also differ by potential shaping and $S'$-redistribution.

To see that $n(c(R)) \leq n(R')$ for all $R'$ such that $c(R) = c(R')$, first note that if $c(R) = c(R')$, then $R = R_\mathcal{X} + R_\mathcal{Y}$ and $R' = R_\mathcal{X} + R_\mathcal{Y}'$, where $R_\mathcal{X} \in \mathcal{X}$ and $R_\mathcal{Y}, R_\mathcal{Y}' \in \mathcal{Y}$. This means that $n(c(R)) = n(R_\mathcal{X})$, and $n(R') = n(R_\mathcal{X} + R_\mathcal{Y}')$. Moreover, since $n$ is a weighted $L_2$-norm, and since $R_\mathcal{X}$ and $R_\mathcal{Y}'$ are orthogonal, we have that $n(R_\mathcal{X} + R_\mathcal{Y}) = \sqrt{n(R_\mathcal{X})^2 + n(R_\mathcal{Y})^2} \geq n(R_\mathcal{X})$. This means that $n(c(R)) \leq n(R')$.

To see that this canonicalisation function is the unique minimal canonicalisation function for any weighted $L_2$-norm $n$, consider an arbitrary reward function $R$. Now, the set of all reward functions that differ from $R$ by potential shaping and $S'$-redistribution forms an affine space of $\R$, and a minimal canonicalisation function must map $R$ to a point $R'$ in this space such that $n(R') \leq n(R'')$ for all other points $R''$ in that space. If $n$ is a weighted $L_2$-norm, then this specifies a convex optimisation problem with a unique solution.
\end{proof}

\occupancymeasureprojectioniscanonicalisation*

\begin{proof}
To show that $c$ is a canonicalisation function, we must show that 
\begin{enumerate}
    \item $c$ is linear,
    \item $c(R)$ and $R$ differ by potential shaping and $S'$-redistribution for all $R$, and
    \item $c(R_1) = c(R_2)$ for all $R_1$ and $R_2$ which differ by potential shaping and $S'$-redistribution.
\end{enumerate}
It is straightforward that $c$ is linear, since it is a projection map. To see that $R$ and $c(R)$ differ by potential shaping and $S'$-redistribution, note that there is a constant $k$ such that $\eta^\pi \cdot R = \eta^\pi \cdot c(R) + k$ for all policies $\pi$, since $\mathrm{Im}(c)$ is parallel to $\Omega$. This means that the policy evaluation functions of $R$ and $c(R)$ differ by a constant $k$. 
By Proposition~\ref{prop:change_from_potentials}, this means that we can create a reward function $R'$ which has the same policy evaluation function as $c(R)$, by applying potential shaping to $R$ with a potential function such that $\Phi(s) = -k$ for all $s \in \mathrm{supp}(\init)$. By Lemma~\ref{lemma:ambiguity_of_J}, this implies that $R'$ and $c(R)$ differ by potential shaping and $S'$-redistribution. Thus $R$ and $c(R)$ differ by potential shaping and $S'$-redistribution.
Finally, note that if $R_1$ and $R_2$ differ by potential shaping and $S'$-redistribution, then there is a constant $k$ such that $\eta^\pi \cdot R_2 = \eta^\pi \cdot R_1 + k$ for all policies $\pi$ (Proposition~\ref{prop:change_from_potentials}). This in turn means that $c(R_1) = c(R_2)$.
\end{proof}

\Jnorm*

\begin{proof}
To show that a function $n$ is a norm on $\mathrm{Im}(c)$, we must show that it satisfies:
\begin{enumerate}
    \item $n(R) \geq 0$ for all $R \in \mathrm{Im}(c)$.
    \item $n(R) = 0$ if and only if $R = R_0$ for all $R \in \mathrm{Im}(c)$.
    \item $n(\alpha \cdot R) = \alpha \cdot n(R)$ for all $R \in \mathrm{Im}(c)$ and all scalars $\alpha$.
    \item $n(R_1 + R_2) \leq n(R_1) + n(R_2)$ for all $R_1, R_2 \in \mathrm{Im}(c)$.
\end{enumerate}

Here $R_0$ is the reward function that is $0$ everywhere.
It is trivial to show that Axioms 1 and 3 are satisfied by $n$. For Axiom 2, note that $n(R) = 0$ exactly when $\max_\pi J(\pi) = \min_\pi J(\pi)$. If $R$ is $R_0$, then $J(\pi) = 0$ for all $\pi$, and so the \enquote{if} part holds straightforwardly. For the \enquote{only if} part, let $R$ be a reward function such that $\max_\pi J(\pi) = \min_\pi J(\pi)$. Then $R$ and $R_0$ induce the same policy ordering under $\tau$ and $\mu_0$, which means that they differ by potential shaping, $S'$-redistribution, and positive linear scaling (Theorem~\ref{thm:policy_ordering}). 
Moreover, since $R_0$ is $0$ everywhere, this means that $R$ and $R_0$ in fact differ by potential shaping and $S'$-redistribution.
However, from the definition of canonicalisation functions, if $R_1, R_2 \in \mathrm{Im}(c)$ differ by potential shaping and $S'$-redistribution, then it must be that $R_1 = R_2$. Hence Axiom 2 holds as well.
We can show that Axiom 4 holds algebraically:
\begin{align*}
    n(R_1 + R_2) &= \max_\pi (J_1(\pi) + J_2(\pi)) - \min_\pi (J_1(\pi) + J_2(\pi))\\
    &\leq \max_\pi J_1(\pi) + \max_\pi J_2(\pi) - \min_\pi J_1(\pi) - \min_\pi J_2(\pi)\\
    &= (\max_\pi J_1(\pi) - \min_\pi J_1(\pi)) + (\max_\pi J_2(\pi) - \min_\pi J_2(\pi))\\
    &= n(R_1) + n(R_2)
\end{align*}
This means that $n(R) = \max_\pi J(\pi) - \min_\pi J(\pi)$ is a norm on $\mathrm{Im}(c)$.
\end{proof}

\subsubsection{Regret Bounds For STARC Metrics}

\STARCpseudometrics*

\begin{proof}
To show that $d$ is a pseudometric, we must show that
\begin{enumerate}
    \item $d(R,R) = 0$
    \item $d(R_1, R_2) = d(R_2, R_1)$
    \item $d(R_1, R_3) \leq d(R_1, R_2) + d(R_2, R_3)$
\end{enumerate}
1 follows from the fact that $m$ is a metric, and 2 follows directly from the fact that the definition of STARC metrics is symmetric in $R_1$ and $R_2$. For 3, the fact that $m$ is a metric again implies that $d(R_1, R_3) = m(s(R_1), s(R_3)) \leq m(s(R_1), s(R_2)) + m(s(R_2), s(R_3)) = d(R_1, R_2) + d(R_2, R_3)$. This completes the proof.
\end{proof}

\STARCdistancezero*

\begin{proof}
This is immediate from Theorem~\ref{thm:policy_ordering}, together with the definition of STARC metrics.%, together with the fact that if $R_1$ and $R_2$ differ by potential shaping, $S'$-redistribution, and positive linear scaling, applied in any order, then $R_2 = \alpha \cdot R_3$ for some scalar $\alpha$ and some $R_3$ that differs from $R_1$ via potential shaping and $S'$-redistribution.
\end{proof}

We will next show that all STARC metrics are sound. To do this, we must first prove a number of supporting lemmas:

\begin{lemma}\label{lemma:same_policy_bound}
For any rewards $R_1$ and $R_2$, and any policy $\pi$, we have that
$$
|\Evaluation_1(\pi) - \Evaluation_2(\pi)| \leq \left(\frac{1}{1-\gamma}\right) L_\infty(R_1, R_2).
$$ 
\end{lemma}

\begin{proof}
This follows from straightforward algebra:
\begin{align*}
|\Evaluation_1(\pi) - \Evaluation_2(\pi)| = &\Big|\mathbb{E}_{\xi \sim \pi}\left[\sum_{t=0}^\infty \gamma^ t R_1(S_t,A_t,S_{t+1})\right] \\
& -\mathbb{E}_{\xi \sim \pi}\left[\sum_{t=0}^\infty \gamma^ t R_2(S_t,A_t,S_{t+1})\right]\Big|\\
= &\Big|\sum_{t=0}^\infty \gamma^t \mathbb{E}_{\xi \sim \pi}[R_1(S_t,A_t,S_{t+1}) - R_2(S_t,A_t,S_{t+1})]\Big|\\
\leq &\sum_{t=0}^\infty \gamma^t \mathbb{E}_{\xi \sim \pi}[ |R_1(S_t,A_t,S_{t+1}) - R_2(S_t,A_t,S_{t+1})|]\\
\leq &\sum_{t=0}^\infty \gamma^t L_\infty(R_1, R_2) = \left(\frac{1}{1-\gamma}\right) L_\infty(R_1, R_2).
\end{align*}
Here the second line follows from the linearity of expectation, and the third line follows from Jensen's inequality.
\end{proof}

Thus, the $L_\infty$-distance between two rewards bounds the difference between their policy evaluation functions. Since all norms are bilipschitz equivalent on any finite-dimensional vector space, this extends to all norms:

\begin{lemma}\label{lemma:norm_to_distance_fixed_pi}
If $p$ is a \emph{norm}, then there is a positive constant $K_p$ such that, for any reward functions $R_1$ and $R_2$, and any policy $\pi$, $|\Evaluation_1(\pi) - \Evaluation_2(\pi)| \leq K_p \cdot p(R_1, R_2)$.
\end{lemma}
\ifshowproofs
\begin{proof}
If $p$ and $q$ are norms on a finite-dimensional vector space, then there are constants $k$ and $K$ such that 
$k \cdot p(x) \leq q(x) \leq K \cdot p(x)$.
Since $\States$ and $\Actions$ are finite, $\mathcal{R}$ is a finite-dimensional vector space. This means that there is a constant $K$ such that $L_\infty(R_1, R_2) \leq K \cdot p(R_1, R_2)$. Together with Lemma~\ref{lemma:same_policy_bound}, this implies that 
$$
|\Evaluation_1(\pi) - \Evaluation_2(\pi)| \leq \left(\frac{1}{1-\gamma}\right) \cdot K \cdot m(R_1, R_2).
$$ 
Letting $K_p = \left( \frac{K}{1-\gamma}\right)$ completes the proof.
\end{proof}
\fi

Next, we show that if the difference between two policy evaluation functions can be bounded, then we can derive a regret bound:

%Note that the constant $K_p$ given by Lemma~\ref{lemma:norm_to_distance_fixed_pi} may not be the smallest value of $K$ for which this statement holds of a given norm $p$. \orange{This fact can be used to compute tighter regret bounds for particular STARC-metrics.}

\begin{lemma}\label{lemma:policy_optimisation_bound}
Let $R_1$ and $R_2$ be reward functions, and $\pi_1, \pi_2$ be two policies. If $|\Evaluation_1(\pi) - \Evaluation_2(\pi)| \leq U$ for $\pi \in \{\pi_1, \pi_2\}$, and if $\Evaluation_2(\pi_2) \geq \Evaluation_2(\pi_1)$, then
$$
\Evaluation_1(\pi_1) - \Evaluation_1(\pi_2) \leq 2 \cdot U.
$$
\end{lemma}
\ifshowproofs
\begin{proof}
First note that $U$ must be non-negative.
Next, note that if $\Evaluation_1(\pi_1) < \Evaluation_1(\pi_2)$ then $\Evaluation_1(\pi_1) - \Evaluation_1(\pi_2) < 0$, and so the lemma holds. 
Now consider the case when $\Evaluation_1(\pi_1) \geq \Evaluation_1(\pi_2)$: 
\begin{align*}
\Evaluation_1(\pi_1) - \Evaluation_1(\pi_2) &= %|\J_1(\pi_1) - \J_1(\pi_2)|\\
%&= 
\Evaluation_1(\pi_1) - \Evaluation_2(\pi_2) + \Evaluation_2(\pi_2) - \Evaluation_1(\pi_2)\\
&\leq |\Evaluation_1(\pi_1) - \Evaluation_2(\pi_2)| + |\Evaluation_2(\pi_2) - \Evaluation_1(\pi_2)|\\
\end{align*}
Our assumptions imply that $|\Evaluation_2(\pi_2) - \Evaluation_1(\pi_2)| \leq U$. We will next show that $|\Evaluation_1(\pi_1) - \Evaluation_2(\pi_2)| \leq U$ as well.
Our assumptions imply that
\begin{align*}
&|\Evaluation_1(\pi_1) - \Evaluation_2(\pi_1)| \leq U\\
\implies &\Evaluation_2(\pi_1) \geq \Evaluation_1(\pi_1) - U\\
\implies &\Evaluation_2(\pi_2) \geq \Evaluation_1(\pi_1) - U
\end{align*}
Here the last implication uses the fact that $\Evaluation_2(\pi_2) \geq \Evaluation_2(\pi_1)$. A symmetric argument also shows that $\Evaluation_1(\pi_1) \geq \Evaluation_2(\pi_2) - U$ (recall that we assume that $\Evaluation_1(\pi_1) \geq \Evaluation_1(\pi_2)$). Together, this implies that $|\Evaluation_1(\pi_1) - \Evaluation_2(\pi_2)| \leq U$. We have thus shown that if $\Evaluation_1(\pi_1) \geq \Evaluation_1(\pi_2)$ then
$$
|\Evaluation_1(\pi_1) - \Evaluation_2(\pi_2)| + |\Evaluation_2(\pi_2) - \Evaluation_1(\pi_2)| \leq 2 \cdot U,$$
and so the lemma holds. This completes the proof.
\end{proof}
\fi

Note that Lemma~\ref{lemma:same_policy_bound} and \ref{lemma:policy_optimisation_bound} together imply that, for any two reward functions $R_1, R_2$, and any two policies $\pi_1, \pi_2$, if $\Evaluation_2(\pi_2) \geq \Evaluation_2(\pi_1)$, then 
$$
\Evaluation_1(\pi_2) - \Evaluation_1(\pi_2) \leq \left(\frac{2}{1-\gamma}\right) L_\infty(R_1, R_2).
$$
Moreover, Lemma~\ref{lemma:norm_to_distance_fixed_pi} says that a similar bound can be derived for any norm. To turn this into a regret bound for STARC metrics, we must consider the effect of the canonicalisation function and normalisation. Our next lemma will be used to account for the difference between the size of a reward function before and after canonicalisation:

\begin{lemma}\label{lemma:smallest_n_ratio}
For any linear function $c : \mathbb{R}^n \to \mathbb{R}^n$ and any norm $n$, there is a positive constant $K_n$ such that $n(c(v)) \leq K_n \cdot n(v)$ for all $v \in \mathbb{R}^n$.
\end{lemma}
\ifshowproofs
\begin{proof}
First consider the case when $n(v) > 0$. In this case, we can find an upper bound for $n(c(v))$ in terms of $n(v)$ by finding an upper bound for $\frac{n(c(R))}{n(R)}$.
Since $c$ is linear, and since $n$ is absolutely homogeneous, we have that for any $v \in \mathbb{R}^n$ and any non-zero $\alpha \in \mathbb{R}$,
$$
\frac{n(c(\alpha \cdot v))}{n(\alpha \cdot v)} = \left(\frac{\alpha}{\alpha}\right)\frac{n(c(v))}{n(v)} = \frac{n(c(v))}{n(v)}.
$$
In other words, $\frac{n(c(v))}{n(v)}$ is unaffected by scaling of $v$. We may thus restrict our attention to the unit ball of $n$.
Next, since the surface of the unit ball of $n$ is a compact set, and since $\frac{n(c(v))}{n(v)}$ is continuous on this surface, the extreme value theorem implies that $\frac{n(c(v))}{n(v)}$ must take on some maximal value $K_n$ on this domain.
Together, the above implies that $n(c(v)) \leq K_n \cdot n(v)$ for all $R$ such that $n(v) > 0$.

%\red{[INSERT IMAGE]}

Next, suppose $n(v) = 0$. In this case, $v$ is the zero vector. Since $c$ is linear, this implies that $c(v) = v$, which means that $n(c(v)) = 0$ as well. Therefore, if $n(v) = 0$, then the statement holds for any $K_n$. In particular, it holds for the value $K_n$ selected above. 
%This completes the proof.
\end{proof}
\fi

Note that the value of $K_n$ depends on how \enquote{tilted} $\mathrm{Im}(c)$ is. If $c$ is an orthogonal projection (as is the case if $c$ is the minimal canonicalisation function for an $L_2$-norm), then $K_n = 1$. Our next lemma has a somewhat complicated statement, but its purpose is simply to derive a bound on the difference between the policy evaluation functions $\J_1, \J_2$ of two reward functions $R_1, R_2$, based on the difference between the policy evaluation functions of their standardised counterparts:

%Note that, for the minimal canonicalisation function, we have that $K_n = 1$. For $C^\mathrm{EPIC}$, we have that $K_n \leq 4$ \citep[see ][]{epic}. In general, the value of $K_n$ 

%For our next lemma, we must once again introduce some new notation. Given a linear function $c : \R \to \R$ and a norm $n$ in $\R$, let $s$ be the function given by $s(R) = \left(\frac{c(R)}{n(c(R))}\right)$ when $n(c(R)) > 0$, and $c(R)$ otherwise. \red{complete}

\begin{lemma}\label{lemma:standardised_to_unstandardised_bound}
Let $c$ be a canonicalisation function, and let $n$ be a norm on $\mathrm{Im}(c)$. 
Let $R$ be any reward function, and let $R_S = \left(\frac{c(R)}{n(c(R))}\right)$ if $n(c(R)) > 0$, and $c(R)$ otherwise. Then $\J(\pi_1) - \J(\pi_2) = n(c(R)) \cdot (\J_S(\pi_1) - \J_S(\pi_2))$, where $J_S$ is the policy evaluation function of $R_S$.
\end{lemma}
\ifshowproofs
\begin{proof}
Let us first consider the case where $n(c(R)) = 0$. Since $n$ is a norm, $c(R)$ must be the reward function that is $0$ everywhere. Since $c$ is a canonicalisation function, we have that $R$ and $c(R)$ have the same ordering of policies. Thus $R$ is \emph{trivial}, which means that $J(\pi_1) = J(\pi_2)$ for all $\pi_1,\pi_2$. Thus $J(\pi_1) - J(\pi_2) = 0$, and so the statement holds.

Let us next consider the case when $n(c(R)) > 0$. Let $R_C = c(R)$. Since $c$ is a canonicalisation function, we have that $R$ and $R_C$ differ by potential shaping and $S'$-redistribution. Thus, for all $\pi$, $J_C(\pi) = J(\pi) - \Expect{S_0 \sim \init}{\Phi(S_0)}$ for some potential function $\Phi$ (Proposition~\ref{prop:change_from_potentials}). Moreover, $J_S = J_C \cdot \left(\frac{1}{n(c(R))}\right)$. This means that
$$
\Evaluation_S(\pi) = \left(\frac{1}{n(c(R))}\right) (\Evaluation(\pi) - \Expect{S_0 \sim \init}{\Phi(S_0)})
$$
for all $\pi$. This further implies that
$$
\Evaluation_S(\pi_1) - \Evaluation_S(\pi_2) = \left(\frac{1}{n(c(R))}\right) (\Evaluation(\pi_1) - \Evaluation(\pi_2))
$$
since the $\Expect{S_0 \sim \init}{\Phi(S_0)}$-terms cancel out. By rearranging, we get that 
$$
\Evaluation(\pi_1) - \Evaluation(\pi_2) =
n(c(R))(\Evaluation_S(\pi_1) - \Evaluation_S(\pi_2)).
$$
This completes the proof.
\end{proof}
\fi

\STARCsound*

\begin{proof}
Consider any transition function $\tau$ and any initial state distribution $\init$, and let $d$ be a STARC metric. We wish to show that there exists a positive constant $U$, such that for any $R_1$ and $R_2$, and any pair of policies $\pi_1$ and $\pi_2$ such that $J_2(\pi_2) \geq J_2(\pi_1)$, we have that
$$
\J_1(\pi_1) - \J_1(\pi_2) \leq (\max_\pi J_1(\pi) - \min_\pi J_1(\pi)) \cdot U \cdot d(R_1, R_2).
$$ 
Recall that $d(R_1, R_2) = m(s(R_1), s(R_2))$, where $m$ is an admissible metric. Since $m$ is admissible, we have that $p(s(R_1), s(R_2)) \leq K_m \cdot m(s(R_1), s(R_2))$ for some norm $p$ and constant $K_m$.
Moreover, since $p$ is a norm, we can apply Lemma~\ref{lemma:norm_to_distance_fixed_pi} to conclude that there is a constant $K_p$ such that for any policy $\pi$, we have that
$$
|\J_1^S(\pi) - \J_2^S(\pi)| \leq K_p \cdot p(s(R_1), s(R_2)),
$$ 
where $\J_1^S$ is the policy evaluation function of $s(R_1)$, and $\J_2^S$ is the policy evaluation function of $s(R_2)$.
Combining this with the fact that $p(s(R_1), s(R_2)) \leq K_m \cdot m(s(R_1), s(R_2))$, we get
\begin{align*}
|\J_1^S(\pi) - \J_2^S(\pi)| &\leq K_p \cdot p(s(R_1), s(R_2))\\
  &\leq K_p \cdot K_m \cdot m(s(R_1), s(R_2))\\  
  &= K_{mp} \cdot d(R_1, R_2)
\end{align*}
where $K_{mp} = K_p \cdot K_m$.
We have thus established that, for any $\pi$, we have
$$
|\J_1^S(\pi) - \J_2^S(\pi)| \leq K_{mp} \cdot d(R_1, R_2).
$$
%Next, consider a fixed transition function $\tau$ and initial state distribution $\mu_0$. If $d$ is an $E$-STARC metric, then let $\tau$ and $\mu_0$ be those that $d$ is defined against. Otherwise, let $\tau$ and $\mu_0$ be chosen arbitrarily.

Let $\pi_1$ and $\pi_2$ be any two policies such that $\J_2(\pi_2) \geq \J_2(\pi_1)$.
Note that $\J_2(\pi_2) \geq \J_2(\pi_1)$ if and only if $\J_2^S(\pi_2) \geq \J_2^S(\pi_1)$. 
We can therefore apply Lemma~\ref{lemma:policy_optimisation_bound} and conclude that 
$$
\J_1^S(\pi_1) - \J_1^S(\pi_2) \leq 2 \cdot K_{mp} \cdot d(R_1, R_2).
$$
%Moreover, since $c$ and $n$ are continuous, Lemma~\ref{lemma:smallest_n_ratio} implies that there is a $K$ such that $n(c(R)) \leq K \cdot n(R)$ for all $R$. 
%
%Since $c$ is a canonicalisation function, we have that $R_1$ and $c(R_1)$ differ by potential shaping and $S'$-redistribution. By Proposition~\ref{prop:change_from_potentials} (and the linearity of expectation), this implies that there is a constant $B$ such that $\Evaluation_1^C = \Evaluation_1 + B$.
We can now apply Lemma~\ref{lemma:standardised_to_unstandardised_bound}:
$$
\J_1(\pi_1) - \J_1(\pi_2) \leq n(c(R_1)) \cdot 2 \cdot  K_{mp} \cdot d(R_1, R_2).
$$ 
We have that $n$ is a norm on $\mathrm{Im}(c)$. Moreover, $\max_\pi \J_1(\pi) - \min_\pi \J_1(\pi)$ is also a norm on $\mathrm{Im}(c)$ (Proposition~\ref{prop:J_norm}).
Since $\mathrm{Im}(c)$ is a finite-dimensional vector space, this means that there is a constant $K_s$ such that $n(c(R_1)) \leq K_s \cdot (\max_\pi \J_1(\pi) - \min_\pi \J_1(\pi))$ for all $R_1 \in \R$. 
Let $U = 2 \cdot K_{mp} \cdot K_s$. We have now established that, for any $\pi_1$ and $\pi_2$ such that $\J_2(\pi_2) \geq \J_2(\pi_1)$, we have
$$
\J_1(\pi_1) - \J_1(\pi_2) \leq (\max_\pi \J_1(\pi) - \min_\pi \J_1(\pi)) \cdot U \cdot d(R_1, R_2).
$$ 
This completes the proof.
\end{proof}

Our second main result for this section is that all STARC metrics are complete. To prove this, we must yet again first prove a number of supporting lemmas: 

\begin{lemma}\label{lemma:lower_bound_lemma_0}
    Let $S \subset \mathbb{R}^n$ be the boundary of a bounded convex set whose interior is non-empty and includes the origin. Then there is an $A > 0$ such that for any $x, y \in S$, if $x \neq y$ then the angle between $x$ and $y-x$ is at least $A$.
\end{lemma}
\begin{proof}
Let $x,y$ be two arbitrary points in $S$ such that $x \neq y$. Let $\alpha$ be the angle between $x$ and $y$, let $\beta$ be the angle between $-y$ and $x-y$, let $\gamma$ be the angle between $-x$ and $y-x$, and let $\delta$ be the angle between $x$ and $y-x$. Note that $\alpha + \beta + \gamma = \pi$, since these angles are the interior angles of the triangle whose corners lie at $x$, $y$, and the origin. We also have that $\gamma + \delta = \pi$, since these two angles add up to the angle between $x$ and $-x$. We seek a lower bound on $\delta$.

\begin{figure}[H]
    \centering
    \includegraphics[width=2\textwidth/3]{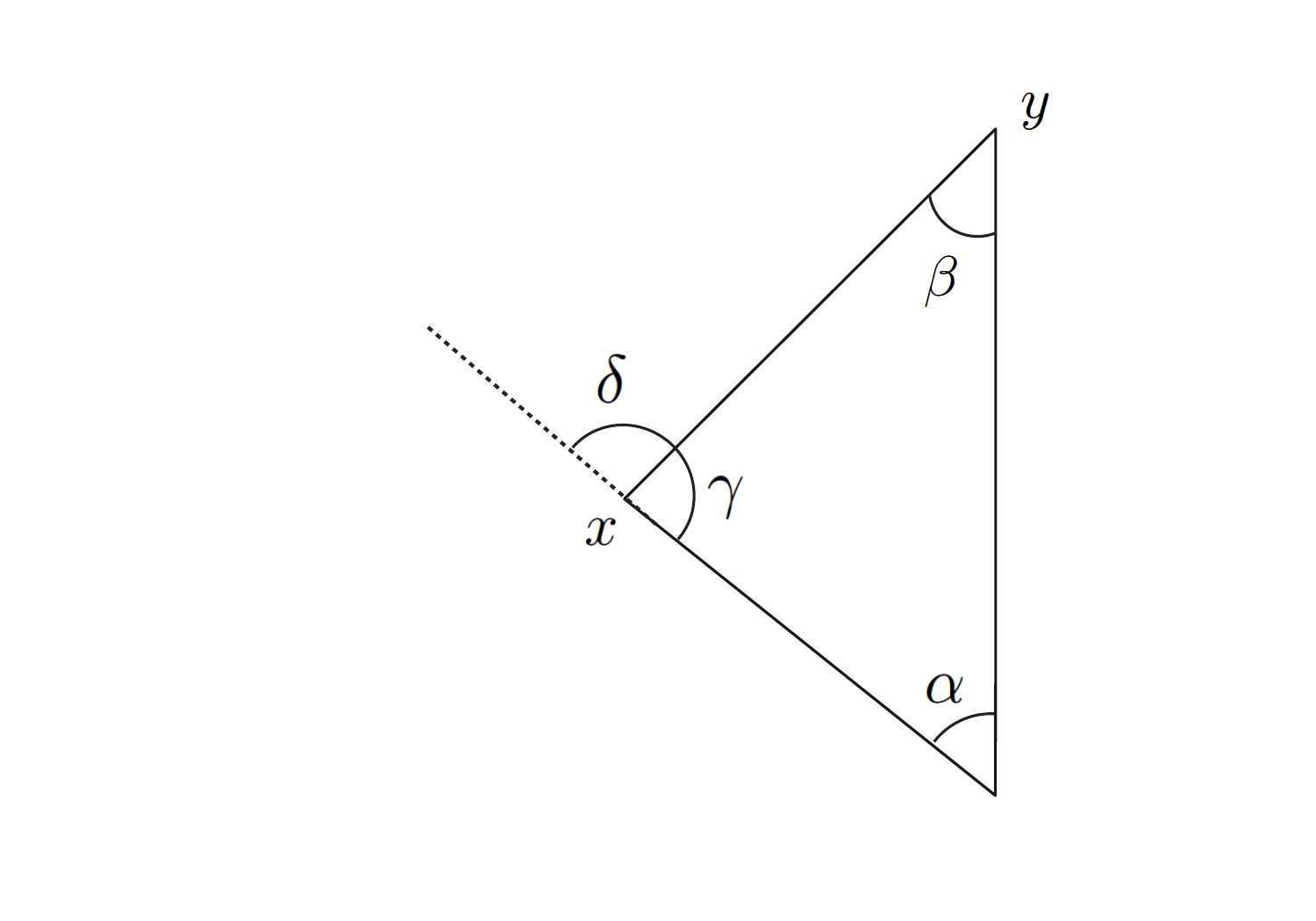}
    %\caption{Caption}
    \label{fig:lemma_62_1}
\end{figure}

First note that if $\alpha > \pi/2$ then $\gamma < \pi/2$, since $\alpha + \beta + \gamma = \pi$. This means that $\delta > \pi/2$, since $\gamma + \delta = \pi$. Next, suppose $\alpha \leq \pi/2$. 
Since $\gamma + \delta = \pi$, we can derive a lower bound for $\delta$ by deriving an upper bound for $\gamma$.
Let $z$ be the point such that the angle between $y$ and $z$ is $\pi/2$, and such that $x$ lies on the line segment between $z$ and $y$. Let $\theta$ be the angle between $-z$ and $y-z$. 

\begin{figure}[H]
    \centering
    \includegraphics[width=2\textwidth/3]{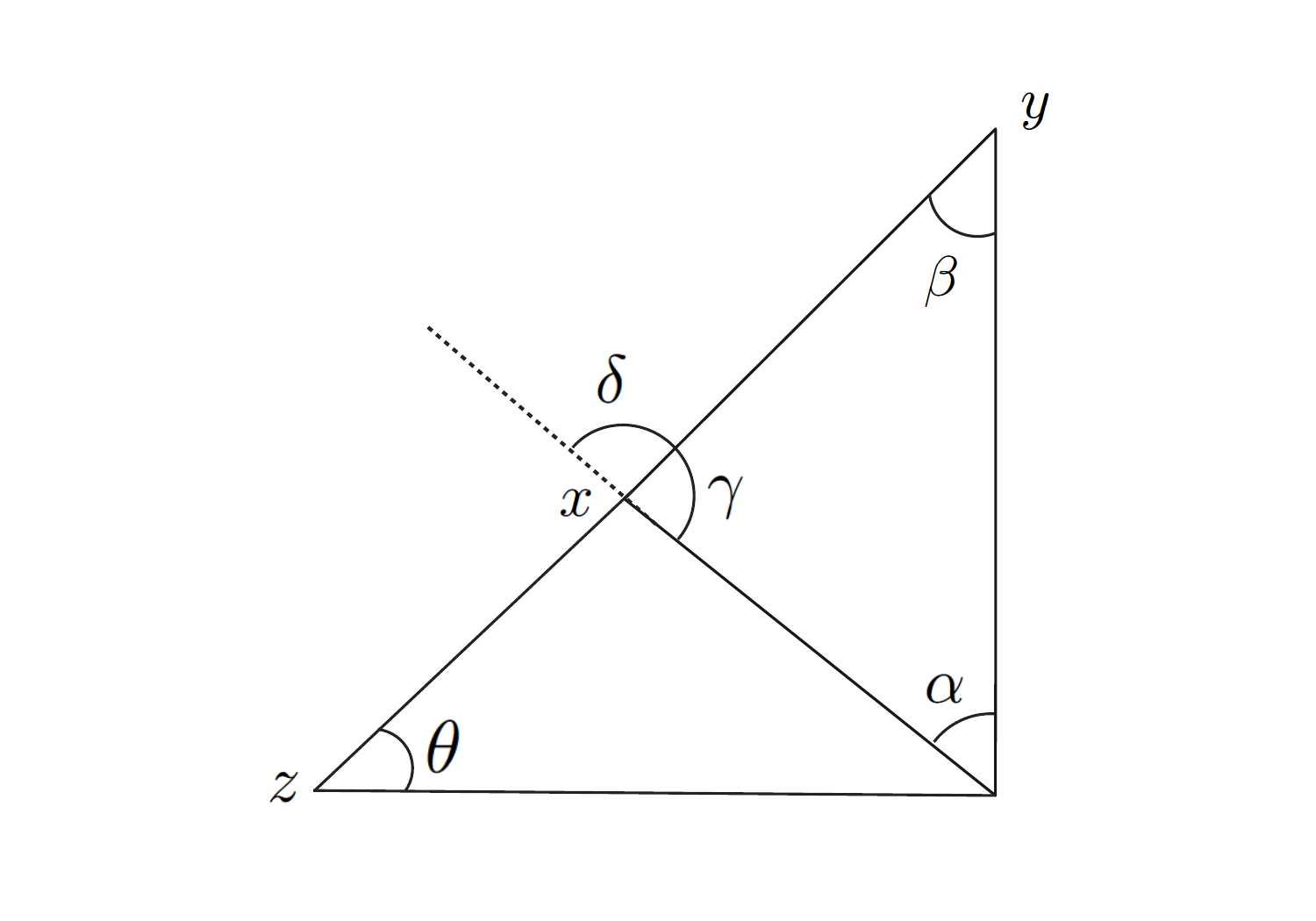}
    %\caption{Caption}
    \label{fig:lemma_62_2}
\end{figure}

Now elementary trigonometry tells us that $\gamma < \pi/2 + \theta$.\footnote{In particular, $\gamma = \pi - \alpha - \beta$, and $\beta = \pi/2 - \theta$. Thus $\gamma$ is maximised when $\alpha = 0$, in which case $\gamma = \pi - (\pi/2 - \theta) = \pi/2 + \theta$. Moreover, if $\alpha = 0$ then $x=y$, which by assumption is not the case. Hence $\gamma < \pi/2 + \theta$.} By deriving an upper bound for $\theta$, we thus obtain an upper bound for $\gamma$ (and hence a lower bound for $\delta$).

Note that $\theta = \arctan(L_2(y)/L_2(z))$.
Moreover, since $S$ is the boundary of some set $X$, and since the interior of $X$ is non-empty, there must be some $\ell > 0$ such that $L_2(x) \geq \ell$ for all $x \in S$. Moreover, since $X$ is bounded, there must be some $u \geq \ell$ such that $L_2(x) \leq u$ for all $x \in S$. We have that $L_2(y) \leq u$, since $y \in S$.

It may be that $z \not\in S$. However, since $S$ is the boundary of a \emph{convex} set, it must still be the case that $L_2(z) \geq \ell$. To see this, suppose $L_2(z) < \ell$, and let $z'$ be the point in $S$ such that $z' = a \cdot z$ for some $a \in \mathbb{R}^+$. $L_2(z') \geq \ell$, since $s' \in S$, and so $L_2(z') > L_2(z)$. Consider the triangle that lies between $z'$, $y$, and the origin. Since $S$ is the boundary of a convex set $X$, every point that lies in the interior of this triangle must lie in the interior of $X$. But if $L_2(z') > L_2(z)$, then $x$ lies in the interior of this triangle. This is a contradiction, since $x$ lies on the boundary of $X$. Thus $L_2(z) \geq \ell$.  

We thus have that $\theta \leq \arctan(u/\ell)$, which means that $\gamma < \pi/2 + \arctan(u/\ell)$, and thus that $\delta > \pi - (\pi/2 + \arctan(u/\ell)) = \pi/2 - \arctan(u/\ell)$. Since this value does not depend on $x$ or $y$, we have that the angle $\delta$ between $x$ and $y-x$ is at least $\pi/2 - \arctan(u/\ell)$ for all $x,y \in S$ such that $x \neq y$, and such that the angle $\alpha$ between $x$ and $y$ is less than or equal to $\pi/2$. Also recall that if $\alpha > \pi/2$ then $\delta > \pi/2$. Since $u/\ell > 0$, we have that $\pi/2 - \arctan(u/\ell) < \pi/2$, and so $\delta > \pi/2 - \arctan(u/\ell)$ for all $x,y \in S$ such that $x \neq y$. Finally, since $\arctan(x) < \pi/2$, we have that $\pi/2 - \arctan(u/\ell) > 0$. Setting $A = \pi/2 - \arctan(u/\ell)$ thus completes the proof.
\end{proof}

Using this, we can now show that we can get a lower bound on the \emph{angle} between two standardised reward functions in terms of their STARC-distance:

\begin{lemma}\label{lemma:lower_bound_lemma_1}
    For any STARC metric $d$, there exist an $\ell_1 \in \mathbb{R}^+$ such that the angle $\theta$ between $s(R_1)$ and $s(R_2)$ satisfies $\ell_1 \cdot d(R_1, R_2) \leq \theta$ for all $R_1$, $R_2$ for which neither $s(R_1)$ or $s(R_2)$ is $0$.
\end{lemma}

\ifshowproofs
\begin{proof}
Let $d$ be an arbitrary STARC-metric, and let $R_1$ and $R_2$ be two arbitrary reward functions for which neither $s(R_1)$ or $s(R_2)$ is $0$. Recall that $d(R_1, R_2) = m(s(R_1),s(R_2))$, where $m$ is a metric that is bilipschitz equivalent to some norm. Since all norms are bilipschitz equivalent on any finite-dimensional vector space, this means that $m$ is bilipschitz equivalent to the $L_2$-norm. Thus, there are positive constants $p$, $q$ such that 
$$
p \cdot m(s(R_1), s(R_2)) \leq L_2(s(R_1), s(R_2)) \leq q \cdot m(s(R_1), s(R_2)).
$$
In particular, the $L_2$-distance between $s(R_1)$ and $s(R_2)$ is at least $\epsilon = p \cdot d(R_1, R_2)$. For the rest of our proof, it will be convenient to assume that $\epsilon < L_2(s(R_1))$; this can be ensured by picking a $p$ that is sufficiently small.

Let us plot the plane which contains $s(R_1)$, $s(R_2)$, and the origin, and orient it so that $s(R_1)$ points straight up, and so that $s(R_2)$ is not on the left-hand side:

\begin{figure}[H]
    \centering
    \includegraphics[width=2\textwidth/3]{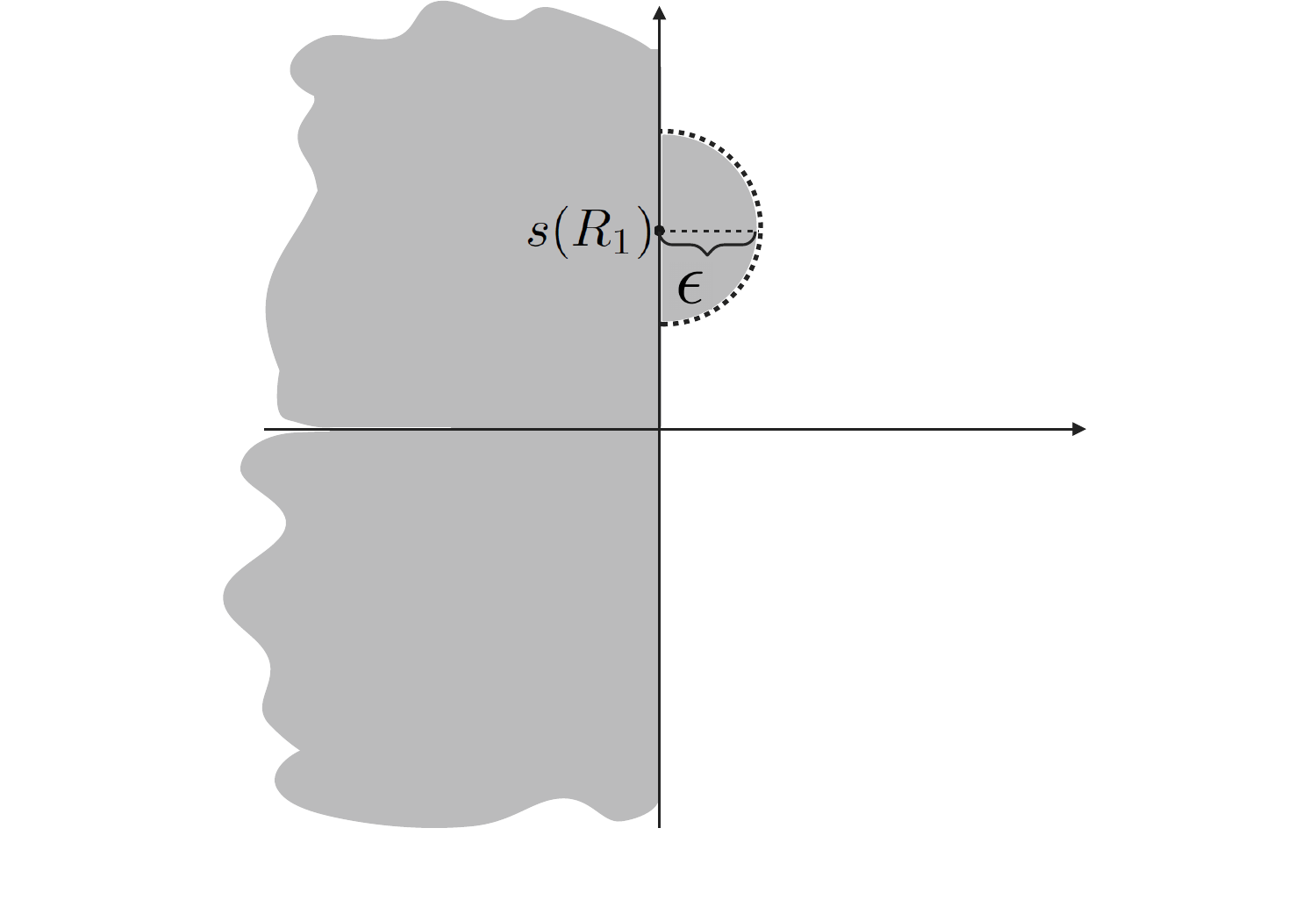}
    %\caption{Caption}
    \label{fig:lemma_64_1}
\end{figure}

Since the distance between $s(R_1)$ and $s(R_2)$ is at least $\epsilon$, and since $s(R_2)$ is not on the left-hand side, we know that $s(R_2)$ cannot be inside of the region shaded grey in the figure above (though it may be on the boundary). Moreover, as per Lemma~\ref{lemma:lower_bound_lemma_0}, we know that there is an $\alpha > 0$ (named $A$ in the statement of Lemma~\ref{lemma:lower_bound_lemma_0}) such that the angle between $s(R_1)$ and $s(R_2) - s(R_1)$ is at least $\alpha$. This means that we also can rule out the following region:

\begin{figure}[H]
    \centering
    \includegraphics[width=2\textwidth/3]{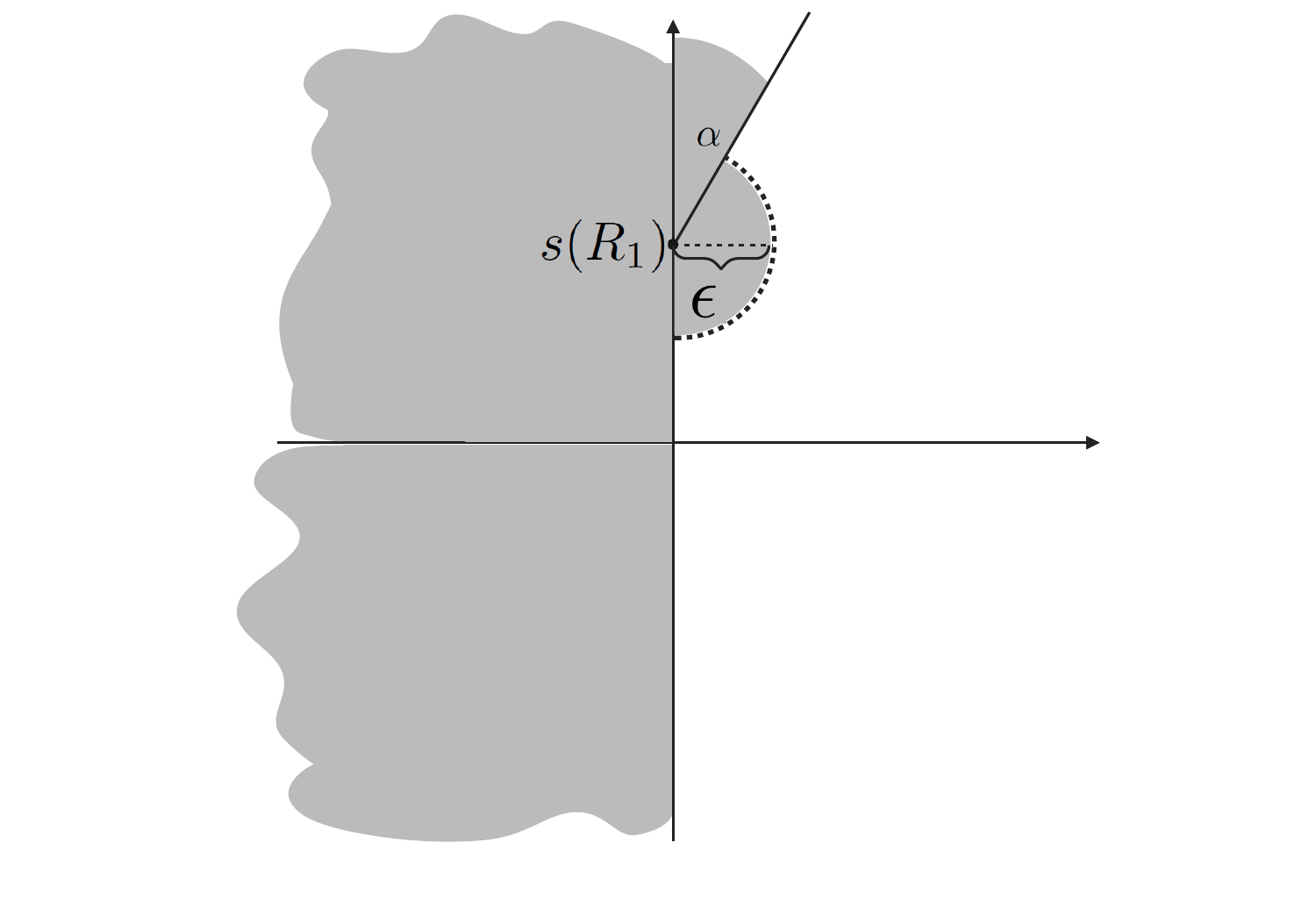}
    %\caption{Caption}
    \label{fig:lemma_64_2}
\end{figure}

Moreover, let $v$ be the element of $\mathrm{Im}(s)$ that is perpendicular to $s(R_1)$, lies on a plane with $s(R_1)$, $s(R_2)$, and the origin, and points in the same direction as $s(R_2)$ within this plane. Since $\mathrm{Im}(s)$ is the boundary of a convex set, we know that $s(R_2)$ cannot lie within the triangle formed by the $x$-axis, the $y$-axis, and the line between $s(R_1)$ and $v$:

\begin{figure}[H]
    \centering
    \includegraphics[width=2\textwidth/3]{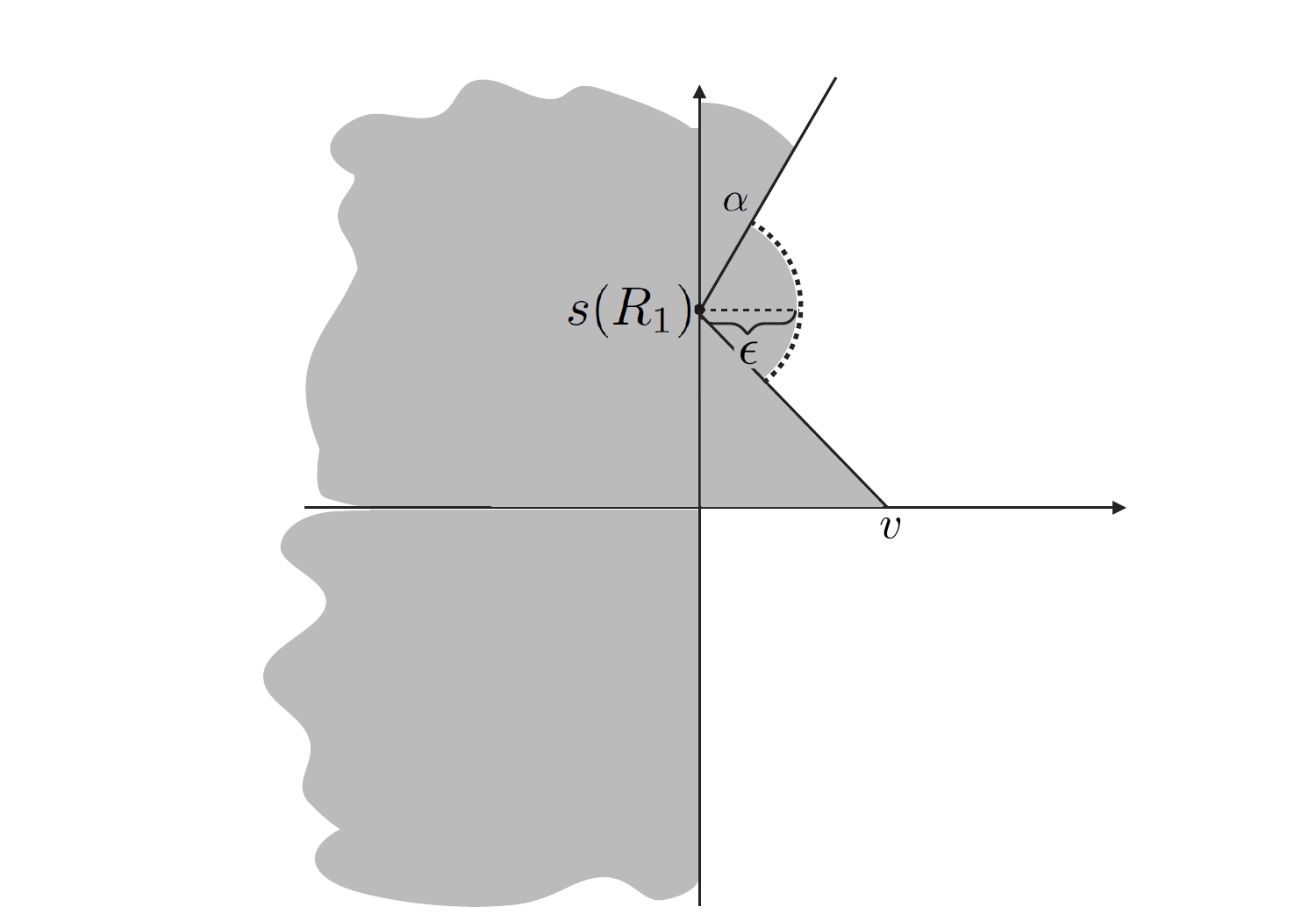}
    %\caption{Caption}
    \label{fig:lemma_64_3}
\end{figure}

Since $\mathrm{Im}(s)$ is compact, we know that there is a vector $a$ in $\mathrm{Im}(s)$ whose $L_2$-norm is bigger than all other vectors in $\mathrm{Im}(s)$, and a (non-zero) vector $b$ in $\mathrm{Im}(s)$ whose $L_2$-norm is smaller than all other (non-zero) vectors in $\mathrm{Im}(s)$, by the extreme value theorem. From this, we can infer that the angle between $s(R_1)$ and $v - s(R_1)$ is at least $\beta = \mathrm{arctan}(b/a)$. Also note that $\beta > 0$.

We now have everything we need to derive a lower bound on the angle $\theta$ between $s(R_1)$ and $s(R_2)$. First note that this angle can be no greater than the angle between $s(R_1)$ and the points marked $A$ and $B$ in the figure below (whichever is smaller):

\begin{figure}[H]
    \centering
    \includegraphics[width=2\textwidth/3]{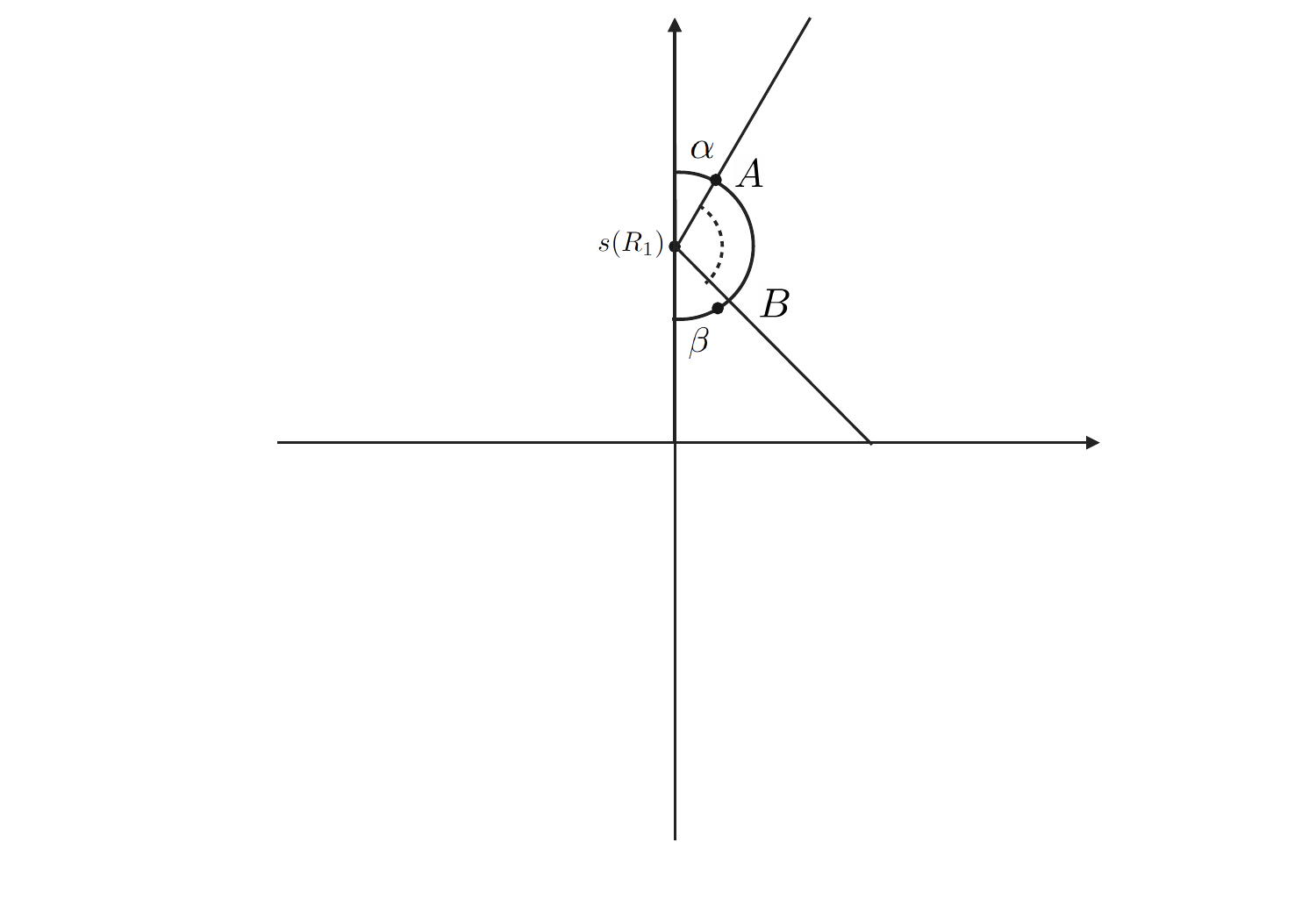}
    %\caption{Caption}
    \label{fig:lemma_64_4}
\end{figure}

To make things easier, replace both $\alpha$ and $\beta$ with $\gamma = \mathrm{min}(\alpha, \beta)$. Since this makes the shaded region smaller, we still have that $s(R_2)$ cannot be in the interior of the new shaded region. Moreover, in this case, we know that the angle between $s(R_1)$ and $s(R_2)$ is no smaller than the angle $\theta'$ between $s(R_1)$ and the point marked $A$:

\begin{figure}[H]
    \centering
    \includegraphics[width=2\textwidth/3]{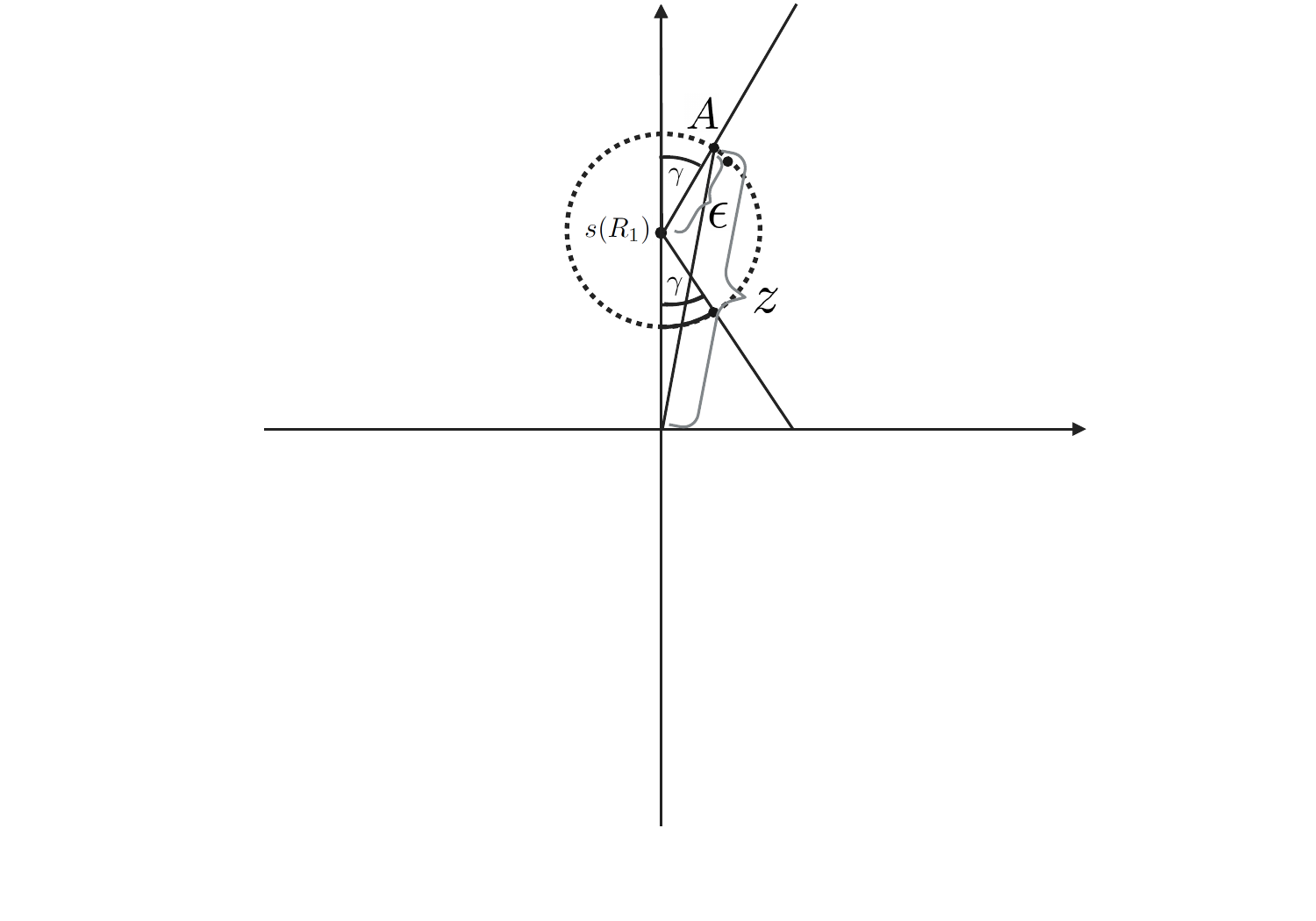}
    %\caption{Caption}
    \label{fig:lemma_64_5}
\end{figure}

Deriving this angle is now just a matter of trigonometry. Letting $z$ denote $L_2(A)$, we have that:
$$
\frac{\epsilon}{\sin(x)} = \frac{z}{\sin(\pi - \gamma)} = \frac{z}{\sin(\gamma)}
$$
From this, we get that 
\begin{align*}
    \theta' &= \arcsin\left(\left(\frac{\epsilon}{z}\right) \sin(\gamma)\right)\\
            &\geq \left(\frac{\epsilon}{z}\right) \sin(\gamma)
\end{align*}
Moreover, it is also straightforward to find an upper bound $z'$ for $z$. Specifically, we have that $z^2 = L_2(s(R_1))^2 + \epsilon^2 - 2L_2(s(R_1))\epsilon \cos(\pi-\gamma)$. Since $\epsilon < L_2(s(R_1))$, this means that 
$$
z < \sqrt{2L_2(s(R_1))^2 - 2L_2(s(R_1))^2 \cos(\pi-\gamma)}.
$$
Moreover, since $\mathrm{Im}(s)$ is compact, there is a vector $a$ in $\mathrm{Im}(s)$ whose $L_2$-norm is bigger than all other vectors in $\mathrm{Im}(s)$. We thus know that 
$$
z < z' = \sqrt{2L_2(a)^2 - 2L_2(a)^2 \cos(\pi-\gamma)}.
$$
Putting this together, we have that 
$$
\theta \geq \theta' \geq \left(\frac{\epsilon}{z'}\right) \sin(\gamma) = m(s(R_1), s(R_2)) \cdot p \cdot \left(\frac{\sin(\gamma)}{z'}\right).
$$
Setting $\ell_1 = p \cdot \left(\frac{\sin(\gamma)}{z'}\right)$ thus completes the proof.
\end{proof}
\fi

Finally, before we can give the full proof, we will also need the following:

\begin{lemma}\label{lemma:lower_bound_lemma_2}
For any invertible matrix $M : \mathbb{R}^n \to \mathbb{R}^n$ there is an $\ell_2 \in (0, 1]$ such that for any $v, w \in \mathbb{R}^n$, the angle $\theta'$ between $M v$ and $M w$ satisfies $\theta' \geq \ell_2 \cdot \theta$, where $\theta$ is the angle between $v$ and $w$.
\end{lemma}

\ifshowproofs
\begin{proof}
We will first prove that this holds in the $2$-dimensional case, and then extend this proof to the general $n$-dimensional case.

Let $M$ be an arbitrary invertible matrix $\mathbb{R}^2 \to \mathbb{R}^2$.
First note that we can factor $M$ via Singular Value Decomposition into three matrices $U$, $\Sigma$, $V$, such that $M = U \Sigma V^\top$, where $U$ and $V$ are orthogonal matrices, and $\Sigma$ is a diagonal matrix with non-negative real numbers on the diagonal. 
Since $M$ is invertible, we also have that $\Sigma$ cannot have any zeroes along its diagonal.
Next, recall that orthogonal matrices preserve angles.
%we have that, for any vectors $v, w \in \mathbb{R}^2$, the angle between $Mv$ and $Mw$ is equal to the angle between $\Sigma V^\top v$ and $\Sigma V^\top w$, and the angle between $V^\top v$ and $V^\top w$ is equal to the angle between $v$ and $w$. This means that there are $v, w \in \mathbb{R}^2$ such that $v, w$ have angle $\theta$ and $Mv$ and $Mw$ have angle $\theta'$, if and only if there are $v', w' \in \mathbb{R}^2$ such that $v', w'$ have angle $\theta$ and $\Sigma v'$ and $\Sigma w'$ have angle $\theta'$ (with $v' + V^\top v$ and $w' + V^\top w$). This means that 
This means that we can restrict our focus to just $\Sigma$.\footnote{If there are vectors $x,y$ such that the angle between $x$ and $y$ is $\theta$ and the angle between $Mx$ and $My$ is $\theta'$, then there are vectors $v,w$ such that the angle between $x$ and $y$ is $\theta$ and the angle between $\Sigma v$ and $\Sigma w$ is $\theta'$, and vice versa.}

Let $\alpha$ and $\beta$ be the singular values of $M$. We may assume, without loss of generality, that 
$$
\Sigma = \begin{pmatrix}\alpha & 0\\0 & \beta\end{pmatrix}.
$$ 
Moreover, since scaling the $x$ and $y$-axes uniformly will not affect the angle between any vectors after multiplication, we can instead equivalently consider the matrix 
$$
\Sigma = \begin{pmatrix}\alpha/\beta & 0\\0 & 1\end{pmatrix}.
$$
Let $v, w \in \mathbb{R}^2$ be two arbitrary vectors with angle $\theta$, and let $\theta'$ be the angle between $\Sigma v$ and $\Sigma w$. We will derive a lower bound on $\theta'$ expressed in terms of $\theta$. Moreover, since the angle between $v$ and $w$ is not affected by their magnitude, we will assume (without loss of generality) that both $v$ and $w$ have length 1 (under the $L_2$-norm).

First, note that if $\theta = \pi$ then $v = -w$. This means that $\Sigma v = -\Sigma w$, since $\Sigma$ is a linear transformation, which in turn means that $\theta' = \pi$. Thus $\theta' \geq \ell_2 \cdot \theta$ as long as $\ell_2 \leq 1$. Next, assume that $\theta < \pi$.

We may assume (without loss of generality) that the angle between $v$ and the $x$-axis is no bigger than the angle between $w$ and the $x$-axis. 
Let $\phi$ be the angle between the $x$-axis and the vector that is in the middle between $v$ and $w$. 
This means that we can express $v$ as $(\cos(\phi - \theta/2), \sin(\phi - \theta/2))$ and $w$ as $(\cos(\phi + \theta/2), \sin(\phi + \theta/2))$. Moreover, since reflection along either of the axes will not change the angle between either $v$ and $w$ or $\Sigma v$ and $\Sigma w$, we may assume (without loss of generality) that $\phi \in [0, \pi/2]$. For convenience, we will also let $\sigma = \alpha/\beta$.

\begin{figure}[H]
    \centering
    \includegraphics[width=3\textwidth/4]{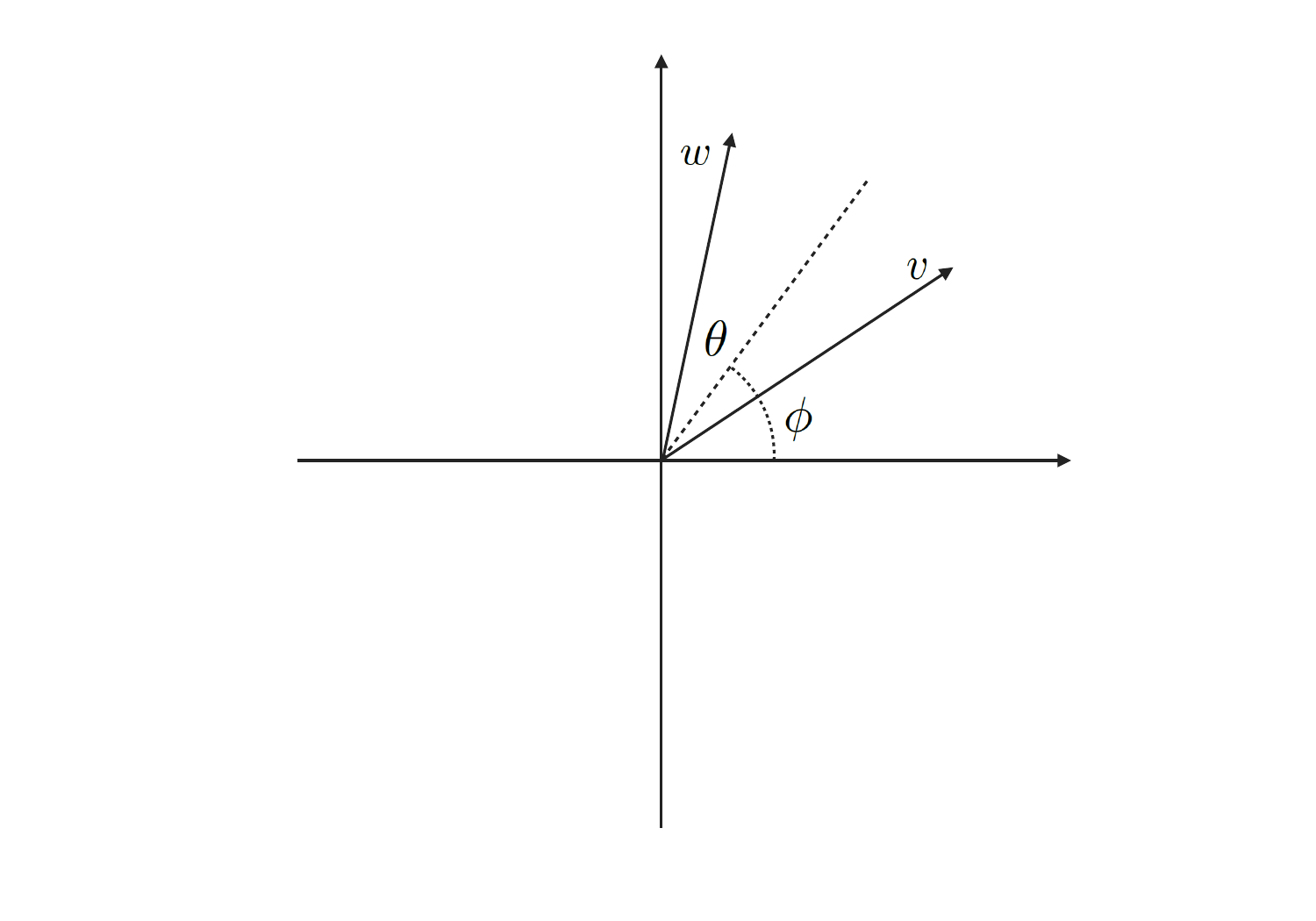}
    %\caption{Caption}
    \label{fig:lemma_65_1}
\end{figure}

(Note that we can visualise the action of $\Sigma$ as scaling the $x$-axis in the figure above by $\sigma$.)

We now have that $\Sigma v = (\sigma \cos(\phi - \theta/2), \sin(\phi - \theta/2))$ and $\Sigma w = (\sigma \cos(\phi + \theta/2), \sin(\phi + \theta/2))$. 
Using the dot product, we get that $\cos(\theta')$ equals
$$
\frac{\sigma^2 \cos(\phi - \theta/2)\cos(\phi + \theta/2) + \sin(\phi - \theta/2)\sin(\phi + \theta/2)}{\sqrt{\sigma^2 \cos^2(\phi - \theta/2) + \sin^2(\phi - \theta/2)}\sqrt{\sigma^2 \cos^2(\phi + \theta/2) + \sin^2(\phi + \theta/2)}}.
$$
We next note that if $\theta \in [0,\pi)$ and $\phi \in [0, \pi/2]$, then the derivative of this expression with respect to $\phi$ can only be 0 when $\phi \in \{0, \pi/2\}$.\footnote{For example, this may be verified using tools such as Wolfram Alpha.} This means that $\cos(\theta')$ must be maximised or minimised when $\phi$ is either $0$ or $\pi/2$, which in turn means that the angle $\theta'$ must be minimised or maximised when $\phi$ is either $0$ or $\pi/2$.

It is now easy to see that if $\sigma >1$ then $\theta'$ is minimised when $\phi = 0$, and that if $\sigma < 1$ then $\theta'$ is minimised when $\phi = \pi/2$.
Moreover, if $\phi = \pi/2$, then 
$$
\theta' = 2 \arctan \left(\frac{\sigma \cos(\pi/2 -\theta/2)}{\sin(\pi/2 -\theta/2)}\right) = 2 \arctan \left(\sigma \tan(\theta/2)\right),
$$
which in turn is greater than $\theta \cdot \sigma$ when $\sigma < 1$.\footnote{To see this, let $x = \tan(\theta/2)$. Now $2 \arctan \left(\sigma \tan(\theta/2)\right) > \sigma \cdot \theta$ for all $\theta \in [0,\pi)$ if and only if $\arctan \left(\sigma x\right) > \sigma \cdot \arctan(x)$ for all $x \geq 0$. This is true, since $\arctan$ is strictly concave on $[0,\infty)$.}
Similarly, if $\phi = 0$, then 
$$
\theta' = 2 \arctan \left(\frac{\sin(\theta/2)}{\sigma \cos(\theta/2)}\right) = 2 \arctan \left(\sigma^{-1} \tan(\theta/2)\right),
$$
which is in turn greater than $\sigma^{-1} \cdot \theta$ when $\sigma > 1$. In either case, we thus have
$$
\theta' \geq \theta \cdot \min(\sigma, \sigma^{-1}) = \theta \cdot \min(\beta/\alpha, \alpha/\beta).
$$
We have therefore show that, for any invertible matrix $M : \mathbb{R}^2 \to \mathbb{R}^2$, there exists a positive constant $\min(\beta/\alpha, \alpha/\beta)$, where $\alpha$ and $\beta$ are the singular values of $M$, such that if $v, w \in \mathbb{R}^2$ have angle $\theta$, then the angle between $Mv$ and $Mw$ is at least $\theta \cdot \min(\beta/\alpha, \alpha/\beta)$.

To generalise this to the general $n$-dimensional case, let $v, w \in \mathbb{R}^n$ be two arbitrary vectors. Consider the 2-dimensional linear subspace given by $S = \mathrm{span}(v, w)$, and note that $M(S)$ also is a 2-dimensional linear subspace of $\mathbb{R}^n$ (since $M$ is linear and invertible).
The linear transformation which $M$ induces between $S$ and $M(S)$ is isomorphic to a linear transformation $M' : \mathbb{R}^2 \to \mathbb{R}^2$.\footnote{To see this, let $A$ be an orthonormal matrix that rotates $\mathbb{R}^2$ to align with $S$, and let $B$ be an orthonormal matrix that rotates $M(S)$ to align with $\mathbb{R}^2$. Now $M' = BMA$ is an invertible linear transformation $\mathbb{R}^2 \to \mathbb{R}^2$. Moreover, since orthonormal matrices preserve the angles between vectors, we have that $v, w \in S$ have angle $\theta$ and $Mv, Mw \in M(S)$ have angle $\theta'$, if and only if $A^{-1}v, A^{-1}w \in \mathbb{R}^2$ have angle $\theta$ and $BMv, BMw \in \mathbb{R}^2$ have angle $\theta'$. Note that $M'A^{-1}v = BMv$ and $M'A^{-1}w = BMw$. This means that there are $v, w \in S$ such that $v, w$ have angle $\theta$ and $Mv, Mw$ have angle $\theta'$, if and only if there are $v', w' \in \mathbb{R}^2$ such that $v', w'$ have angle $\theta$ and $M'v'$ and $M'w'$ have angle $\theta'$ (with $v' = A^{-1}v$ and $w' = A^{-1}w$).}
We can thus apply our previous result for the two-dimensional case, and conclude that if the angle between $v$ and $w$ is $\theta$, then the angle between $Mv$ and $Mw$ is at least $\theta \cdot \min(\beta/\alpha, \alpha/\beta)$, where $\alpha$ and $\beta$ are the singular values of $M'$. Next, note that the singular values of $M'$ cannot be smaller than the smallest singular values of $M$ or bigger than the biggest singular values of $M$. We can therefore let $\ell_2 = \alpha/\beta$, where $\alpha$ is the smallest singular value of $M$ and $\beta$ is the greatest singular value of $M$, and conclude that the angle between $Mv$ and $Mw$ must be at least $\ell_2 \cdot \theta$. Since the value of $\ell_2$ does not depend on $v$ or $w$, this completes the proof.
\end{proof}
\fi

\STARCcomplete*

\begin{proof}
Let $d$ be an arbitrary STARC metric. We need to show that there exists a positive constant $L$ such that, for any reward functions $R_1$ and $R_2$, there are two policies $\pi_1$, $\pi_2$ with $J_2(\pi_2) \geq J_2(\pi_1)$ and
$$
J_1(\pi_1) - J_1(\pi_2) \geq L \cdot (\max_\pi J_1(\pi) - \min_\pi J_1(\pi)) \cdot d(R_1, R_2),
$$
and if both $R_1$ and $R_2$ are trivial, then we have that $d(R_1, R_2) = 0$.
%$\max_\pi J_1(\pi) - \min_\pi J_1(\pi) = 0$ and $\max_\pi J_2(\pi) - \min_\pi J_2(\pi) = 0$, then we have that $d(R_1, R_2) = 0$.

Let $c$ be the canonicalisation function of $d$, and let $s : \R \to \R$ be the function such that $s(R) = c(R)/n(c(R))$ if $n(c(R)) \neq 0$, and $c(R)$ otherwise, where $n$ is the norm used in the normalisation step of $c$.

First note that the last condition holds straightforwardly. %If $\max_\pi J_1(\pi) - \min_\pi J_1(\pi) = 0$ and $\max_\pi J_2(\pi) - \min_\pi J_2(\pi) = 0$, 
If both $R_1$ and $R_2$ are trivial, then $c(R_1) = c(R_2) = R_0$, which implies that $d(R_1, R_2) = 0$. %This condition is therefore satisfied. 

For the first condition, let us first consider the case when $R_1$ is trivial (and $R_2$ may be trivial or non-trivial). In this case the left-hand side is $0$ for all $\pi_1$ and $\pi_2$. Moreover, $\max_\pi J_1(\pi) - \min_\pi J_1(\pi) = 0$, and so the right-hand side is also $0$ (for any value of $L$). In this case, the inequality is therefore satisfied for any $L$.

%then the statement holds trivially for any non-negative $L$ (since the other two terms on the right-hand side of the inequality are strictly non-negative). 

Let us next consider the case where $R_2$ is trivial, but where $R_1$ is not. In this case, $J_2(\pi_2) \geq J_2(\pi_1)$ for all $\pi_1$ and $\pi_2$, which means that 
$$
\max_{\pi_1, \pi_2 : J_2(\pi_2) \geq J_2(\pi_1)} J_1(\pi_1) - J_1(\pi_2) = \max_\pi J_1(\pi) - \min_\pi J_1(\pi).
$$
Therefore, the inequality is satisfied as long as we pick an $L$ such that $L \cdot d(R_1, R_2) \leq 1$ for all $R_1$ and all trivial $R_2$. In other words, we need that $L \leq 1/\max_R d(R,R_0)$. This can be ensured by picking an $L$ that is sufficiently small (noting that any STARC metric $d$ is bounded).

Finally, let us consider the case where neither $R_1$ or $R_2$ is trivial, i.e., the case where $\max_\pi J_1(\pi) - \min_\pi J_1(\pi) > 0$ and $\max_\pi J_2(\pi) - \min_\pi J_2(\pi) > 0$. 
Let $m : \Pi \to \mathbb{R}^{|\States||\Actions||\States|}$ be the function that takes a policy $\pi$ and returns its occupancy measure $\eta^\pi$, and let $\Omega = \mathrm{Im}(m)$. Note that $m$ implicitly depends on $\tfunc$ and $\gamma$.
We will use $d$ to derive a lower bound on the angle $\theta$ between the level sets of $J_1$ and $J_2$ in $\Omega$. We will then show that $\Omega$ contains an open set with a certain diameter. From this, we can find two policies that incur a certain amount of regret.

%Let $s$ be the standardisation function of $d$. $R$ and $s(R)$ differ by $S'$-redistribution, potential shaping, and positive linear scaling, and none of these transformations change the policy order of $R$ under $\tau$ and $\init$ (Proposition~\ref{prop:policy_order}). This means that the level sets of $R_1$ are the same as those of $s(R_1)$ in $\mathrm{Im}(m)$, and likewise for $R_2$ and $s(R_2)$. Therefore, the angle between the level sets of $R_1$ and $R_2$ in $\mathrm{Im}(m)$ is the same as the angle between the level sets of $s(R_1)$ and $s(R_2)$.

%Note that $d(R_1, R_2) = d(s(R_1), s(R_2))$. 
%We can therefore apply Lemma~\ref{lemma:lower_bound_lemma_1}, and conclude that there exists an $\ell_1$, such that the angle between $s(R_1)$ and $s(R_2)$ is at least $\ell_1 \cdot d(R_1, R_2)$, and where $\ell_1 \cdot d(R_1, R_2)$ is also at most $\pi/2$.

First, by Lemma~\ref{lemma:lower_bound_lemma_1}, there exists an $\ell_1$ such that for any non-trivial $R_1$ and $R_2$, the angle between $s(R_1)$ and $s(R_2)$ is at least $\ell_1 \cdot d(R_1, R_2)$. To make our proof easier, we will assume that we pick an $\ell_1$ that is small enough to ensure that $\ell_1 \cdot d(R_1, R_2) \leq \pi/2$ for all $R_1, R_2$. Since $d$ is bounded, this is possible.

Note that $s(R_1)$ and $s(R_2)$ may not be parallel with $\Omega$, which means that the angle between $s(R_1)$ and $s(R_2)$ may not be the same as the angle between the level sets of $s(R_1)$ and $s(R_2)$ in $\Omega$.
Therefore, consider the matrix $M$ that projects $\R$ onto the linear subspace of $\R$ that is parallel to $\Omega$, where $c$ is the canonicalisation function of $d$.
Now the angle between $M s(R_1)$ and $M s(R_2)$ is the same as the angle between the level sets of the linear functions which $J_1$ and $J_2$ induce on $\Omega$.
Moreover, recall that $M$ is a canonicalisation function (Proposition~\ref{prop:occupancy_measure_projection_is_canonicalisation}). This means that the elements of $\mathrm{Im}(M)$ and $\mathrm{Im}(c)$ can be put in a one-to-one correspondence, and so $M$ is invertible when viewed as a function from $\mathrm{Im}(c)$. Also note that $s(R_1), s(R_2) \in \mathrm{Im}(c)$.
%Note that $M$ is invertible, since any two reward functions in $\mathrm{Im}(c)$ induce distinct policy evaluation functions (Theorem~\ref{thm:policy_ordering}).
We can therefore apply Lemma~\ref{lemma:lower_bound_lemma_1}, and conclude that there exists an $\ell_2 \in (0,1]$, such that the angle $\theta$ between the normal vectors (and hence the level sets) of $s(R_1)$ and $s(R_2)$ in $\Omega$ is at least $\ell_2 \cdot \ell_1 \cdot d(R_1, R_2)$.
Moreover, since $\ell_1 \cdot d(R_1, R_2) \leq \pi/2$, and since $\ell_2 \leq 1$, we have that $\ell_2 \cdot \ell_1 \cdot d(R_1, R_2) \leq \pi/2$.

This gives us that, for any two policies $\pi_1, \pi_2$, we have:
\begin{align*}
    J_1(\pi_1) - J_1(\pi_2) &= J_1^C(\pi_1) - J_1^C(\pi_2)\\
    &=c(R_1) \cdot m(\pi_1) - c(R_1) \cdot m(\pi_2)\\
    &= c(R_1) \cdot (m(\pi_1) - m(\pi_2))\\
    &= M(c(R_1)) \cdot (m(\pi_1) - m(\pi_2))\\
    &= L_2(M(c(R_1))) \cdot L_2(m(\pi_1) - m(\pi_2)) \cdot \cos(\phi)
\end{align*}
where $\phi$ is the angle between $M(c(R_1))$ and $m(\pi_1) - m(\pi_2)$, and $J_1^C$ is the evaluation function of $c(R_1)$. Note that the first and fourth line follow from Proposition~\ref{prop:change_from_potentials}. We can thus derive a lower bound on worst-case regret by deriving a lower bound for the greatest value of this expression.

By Lemma~\ref{lemma:homeomorphism}, we have that $\Omega$ contains a set that is open in the smallest affine space which contains $\Omega$. 
This means that there is an $\epsilon$, such that $\Omega$ contains a sphere of diameter $\epsilon$. 
We will show that we always can find two policies within this sphere that incur a certain amount of regret.
Consider the 2-dimensional cut which goes through the middle of this sphere and is parallel with the normal vectors of the level sets of $J_1$ and $J_2$ in $\Omega$. The intersection between this cut and the $\epsilon$-sphere forms a 2-dimensional circle with diameter $\epsilon$.
Let $\pi_1, \pi_2$ be the two policies for which $m(\pi_1)$ and $m(\pi_2)$ lie opposite to each other on this circle, and satisfy that $J_2(\pi_1) = J_2(\pi_2)$ (or, equivalently, that $Mc(R_1) \cdot m(\pi_1) = Mc(R_1) \cdot m(\pi_2)$). Without loss of generality, we may assume that $J_1(\pi_1) \geq J_1(\pi_2)$.

\begin{figure}[H]
    \centering
    \includegraphics[width=2\textwidth/3]{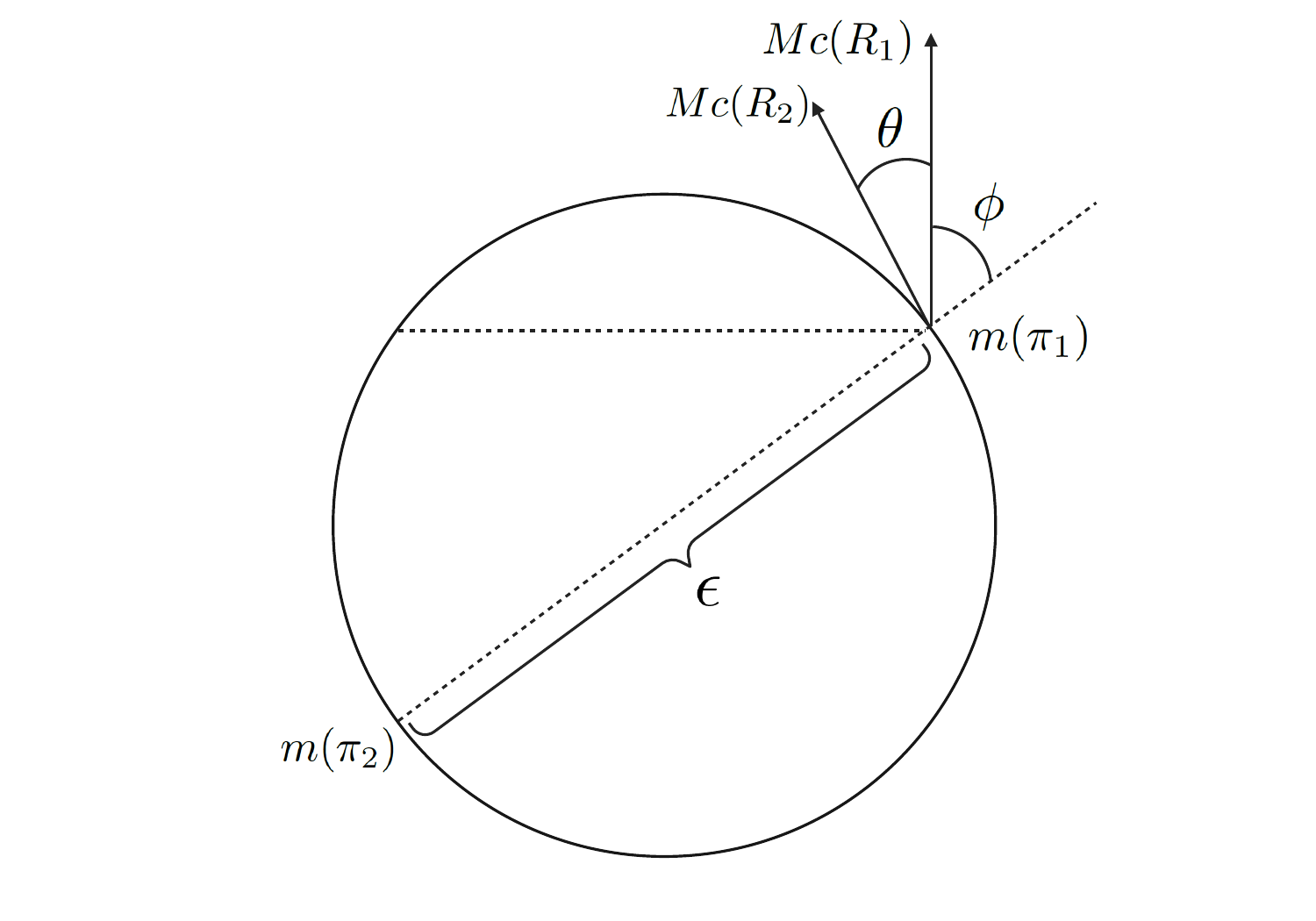}
    %\caption{Caption}
    \label{fig:thm_66_1}
\end{figure}

Now note that $L_2(m(\pi_1) - m(\pi_2)) = \epsilon$. Moreover, recall that the angle $\theta$ between $Mc(R_1)$ and $Mc(R_2)$ is at least $\theta' = \ell_1 \cdot \ell_2 \cdot d(R_1, R_2)$, and that this quantity is at most $\pi/2$. This means that the angle $\phi$ is at most $\pi/2-\theta'$, and so $\cos(\phi)$ is at least $\cos(\pi/2 - \theta') = \cos(\pi/2 - \ell_1 \cdot \ell_2 \cdot d(R_1, R_2))$.
This means that we have two policies $\pi_1, \pi_2$ where $J_2(\pi_2) = J_2(\pi_1)$ and such that 
\begin{align*}
    J_1(\pi_1) - J_1(\pi_2) &= L_2(M(c(R_1))) \cdot L_2(m(\pi_1) - m(\pi_2)) \cdot \cos(\phi)\\
    &\geq L_2(M(c(R_1))) \cdot \epsilon \cdot \cos(\pi/2 - \ell_2 \cdot \ell_1 \cdot d(R_1, R_2))\\
    &= L_2(M(c(R_1))) \cdot \epsilon \cdot \sin(\ell_2 \cdot \ell_1 \cdot d(R_1, R_2)).
\end{align*}
Note that $\sin(x) \geq x \cdot 2/\pi$ when $x \leq \pi/2$, and that $\ell_2 \cdot \ell_1 \cdot d(R_1, R_2) \leq \pi/2$.
Putting this together, we have that there exists $\pi_1$, $\pi_2$ with $J_2(\pi_2) = J_2(\pi_1)$ such that
$$
J_1(\pi_1) - J_1(\pi_2) \geq L_2(M(c(R_1))) \cdot \left( \frac{\epsilon \cdot \ell_1 \cdot \ell_2 \cdot 2}{\pi} \right) \cdot d(R_1, R_2).
$$
Next, note that, if $p$ is a norm and $M$ is an invertible matrix, then $p \circ M$ is also a norm. Furthermore, recall that $\max_\pi J_1(\pi) - \min_\pi J_1(\pi)$ is a norm on $\mathrm{Im}(c)$, when $c$ is a canonicalisation function (Proposition~\ref{prop:J_norm}). 
Since all norms are equivalent on a finite-dimensional vector space, this means that there must exist a positive constant $\ell_3$ such that 
$$
L_2(M(c(R_1))) \geq \ell_3 \cdot (\max_\pi J_1^C(\pi) - \min_\pi J_1^C(\pi)) = \ell_3 \cdot (\max_\pi J_1(\pi) - \min_\pi J_1(\pi)).
$$
We can therefore set $L \leq \left( \epsilon \cdot \ell_1 \cdot \ell_2 \cdot \ell_3 \cdot 2 / \pi \right)$, and obtain the result we want:
$$
J_1(\pi_1) - J_1(\pi_2) \geq L \cdot (\max_\pi J_1(\pi) - \min_\pi J_1(\pi)) \cdot d(R_1, R_2).
$$
This completes the proof.
\end{proof}

\subsection{Partial Identifiability}

In this section, we provide the proofs of our results concerning partial identifiability. In the main text, these results are presented in Section~\ref{sec:partial_identifiability}.

\subsubsection{Invariances of Intermediate Objects}\label{appendix:intermediate_invariances}

Many types of policies can be computed via some intermediate objects. For example, the Boltzmann-rational policy can be computed by first computing the optimal advantage function $A^\star$, and then applying the softmax function.
Moreover, recall that for any two reward objects $f : \mathcal{R} \to X$ and $g : \mathcal{R} \to Y$, if there exists a function $h : X \to Y$ such that $h \circ f = g$, then $\mathrm{Am}(f) \refines \mathrm{Am}(g)$ (Lemma~\ref{lemma:ambiguity_inherited}).
This means that if $g(R)$ can be computed by first computing some intermediate object $f(R)$, then $g$ inherits all of the invariances of $f$. For example, $b_{\tfunc,\init,\beta}$ inherits the invariances of $A^\star$. For this reason, it will be useful to catalogue the invariances of a number of such objects, which we will do in this section.

We begin by deriving the invariances of different forms of $Q$-functions:

\begin{lemma}\label{lemma:ambiguity_Q_function}
For any transition function $\tfunc$, any discount factor $\gamma$, and any policy $\pi$, the $Q$-function $Q^\pi$ determines $R$ up to $\SR$.
\end{lemma}
\ifshowproofs\begin{proof}
Recall that $\Q^\pi$ is the only function which satisfies the Bellman equation
(Equation~\ref{equation:Q_pi_recursion}) for all $s\in\States$ and $a\in\Actions$:
$$
Q^\pi(s,a)
=
\Expect{
    S' \sim \tfunc(s,a),
    A' \sim \policy(S')
}
{
    R(s,a,S') + \gamma \cdot Q^\pi(S', A')
}.
$$
This equation can be rewritten as
$$
\Expect
    {S' \sim \tfunc(s,a)}
    {R(s,a,S')}
=
Q^\pi(s,a) - \gamma \cdot \Expect
    {S' \sim \tfunc(s,a), A' \sim \pi(S')} {Q^\pi(S', A')}.
$$
Since $Q^\pi$ is the only function which satisfies this equation for all
$s \in \States, a \in \Actions$, we have that the values of the left-hand
side for each $s \in \States, a \in \Actions$ together determine
$Q^\pi$, and vice versa. Since the left-hand side values are
preserved by $S'$-redistribution of $R$, and no other transformations,
we have that $Q^\pi$ is preserved by $S'$-redistribution of $R$, and no
other transformations.
\end{proof}\fi

\begin{lemma}\label{lemma:ambiguity_optimal_Q_function}
For any transition function $\tfunc$ and any discount factor $\gamma$, the optimal $Q$-function $Q^\star$ determines $R$ up to $\SR$.
\end{lemma}
\ifshowproofs\begin{proof}
Analogous to Lemma~\ref{lemma:ambiguity_Q_function}, noting that $Q^\star$ is the only function which satisfies the Bellman optimality equation for all $s\in\States$ and $a\in\Actions$:
$$
Q^\star(s,a)
=
\Expect{
    S' \sim \tfunc(s,a),
    A' \sim \policy(S')
}
{
    R(s,a,S') + \gamma \cdot \max_{a' \in \Actions}Q^\pi(S', a')
}.
$$
\end{proof}\fi

\begin{lemma}\label{lemma:ambiguity_soft_Q_function}
For any transition function $\tfunc$, any discount $\gamma$, and any weight $\alpha$, the soft $Q$-function $\QSoft$ determines $R$ up to $\SR$.
\end{lemma}
\ifshowproofs\begin{proof}
Analogous to Lemma~\ref{lemma:ambiguity_Q_function}, noting that $\QSoft$ is the only function which satisfies the following modified Bellman equation for all $s\in\States$ and $a\in\Actions$:
$$
\QSoft(s,a) = 
    \Expect{S' \sim \tfunc(s,a)}
    {
        R(s,a,S') + \gamma \alpha \log \sum_{a' \in \Actions} \exp \left(\frac{1}{\alpha}\right)\QSoft (S', a')
    }\,.
$$
\end{proof}\fi

By Lemma~\ref{lemma:ambiguity_inherited}, this means that any type of policy which can be derived from one of these three $Q$-functions will be invariant to $S'$-redistribution. We next derive the invariances of different forms of advantage functions:

\begin{lemma}\label{lemma:ambiguity_A_function}
For any transition function $\tfunc$, any discount $\gamma$, and any policy $\policy$, the advantage function $A^\pi$ determines $R$ up to $\PS \bigodot \SR$.
\end{lemma}
\ifshowproofs\begin{proof}
First, recall that $A^\pi$ can be derived from $Q^\pi$, via $A^\pi(s,a) = Q^\pi(s,a) - \Expect{A \sim \pi(s)}{Q(s,A)}$. Since $Q^\pi$ is invariant to $S'$-redistribution (Lemma~\ref{lemma:ambiguity_Q_function}), this means that $A^\pi$ is invariant to $S'$-redistribution (Lemma~\ref{lemma:ambiguity_inherited}). Next, recall that if $R_1$ and $R_2$ differ by potential shaping with $\Phi$, then $Q_1(s,a) = Q_2(s,a) - \Phi(s)$ (Proposition~\ref{prop:change_from_potentials}). This means that
\begin{align*}
    A^\pi_1(s,a) &= Q^\pi_1(s,a) - \Expect{A \sim \pi(s)}{Q_1(s,A)}\\
    &= Q^\pi_2(s,a) - \Phi(s) - \Expect{A \sim \pi(s)}{Q_2(s,A) - \Phi(s)}\\
    &= Q^\pi_2(s,a) - \Expect{A \sim \pi(s)}{Q_2(s,A)}\\
    &= A^\pi_2(s,a).
\end{align*}
Together, this means that $A^\pi$ is invariant to both $S'$-redistribution and potential shaping.

We next need to show that $A^\pi$ is not invariant to any transformations that cannot be expressed as a combination of $S'$-redistribution and potential shaping.
Let $R_1$ and $R_2$ be two reward functions such that $A^\pi_1 = A^\pi_2$, and let $Q^\pi_1$, $Q^\pi_2$ be their $Q$-functions. Define
$$
\Phi(s) := \Expect{A\sim\pi(s)}{Q^\pi_1(s,A) - Q^\pi_2(s,A)}
$$
and let 
$$
R_3(s,a,s') = R_2(s,a,s') + \gamma \cdot \Phi(s) - \Phi(s').
$$
Now $R_2$ and $R_3$ differ by potential shaping (with $\Phi$). Moreover, by Proposition~\ref{prop:change_from_potentials}, we have that $Q^\pi_3 = Q^\pi_1$. Therefore, by Lemma~\ref{lemma:ambiguity_Q_function}, we have that $R_3$ and $R_1$ differ by $S'$-redistribution. This implies that $R_1$ and $R_2$ differ by a combination of potential shaping and $S'$-redistribution.
\end{proof}\fi

\begin{lemma}\label{lemma:ambiguity_optimal_A_function}
For any transition function $\tfunc$ and any discount $\gamma$, the optimal advantage function $A^\star$ determines $R$ up to $\PS \bigodot \SR$.
\end{lemma}
\ifshowproofs\begin{proof}
Analogous to Lemma~\ref{lemma:ambiguity_A_function}.
\end{proof}\fi

%\begin{lemma}\label{lemma:ambiguity_soft_A_function}
%Given any weight $\alpha > 0$, the soft advantage function $A^\mathrm{S}_\alpha$ determines $R$ up to $S'$-redistribution and potential shaping.
%\end{lemma}
%\ifshowproofs\begin{proof}
%    \red{TODO}
%\end{proof}\fi

By Lemma~\ref{lemma:ambiguity_inherited}, this means that any type of policy which can be derived from one of these two advantage functions will be invariant to $S'$-redistribution and potential shaping.

\subsubsection{Partial Identifiability of Policies}

We first note that the softmax function is invariant to constant shift, and to no other transformations: %\red{[move below def/first-use of softmax, and defer proof to appendix?]}
\begin{proposition}\label{prop:softmax_shift}
Let $v, w \in \mathbb{R}^n$ be two vectors, and let $\beta \in \mathbb{R}^+$. Then
$$
\frac{\exp \beta v_i}{\sum_{j=1}^n \exp \beta v_j} = \frac{\exp \beta w_i}{\sum_{j=1}^n \exp \beta w_j}
$$
for all $i \in \{1 \dots n\}$ if and only if there is a constant scalar $c$ such that $v_i = w_i + c$ for all $i \in \{1 \dots n\}$.
\end{proposition}
\ifshowproofs
\begin{proof}
For the first direction, suppose there is a constant scalar $c$ such that $v_i = w_i + c$ for all $i \in \{1 \dots n\}$. Then
\begin{align*}
    \frac{\exp \beta v_i}{\sum_{j=1}^n \exp \beta v_j} &= \frac{\exp \beta (w_i + c)}{\sum_{j=1}^n \exp \beta (w_j + c)}\\
    &= \frac{\exp (\beta c) \cdot \exp \beta w_i}{\exp (\beta c) \cdot \sum_{j=1}^n \exp \beta w_j}\\
    &= \frac{\exp \beta w_i}{\sum_{j=1}^n \exp \beta w_j}.
\end{align*}
For the other direction, suppose
$$
\frac{\exp \beta v_i}{\sum_{j=1}^n \exp \beta v_j} = \frac{\exp \beta w_i}{\sum_{j=1}^n \exp \beta w_j}
$$
for all $i \in \{1 \dots n\}$. Note that this can be rewritten as follows:
\begin{align*}
    \frac{\exp \beta v_i}{\exp \beta w_i} &= \frac{\sum_{j=1}^n \exp \beta v_j}{\sum_{j=1}^n \exp \beta w_j}\\
    \exp (\beta v_i - \beta w_i) &= \frac{\sum_{j=1}^n \exp \beta v_j}{\sum_{j=1}^n \exp \beta w_j}\\
    v_i - w_i &= \left(\frac{1}{\beta}\right) \log \left(\frac{\sum_{j=1}^n \exp \beta v_j}{\sum_{j=1}^n \exp \beta w_j}\right)
\end{align*}
Since the right-hand side of this expression does not depend on $i$, it follows that $v_i - w_i$ is constant for all $i$.
\end{proof}
\fi

\ambiguityboltzmannrational*

\begin{proof}
Let $a_{\tfunc, \gamma} : \R \to (\SxA \to \mathbb{R})$ be the function that takes $R$ and returns the optimal advantage function $A^\star$ for $R$, given transition function $\tfunc$ and discount factor $\gamma$. We will show that $\mathrm{Am}(b_{\tfunc, \gamma, \beta}) = \mathrm{Am}(a_{\tfunc, \gamma})$, by showing that the Boltzmann-rational policy can be derived from the optimal advantage function, and vice versa.

First, recall that the Boltzmann-rational policy is given by applying the softmax function (with temperature $\beta$) to the optimal $Q$-function, $Q^\star$, in each state. Moreover, since the the softmax function is invariant to constant shift (Proposition~\ref{prop:softmax_shift}), and since $A^\star$ and $Q^\star$ differ by constant shift in each state (given by $V^\star$), we have that the Boltzmann-rational policy also can be obtained by applying the softmax function with temperature $\beta$ to the advantage function in each state. This means that there exists a function $h$ such that $b_{\tfunc, \gamma, \beta} = h \circ a_{\tfunc, \gamma}$. Thus, by Lemma~\ref{lemma:ambiguity_inherited}, we have that $\mathrm{Am}(a_{\tfunc, \gamma}) \refines \mathrm{Am}(b_{\tfunc, \gamma, \beta})$.

For the other direction, recall that any softmax function is invariant to constant shift, and no other transformations (Proposition~\ref{prop:softmax_shift}). Since the Boltzmann-rational policy can be derived from $A^\star$ by applying a softmax function in each state, this means that if $b_{\tfunc, \gamma, \beta}(R_1) = b_{\tfunc, \gamma, \beta}(R_2)$, then there is a function $B : \States \to \mathbb{R}$ such that $A^\star_1(s,a) = A^\star_2(s,a) + B(s)$ for all $s,a$. Moreover, since $\max_{a \in \Actions} A^\star(s,a) = 0$ for all $s$, it follows that $B(s) = 0$ for all $s$, which means that $A^\star_1 = A^\star_2$. In other words, if $b_{\tfunc, \gamma, \beta}(R_1) = b_{\tfunc, \gamma, \beta}(R_2)$ then $a_{\tfunc, \gamma}(R_1) = a_{\tfunc, \gamma}(R_2)$, which means that there exists a function $h$ such that $a_{\tfunc, \gamma} = h \circ b_{\tfunc, \gamma, \beta}$. Thus, by Lemma~\ref{lemma:ambiguity_inherited}, we have that $\mathrm{Am}(b_{\tfunc, \gamma, \beta}) \refines \mathrm{Am}(a_{\tfunc, \gamma})$.

Since $\mathrm{Am}(a_{\tfunc, \gamma}) \refines \mathrm{Am}(b_{\tfunc, \gamma, \beta})$ and $\mathrm{Am}(b_{\tfunc, \gamma, \beta}) \refines \mathrm{Am}(a_{\tfunc, \gamma})$, we have that $\mathrm{Am}(b_{\tfunc, \gamma, \beta}) = \mathrm{Am}(a_{\tfunc, \gamma})$. In other words, since the Boltzmann-rational policy can be derived from the optimal advantage function, and vice versa, it must be the case that they have the same invariances. Applying Lemma~\ref{lemma:ambiguity_optimal_A_function} completes the proof.
\end{proof}

\ambiguityMCE*

\begin{proof}
First, recall that the MCE policy is given by applying the softmax function with temperature $(1/\alpha)$ to the soft $Q$-function $\QSoft$ in each state. Next, recall that any softmax function is invariant to constant shift, and no other transformations (Proposition~\ref{prop:softmax_shift}). This means that $c_{\tfunc, \gamma, \alpha}$ is invariant to all transformations that induce a constant shift of $Q^\mathrm{S}_\alpha$ in each state, and no other transformations.

The first direction is straightforward; Lemma~\ref{lemma:ambiguity_soft_Q_function} and Proposition~\ref{prop:potentials_to_soft_Q} together imply that $S'$-redistribution and potential shaping results in a constant shift of $Q^\mathrm{S}_\alpha$ in all states. This means that $c_{\tfunc, \gamma, \alpha}$ is invariant to these transformations.

For the other direction, let $R_1$ and $R_2$ be two reward functions such that the corresponding soft $Q$-functions satisfy $\QSoftN{1}(s,a) = \QSoftN{2}(s,a) + B(s)$ for some function $B : \States \to \mathbb{R}$. Recall that $\QSoftN{1}$ is the unique function which satisfies the following equation for all $s$ and $a$:
$$
\QSoftN{1}(s,a) = 
    \Expect{S' \sim \tfunc(s,a)}
    {
        R(s,a,S') + \gamma \alpha \log \sum_{a' \in \Actions} \exp \left(\frac{1}{\alpha}\right)\QSoftN{1}(S', a')
    }\,.
$$
This can be rewritten as
$$
\Expect{}{R_1(s,a,S')}
    = \QSoftN{1}(s,a)
        - \Expect{}{
            \gamma \alpha
            \log \sum_{a' \in A} \exp \left(\frac{1}{\alpha}\right) \QSoftN{1} (S', a')
        }.
$$
We can now rewrite the right-hand side as follows:
\begin{align*}
    &\QSoftN{1}(s,a)
        - \Expect{}{
            \gamma \alpha
            \log \sum_{a' \in A} \exp \left(\frac{1}{\alpha}\right) \QSoftN{1} (S', a')
        }
\\
    = &\QSoftN{2}(s,a) + B(s)
        - \Expect{}{
            \gamma \alpha
            \log \sum_{a' \in A} \exp \left(\frac{1}{\alpha}\right)\left(
                \QSoftN{2} (S', a') + B(S')
            \right)
        }
\\
    = &\QSoftN{2}(s,a) + B(s)
        - \Expect{}{
            \gamma \alpha \log \left(
                \sum_{a' \in A} \exp \left(\frac{1}{\alpha}\right)\QSoftN{2} (S', a')
            \right)
            + \gamma B(S')
        }
\\
    = &\Expect{}{
            R_2(s,a,S') + B(s) - \gamma B(S')
        }\,.
\end{align*}
Now set $\Phi(s) = -B(s)$, and we can see that the difference between $R_1$ and $R_2$ is described by potential shaping and $S'$-redistribution.
\end{proof}

\ambiguityoptimal*

\begin{proof}
Immediate from Theorem~\ref{thm:OPT_ambiguity}, since $o_{\tfunc, \gamma}^\star(R_1) = o_{\tfunc, \gamma}^\star(R_2)$ if and only if $R_1$ and $R_2$ have the same optimal policies.
\end{proof}

\subsubsection{Ambiguity Tolerance and Applications}

\diameterofamforBRandMCE*

\begin{proof}
    As per Theorem~\ref{thm:ambiguity-boltzmann-rational} and \ref{thm:ambiguity-MCE}, if $f$ is either $b_{\tfunc, \gamma, \beta}$ or $c_{\tfunc, \gamma, \alpha}$, and $f(R_1) = f(R_2)$, then $R_1$ and $R_2$ differ by a transformation in $\PS \bigodot \SR$. As per theorem~\ref{thm:policy_ordering}, this implies that $R_1 \eq{\ORD} R_2$, and so $\mathrm{Am}(f) \refines \ORD$. Since $\ORD \refines \OPT$, we also have that $\mathrm{Am}(f) \refines \OPT$. Finally, as per Proposition~\ref{prop:sound_and_complete_means_distance_0_iff_same_order}, all sound and complete pseudometrics metrics have the property that $d^\R(R_1, R_2) = 0$ if $R_1 \eq{\ORD} R_2$. Thus the upper (and hence also the lower) diameter of $\mathrm{Am}(f)$ under $d^\R$ is $0$. 
\end{proof}

\optimalpoliciesdonotpreserveord*

\begin{proof}
If $|\States| \geq 2$ or $|\Actions| \geq 3$, then there exists uncountably many reward functions that do not have the same ordering of policies (this is immediate from Theorem~\ref{thm:policy_ordering}). Moreover, $\mathrm{Im}(o_{\tau,\gamma}^\star)$ is finite. By the pigeonhole principle, this means that there must exist reward functions $R_1$, $R_2$ such that $o_{\tfunc,\gamma}^\star(R_1) = o_{\tfunc,\gamma}^\star(R_2)$ but $R_1 \not\eq{\mathrm{ORD}_{\tau,\gamma}} R_2$.
\end{proof}

\diameterofop*

\begin{proof}
    The first part follows from Proposition~\ref{prop:optimal_policies_do_not_preserve_ord}, and the second part follows from Theorem~\ref{thm:ambiguity-optimal}. For the third part, first note that Proposition~\ref{prop:sound_and_complete_means_distance_0_iff_same_order} implies that if $d^\R$ is both sound and complete, then $d^\R(R_1, R_2) = 0$ if and only if $R_1 \eq{\ORD} R_2$. Thus the fact that $\mathrm{Am}(o_{\tfunc,\gamma}^\star) \not\refines \ORD$ implies that the upper diameter of $\mathrm{Am}(o_{\tfunc,\gamma}^\star)$ under $d^\R$ is greater than $0$. To see that the lower diameter is $0$, consider the reward function $R_0$ that is $0$ everywhere. Then $o_{\tfunc,\gamma}^\star(R_0)$ must indicate that all actions are optimal in all states, which means any reward function $R$ such that $o_{\tfunc,\gamma}^\star(R) = o_{\tfunc,\gamma}^\star(R_0)$ must be trivial. All trivial reward functions have the same policy order, and so $d^\R(R,R_0) = 0$.
\end{proof}

\nonmaxsupportiveoptveryambiguous*
\begin{proof}
This can be demonstrated by a pigeonhole argument. Specifically, the codomain of each $o \in \mathcal{O}_{\tfunc,\gamma}$ has $(2^{|\Actions|}-1)^{|\States|}$ elements, and there are $(2^{|\Actions|}-1)^{|\States|}$ $\mathrm{OPT}_{\tfunc,\gamma}$-equivalence classes.
This means that if $\mathrm{Am}(o) \refines \mathrm{OPT}_{\tfunc,\gamma}$, then there must be a one-to-one correspondence between $\mathrm{OPT}_{\tfunc,\gamma}$-equivalence classes and elements of $o$'s codomain, so that there for each equivalence class $C \in \OPT$ is a $y_C \in \mathrm{Im}(o)$ such that $o(R) = y_C$ if and only if $R \in C$.
Further, say that if $f, g : X \rightarrow \mathcal{P}(Y)$ are set-valued functions, then $f \subseteq g$ if $f(x) \subseteq g(x)$ for all $x \in X$, and $f \subset g$ if $f \subseteq g$ but $g \not\subseteq f$.
Then if $o \in \mathcal{O}_{\tfunc,\gamma}$ we have that $o(R) \subseteq o_{\tfunc,\gamma}^\star(R)$ for all $R$ --- a policy is optimal if and only if it takes only optimal actions, but it need not take all optimal actions.
Moreover, if $o \neq o_{\tfunc,\gamma}^\star$ then there is an $R_1$ such that $o(R_1) \subset o_{\tfunc,\gamma}^\star(R_1)$.
Let $R_2$ be a reward function so that $o_{\tfunc,\gamma}^\star(R_2) = o(R_1)$ --- for any function $\States \rightarrow \mathcal{P}(\Actions) - \varnothing$, there is a reward function for which those are the optimal actions, so there is always some $R_2$ such that $o_{\tfunc,\gamma}^\star(R_2) = o(R_1)$.
Now either $o(R_2) = o(R_1)$ or $o(R_2) \subset o(R_1)$, since all actions that are optimal under $R_2$ are optimal under $R_1$. In the first case, since $o(R_1) = o(R_2)$ but $R_1 \not\eq{\mathrm{OPT}_{\tfunc,\gamma}} R_2$, we have that $\mathrm{Am}(o) \not\refines \mathrm{OPT}_{\tfunc,\gamma}$. In the second case, let $R_3$ be a reward function so that $o_{\tfunc,\gamma}^\star(R_3) = o(R_2)$, and repeat the same argument. Since there can only be a finite sequence $o(R_n) \subset \dots \subset o(R_2) \subset o(R_1)$, we have that we must eventually find two $R_n, R_{n-1}$ such that $o(R_n) = o(R_{n-1})$ but $R_n \not\eq{\mathrm{OPT}_{\tfunc,\gamma}} R_{n-1}$.
This means that it cannot be the case that $\mathrm{Am}(o) \refines \mathrm{OPT}_{\tfunc,\gamma}$.
\end{proof}

\subsubsection{Transfer Learning}

We will begin by proving the following intermediate result:

%We will formalise this by assuming that the transition function, $\tfunc$, changes between training and application. This captures, for example, the common sim-to-real setting, where learning occurs in a simulation whose dynamics differ slightly from those of the real environment. It also models the case where training data for the reward learning algorithm is not available for the deployment environment, but is available for a different environment.

%It can also indirectly capture the case where the new environment has new states, since parts of $\States$ can become reachable if $\TransitionDistribution$ or $\InitStateDistribution$ changes.
    
%The simplest transformations to consider are masks of \emph{impossible} or \emph{unreachable} transitions. These are parametrised by transition dynamics. In general, the ambiguity corresponding to a mask of $\mathcal{X}$ is less than for a mask of $\mathcal{X'}\supset\mathcal{X}$. For example, if the new dynamics supports previously impossible transitions, then sources with invariance to an impossible-transition mask from the original MDP may not be tolerable for applications in the new MDP.

%We can begin by noting that no guarantees can be obtained if the two MDPs have very different reachable state spaces; in that case, the learnt reward function will mostly depend on the inductive bias of the learning algorithm. However, we will demonstrate that similar problems can occur even with much smaller differences between the learning environment and the deployment environment:% also can cause similar problems:

\begin{lemma}\label{lemma:s-redistribution_and_transfer}
Let $r : \SxA \to \mathbb{R}$ be any function, $R_1$ any reward function, and $\tfunc_1, \tfunc_2$ any transition functions.
Then there exists a reward function $R_2$ such that $\Expect{S' \sim \tfunc_1(s,a)}{R_2(s,a,S')} = \Expect{S' \sim \tfunc_1(s,a)}{R_1(s,a,S')}$ for all $s,a$, and
such that 
$\Expect{S' \sim \tfunc_2(s,a)}{R_2(s,a,S')} = r(s,a)$
for all $s,a$ such that $\tfunc_1(s,a) \neq \tfunc_2(s,a)$.
\end{lemma}
\ifshowproofs
\begin{proof}
The requirement that $R_2$ is produced from $R_1$ by $S'$-redistribution under $\tfunc$ is satisfied if, for all $s\in\States$ and $a\in\Actions$,
\begin{equation*}
\label{eq:thm:sr-transfer:condition-1}
    \Expect{S' \sim \tfunc(s,a)}{R_1(s,a,S')}
    =
    \Expect{S' \sim \tfunc(s,a)}{R_2(s,a,S')}
\,.
\end{equation*}
Let $s\in\States$ and $a\in\Actions$ be any state and action such that $\tfunc_1(s, a) \neq \tfunc_2(s, a)$.
Let $\Vec{\tfunc_1}_{s,a}$ and
$\Vec{\tfunc_2}_{s,a}$ be
$\tfunc_1(s, a)$ and $\tfunc_2(s, a)$
expressed as vectors,
and let $\Vec{R_1}_{s,a}$ be the vector where
$\Vec{R_1}_{s,a}^{(i)} = R_1(s,a,s_i)$.
The question is then if there is an analogous vector
$\Vec{R_2}_{s,a}$ such that:
\begin{equation*}
\begin{aligned}
    \Vec{\tfunc_1}_{s,a} \cdot \Vec{R_2}_{s,a}
    &=
    \Vec{\tfunc_1}_{s,a} \cdot \Vec{R_1}_{s,a}\,,
\\
    \Vec{\tfunc_2}_{s,a} \cdot \Vec{R_2}_{s,a}
    &= r(s,a).
\end{aligned}
\end{equation*}
Since $\Vec{\tfunc_1}_{s,a}$ and $\Vec{\tfunc_2}_{s,a}$ differ and are valid probability distributions, they are linearly independent (recall also that $|\Actions| \geq 2$).
Therefore, the system of equations always has a solution for $\Vec{R_2}_{s,a}$.
Form the required $R_2$ as $R_1$ modified to have the values of $\Vec{\reward_2}_{s,a}$ in these states where the $\tfunc_1$ and $\tfunc_2$ differ.
\end{proof}
\fi

To unpack this, let $R_1$ be the true reward function, $\tfunc_1$ be the transition dynamics of the training environment, and $\tfunc_2$ be the transition dynamics of the deployment environment. Lemma~\ref{lemma:s-redistribution_and_transfer} then says, roughly, that if the training data is invariant to $S'$-redistribution, and $\tfunc_1$ and $\tfunc_2$ differ for enough states, then the learnt reward function is essentially unconstrained in the deployment environment. Specifically, for every reward function $R_1$ there exists a reward function $R_2$ such that $R_1$ and $R_2$ are indistinguishable in the training environment, but such that $R_2$ may have any value for any state-action pair for which $\tfunc_1 \neq \tfunc_2$.
This means that no guarantees can be obtained.
Moreover, note that Lemmas~\ref{lemma:ambiguity_Q_function}-\ref{lemma:ambiguity_soft_Q_function} and Lemma~\ref{lemma:ambiguity_inherited} imply that this result extends to \emph{any} object that can be computed from a $Q$-function, which is a very broad class.
Lemma~\ref{lemma:s-redistribution_and_transfer} then suggests that any such data source is too ambiguous to guarantee transfer to a different environment. From this, we can immediately derive the following:

\wrongtautooambiguous*

\begin{proof}
Let $R_1$ be an arbitrary reward. If $\tfunc_1 \neq \tfunc_2$, then there exists some $s,a$ such that $\tfunc_1(s,a) \neq \tfunc_2(s,a)$. Using the construction in Lemma~\ref{lemma:s-redistribution_and_transfer}, we can find a reward function $R_2$ such that $R_1$ and $R_2$ differ by $S'$-redistribution with $\tfunc_1$, and such that $A^\star_2(s,a)$ has any arbitrary value when computed under $\tfunc_2$ (and any discount $\gamma$). In particular, if $a \not\in \mathrm{argmax}_{a'} A_1^\star(s,a')$ under $\tfunc_2$ and $\gamma$, then we can let $a \in \mathrm{argmax}_{a'} A_2^\star(s,a')$ under $\tfunc_2$ and $\gamma$, and vice versa. This means that $\mathrm{argmax}_{a} A_1^\star(s,a) \neq \mathrm{argmax}_{a} A_2^\star(s,a)$ under $\tfunc_2$ and $\gamma$, and so $R_1 \not\eq{\mathrm{OPT}_{\tfunc_2, \gamma}} R_2$. However, $R_1$ and $R_2$ differ by $S'$-redistribution with $\tfunc_1$, and so $f_{\tfunc_1(R_1)} = f_{\tfunc_1(R_2)}$.
\end{proof}

We can also extend this result to a stronger statement, expressed in terms of Definition~\ref{def:ambiguity_diameter}. To do this, we will need the following lemma: 

\begin{lemma}\label{lemma:non_trivial_additive_big_diameter}
Let $f$ be a reward object such that, for every reward $R$ there exists a reward $R^\dagger$ such that $R^\dagger$ is non-trivial, and such that $f(R) = f(R + \alpha R^\dagger)$ for all $\alpha \in \mathbb{R}$. Then the lower and upper diameter of $\mathrm{Am}(f)$ under $\starc$ is 1.
\end{lemma}
\ifshowproofs
\begin{proof}
Let $R$ be an arbitrary reward function, and let $S$ be the set given by
$$
S = \{R + \alpha R^\dagger : \alpha \in \mathbb{R}\}.
$$
Note that $S$ forms a line through $\R$, and let us consider what happens when the canonicalisation and normalisation of $\starc$ is applied to $S$. Specifically, recall that $c^\mathrm{STARC}_{\tau,\gamma}$ only collapses dimensions along which every reward differs by potential shaping and $S'$-redistribution. Moreover, $R$ and $R + R'$ differ by potential shaping and $S'$-redistribution if and only if $R'$ is trivial. Since $R^\dagger$ is non-trivial, this means that $c^\mathrm{STARC}_{\tau,\gamma}(S)$ forms a line through $\mathrm{Im}(c^\mathrm{STARC}_{\tau,\gamma})$. After the normalisation step, we have that $c^\mathrm{STARC}_{\tau,\gamma}(S)$ is projected onto the unit ball of $n$, where $n$ is the norm used in the normalisation step of $d^\mathrm{STARC}_{\tau,\gamma}$. If $c^\mathrm{STARC}_{\tau,\gamma}(S)$ intersects the origin, then $s^\mathrm{STARC}_{\tau,\gamma}(S)$ will contain two points that are on the opposite sides of $\mathrm{Im}(s^\mathrm{STARC}_{\tau,\gamma})$, and these points have an $L_2$-distance of 2. If $c^\mathrm{STARC}_{\tau,\gamma}(S)$ does not intersect the origin, then $s^\mathrm{STARC}_{\tau,\gamma}(S)$ forms an arc along the surface of $\mathrm{Im}(s^\mathrm{STARC}_{\tau,\gamma})$. For every $\epsilon > 0$, there are two points on the far ends of this arch whose $L_2$-distance is at least $2-\epsilon$. Recall that the STARC-distance between $R_1$ and $R_2$ is \emph{half} of the $L_2$-distance between $s^\mathrm{STARC}_{\tau,\gamma}(R_1)$ and $s^\mathrm{STARC}_{\tau,\gamma}(R_2)$.

\begin{figure}[H]
    \centering
    \includegraphics[width=3\textwidth/4]{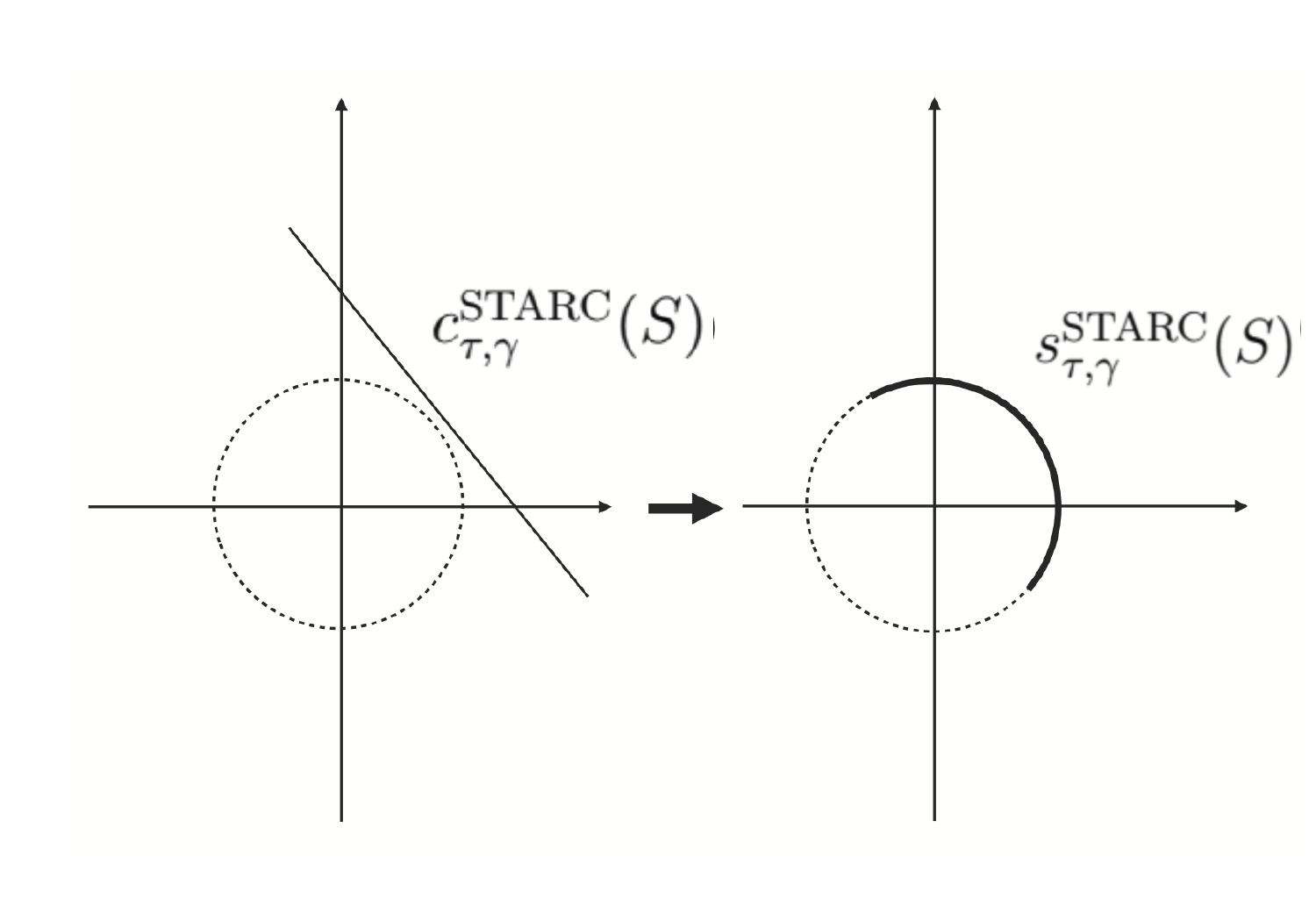}
    %\caption{\red{TODO}}
    \label{fig:thm83fig}
\end{figure}

Since $R$ was chosen arbitrarily, we have that there for any $x \in \mathrm{Im}(f)$ and $\epsilon > 0$ exists reward functions $R_1, R_2$ such that $f(R_1) = f(R_2) = x$, and such that $\starc(R_1, R_2) > 1- \epsilon$. This means that the lower diameter of $\mathrm{Am}(f)$ under $\starc$ is 1. Since 1 is the maximal distance under $\starc$, we also have that the upper diameter of $\mathrm{Am}(f)$ under $\starc$ is 1.    
\end{proof}
\fi

We also need the following lemma:
\begin{lemma}\label{lemma:SR_additive_invariance}
If $f_{\tfunc_1}$ is invariant to $S'$-redistribution with $\tfunc_1$, and $\tau_1 \neq \tau_2$, then for all $\gamma$ and all reward functions $R$, there exists a reward function $R^\dagger$ such that $R \not\eq{\mathrm{ORD}_{\tfunc_2,\gamma}}R^\dagger$, such that $R^\dagger$ is non-trivial under $\tau_2$ and $\gamma$, and such that $f_{\tau_1}(R) = f_{\tau_1}(R + \alpha R^\dagger)$ for all $\alpha \in \mathbb{R}$.
\end{lemma}
\ifshowproofs
\begin{proof}
This follows from Lemma~\ref{lemma:s-redistribution_and_transfer}; pick $R^\dagger$ such that
\begin{align*}
    \Expect{S' \sim \tfunc_1}{R^\dagger(s,a,S')} &= 0\\
    \Expect{S' \sim \tfunc_2}{R^\dagger(s,a,S')} &= \Expect{S' \sim \tfunc_2}{R'(s,a,S')}
\end{align*}
for all $s,a$, where $R'$ is some reward function such that $R'$ is non-trivial under $\tau_2$ and $\gamma$, and such that $R \not\eq{\mathrm{ORD}_{\tfunc_2,\gamma}} R'$. Since $R'$ is non-trivial under $\tau_2$ and $\gamma$, we have that $R^\dagger$ is non-trivial under $\tau_2$ and $\gamma$, and since $R \not\eq{\mathrm{ORD}_{\tfunc_2,\gamma}} R'$, we have that $R \not\eq{\mathrm{ORD}_{\tfunc_2,\gamma}}R^\dagger$. Moreover, for all $R$ and all $\alpha$, we have that $R$ and $R + \alpha R^\dagger$ differ by $S'$-redistribution (with $\tfunc_1$), and so $f_{\tau_1}(R) = f_{\tau_1}(R + \alpha R^\dagger)$.
\end{proof}
\fi

Using this lemma, we can now derive a quantitative result.

\wrongtaulargediameter*

\begin{proof}
Immediate from Lemma~\ref{lemma:non_trivial_additive_big_diameter} and \ref{lemma:SR_additive_invariance}.
\end{proof}

We say that a state $s$ is \emph{controllable} relative to a transition function $\tfunc$, initial state distribution $\init$, and discount $\gamma$, if there exist two policies $\pi$, $\pi'$ such that 
$$
%\sum_{t=0}^\infty \gamma^t \mathbb{P}(\pi\text{ enters \textit{s} at time \textit{t}})
%\neq
%\sum_{t=0}^\infty \gamma^t \mathbb{P}(\pi'\text{ enters \textit{s} at time \textit{t}}).
\sum_{t=0}^\infty \gamma^t \mathbb{P}_{\xi \sim \pi}(S_t = s)
\neq
\sum_{t=0}^\infty \gamma^t \mathbb{P}_{\xi \sim \pi'}(S_t = s).
$$
In other words, a state is controllable if the agent can influence how often it visits that state in expectation. %\orange{Note that a state may be controllable }

%Note that the sum starts from $t=1$. It can therefore be viewed as summing the discounted probability that $\pi$ and $\pi'$ enter $s$ at each time step.
%Recall also that $\tfunc$ is trivial if for all $s \in \States$ and $a,a' \in \Actions$, we have $\tfunc(s,a) = \tfunc(s,a')$. 

\begin{restatable}[]{lemma}{tautocontrol}
\label{lemma:tau_to_control}
For any $\init$, $\discount$, and $\tfunc$, there exists a controllable state if and only if $\tfunc$ is non-trivial.
\end{restatable}

\begin{proof}
It is straightforward to see that if $\tfunc$ is trivial then there are no controllable states. For the other direction, suppose there are no controllable states. Let a \enquote{state-valued} reward function be a reward function $R$ such that for each $s \in \States$, we have that $R(s,a_1,s_1) = R(s,a_2,s_2)$ for all $s_1,s_2 \in \States, a_1, a_2 \in \Actions$. Given a state-valued reward $R$, let $\vec{R} \in \mathbb{R}^{|S|}$ be the vector such that $\vec{R}[s]$ is the reward that $R$ assigns to transitions leaving $s$. Moreover, given a policy $\pi$, let $T^\pi$ be the $|\States|\times|\States|$-dimensional transition matrix that describes the transitions of $\pi$ under $\tfunc$, so that $T^\pi[s,s'] = \mathbb{P}_{A \sim \pi(s), S' \sim \tfunc(s,A)}(S' = s')$, and let $\vec{V^\pi} \in \mathbb{R}^{|S|}$ be the vector such that $\vec{V^\pi}[s] = V^\pi(s)$.
Using the Bellman equation for $V^\pi$ (Equation~\ref{equation:V_pi_recursion}), we now have that that:
\begin{align*}
    \vec{V^\pi} &= \Vec{R} + \gamma T^\pi \vec{V^\pi}\\
    \vec{V^\pi} - \gamma T^\pi \vec{V^\pi}&= \vec{R}\\
    (I - \gamma T^\pi)\vec{V^\pi} &= \vec{R}\\
    \vec{V^\pi} &= (I - \gamma T^\pi)^{-1}\vec{R}
\end{align*}
To see that $(I - \gamma T^\pi)$ always is invertible, note that the identity $(I - \gamma T^\pi)\vec{V^\pi} = \vec{R}$ implies that we, for any value function $V^\pi$, can find a state-valued reward function $R$ such that $V^\pi$ is the value function for $R$. Moreover, via the Bellman optimality equation (Equation~\ref{equation:V_pi_recursion}), we have that we, for any state-valued reward function $R$, can find a value function $V^\pi$ such that $V^\pi$ is the value function for $R$. There is therefore a one-to-one correspondence between value functions and state-valued reward functions, and so  $(I - \gamma T^\pi)$ must be invertible.

Next, note that if there are no controllable states, and $R$ is state-valued, then every policy $\pi$ has the same value function $V^\pi$. This means that
$$
(I - \gamma T^\pi)^{-1}\vec{R} = (I - \gamma T^{\pi'})^{-1}\vec{R}
$$
for all policies $\pi, \pi'$. Next, since $R$ was chosen arbitrarily, this identity must hold for all state-valued reward functions (i.e., all vectors in $\mathbb{R}^{|\States|}$), which means that 
$$
(I - \gamma T^\pi)^{-1} = (I - \gamma T^{\pi'})^{-1}
$$
for all policies $\pi, \pi'$. From this, it follows that $T^\pi = T^{\pi'}$ for all $\pi, \pi'$, which in turn implies that the transition function $\tfunc$ must be trivial.
\end{proof}

\wronggammatooambiguous*

\begin{proof}
As per Lemma~\ref{lemma:tau_to_control}, if $\tfunc$ is non-trivial then there is a state $s$ that is controllable relative to $\tfunc$ and $\gamma_2$ (and any $\init$ under which all states are reachable).
Let $\Phi_x : \States \to \mathbb{R}$ be the function given by $\Phi_x(s) = X$ and $\Phi_x(s') = 0$ for all $s' \neq s$, where $X \in \mathbb{R}$ and $X \neq 0$, and let $R$ be the reward function given by
$$
R(s,a,s') = \gamma_1 \cdot \Phi_x(s') - \Phi_x(s).
$$
Now $R_0$ and $R$ differ by potential shaping with $\gamma_1$, where $R_0$ is the reward function that is zero everywhere, which  means that $f^{\gamma_1}(R_0) = f^{\gamma_1}(R)$. 
Let $\Evaluation$ be the policy value function of $R$, evaluated under $\tfunc$, $\gamma_2$, and some initial state distribution $\init$ under which all states are reachable.
Moreover, given a policy $\pi$, let
\begin{align*}
    n^\pi &= \sum_{t=0}^\infty \gamma_2^t \mathbb{P}_{\xi \sim \pi}(S_{t+1} = s),\\
    x^\pi &= \sum_{t=0}^\infty \gamma_2^t \mathbb{P}_{\xi \sim \pi}(S_t = s).  
\end{align*}
%Some straightforward algebra shows that %if $s \not\in \mathrm{supp}(\init')$ then 
We then have that
$\Evaluation(\pi) = X \cdot (\gamma_1 n^\pi - x^\pi)$. Let $p$ denote $\init(s)$. If $\gamma_1 = \gamma_2$ then we know that 
$\Evaluation(\pi) = -X \cdot p$ (Proposition~\ref{prop:change_from_potentials}), which gives that
\begin{align*}
    X \cdot (\gamma_2 n^\pi - x^\pi) &= -X \cdot p\\
    \gamma_2 n^\pi - x^\pi &= - p\\
    x^\pi &= \gamma_2 n^\pi + p
\end{align*}
By plugging this into the above, and rearranging, we obtain 
$$
\Evaluation(\pi) = X n^\pi(\gamma_1 - \gamma_2) - pX.
$$
Moreover, if $s$ is controllable then there are $\pi_1, \pi_2$ such that $n^{\pi_1} \neq n^{\pi_2}$, which means that $\Evaluation(\pi_1) \neq \Evaluation(\pi_2)$. 
Thus $R$ is not trivial under $\tfunc$ and $\gamma_2$. Since $R_0$ is trivial under $\tfunc$ and $\gamma_2$, this means that $R \not\eq{\mathrm{OPT}_{\tfunc,\gamma_2}} R_0$. Thus, there exists reward functions $R, R_0$ such that $f_{\gamma_1}(R) = f_{\gamma_1}(R_0)$ but $R \not\eq{\mathrm{OPT}_{\tfunc,\gamma_2}} R_0$, which means that $\mathrm{Am}(f^{\gamma_1}) \not\refines \mathrm{OPT}_{\tfunc,\gamma_2}$.
\end{proof}

%To do this, we need the following lemma:

\begin{lemma}\label{lemma:PS_additive_invariance}
If $f_\gamma$ is invariant to potential shaping with $\gamma$, $\gamma_1 \neq \gamma_2$, and $\tfunc$ is non-trivial, then for all reward functions $R$, there exists a reward function $R^\dagger$ such that $R \not\eq{\mathrm{ORD}_{\tfunc,\gamma_2}}R^\dagger$, such that $R^\dagger$ is non-trivial under $\tfunc$ and $\gamma_2$, and such that $f_{\gamma_1}(R) = f_{\gamma_1}(R + \alpha R^\dagger)$ for all $R$ and all $\alpha \in \mathbb{R}$. 
%If $f_\gamma : \mathcal{R} \to \Pi$ is invariant to $\PS$, and $\gamma_1 \neq \gamma_2$, then for each $R_1$ there is an $R_2$ such that $f_\gamma(R_1) = f_\gamma(R_1 + \alpha R_2)$ for all $\alpha \in (-\infty, \infty)$, but $d^\mathrm{STARC}_{\tau,\gamma_2}(R_2, 0) > 0$.
\end{lemma}

\ifshowproofs
\begin{proof}
Let $R$ be an arbitrary reward function.
As per Lemma~\ref{lemma:tau_to_control}, if $\tfunc$ is non-trivial then there is a state $s$ that is controllable relative to $\tfunc$ and $\gamma_2$ (and any $\init$ under which all states are reachable). Specifically, there is a state $s_1$ for which there are two policies $\pi_1$, $\pi_2$ such that 
$$
\sum_{t=0}^\infty \gamma_2^t \mathbb{P}_{\xi \sim \pi_1}(S_t = s_1)
>
\sum_{t=0}^\infty \gamma_2^t \mathbb{P}_{\xi \sim \pi_2}(S_t = s_1).
$$
Moreover, since 
$$
\sum_{s \in \States}\sum_{t=0}^\infty \gamma_2^t \mathbb{P}_{\xi \sim \pi}(S_t = s) = \frac{1}{1-\gamma_2}
$$
for all $\pi$, we have that there also must be another state $s_2$ such that
$$
\sum_{t=0}^\infty \gamma_2^t \mathbb{P}_{\xi \sim \pi_1}(S_t = s_2)
<
\sum_{t=0}^\infty \gamma_2^t \mathbb{P}_{\xi \sim \pi_2}(S_t = s_2).
$$
In other words, there must be at least two controllable states $s_1, s_2$. 
Let $\Phi_{x,1} : \States \to \mathbb{R}$ be the function given by $\Phi_{x,1}(s_1) = X$ and $\Phi_{x,1}(s') = 0$ for all $s' \neq s_1$, where $X \in \mathbb{R}$ and $X \neq 0$, and let $R_1$ be the reward given by
$$
R_1(s,a,s') = \gamma_1 \cdot \Phi_{x,1}(s') - \Phi_{x,1}(s).
$$
Define $\Phi_{x,2}$ and $R_2$ analogously, using $s_2$ instead of $s_1$. $R_1$ and $R_2$ are non-trivial under $\tfunc$ and $\gamma_2$, as shown in Theorem~\ref{thm:wrong_gamma_too_ambiguous}. Moreover, $R$ and $R + \alpha R_1$ differ by potential shaping with $\gamma_1$, and likewise for $R$ and $R + \alpha R_2$, for all $\alpha \in \mathbb{R}$. It remains to show that $R_1 \not\eq{\mathrm{ORD}_{\tfunc,\gamma_2}} R_2$, which implies that either $R \not\eq{\mathrm{ORD}_{\tfunc,\gamma_2}} R_1$ or $R \not\eq{\mathrm{ORD}_{\tfunc,\gamma_2}} R_2$. When we have shown this, we can set $R^\dagger$ to either $R_1$ or $R_2$, which completes the proof.

To see that $R_1 \not\eq{\mathrm{ORD}_{\tfunc,\gamma_2}} R_2$, for a given policy $\pi$, let
\begin{align*}
    n_1^\pi &= \sum_{t=0}^\infty \gamma_2^t \mathbb{P}_{\xi \sim \pi}(S_{t+1} = s_1),\\
    n_2^\pi &= \sum_{t=0}^\infty \gamma_2^t \mathbb{P}_{\xi \sim \pi}(S_{t+1} = s_2).  
\end{align*}
and recall that
\begin{align*}
\Evaluation_1(\pi) &= X n_1^\pi(\gamma_1 - \gamma_2) - pX\\
\Evaluation_2(\pi) &= X n_2^\pi(\gamma_1 - \gamma_2) - pX
\end{align*}
as shown in the proof of Theorem~\ref{thm:wrong_gamma_too_ambiguous}. The proof of Theorem~\ref{thm:wrong_gamma_too_ambiguous} also shows that $x^\pi_1 = \gamma_2 n^\pi_1 + p_1$, where $p_1$ is $\init(s_1)$ and $x_1^\pi = \sum_{t=0}^\infty \gamma_2^t \mathbb{P}_{\xi \sim \pi}(S_t = s_1)$. Thus, if 
$$
\sum_{t=0}^\infty \gamma_2^t \mathbb{P}_{\xi \sim \pi_1}(S_t = s_1)
>
\sum_{t=0}^\infty \gamma_2^t \mathbb{P}_{\xi \sim \pi_2}(S_t = s_1).
$$
then $x_1^{\pi_1} > x_1^{\pi_2}$, which means that $n_1^{\pi_1} > n_1^{\pi_2}$. Similarly, we also have that $n_2^{\pi_1} < n_2^{\pi_2}$. Thus, if $\gamma_1 > \gamma_2$, then we have that $J_1(\pi_1) > J_1(\pi_2)$ but $J_2(\pi_1) < J_2(\pi_2)$, and if $\gamma_1 < \gamma_2$, then we have that $J_1(\pi_1) < J_1(\pi_2)$ but $J_2(\pi_1) > J_2(\pi_2)$. In either case, we have that $J_1$ and $J_2$ induce different policy orderings, and so $R_1 \not\eq{\mathrm{ORD}_{\tfunc,\gamma_2}} R_2$. This completes the proof.
\end{proof}
\fi

Using this lemma, we can now derive the following:

\wronggammalargediameter*

\begin{proof}
In the proof of Theorem~\ref{thm:wrong_gamma_too_ambiguous} we show that if $\gamma_1 \neq \gamma_2$ and $\tfunc$ is non-trivial, then there exists a reward $R^\dagger$ such that $R$ and $R + \alpha R^\dagger$ differ by potential shaping with $\gamma_1$ (for all $R$ and all $\alpha \in \mathbb{R}$), and such that $R^\dagger$ is non-trivial under $\gamma_2$. We can thus apply Lemma~\ref{lemma:non_trivial_additive_big_diameter}.
\end{proof}

\subsection{Misspecification Analysis With Equivalence Relations}

In this section, we provide the proofs of our results concerning the analysis of misspecification robustness with equivalence relations. In the main text, these results are presented in Section~\ref{sec:misspecification_1}.

\subsubsection{Misspecification of Policies}

\boltzmannweakrobustness*

\begin{proof}
Let $f_{\tfunc,\gamma} \in F_{\tfunc,\gamma}$ be surjective onto $\Pi^+$.
By definition, we have that $f_{\tfunc,\gamma}$ is $\mathrm{OPT}_{\tfunc,\gamma}$-robust to misspecification with $g$ if and only if $\mathrm{Am}(f_{\tfunc,\gamma}) \refines \mathrm{OPT}_{\tfunc,\gamma}$, $g \neq f_{\tfunc,\gamma}$, $\mathrm{Im}(g) \subseteq \mathrm{Im}(f)$, and if $f_{\tfunc,\gamma}(R_1) = g(R_2)$ then $R_1$ and $R_2$ have the same optimal policies under transition function $\tfunc$ and discount $\gamma$.
%for all $\pi \in \mathrm{Im}(g)$ we have that all $R \in \mathrm{Am}_g(\pi)$ and all $R \in \mathrm{Am}_f(\pi)$ have the same optimal policies in $\M$.

Since $f_{\tfunc,\gamma} \in F_{\tfunc,\gamma}$, we have that for all $R$,
$$
\mathrm{argmax}_{a \in \Actions} f_{\tfunc,\gamma}(R)(a \mid s) = \mathrm{argmax}_{a \in \Actions} Q^\star(s,a).
$$
Moreover, $R_1$ and $R_2$ have the same optimal policies under $\tfunc$ and $\gamma$ if and only if $\mathrm{argmax}_{a \in \Actions} Q^\star_1(s,a) = \mathrm{argmax}_{a \in \Actions} Q^\star_2(s,a)$ under $\tfunc$ and $\gamma$. 
Thus, if $f_{\tfunc,\gamma}(R_1) = f_{\tfunc,\gamma}(R_2)$ then $R_1 \equiv_{\mathrm{OPT}_{\tfunc,\gamma}} R_2$, and so $\mathrm{Am}(f_{\tfunc,\gamma}) \refines \mathrm{OPT}_{\tfunc,\gamma}$.

Let $g \in F_{\tfunc,\gamma}$ and $g \neq f_{\tfunc,\gamma}$. Since $g$ is a function $\mathcal{R} \to \Pi^+$, and since $f_{\tfunc,\gamma}$ is surjective onto $\Pi^+$, we have that  $\mathrm{Im}(g) \subseteq \mathrm{Im}(f_{\tfunc,\gamma})$. Next, by the same argument as above, if $f_{\tfunc,\gamma}(R_1) = g(R_2)$ then $\mathrm{argmax}_{a \in \Actions} Q^\star_1(s,a) = \mathrm{argmax}_{a \in \Actions} Q^\star_2(s,a)$, which implies that $R_1 \equiv_{\mathrm{OPT}_{\tfunc,\gamma}} R_2$. This means that $f_{\tfunc,\gamma}$ is $\mathrm{OPT}_{\tfunc,\gamma}$-robust to misspecification with $g$.

Next, suppose $f_{\tfunc,\gamma}$ is $\mathrm{OPT}_{\tfunc,\gamma}$-robust to misspecification with $g$. This means that $\mathrm{Im}(g) \subseteq \mathrm{Im}(f)$, that $f_{\tfunc,\gamma} \neq g$, and that if $f_{\tfunc,\gamma}(R_1) = g(R_2)$ then $R_1 \equiv_{\mathrm{OPT}_{\tfunc,\gamma}} R_2$. First, note that $\mathrm{Im}(g) \subseteq \mathrm{Im}(f)$ implies that $g$ is a function $\mathcal{R} \to \Pi^+$.
Next, let $R_1$ be an arbitrary reward function, and let $R_2$ be a reward function such that $f_{\tfunc,\gamma}(R_2) = g(R_1)$. Since $\mathrm{Im}(g) \subseteq \mathrm{Im}(f)$, we have that such a reward function $R_2$ must exist. Next, since $f_{\tfunc,\gamma}(R_2) = g(R_1)$,
$$
\mathrm{argmax}_{a \in \Actions} f_{\tfunc,\gamma}(R_2)(a \mid s) = \mathrm{argmax}_{a \in \Actions} g(R_1)(a \mid s).
$$
Moreover, we have that $R_1 \equiv_{\mathrm{OPT}_{\tfunc,\gamma}} R_2$, since $f_{\tfunc,\gamma}$ is $\mathrm{OPT}_{\tfunc,\gamma}$-robust to misspecification with $g$. This means that
$$
\mathrm{argmax}_{a \in \Actions} Q^\star_1(s,a) = \mathrm{argmax}_{a \in \Actions} Q^\star_2(s,a).
$$
Now, since $f_{\tfunc,\gamma} \in F_{\tfunc,\gamma}$, we have that
$$
\mathrm{argmax}_{a \in \Actions} f_{\tfunc,\gamma}(R_2)(a \mid s) = \mathrm{argmax}_{a \in \Actions} Q^\star_2(s,a).
$$
By transitivity, this implies that
$$
\mathrm{argmax}_{a \in \Actions} g(R_1)(a \mid s) = \mathrm{argmax}_{a \in \Actions} Q^\star_1(s,a).
$$
Since $R_1$ was chosen arbitrarily, this must hold for all $R_1$. Thus $g \in F_{\tfunc,\gamma}$.
\end{proof}

\boltzmannordrobustness*

\begin{proof}
As per Theorem~\ref{thm:ambiguity-boltzmann-rational}, $\mathrm{Am}(b_{\tfunc,\gamma,\beta})$ is characterised by $\PS \bigodot \SR$, and as per Theorem~\ref{thm:policy_ordering}, $\ORD$ is characterised by $\PS \bigodot \SR \bigodot \LS$. 
Hence $\mathrm{Am}(b_{\tfunc,\gamma,\beta}) \refines \mathrm{ORD}_{\tfunc,\gamma}$, which means that Lemma~\ref{lemma:P_robustness_function_composition} implies that $b_{\tfunc,\gamma,\beta}$ is $\mathrm{ORD}_{\tfunc,\gamma}$-robust to misspecification with $g$
if and only if $g \neq b_{\tfunc,\gamma,\beta}$, and there exists a $t \in \PS \bigodot \SR \bigodot \LS$ such that $g = b_{\tfunc,\gamma,\beta} \circ t$.

For the first direction, assume that there exists a $t \in \PS \bigodot \SR \bigodot \LS$ such that $g = b_{\tfunc,\gamma,\beta} \circ t$ and $g \neq b_{\tfunc,\gamma,\beta}$.
Now $b_{\tfunc,\gamma,\beta}(R)$ is the policy given by
\begin{equation*}
    b_{\tfunc,\gamma,\beta}(R)(a \mid s) = \frac{\exp \beta A_R(s,a)}{\sum_{a \in \Actions} \exp \beta A_R(s,a)}.
\end{equation*}
where $A_R$ is the optimal advantage function of $R$ under $\tfunc$ and $\gamma$. If $g(R) = b_{\tfunc,\gamma,\beta} \circ t(R)$ for some $t \in \mathrm{PS}_\discount \bigodot \mathrm{LS} \bigodot S'\mathrm{R}_\tfunc$, then we have that
\begin{align*}
    g(R)(a \mid s) &= \frac{\exp \beta A_{t(R)}(s,a)}{\sum_{a \in \Actions} \exp \beta A_{t(R)}(s,a)} \\
    &= \frac{\exp \beta c_R A_R(s,a)}{\sum_{a \in \Actions} \exp \beta c_R A_R(s,a)},
\end{align*}
where $c_R$ is the linear scaling factor that $t$ applies to $R$. 
Note that the advantage function $\A$ is preserved by both potential shaping and $S'$-redistribution (Lemma~\ref{lemma:ambiguity_A_function}).
Now let $\psi(R) = \beta \cdot c_R$, and we can see that $g = b_{\tfunc,\gamma,\psi} \in B_{\tfunc,\gamma}$. Thus, if $b_{\tfunc,\gamma,\beta}$ is $\mathrm{ORD}_{\tfunc,\gamma}$-robust to misspecification with $g$, then $g \in B_{\tfunc,\gamma}$ and $g \neq b_{\tfunc,\gamma,\psi}$.

For the other direction, assume that $g \in B_{\tfunc,\gamma}$ and $g \neq b_{\tfunc,\gamma,\beta}$. Since $g \in B_{\tfunc,\gamma}$, there is a function $\psi : \R \to \mathbb{R}^+$ such that $g(R)$ is the policy given by applying a softmax function with temperature $\psi(R)$ to the optimal advantage function of $R$. Now let $t \in \LS$ be the function that scales each $R \in \R$ by a factor of $\psi(R)/\beta$, and we can see that $g = b_{\tfunc,\gamma,\beta} \circ t$. This completes the proof.
\end{proof}

\optimalpolicyrobustness*

\begin{proof}
%This Theorem largely follows from Lemma~\ref{lemma:less_ambiguity_less_robustness}.
The first part follows from Proposition~\ref{prop:optimal_policies_do_not_preserve_ord}, which says that $\mathrm{Am}(o) \not\refines \mathrm{ORD}_{\tfunc,\gamma}$, unless $|\States| = 1$ and $|\Actions| = 2$.
Moreover, $\mathrm{Am}(o_{\tfunc,\gamma}^\star) = \mathrm{OPT}_{\tfunc,\gamma}$ (Theorem~\ref{thm:ambiguity-optimal}). Therefore, by Lemma~\ref{lemma:less_ambiguity_less_robustness}, there is no function $g$ such that $o_{\tfunc,\gamma}^\star$ is $\mathrm{OPT}_{\tfunc,\gamma}$-robust to misspecification with $g$.
Finally, Theorem~\ref{thm:non_max_supportive_opt_very_ambiguous} says that if $o \in \mathcal{O}_{\tfunc,\gamma}$ but $o \neq o_{\tfunc,\gamma}^\star$, then $\mathrm{Am}(o) \not\refines \mathrm{OPT}_{\tfunc,\gamma}$.
\end{proof}

\mcestrongrobustness*

\begin{proof}
As per Theorem~\ref{thm:ambiguity-MCE}, $\mathrm{Am}(c_{\tfunc,\gamma,\alpha})$ is characterised by $\PS \bigodot \SR$, 
and as per Theorem~\ref{thm:policy_ordering}, 
$\ORD$ is characterised by $\PS \bigodot \SR \bigodot \LS$. Hence $\mathrm{Am}(c_{\tfunc,\gamma,\alpha}) \refines \ORD$, which means that Lemma~\ref{lemma:P_robustness_function_composition} implies that $c_{\tfunc,\gamma,\alpha}$ is $\ORD$-robust to misspecification with $g$
if and only if $g \neq c_{\tfunc,\gamma,\alpha}$, and there exists a $t \in \PS \bigodot \SR \bigodot \LS$ such that $g = c_{\tfunc,\gamma,\alpha} \circ t$.

For the first direction, assume that $g \neq c_{\tfunc,\gamma,\alpha}$, and that there exists a $t \in \PS \bigodot \SR \bigodot \LS$ such that $g = c_{\tfunc,\gamma,\alpha} \circ t$.
Recall that $c_{\tfunc,\gamma,\alpha}(R)$ is the unique policy that maximises the maximal causal entropy objective;
\begin{equation*}
    \Evaluation^\mathrm{MCE}_{R}(\pi) = \Evaluation_R(\pi) - \alpha \sum_{t=0}^\infty \mathbb{E}_{S_t \sim \pi, \tfunc,\init}[\discount^t H(\pi(S_t))],
\end{equation*}
where $J_R$ is the policy evaluation function for the reward function $R$. Therefore, if $g(R) = c_{\tfunc,\gamma,\psi} \circ t(R)$ then $g(R)$ is the policy 
\begin{align*}
    &\max_\pi\Evaluation^\mathrm{MCE}_{t(R)}(\pi)\\
    = &\max_\pi\Evaluation_{t(R)}(\pi) - \alpha \sum_{t=0}^\infty \mathbb{E}_{S_t \sim \pi, \tfunc,\init}[\discount^t H(\pi(S_t))] \\
    = &\max_\pi c_R \cdot \Evaluation_{R}(\pi) - \alpha \sum_{t=0}^\infty \mathbb{E}_{S_t \sim \pi, \tfunc,\init}[\discount^t H(\pi(S_t))]
\end{align*}
where $c_R$ is the linear scaling factor that $t$ applies to $R$. 
Note that $\Evaluation_R$ is preserved by $S'$-redistribution, and potential shaping can only change $\Evaluation_R$ by inducing a uniform constant shift of $\Evaluation_R$ for all policies (Proposition~\ref{prop:change_from_potentials}). Thus linear scaling is the only transformation in $\PS \bigodot \SR \bigodot \LS$ that could affect the MCE objective.
Finally, let $\psi$ be the function $\psi(R) = \alpha / c_R$, and we can see that $g = c_{\tfunc,\gamma,\psi} \in C_{\tfunc,\gamma}$. 
%Thus, if $c_{\tfunc,\gamma,\alpha}$ is $\ORD$-robust to misspecification with $g$ and $g \neq c_{\tfunc,\gamma,\psi}$ then $g \in C_{\tfunc,\gamma}$. 

For the other direction, assume that $g \in C_{\tfunc,\gamma}$ and $g \neq c_{\tfunc,\gamma,\psi}$. Then there is a function $\psi : \R \to \mathbb{R}^+$ such that $g(R)$ is the unique policy that maximises the MCE objective given by 
\begin{equation*}
    \Evaluation_R(\pi) - \psi(R) \sum_{t=0}^\infty \mathbb{E}_{S_t \sim \pi, \tfunc,\init}[\discount^t H(\pi(S_t))].
\end{equation*}
Now let $t \in \LS$ be the function that applies a positive linear scaling factor of $\psi(R)/\alpha$ to each reward function $R$, and we can see that $g = c_{\tfunc,\gamma,\psi} \circ t$. Since $t \in \LS$, this completes the other direction, and the proof.
\end{proof}

\subsubsection{Wider Classes of Policies}

\continuousnotsurjective*

\begin{proof}
Assume for contradiction that $f : \mathcal{R} \to \Pi^+$ is continuous and surjective, and that $f(R_1) = f(R_2)$ if and only if $R_1 \eq{\mathrm{ORD}_{\tfunc, \gamma}} R_2$. 
Then $f$ is a continuous bijection from $\mathrm{Im}(s^\mathrm{STARC}_{\tfunc,\gamma})$ to $\Pi^+$, where $s^\mathrm{STARC}_{\tfunc,\gamma}$ is the standardisation function of $\starc$ (Definition~\ref{def:standard_starc}). Moreover, $\mathrm{Im}(s^\mathrm{STARC}_{\tfunc,\gamma})$ is compact (because it is a closed and bounded subset of a finite-dimensional Euclidean space), and $\Pi^+$ is Hausdorff. It thus follows that $f$ is a homeomorphism. This is a contradiction, since $\Pi^+$ is not homeomorphic to $\mathrm{Im}(s^\mathrm{STARC}_{\tfunc,\gamma})$. For example, $\mathrm{Im}(s^\mathrm{STARC}_{\tfunc,\gamma})$ contains an isolated point, which $\Pi^+$ does not.
\end{proof}

\continuousandinvarianttolinear*

\begin{proof}
This is straightforward.
Suppose $f : \mathcal{R} \to \Pi^+$ is continuous and invariant to positive linear scaling.
Let $R_1, R_2$ be two arbitrary reward functions, and consider a sequence $c_t$ where $c_t > 0$, but $c_t \to 0$ as $t \to \infty$. 
Next, consider the sequence given by $c_t \cdot R_1$. Since $c_t \cdot R_1 \to R_0$ as $t \to \infty$, and since $f$ is continuous, we have that $f(c_t \cdot R_1) \to f(R_0)$ as $t \to \infty$. 
Moreover, since $f$ is invariant to positive linear scaling, we have that $f(c_t \cdot R_1) = f(R_1)$ for all $c_t$. This implies that $f(R_1) = f(R_0)$. By an analogous argument, we also have that $f(R_2) = f(R_0)$, and hence that $f(R_1) = f(R_2)$.
\end{proof}

\subsubsection{Misspecified Environments}

\notProbusttomisspecifiedt*

\begin{proof}
Suppose for contradiction that $f_{\tfunc_1}$ is $\mathrm{OPT}_{\tfunc_3,\gamma}$-robust to misspecification with $f_{\tfunc_2}$. If $\tfunc_1 \neq \tfunc_2$, then $\tfunc_3 \neq \tfunc_1$, or $\tfunc_3 \neq \tfunc_2$, or both. Theorem~\ref{thm:wrong_tau_too_ambiguous} then implies that $\mathrm{Am}(f_{\tfunc_1}) \not\refines \mathrm{OPT}_{\tfunc_3,\gamma}$, or $\mathrm{Am}(f_{\tfunc_2}) \not\refines \mathrm{OPT}_{\tfunc_3,\gamma}$, or both. The former violates condition 3 in Definition~\ref{def:misspecification_eq}, and the latter violates Lemma~\ref{lemma:P_robustness_implies_refinement}. Thus $f_{\tfunc_1}$ cannot be $\mathrm{OPT}_{\tfunc_3,\gamma}$-robust to misspecification with $f_{\tfunc_2}$.
\end{proof}

\notProbusttomisspecifiedgamma*

\begin{proof}
Suppose for contradiction that $f_{\gamma_1}$ is $\mathrm{OPT}_{\tfunc,\gamma_3}$-robust to misspecification with $f_{\gamma_2}$. If $\gamma_1 \neq \gamma_2$, then $\gamma_3 \neq \gamma_1$, or $\gamma_3 \neq \gamma_2$, or both. Theorem~\ref{thm:wrong_gamma_too_ambiguous} then implies that $\mathrm{Am}(f_{\gamma_1}) \not\refines \mathrm{OPT}_{\tfunc,\gamma_3}$, or $\mathrm{Am}(f_{\gamma_2}) \not\refines \mathrm{OPT}_{\tfunc,\gamma_3}$, or both. The former violates condition 3 in Definition~\ref{def:misspecification_eq}, and the latter violates Lemma~\ref{lemma:P_robustness_implies_refinement}. Thus $f_{\gamma_1}$ cannot be $\mathrm{OPT}_{\tfunc,\gamma_3}$-robust to misspecification with $f_{\gamma_2}$.
\end{proof}

\notProbustbroadcase*

\begin{proof}
By Lemma~\ref{lemma:ambiguity_optimal_Q_function}, we have that $f_{\tfunc,\gamma}$ determines $R$ up to $\SR$. Thus, by Theorem~\ref{thm:not_P_robust_to_misspecified_t}, we have that $f_{\tfunc_1,\gamma}$ is not $\mathrm{OPT}_{\tfunc_3,\gamma}$-robust to misspecification with $f_{\tfunc_2,\gamma}$. Moreover, $\mathrm{Im}(f_{\tfunc_1,\gamma}) = \mathrm{Im}(f_{\tfunc_2,\gamma})$, since for any function $q : \SxA \to \mathbb{R}$, we can always find a reward function $R$ such that the optimal $Q$-function for $R$ is $q$ (via the Bellman optimality equation). We thus apply Lemma~\ref{lemma:P_robustness_inheritance}, and conclude the proof.
\end{proof}

\subsection{Misspecification Analysis With Metrics}

In this section, we provide the proofs of our results concerning the analysis of misspecification robustness with distance (pseudo)metrics. In the main text, these results are presented in Section~\ref{sec:misspecification_2}.

\subsubsection{Necessary and Sufficient Conditions}

\nasforsmallstarc*

\begin{proof}
For the first direction, suppose $\starc(R,t(R)) \leq \epsilon$ for all $R \in \mathcal{R}$, and let $R$ be an arbitrarily selected reward function. We will show that it is possible to navigate from $R$ to $t(R)$ using the described transformations.

Recall that $\starc(R, t(R))$ is computed by first applying $c^\mathrm{STARC}_{\tfunc,\gamma}$ to both $R$ and $t(R)$, normalising the resulting vectors, measuring their $L_2$-distance, and dividing the result by 2.
This means that if $\starc(R, t(R)) < 0.5$, then the $L_2$-distance between $s^\mathrm{STARC}_{\tfunc,\gamma}(R)$ and $s^\mathrm{STARC}_{\tfunc,\gamma}(t(R))$ is less than 1. Note also that if $R$ is trivial and $t(R)$ is non-trivial, or vice versa, then the $L_2$-distance between $s^\mathrm{STARC}_{\tfunc,\gamma}(R)$ and $s^\mathrm{STARC}_{\tfunc,\gamma}(t(R))$ is exactly 1. Thus, either $R$ and $t(R)$ are both trivial, or they are both non-trivial.

If $R$ and $t(R)$ are both trivial, then they differ by some transformation in $\PS \bigodot \SR$ (as implied by Theorem~\ref{thm:policy_ordering}), and so the theorem holds. Next, if $R$ and $t(R)$ are both non-trivial, then $s^\mathrm{STARC}_{\tfunc,\gamma}(R)$ and $s^\mathrm{STARC}_{\tfunc,\gamma}(t(R))$ can be placed in the following diagram, where $\epsilon' \leq 2\epsilon$:
\begin{figure}[H]
    \centering
    \hspace*{2cm}
    \includegraphics[width=1\textwidth/2]{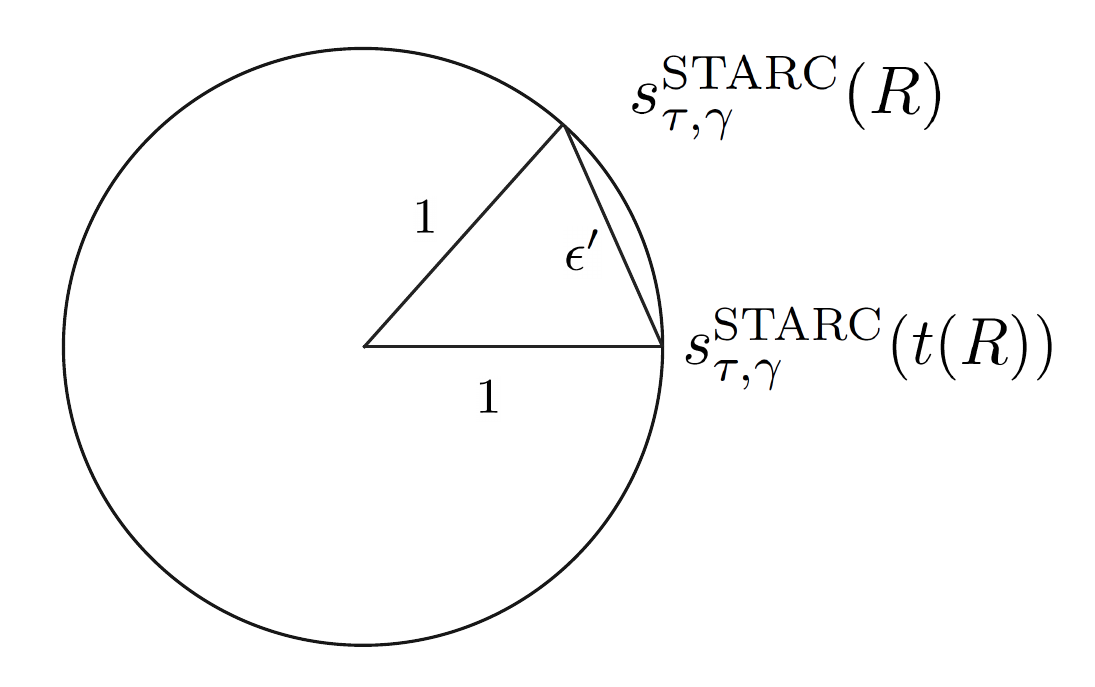}
    %\caption{Caption}
    \label{fig:fig1}
\end{figure}
Now, elementary trigonometry tells us that the angle $\theta$ between $s^\mathrm{STARC}_{\tfunc,\gamma}(R)$ and $s^\mathrm{STARC}_{\tfunc,\gamma}(t(R))$ is $\theta = 2 \arcsin (\epsilon'/2)$.
\footnote{This can be seen by bisecting the triangle along the vertex between $s^\mathrm{STARC}_{\tfunc,\gamma}(R)$ and $s^\mathrm{STARC}_{\tfunc,\gamma}(t(R))$, to form two right triangles. Since the hypotenuse is $1$, we have that $\sin(\theta/2) = \epsilon'/2$.} 
Moreover, suppose we make a right triangle by extending $s^\mathrm{STARC}_{\tfunc,\gamma}(R)$ as follows:

\begin{figure}[H]
    \centering
    \includegraphics[width=1\textwidth/3]
    {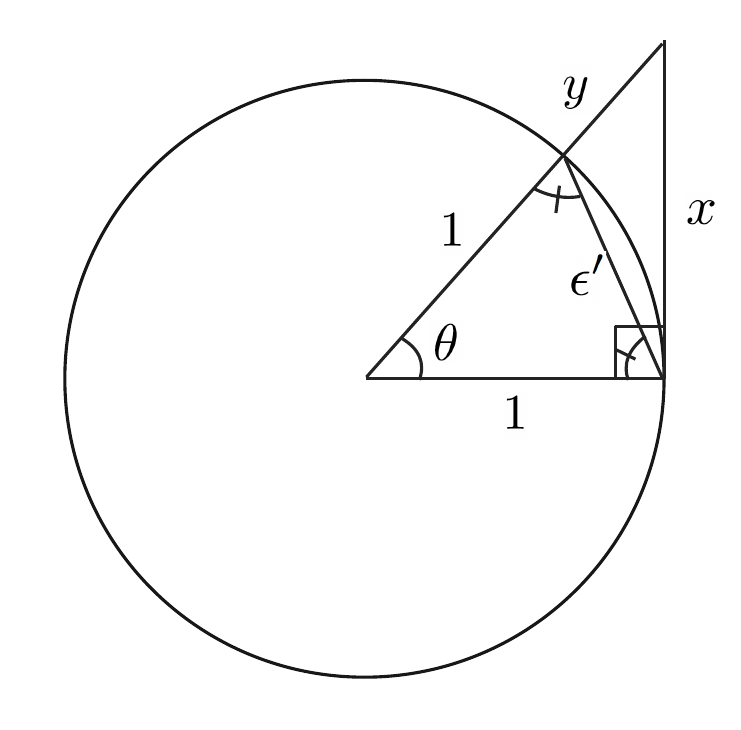}
    %\caption{Caption}
    \label{fig:fig2}
\end{figure}

Note that this can be done, since $\epsilon' \leq 2\epsilon < 2 \cdot 0.5 = 1 < \sqrt{2}$. Here elementary trigonometry again tells us that 
$$
x/(1+y) = \sin(\theta) = \sin(2 \arcsin (\epsilon'/2)),
$$
or that $x = (1+y)\sin(2 \arcsin (\epsilon'/2))$. This means that we can go from $R$ to $t(R)$ as follows:
\begin{enumerate}
    \item Apply $c^\mathrm{STARC}_{\tfunc,\gamma}$. Since $R$ and $c^\mathrm{STARC}_{\tfunc,\gamma}(R)$ differ by potential shaping and $S'$-redistribution, this transformation can be expressed as a combination of potential shaping and $S'$-redistribution. Call the resulting vector $R'$.
    \item Normalise $R'$, so that its magnitude is 1. This transformation is an instance of positive linear scaling. Call the resulting vector $R''$.
    \item Scale $R''$ until it forms a right triangle with $s^\mathrm{STARC}_{\tfunc,\gamma}(t(R))$. This transformation is an instance of positive linear scaling. Call the resulting vector $R'''$.
    \item Move from $R'''$ to $s^\mathrm{STARC}_{\tfunc,\gamma}(t(R))$. This will move $R'''$ by a distance equal to $(1+y)\sin(2 \arcsin (\epsilon'/2))$, where $(1+y) = L_2(R''')$. Moreover, since $R'''$ is in the image of $c^\mathrm{STARC}_{\tfunc,\gamma}$, we have that $R''' = c^\mathrm{STARC}_{\tfunc,\gamma}(R''')$, and so $L_2(R''') = L_2(c^\mathrm{STARC}_{\tfunc,\gamma}(R'''))$. This means that $R'''$ is moved by $L_2(c^\mathrm{STARC}_{\tfunc,\gamma}(R''')) \cdot \sin(2 \arcsin (\epsilon'/2))$. Since $0 \leq \epsilon' \leq 2\epsilon < 1 < \sqrt{2}$, and since $\sin(2 \arcsin (x/2))$ is growing monotonically on $[0,\sqrt{2}]$ this means that 
    $$
    L_2(R''', s^\mathrm{STARC}_{\tfunc,\gamma}(t(R))) \leq L_2(c^\mathrm{STARC}_{\tfunc,\gamma}(R''')) \cdot \sin(2 \arcsin (\epsilon)).
    $$
    \item Move from $s^\mathrm{STARC}_{\tfunc,\gamma}(t(R))$ to $c^\mathrm{STARC}_{\tfunc,\gamma}(t(R))$. Since $s^\mathrm{STARC}_{\tfunc,\gamma}(t(R))$ is simply a normalised version of $c^\mathrm{STARC}_{\tfunc,\gamma}(t(R))$, this is an instance of positive linear scaling.
    \item Move from $c^\mathrm{STARC}_{\tfunc,\gamma}(t(R))$ to $t(R)$. Since $t(R)$ and $c^\mathrm{STARC}_{\tfunc,\gamma}(t(R))$ differ by potential shaping and $S'$-redistribution, this transformation can be expressed as a combination of potential shaping and $S'$-redistribution.
\end{enumerate}
Thus, for an arbitrary reward function $R$, we can find a series of transformations that fit the given description. This completes the first direction.

For the other direction, suppose $t$ can be expressed as $t_1 \circ t_2 \circ t_3$ where 
$$
L_2(R, t_2(R)) \leq L_2(c^\mathrm{STARC}_{\tfunc,\gamma}(R)) \cdot \sin(2 \arcsin (\epsilon))
$$
for all $R$, and where $t_1,t_3 \in \SR \bigodot \PS \bigodot \LS$.

Recall that $\starc$ is invariant order-preserving transformations, and that all transformations in $\SR \bigodot \PS \bigodot \LS$ are order-preserving. This means that $\starc(R, t_i(R)) = 0$ for $i \in \{1,3\}$. 

For $t_2$, let $R$ be an arbitrary reward function. First note that if $R$ is trivial, then $c^\mathrm{STARC}_{\tfunc,\gamma}(R) = R_0$, and so $L_2(c^\mathrm{STARC}_{\tfunc,\gamma}(R)) = 0$. This implies that $L_2(R, t_2(R)) = 0$, and so $\starc(R, t_2(R)) = 0$. By the triangle inequality, we then have that $\starc(R, t(R)) = 0 \leq \epsilon$.

Next, assume that $R$ is non-trivial. Recall that $c^\mathrm{STARC}_{\tfunc,\gamma}$ is a linear orthogonal projection; this means that $L_2(c^\mathrm{STARC}_{\tfunc,\gamma}(R_1), c^\mathrm{STARC}_{\tfunc,\gamma}(R_2)) \leq L_2(R_1, R_2)$. As such, if $L_2(R, t_2(R)) \leq L_2(c^\mathrm{STARC}_{\tfunc,\gamma}(R)) \cdot \sin(2 \arcsin (\epsilon))$, then 
$$
L_2(c^\mathrm{STARC}_{\tfunc,\gamma}(R), c^\mathrm{STARC}_{\tfunc,\gamma}(t_2(R))) \leq L_2(c^\mathrm{STARC}_{\tfunc,\gamma}(R)) \cdot \sin(2 \arcsin (\epsilon))
$$
as well. Consider the set of all reward functions in $\mathrm{Im}(c^\mathrm{STARC}_{\tfunc,\gamma})$ whose distance to $c^\mathrm{STARC}_{\tfunc,\gamma}(R)$ is at most $L_2(c^\mathrm{STARC}_{\tfunc,\gamma}(R)) \cdot \sin(2 \arcsin (\epsilon))$, as in the following diagram:

\begin{figure}[H]
    \centering
    \includegraphics[width=3\textwidth/4]{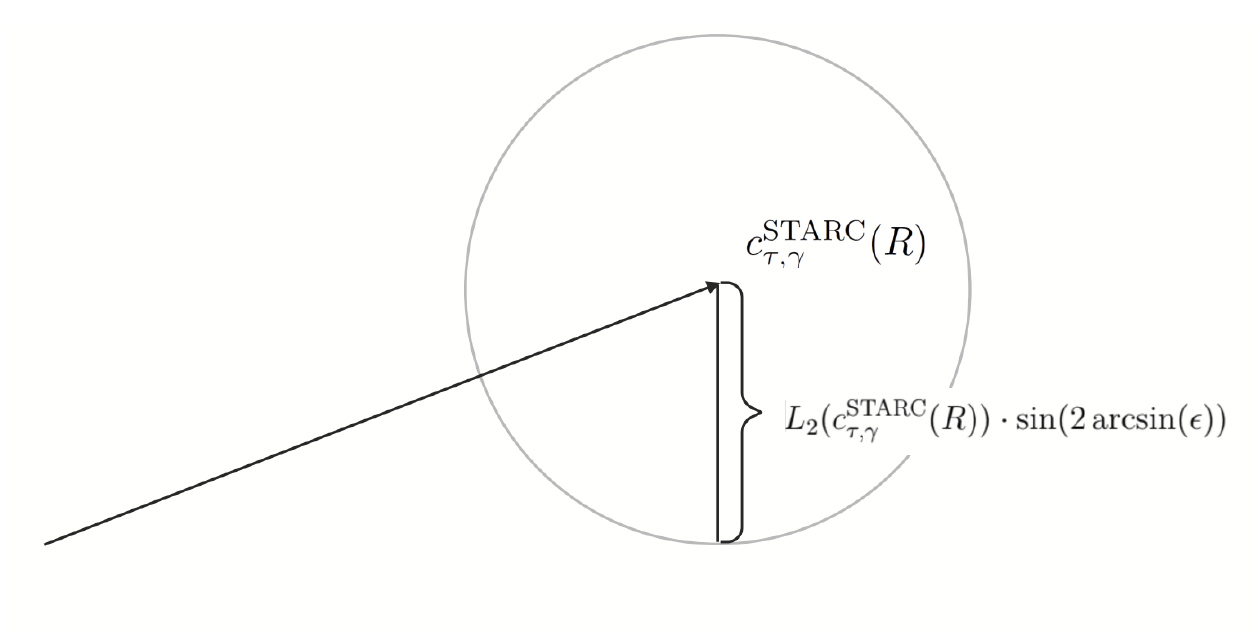}
    %\caption{Caption}
    \label{fig:fig3}
\end{figure}

Now $c^\mathrm{STARC}_{\tfunc,\gamma}(t_2(R))$ is located within the sphere depicted in the diagram above.\footnote{Note that $L_2(c^\mathrm{STARC}_{\tfunc,\gamma}(R)) \cdot \sin(2 \arcsin (\epsilon)) < L_2(c^\mathrm{STARC}_{\tfunc,\gamma}(R))$ for $\epsilon < 0.5$.}
The vectors $R'$ within this sphere that maximise the distance to $c^\mathrm{STARC}_{\tfunc,\gamma}(R)$ \emph{after normalisation} lie on the tangents of this sphere and the origin. That is, we wish to find an $R' \in \mathrm{Im}(c^\mathrm{STARC}_{\tfunc,\gamma})$ which maximises 
$$
L_2\left(\frac{R'}{L_2(R')}, \frac{c^\mathrm{STARC}_{\tfunc,\gamma}(R)}{L_2(c^\mathrm{STARC}_{\tfunc,\gamma}(R))}\right),
$$
subject to the constraint that 
$$
L_2(c^\mathrm{STARC}_{\tfunc,\gamma}(R), R') \leq L_2(c^\mathrm{STARC}_{\tfunc,\gamma}(R)) \cdot \sin(2 \arcsin (\epsilon)).
$$
This optimisation problem is solved by the rewards $R'$ such that $L_2(c^\mathrm{STARC}_{\tfunc,\gamma}(R), R') = L_2(c^\mathrm{STARC}_{\tfunc,\gamma}(R)) \cdot \sin(2 \arcsin (\epsilon))$, and such that $R'$ forms a right triangle with $c^\mathrm{STARC}_{\tfunc,\gamma}(R)$ and the origin.

Let $R'$ be any such reward, and let $\delta$ be the distance between this reward and $c^\mathrm{STARC}_{\tfunc,\gamma}(R)$ after normalisation (i.e., $\delta = L_2\left(R'/L_2(R'),c^\mathrm{STARC}_{\tfunc,\gamma}(R)/L_2(c^\mathrm{STARC}_{\tfunc,\gamma(R))}\right)$). Let $\theta$ be the angle between $c^\mathrm{STARC}_{\tfunc,\gamma}(R)$ and $R'$. %We can now draw the following diagram:

\begin{figure}[H]
    \centering
    \includegraphics[width=3\textwidth/4]{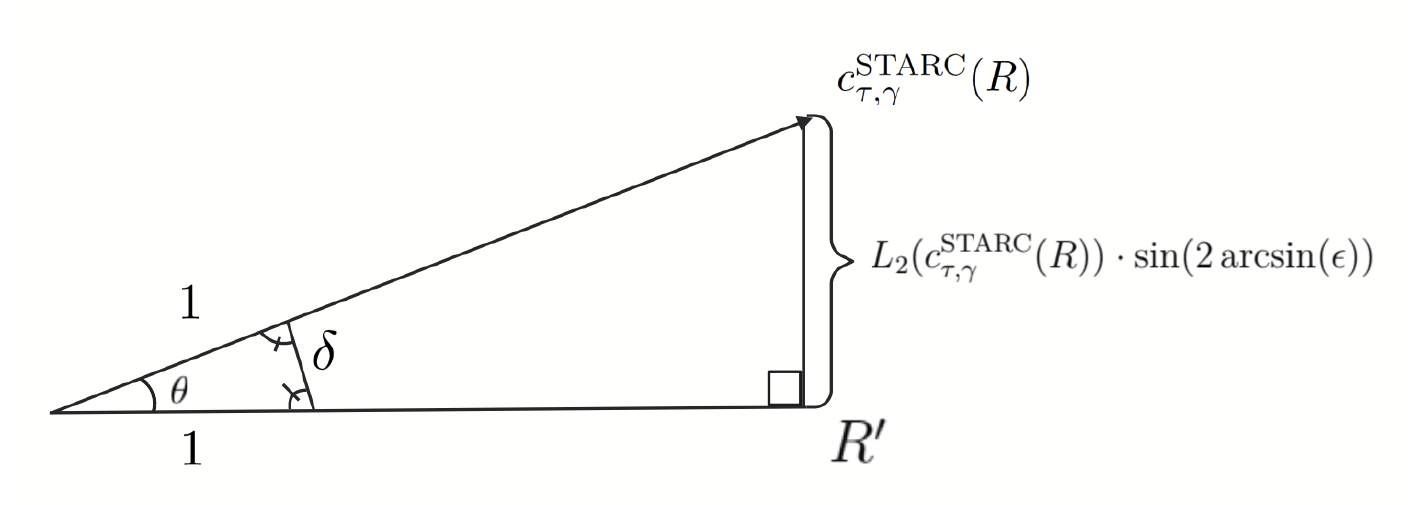}
    %\caption{Caption}
    \label{fig:fig4}
\end{figure}

Note that $\starc(R, R') = 0.5\cdot \delta$, from the definition of $\starc$ (also noting that $R' = c(R')$).
Elementary trigonometry now tells us that 
$$
\sin(\theta) = \frac{L_2(c^\mathrm{STARC}_{\tfunc,\gamma}(R)) \cdot \sin(2 \arcsin (\epsilon))}{L_2(c^\mathrm{STARC}_{\tfunc,\gamma}(R))},
$$
which gives that $\theta = 2 \arcsin (\epsilon)$. From this, we have that $\delta = 2\epsilon$, and so $\starc(R, R') = \epsilon$. Since $R'$ was selected to maximise the $\starc$-distance to $R$ among all rewards $R''$ such that $L_2(R,R'') \leq L_2(c^\mathrm{STARC}_{\tfunc,\gamma}(R)) \cdot \sin(2 \arcsin(\epsilon))$, we conclude that if $L_2(R,R'') \leq L_2(c^\mathrm{STARC}_{\tfunc,\gamma}(R)) \cdot \sin(2 \arcsin(\epsilon))$, then $\starc(R,R'') \leq \epsilon$.
Since $R$ was selected arbitrarily, this proves that $\starc(R,t_2(R)) \leq \epsilon$ for all $R$. Since $\starc(R,t_1(R)) = 0$ and $\starc(R,t_3(R)) = 0$ for all $R$, and since $\starc$ is a pseudometric, we thus have that $\starc(R,t(R)) \leq \epsilon$ for all $R$. This completes the other direction, and hence the proof.
\end{proof}

\nasforbrmce*

\begin{proof}
    Immediate from Lemma~\ref{lemma:weak_epsilon_robustness_function_composition}, Proposition~\ref{prop:nas_for_small_starc}, Theorem~\ref{thm:ambiguity-boltzmann-rational}, and Theorem~\ref{thm:ambiguity-MCE}.
\end{proof}

\optnotrobust*

\begin{proof}
By Corollary~\ref{cor:diameter_of_op}, if $|\States| \geq 2$ or $|\Actions| \geq 3$, and $d^\R$ is both sound and complete, then there exists reward functions $R_1, R_2$ such that $o_{\tau,\gamma}^\star(R_1) = o_{\tau,\gamma}^\star(R_2)$, and such that $d^\R(R_1, R_2) > 0$. Thus $d^\R(R_1, R_2) = E > 0$, and so $o_{\tau,\gamma}^\star$ violates condition 3 of Definition~\ref{def:misspecification_metric} for all $\epsilon < E$.
\end{proof}

\subsubsection{Perturbation Robustness}

\perturbationrobustnessnas*

\begin{proof}
For the first direction, suppose $f$ is $\epsilon/\delta$-separating, and let $g$ be a $\delta$-perturbation of $f$ with $\mathrm{Im}(g) \subseteq \mathrm{Im}(f)$. We will show that $f$ and $g$ satisfy the conditions of Definition~\ref{def:misspecification_metric}.
For the first condition, let $R_1, R_2$ be two arbitrary reward functions such that $f(R_1) = g(R_2)$. Since $g$ is a $\delta$-perturbation of $f$, we have that $d^\Pi(g(R_2), f(R_2)) \leq \delta$. Since $f(R_1) = g(R_2)$, straightforward substitution thus gives us that $d^\Pi(f(R_1), f(R_2)) \leq \delta$. Since $f$ is $\epsilon/\delta$-separating, this means that $d^\mathcal{R}(R_1, R_2) \leq \epsilon$. Since $R_1$ and $R_2$ were chosen arbitrarily, this means that if $f(R_1) = g(R_2)$ then $d^\mathcal{R}(R_1, R_2) \leq \epsilon$. Thus, the first condition of Definition~\ref{def:misspecification_metric} holds. For the third condition, note that if $f(R_1) = f(R_2)$, then $d^\Pi(f(R_1), f(R_2)) = 0 \leq \delta$. Since $f$ is $\epsilon/\delta$-separating, this means that $d^\mathcal{R}(R_1, R_2) \leq \epsilon$, which means that the second condition is satisfied as well.
The second condition is satisfied, since we assume that $\mathrm{Im}(g) \subseteq \mathrm{Im}(f)$, and the fourth condition is satisfied by the definition of $\delta$-perturbations. This means that $f$ and $g$ satisfy all the conditions of Definition~\ref{def:misspecification_metric}, and thus $f$ is $\epsilon$-robust to misspecification with $g$. Since $g$ was chosen arbitrarily, this means that $f$ is $\epsilon$-robust to misspecification with any $\delta$-perturbation $g$ such that $\mathrm{Im}(g) \subseteq \mathrm{Im}(f)$. Thus $f$ is $\epsilon$-robust to $\delta$-perturbation.

For the second direction, suppose $f$ is \emph{not} $\epsilon/\delta$-separating. This means that there exist $R_1, R_2 \in \R$ such that $d^\mathcal{R}(R_1, R_2) > \epsilon$ and $d^\Pi(f(R_1), f(R_2)) \leq \delta$. Now let $g : \R \to \R$ be the behavioural model where $g(R_1) = f(R_2)$,  $g(R_2) = f(R_1)$, and $g(R) = f(R)$ for all $R \not\in \{R_1, R_2\}$. Now $g$ is a $\delta$-perturbation of $f$. However, $f$ is not $\epsilon$-robust to misspecification with $g$, since $g(R_1) = f(R_2)$, but $d^\mathcal{R}(R_1, R_2) > \epsilon$. Thus, if $f$ is not $\epsilon/\delta$-separating then $f$ is not $\epsilon$-robust to $\delta$-perturbation, which in turn means that if $f$ is $\epsilon$-robust to $\delta$-perturbation, then $f$ is must be $\epsilon/\delta$-separating.
\end{proof}

\notseparating*

\begin{proof}
Let $\delta$ be any positive constant. By assumption, there exists a $\delta'$ such that if $L_2(\pi_1, \pi_2) < \delta'$ then $d^\Pi(\pi_1, \pi_2) < \delta$. Moreover, since $f$ is continuous, there exists an $\epsilon$ such that if $L_2(R_1, R_2) < \epsilon$, then $L_2(f(R_1), f(R_2)) < \delta'$. Next, let $R$ be any reward function that is non-trivial under $\tfunc$ and $\gamma$. We now have that, for any positive constant $c$, the reward functions $c \cdot R$ and $-c \cdot R$ have the opposite policy ordering, which means that $\starc(c \cdot R, - c \cdot R) = 1$. Moreover, by making $c$ sufficiently small, we can ensure that $L_2(c \cdot R, -c \cdot R) < \epsilon$. Thus, for any positive $\delta$ there exist reward functions $c \cdot R$ and $-c \cdot R$ such that $d^\Pi(f(c \cdot R), f(-c \cdot R)) < \delta$, and such that $d^\mathrm{STARC}_{\tau,\gamma}(c \cdot R, -c \cdot R) = 1$. Hence $f$ is not $\epsilon/\delta$-separating for any $\delta > 0$ and any $\epsilon < 1$.
\end{proof}

\subsubsection{Misspecified Parameters}

\misspecifiedenv*

\begin{proof}
If $\tau_1 \neq \tau_2$, then either $\tau_1 \neq \tau_3$ or $\tau_2 \neq \tau_3$. 
If $\tau_1 \neq \tau_3$, then Theorem~\ref{thm:wrong_tau_large_diameter} implies that there exists reward functions $R_1, R_2$ such that $f_{\tau_1}(R_1) = f_{\tau_1}(R_2)$, but such that $d^\mathrm{STARC}_{\tfunc_3,\gamma}(R_1, R_2)$ is arbitrarily close to 1. 
Thus $f_{\tau_1}$ violates condition 2 of Definition~\ref{def:misspecification_metric} for all $\epsilon < 1$.
Similarly, if $\tau_2 \neq \tau_3$, then Theorem~\ref{thm:wrong_tau_large_diameter} implies that there exists reward functions $R_1, R_2$ such that $f_{\tau_2}(R_1) = f_{\tau_2}(R_2)$, but such that $d^\mathrm{STARC}_{\tau_3,\gamma}(R_1, R_2)$ is arbitrarily close to 1.
Then Lemma~\ref{lemma:epsilon_robustness_implies_weak_refinement} implies that there can be no $f$ that is $\epsilon$-robust to misspecification with $f_{\tau_2}$ (as defined by $d^\mathrm{STARC}_{\tau_3,\gamma}$) for any $\epsilon < 0.5$. 
\end{proof}

\misspecifieddiscounting*

\begin{proof}
If $\gamma_1 \neq \gamma_2$, then either $\gamma_1 \neq \gamma_3$ or $\gamma_2 \neq \gamma_3$. 
If $\gamma_1 \neq \gamma_3$, then Theorem~\ref{thm:wrong_gamma_large_diameter} implies that there exists reward functions $R_1, R_2$ such that $f_{\gamma_1}(R_1) = f_{\gamma_1}(R_2)$, but such that $d^\mathrm{STARC}_{\tau,\gamma_3}(R_1, R_2)$ is arbitrarily close to 1. Thus $f_{\gamma_1}$ violates condition 2 of Definition~\ref{def:misspecification_metric} for all $\epsilon < 1$.
Similarly, if $\gamma_2 \neq \gamma_3$, then Theorem~\ref{thm:wrong_gamma_large_diameter} implies that there exists reward functions $R_1, R_2$ such that $f_{\gamma_2}(R_1) = f_{\gamma_2}(R_2)$, but such that $d^\mathrm{STARC}_{\tau,\gamma_3}(R_1, R_2)$ is arbitrarily close to 1.
Then Lemma~\ref{lemma:epsilon_robustness_implies_weak_refinement} implies that there can be no $f$ that is $\epsilon$-robust to misspecification with $f_{\gamma_2}$ (as defined by $d^\mathrm{STARC}_{\tau,\gamma_3}$) for any $\epsilon < 0.5$. 
\end{proof}


%\begin{thebibliography}{00}



%\end{thebibliography}
\end{document}